\RequirePackage{fix-cm}

\documentclass[twocolumn,final]{svjour3} 
\smartqed  
\usepackage[utf8]{inputenc}
\usepackage[bookmarks,breaklinks=true,colorlinks=false, hidelinks,draft]{hyperref}
\usepackage{graphicx}
\usepackage{caption, subcaption}
\usepackage[authoryear]{natbib}
\usepackage{enumitem}

\usepackage{amsmath,amsthm,amssymb,amsfonts,color,bm}
\usepackage{graphicx}
\usepackage{mathtools}
\usepackage{physics}
\usepackage{xcolor}

\newcommand{\cu}[1]{\mathbf{#1}}
\newcommand{\tash}[2]{\frac{\partial #1}{\partial #2}}
\newcommand{\diff}{\mathop{}\!\mathrm{d}}

\DeclareMathOperator*{\argmin}{arg\,min}  

\newcommand{\expec}{\mathbb{E}}           
\newcommand{\cov}{\operatorname{Cov}}     
\newcommand{\varr}{\operatorname{Var}}     
\newcommand{\diag}{\operatorname{diag}}     
\newcommand*{\trans}{{\mkern-1.5mu\mathsf{T}}}

\newcommand{\matern}{Mat\'{e}rn }
\newcommand{\R}{\mathbb{R}} 

\usepackage{booktabs}
\usepackage{tabularx}

\usepackage{tikz}
\usepackage{booktabs}
\usetikzlibrary{fadings}
\usetikzlibrary{patterns}
\usetikzlibrary{shadows.blur}

\usepackage[title]{appendix}
\newtheorem{assumption}{Assumption}
\newtheorem{algorithm}{Algorithm}

\begin{document}
	\sloppy
	
	\title{Deep State-Space Gaussian Processes\thanks{This project was supported by ELEC doctoral school and Academy of Finland.}
	}
	
	
	\author{Zheng Zhao \and Muhammad Emzir \and Simo S\"{a}rkk\"{a}
	}
	
	
	\institute{Zheng Zhao, Muhammad Emzir, and Simo S\"{a}rkk\"{a} \at
		Department of Electrical Engineering and Automation\\
		Aalto University\\
		\email{zheng.zhao@aalto.fi}           
	}
	
	\date{Received: date / Accepted: date}

	\maketitle
	
	\begin{abstract}
		This paper is concerned with a state-space approach to deep Gaussian process (DGP) regression. We construct the DGP by hierarchically putting transformed Gaussian process (GP) priors on the length scales and magnitudes of the next level of Gaussian processes in the hierarchy. The idea of the state-space approach is to represent the DGP as a non-linear hierarchical system of linear stochastic differential equations (SDEs), where each SDE corresponds to a conditional GP. The DGP regression problem then becomes a state estimation problem, and we can estimate the state efficiently with sequential methods by using the Markov property of the state-space DGP. The computational complexity scales linearly with respect to the number of measurements. Based on this, we formulate state-space MAP as well as Bayesian filtering and smoothing solutions to the DGP regression problem. We demonstrate the performance of the proposed models and methods on synthetic non-stationary signals and apply the state-space DGP to detection of the gravitational waves from LIGO measurements. 
	\end{abstract}
	
	\section{Introduction}
	\label{sec:intro}
	
	Gaussian processes (GP) are popular models for probabilistic non-parametric regression, especially in the machine learning field~\citep{gp-carl-edward}. As opposed to parametric models, such as deep neural networks~\citep{dl-goodfellow}, GPs put prior distributions on the unknown functions. As the mean and covariance functions characterize a GP entirely, the design of those two functions determines how well the GP learns the structure of data. However, GPs by using, for example, radial basis functions (RBFs) and \matern class of covariance functions are stationary, and hence those conventional GPs have limitations on learning non-stationary structures in data. Heteroscedastic GPs~\citep{leQuo2005, ICML2011Lazaro-Gredilla} are designed to tackle with the non-stationarity in measurement noise. To model the non-stationary of the unknown process, we often need to manipulate the covariance function to give non-stationary GPs.
	
	One approach to construct non-stationary GPs is to transform the domain/input space by compositions. For example, \citet{Wilson2016DeepKernel, Wilson2016StochasticVBDKL, Shedivat2017RecurrentDKL} transform the inputs by deterministic deep architectures and then feed to GPs, where the deep transformations are responsible for capturing the non-stationarity from data. The resulting GP posterior distribution is in closed form. Similarly, \citet{Calandra2016ManifoldGP} transform the input space to manifold feature spaces. \citet{damianou2013} construct deep Gaussian processes (DGPs) by feeding the outputs of GPs to another layer of GPs as (transformed) inputs. However, the posterior inference requires complicated approximations and does not scale well with a large number of measurements~\citep{Salimbeni2017Doubly}. 
	
	Another commonly used non-stationary GP construction is to have input-dependent covariance function hyperparameters, so that the resulting covariance function is non-stationary~\citep{OLDPaul1992, higdon1999non, paciorek2004nonstationary}. For example, one can parametrize the lengthscale as a function of time. This method grants GPs the capability of changing behaviour depending on the input. However, one needs to be careful to ensure that the construction leads to valid (positive definite) covariance functions \citep{paciorek2004nonstationary}. It is also possible to put GP priors on the covariance parameters~\citep{ville2014MLSP, ICML2011Lazaro-Gredilla, heinonen2016, lassi2019Matern, Karla2020}. For example,~\citet{Salimbeni2017ns} model the lengthscale as a GP by using the non-stationary covariance function of~\citet{paciorek2006spatial}, and they approximate the posterior density via the variational Bayes approach. 
	
	The idea of putting GP priors on the hyperpameters of a GP~\citep{heinonen2016, lassi2019Matern, Salimbeni2017ns} can be continued hierarchically, which leads to one type of deep Gaussian process (DGP) construction~\citep{Matthew2018JMLR, emzir2019, Emzir2020}. Namely, the GP is conditioned on another GP, which again depends on another GP, and so forth. It is worth emphasizing that this hyperparameter-based (or covariance-operator) construction of DGP is different from the compositional DGPs as introduced by~\citet{damianou2013} and~\citet{duvenaud2014}. In these composition-based DGPs, the output of each GP is fed as an input to another GP. Despite the differences, these two types of DGP constructions are similar in many aspects and are often analyzed under the same framework~\citep{Matthew2018JMLR}. 
	
	This paper focuses on hyperparameter-based (or covariance-operator) temporal DGPs. In particular, we introduce the state-space representations of DGPs by using non-linear stochastic differential equations (SDEs). The SDEs form a hierarchical non-linear system of conditionally linear SDEs which results from the property that a temporal GP can be constructed as a solution to a linear SDE~\citep{jouni2010, simoMagazine2013, sarkka2019}. More generally, it is related to the connection of Gaussian fields and stochastic partial differential equations~\citep[SPDEs,][]{lindtrom2011}. (D)GP regression then becomes equivalent to the smoothing problem on the corresponding continuous-discrete state-space model~\citep{simoMagazine2013}. Additionally, by using the SDE representations of DGPs we can avoid to explicitly choose/design the covariance function.
	
	However, the posterior distribution of (state-space) DGPs does not admit a closed-form solution as a plain GP does. Hence we need to use approximations such as maximum a posteriori (MAP) estimates, Laplace approximations, Markov chain Monte Carlo \citep[MCMC,][]{heinonen2016, brooks2011handbook,luengo2020}, or variational Bayes methods~\citep{ICML2011Lazaro-Gredilla, Salimbeni2017Doubly, Paul2020VariationalSS}. However, those methods are often computationally heavy. The another benefit of using state-space DGPs is that we can use the Bayesian filtering and smoothing solvers which are particularly efficient for solving temporal regression/smoothing problems~\citep{sarkka2013}. 
	
	In short, we formulate the (temporal) DGP regression as a state-estimation problem on a non-linear continuous-discrete state-space model. For this purpose, various well-established filters and smoothers are available, for example, the Gaussian (assumed density) filters and smoothers~\citep{sarkka2013, sarkka2013gaussian, zhao2020taylor}. For temporal data, the  computational complexity of using filtering and smoothing approaches is $\mathcal{O}(N)$, where $N$ is the number of measurements. 
	
	The contributions of the paper are the follows. 1) We construct a general hyperparameter-based deep Gaussian process (DGP) model and formulate a batch MAP solution for it as a standard reference approach. 2) We convert the DGP into a state-space form consisting of a system of stochastic differential equations. 3) For the state-space DGP, we formulate the MAP and Bayesian filtering and smoothing solutions. The resulting computational complexity scales linearly with respect to the number of measurements. 4) We prove that for a class of DGP constructions and Gaussian approximations on the DGP posterior, certain nodes of the DGP (e.g., the magnitude $\sigma$ of \matern GP) will not be asymptotically updated from measurements. 5) We conduct experiments on synthetic data and also apply the methods to gravitational wave detection. 
	
	\section{Deep Gaussian Processes}
	\label{sec:dgp-sde}
	\subsection{Non-stationary Gaussian Processes}
	\label{sec:non-stat-GP}
	We start by reviewing the classical Gaussian process (GP) regression problem \citep{gp-carl-edward}. We consider the model
	\begin{equation}
		\begin{split}
			f(t) &\sim \mathcal{GP}(0, C(t, t')),\\
			y_k &= f(t_k) + r_k,
		\end{split}
		\label{equ:vanilla-gp-reg}
	\end{equation}
	where $f \colon \mathbb{T}\to\R$ is a zero-mean GP on $\mathbb{T}=\left\lbrace t\in\R\colon t\geq t_0 \right\rbrace $ with a covariance function $C$. The observation $y_k\coloneqq y(t_k)$ of $f(t_k)$ is contaminated by a Gaussian noise $r_k\sim\mathcal{N}(0, R_k)$. We let $\cu{R} = \diag(R_1,\ldots,R_N)$. Given a set of $N$ measurements $\cu{y}_{1:N} = \left\lbrace y_1,\ldots, y_N \right\rbrace $, GP regression aims at obtaining the posterior distribution
	\begin{equation}
		p(f\mid \cu{y}_{1:N}),\nonumber
	\end{equation}
	which is again Gaussian with closed-form mean and covariance functions~\citep{gp-carl-edward}. In this model, the choice of covariance function $C$ is crucial to the GP regression as it determines, for example, the smoothness and stationarity of the process. Typical choices, such as radial basis or \matern covariance functions, give stationary GPs. 
	
	However, it is difficult for a stationary GPs to tackle with non-stationary data. The main problem arises from the covariance function~\citep{gp-carl-edward} as the value of a  stationary covariance function only depends on the difference of inputs. That is to say, the correlations of any pairs of two inputs are the same when the differences are the same. This feature is not beneficial for non-stationary signals, as the correlation might vary depending on the input.
	
	A solution to this problem is using a non-stationary covariance function~\citep{higdon1999non, paciorek2004nonstationary, paciorek2006spatial, lindtrom2011}. That grants GP with the capability of adaption by learning hyperparameter functions from data. However, one needs to carefully design the non-stationary covariance function such that it is positive definite. Recent studies by, for example, \citet{heinonen2016, lassi2019Matern} and~\citet{Karla2020}, propose to put GP priors on the covariance function hyperparameters. In this article, we follow these approaches to construct hierarchy of GPs which becomes the construction of the deep GP model. 
	
	\subsection{Deep Gaussian Process Construction}
	\label{sec:dgp-construct-reg}
	We define a deep Gaussian process (DGP) in a general perspective as follows. Suppose that the DGP has $L$ layers, and each layer ($i=1,\ldots, L$) is composed of $L_i$ nodes.  Each node of the DGP is conditionally a GP, denoted by $u^i_{j,k}$, where $k=1,2,\ldots, L_i$. We give three indices for the node. The indices $i$ and $k$ specify the layer and the position of the GP, respectively. As an example, $u^i_{j,k}$ is located in the $i$-th layer of the DGP and is the $k$-th node in the $i$-th layer. The index $j$ is introduced to indicate the conditional connection to its unique child node on the previous layer. That is to say, $u^i_{j,k}$ is the child of nodes $u^{i+1}_{k,k'}$ for all suitable $k'$. The terminologies ``child'' and ``parent'' follow from the graphical model conventions~\citep{bishop-ml-pr}. To keep the notation consistent, we also use $u^1_{1,1} \coloneqq f$ for the top layer GP. The nodes $u^{L+1}_{j,k}$ outside of the DGP, we treat as degenerate random variables (i.e., constants or trainable hyperparameters). Remark that every node in the DGP is uniquely indexed by $i$ and $k$, whereas $j$ only serves the purpose of showing the dependency instead of indexing.
	
	We call the vector process
	\begin{equation}
		U \colon\mathbb{T}\to \R^{\sum^L_{i=1} L_i},\nonumber
	\end{equation}
	the DGP, where each element of $U$ corresponds (one to one and onto) to the element of the set of all nodes $\left\lbrace u^i_{j,k}\colon i=1,\ldots,L, k=1,2,\ldots,L_i \right\rbrace $. Similarly, each element of the vector $U^i\colon\mathbb{T}\to\R^{L_i}$ corresponds to the element of the set of all nodes from the $i$-th layer. We denote by $U^i_{k,\cdot}=\left\lbrace u^{i}_{k, k'}\colon \text{for all suitable } k' \right\rbrace $ the set of all parent nodes of $u^{i-1}_{j, k}$. 

	\begin{figure}[h!]
		\centering
		\resizebox{.95\linewidth}{!}{%
			\input{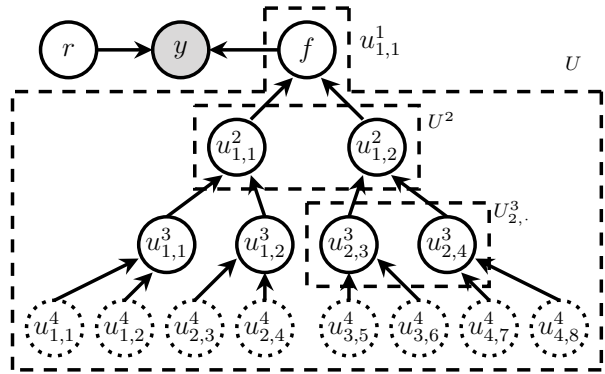}
		}
		\caption{Example of a 3-layer DGP regression model, where each (conditional) GP depends on two other GPs. Variable $y$ is the measurement, and the nodes in $U^4$ are degenerate random variables. }
		\label{fig:dgp-example-tree}
	\end{figure}
	In this tree-like general construction of DGP $U$, there are $\sum^L_{i=1} L_i$ nodes in total. Every $u^i_{j,k}$ is independent of other nodes in the same $i$-th layer, and depends on the nodes $U^{i+1}_{j,\cdot}$ on the next $\left(i+1\right) $-th layer. When there is only one layer, the DGP reduces to a conventional GP. Figure~\ref{fig:dgp-example-tree} illustrates the DGP construction.
	
	The realization of the DGP depends on how each of the conditionally GP nodes is constructed. In the following sections, we discuss two realizations of this DGP, by either constructing the conditional GPs by specifying the mean and covariance functions, or by stochastic differential equations. These two constructions lead to DGP regression in batch and sequential forms, respectively. 
	
	\section{A Batch Deep Gaussian Process Regression Model}
	\label{sec:dgp-batch-reg}
	In this section, we present a batch DGP construction which uses the construction of non-stationary GPs presented in \citet{paciorek2006spatial} to form the DGP. To emphasize the difference to the SDE construction which is the main topic of this article, we call this the batch-DGP. Let us assume that every conditional GP in the DGP has zero mean and we observe the top GP $f$ with additive Gaussian noise. We write down the DGP regression model as
	\begin{equation}
		\begin{split}
			f\mid U^2 &\sim \mathcal{GP}( 0, C(t, t';U^2))  \\
			u^2_{1,1} \mid U^3_{1,\cdot} &\sim \mathcal{GP}(0, C^2_1(t,t';U^3_{1,\cdot}))\\
			u^2_{1,2} \mid U^3_{2,\cdot} &\sim \mathcal{GP}(0, C^2_2(t,t';U^3_{2,\cdot}))\\
			&\vdots\\
			u^i_{\cdot,k} \mid U^{i+1}_{k,\cdot} &\sim \mathcal{GP}(0, C^i_k(t,t';U^{i+1}_{k,\cdot}))\\
			&\vdots\\
			y_k &= f(t_k) + r_k,
		\end{split}
		\label{equ:gp-model-fancy}
	\end{equation}
	where each covariance function $C^i_k\colon \mathbb{T}\times\mathbb{T}\to\R$ is parameterized by next layer's (conditional) GPs. That is to say, the covariance function $C^i_k$ takes the nodes in $U^{i+1}_{k,\cdot}$ as parameters. 
	
	This DGP construction requires positive covariance function at each node. One option is the non-stationary exponential covariance function which has the form \citep[cf.][]{paciorek2006spatial}
	\begin{equation}
		\begin{split}
			C_{NS}(t, t';\ell,\sigma) &= \frac{\sigma(t)\sigma(t')}{\Gamma(\nu)2^{\nu-1}}\ell^{\frac{1}{4}}(t)\ell^{\frac{1}{4}}(t') \\
			&\hspace{-1cm}\times\sqrt{2} \exp\left( \frac{-\sqrt{2}\,\abs{t-t'}}{\sqrt{\ell(t) + \ell(t')}} \right) \left(\ell(t) + \ell(t')\right)^{-\frac{1}{2}}.\nonumber
			\label{equ:ns-matern}
		\end{split}
	\end{equation}
	In the above covariance function $C_{NS}$, the length scale $\ell(t)$ and magnitude $\sigma(t)$ are functions of input $t$. \citet{paciorek2006spatial} also generalize $C_{NS}$ to the \matern class.
	
	For the DGP construction in~\eqref{equ:gp-model-fancy} we need to ensure the positivity of the hyperparameter functions. For that purpose we introduce a wrapping function $g\colon\R\to(0,\infty)$ which is positive and smooth, and we put $\ell(t) = g(u_\ell(t))$ and $\sigma(t) = g(u_\sigma(t))$ where $u_\ell$ and $u_\sigma$ are the conditionally Gaussian processes from the next layer. The exponential or squaring functions are typical options for $g$. In Example~\ref{example:matern-gp}, we show a two-layer DGP by using the covariance function $C_{NS}$.
	
	\begin{example}
		\label{example:matern-gp}
		Consider a two layer exponential (Mat\'ern) DGP 
		\begin{equation}
			\begin{split}
				f\mid u^2_{1,1},u^2_{1,2} &\sim \mathcal{GP}(0, C_{NS}(t,t';g(u^2_{1,1}),g(u^2_{1,2}))),\\
				u^2_{1,1} &\sim \mathcal{GP}(0, C_{NS}(t,t';g(u^3_{1,1}),g(u^3_{1,2}))), \\
				u^2_{1,2} &\sim \mathcal{GP}(0, C_{NS}(t,t';g(u^3_{2,3}),g(u^3_{2,4}))).
			\end{split}\nonumber
		\end{equation}
		In this case, we have the so-called length scale $\ell^2_{1,1} = g(u^2_{1,1})$ and magnitude $\sigma^2_{1,2} = g(u^2_{1,2})$. Also, $U = \begin{bmatrix}
			f & u^2_{1,1} & u^2_{1,2}\end{bmatrix}^\trans$ and $U^2= U^2_{1,\cdot} = \begin{bmatrix}
			u^2_{1,1} & u^2_{1,2}\end{bmatrix}^\trans$.
	\end{example}
	
	Given a set of measurements $\cu{y}_{1:N} = \left\lbrace y_1, y_2,\ldots, y_N \right\rbrace$, the aim of DGP regression is to obtain the posterior density
	\begin{equation}
		p(U\mid \cu{y}_{1:N}) = \frac{p(\cu{y}_{1:N}\mid U)\,p(U)}{p(\cu{y}_{1:N})},
		\label{equ:posterior}
	\end{equation}
	for any input $t\in\mathbb{T}$. Moreover, by the construction of DGP (conditional independence) we have
	\begin{equation}
		p(U) = p(f\mid U^2)\,\prod_{i=2}^L\prod_{k=1}^{L_i} p\left(u^i_{j, k}\mid U^{i+1}_{k,\cdot}\right),
		\label{equ:gp-field}
	\end{equation}
	where each $u^i_{j,k} \mid U^{i+1}_{k,\cdot}$ is a GP as defined in~\eqref{equ:gp-model-fancy}. We isolate $f\mid U^2$ out of the above factorization because we are particularly interested in the observed $f$. It is important to remark that the distribution of $U$ is (usually) not Gaussian because of the non-Gaussianility induced by the conditional hierarchy of Gaussian processes which depend on each other non-linearly. 
	
	\subsection{Batch MAP Solution}
	\label{sec:batch-MAP}
	The maximum a posteriori (MAP) estimate gives a point estimate of $U$ as the maximum of the posterior distribution~\eqref{equ:posterior}. Let us denote $f_{1:N} = \begin{bmatrix}
		f(t_1)&f(t_2)&\cdots& f(t_N) \end{bmatrix}^\trans \in\R^{N}$, $U_{1:N} = \left\lbrace u^i_{j,k\mid 1:N}\colon \text{for all } i, k\right\rbrace $, where $u^i_{j,k\mid 1:N}=\begin{bmatrix}
		u^i_{j,k}(t_1) & \cdots & u^i_{j,k}(t_N)
	\end{bmatrix}^\trans\in\R^{N}$. We are targeting at the posterior density $p( U_{1:N}\mid \cu{y}_{1:N})$ evaluated at $t_1, \ldots, t_N$. The MAP estimate is then obtained by
	\begin{equation}
		U_{1:N}^{\mathrm{BMAP}} = \argmin_{U_{1:N}} \mathcal{L}^{\mathrm{BMAP}}( U_{1:N};\cu{y}_{1:N}),
		\label{equ:MAP-old-opt}
	\end{equation}
	where $\mathcal{L}^{\mathrm{BMAP}}$ is the negative logarithm of the unnormalized posterior distribution given by
	\begin{equation}
		\begin{split}
			&\mathcal{L}^{\mathrm{BMAP}}(U_{1:N};\cu{y}_{1:N}) \\
			&= -\log \left[ p(\cu{y}_{1:N}\mid U_{1:N}) \,p(U_{1:N})\right] \\
			&= \frac{1}{2}\left[ \left(\cu{y}_{1:N}-f_{1:N}\right)^\trans\,\cu{R}^{-1}\,\left(\cu{y}_{1:N}-f_{1:N}\right) + \log\abs{2\pi\,\cu{R}}\right]  \\
			&\quad+ \frac{1}{2}\left[ f_{1:N}^\trans\, \cu{C}^{-1}\,f_{1:N} + \log \abs{2\pi\,\cu{C}}\right] \\
			&\quad+ \frac{1}{2}\sum^L_{i=2}\sum^{L_i}_{k=1}\left[ (u^i_{j,k\mid 1:N})^\trans(\cu{C}^i_k)^{-1}u^i_{j,k\mid 1:N} +\log\abs{2\pi\,\cu{C}^i_k} \right].
		\end{split}
		\label{equ:MAP-loss-old}
	\end{equation}
	In the above Equation~\eqref{equ:MAP-loss-old}, $\cu{C}$ and $\cu{C}^i_k$ are the covariance matrices formed by evaluating the corresponding GP covariance functions at $(t_1,\ldots,t_N)\times (t_1,\ldots,t_N)$. The computational complexity for computing \eqref{equ:MAP-loss-old} is $\mathcal{O}(N^3\,\sum^{L}_{i=1}L_i)$, which scales cubically with the number of measurements. 
	
	It is important to recall from~\eqref{equ:gp-model-fancy} that the covariance matrices also depend on the other GP nodes (i.e., $f_{1:N}$ and $U_{1:N}$ are in $\cu{C}^i_k$). Therefore the objective function $\mathcal{L}^{\mathrm{BMAP}}$ is non-quadratic. Additional non-linear terms are also introduced by the determinants of the covariance matrices. However, quasi-Newton methods \citep{numerical-opt} can be used to solve the optimization problem. The required gradients of~\eqref{equ:MAP-loss-old} are provided in Appendix~\ref{append:MAP-old-deriv}.
	
	There are two major challenges in this MAP solution. Firstly, the optimization of~\eqref{equ:MAP-old-opt} is not computationally cheap. It requires to evaluate and store $\sum^{L}_{i=1}L_i$ inversions of $N$-dimensional matrices for every optimization iteration. This prevents the use of the DGP on large-scale datasets and large models. Moreover, \citet{paciorek2006spatial} state that the optimization of \eqref{equ:MAP-old-opt} is difficult and prone to overfitting, which we also confirm in the experiment section. Another problem is the uncertainty quantification and prediction (interpolation) with the MAP estimate which is degenerate.
	
	\section{Deep Gaussian Processes in State-space}
	\label{sec:gp-sde}
	Stochastic differential equations (SDEs) are common models to construct stochastic processes~\citep{SDE-Friedman1975, Williams2000Vol1, Williams2000Vol2, sarkka2019}. Instead of constructing the process by specifying, for example, the mean and covariance functions, an SDEs characterizes a process by describing the dynamics with respect to a Wiener process. In this section, we show how a DGP as defined in Section~\ref{sec:dgp-construct-reg} can be realized using a hierarchy of SDEs. To highlight the difference to the previous batch-DGP realization, we call this the SS-DGP. The regression problem on this class of DGPs can be seen as a state estimation problem. 
	
	\subsection{Gaussian Processes as Solutions of Linear SDEs}
	Consider a linear time invariant (LTI) SDE
	\begin{equation}
		\begin{split}
			\diff \cu{f}(t) &= \cu{A}\,\cu{f}(t) \diff t + \cu{L}\diff \cu{W}_f(t),  \\
			\cu{f}(t_0) &\sim \mathcal{N}(\cu{0}, \cu{P}_\infty),\\
		\end{split}
		\label{equ:gp-sde}
	\end{equation}
	where coefficients $\cu{A}\in\R^{d\times d}$ and $\cu{L}^{d\times S}$ are constant matrices, $\cu{W}_f(t)\in\R^{S }$ is a Wiener process with unit spectral density, and $\cu{f}(t_0)$ is a Gaussian initial condition with zero mean and covariance $\cu{P}_\infty$, which is obtained as the solution to
	\begin{equation}
		\begin{split}
			\cu{A}\,\cu{P}_\infty + \cu{P}_\infty\,\cu{A}^\trans + \cu{L}\,\cu{L}^\trans = \cu{0}.
		\end{split}
		\label{equ:lyapunov}
	\end{equation}
	When the stationary covariance $\cu{P}_\infty$ exists, the vector process $\cu{f}$ is a stationary Gaussian process with the (cross-)covariance function
	\begin{equation}
		\cov\left[\cu{f}(t),\cu{f}(t')\right] =  \begin{cases}
			\cu{P}_\infty\,\exp\left[ (t'-t)\,\cu{A}\right]^\trans, & t<t',\\
			\exp\left[ -(t'-t)\,\cu{A}\right]\,\cu{P}_\infty, & t\geq t'.
		\end{cases}
		\label{equ:LTI-cross-cov}
	\end{equation}
	It now turns out that we can construct matrices $\cu{A}$, $\cu{L}$, and $\cu{H}$ such that $f = \cu{H}\, \cu{f}$ is a Gaussian process with a given covariance function \citep{jouni2010,simoMagazine2013,sarkka2019}. The marginal covariance of $f$ can be extracted by $\cov[f(t), f(t')] = \cu{H} \cov\left[\cu{f}(t),\cu{f}(t')\right] \cu{H}^\trans$. In order to construct non-stationary GPs, we can let the SDE coefficients (i.e., $\cu{A}$ and $\cu{L}$) be functions of time.
	
	In particular, if $f$ is a \matern GP, then we can select the state
	\begin{equation}
		\cu{f}(t) = \begin{bmatrix}
			f(t) & Df(t) & \cdots & D^\alpha f(t)
		\end{bmatrix}^\trans \in\R^{d },
		\label{equ:state-f}
	\end{equation}
	and the corresponding $\cu{H} = \begin{bmatrix}
		1 & 0 & 0 &\cdots
	\end{bmatrix}$, where $D$ is the time derivative, $\alpha$ is the smoothness factor, and dimension $d = \alpha + 1$. We can also generalize the results to spatial-temporal Gaussian processes, and hence the corresponding SDEs will become stochastic partial differential equations~\citep[SPDEs,][]{simo-jouni-2012,simoMagazine2013}.
	
	When constructing a GP using SDEs, we sometimes need to select the SDE coefficients suitably so that the resulting covariance function~\eqref{equ:LTI-cross-cov} admits a desired form (e.g., Mat\'ern). One way to proceed is to find the spectral density function of the GP covariance function (via Wiener--Khinchin theorem) and translate the resulting transfer function into the state-space form \citep{jouni2010}. The results are known for many classes of GPs, for example, the M\'atern and periodic GPs~\citep{sarkka2019}.
	
	As an alternative to the batch-GP construction in Section~\ref{sec:dgp-batch-reg}, the SDE approach offers more freedom to certain extent because the corresponding covariance functions are positive definite and non-stationary by construction. It is also computationally beneficial in regression, as we can leverage the Markov properties of the SDEs in the computations. 
	
	\subsection{Deep Gaussian Processes as Hierarchy of SDEs}
	\label{sec:dgp-system-sdes}
	So far, we have only considered the SDE construction of a single stationary/non-stationary GP. To realize a DGP as defined in Section~\ref{sec:dgp-construct-reg}, we need to formulate a hierarchical system composed of linear SDEs. Namely, we parametrize the SDE coefficients as functions of other GPs in a hierarchical structure. Followed from the SDE expression of GP $f$ in Equation~\eqref{equ:gp-sde}, let us similarly define the state
	\begin{equation}
		\cu{u}^i_{j,k} \colon \mathbb{T} \to \R^{d },\nonumber
	\end{equation}
	for any GP $u^i_{j,k}$ in the DGP $U$. We then construct the DGP by finding the SDE representation for each $u^i_{j,k}$ to yield
	\begin{equation}
		\begin{split}
			\diff \cu{f} &= \cu{A}(\cu{U}^{2}_{1,\cdot})\,\cu{f} \diff t + \cu{L}(\cu{U}^{2}_{1,\cdot})\diff \cu{W}_f, \\
			\diff \cu{u}^2_{1,1} &= \cu{A}^2_1(\cu{U}^{3}_{1,\cdot}) \, \cu{u}^2_{1,1} \diff t + \cu{L}^2_1(\cu{U}^{3}_{1,\cdot}) \diff \cu{W}^2_1,\\
			\diff \cu{u}^2_{1,2} &= \cu{A}^2_2(\cu{U}^{3}_{2,\cdot}) \, \cu{u}^2_{1,2} \diff t + \cu{L}^2_2(\cu{U}^{3}_{2,\cdot})\diff \cu{W}^2_2 , \\
			&\vdots \\
			\diff \cu{u}^i_{j,k} &= \cu{A}^i_k(\cu{U}^{i+1}_{k,\cdot})\, \cu{u}^i_{j,k} \diff t + \cu{L}^i_k(\cu{U}^{i+1}_{k,\cdot})\diff \cu{W}^i_{k},\\
			&\vdots
		\end{split}
		\label{equ:sdes-ensemble}
	\end{equation}
	where $\cu{W}_f\in\R^{S }$ and $\cu{W}^i_k\in\R^{S }$ for all $i$ and $k$ are mutually independent standard Wiener processes. Note that the above SDE system~\eqref{equ:sdes-ensemble} is \textit{non-linear}, and the coefficients are state-dependent. We denote by $\cu{U}^{i+1}_{k,\cdot}$ the collection states for all parent states of $\cu{u}^i_{j,k}$. For example, if $\cu{u}^2_{1,1}$ is conditioned on $\cu{u}^3_{1, 1}$ and $\cu{u}^3_{1, 2}$, then $\cu{U}^{3}_{1,\cdot} = \begin{bmatrix}
		\left( \cu{u}^3_{1, 1}\right)^\trans  & \left( \cu{u}^3_{1, 2}\right)^\trans 
	\end{bmatrix}^\trans$. To further condense the notation, we rearrange the above SDEs~\eqref{equ:sdes-ensemble} into 
	\begin{equation}
		\begin{split}
			\diff \cu{U}(t) &= \bm{\Lambda}(\cu{U}(t))\diff t + \bm{\beta}(\cu{U}(t))\diff\cu{W}(t), \\
			\cu{U}(t_0) &\sim \mathcal{N}(\cu{0}, \cu{P}_0),
		\end{split}
		\label{equ:sdes-ensemble-one-line}
	\end{equation}
	where 
	\begin{equation}
		\begin{split}
			&\cu{U}(t) = \begin{bmatrix}\cu{f}^\trans & \left( \cu{u}^2_{1,1}\right)^\trans  & \left( \cu{u}^2_{1,2}\right)^\trans & \cdots & \left( \cu{u}^i_{j,k}\right)^\trans  & \cdots \end{bmatrix}^\trans \in \R^{\varrho },\nonumber
		\end{split}
	\end{equation}
	is the SDE state of the entire DGP, $\cu{U}(t_0)$ is the Gaussian initial condition, $\varrho = d\sum^L_{i=1}L_i$ is the total dimension of the state, and 
	\begin{equation}
		\cu{W}(t)  = \begin{bmatrix}\cu{W}^\trans_f & (\cu{W}^2_1)^\trans & \cdots & (\cu{W}^i_k)^\trans & \cdots\end{bmatrix}^\trans \in \R^{\varrho }. \nonumber
	\end{equation}
	The drift $\bm{\Lambda}\circ\cu{U}\colon\mathbb{T}\to\R^{\varrho  }$ and dispersion $\bm{\beta}\circ\cu{U}\colon\mathbb{T}\to\R^{\varrho  }$ functions can be written as
	\begin{equation}
		\bm{\Lambda}(\cu{U}(t)) = \begin{bmatrix}
			\cu{A}(\cu{U}^{2}_{1,\cdot}) & & & &\\
			& \cu{A}^2_1(\cu{U}^{3}_{1,\cdot}) & & &\\
			& & \ddots & &\\
			& & &\cu{A}^i_k(\cu{U}^{i+1}_{k,\cdot}) &\\
			& & & & \ddots\\
		\end{bmatrix} \, \cu{U}(t),\nonumber
	\end{equation}
	and
	\begin{equation}
		\bm{\beta}(\cu{U}(t)) = \diag\left(\cu{L}(\cu{U}^{2}_{1,\cdot}), \cu{L}^2_1(\cu{U}^{3}_{1,\cdot}), \ldots, \cu{L}^i_k(\cu{U}^{i+1}_{k,\cdot}), \ldots \right),\nonumber
	\end{equation}
	respectively.
	
	The above SDE representation of DGP is general in the sense that the SDE coefficients of each GP and the number of layers are free. However, they cannot be completely arbitrary as we at least need to require that the SDE has a weakly unique solution. A classical sufficient condition is to have the coefficients globally Lipschitz continuous and have at most linear growth~\citep{SDE-Friedman1975, xu-sde-exist2008, sde-mao2008, oksendal}. These restrictive conditions can be further weakened, for example, to locally Lipschitz~\citep[Ch.~5]{SDE-Friedman1975} and weaker growth condition~\citep[Theorem 2.2]{shen-mao2006-SDE-weaker}. Alternatively, requiring the coefficients to be Borel measurable and locally bounded is enough for a unique solution~\citep[][Theorem 21.1 and Equation 21.9]{Williams2000Vol2}.
	
	It is also worth remarking that the SDE system~\eqref{equ:sdes-ensemble-one-line} and hence the DGP is a well-defined It\^{o} diffusion, provided that the coefficients are regular enough~\citep[Definition 7.1.1,][]{oksendal}. This feature is valuable, as being an It\^{o} diffusion offers many fruitful properties that we can use in practice, for example, continuity, Markov property, and the existence of infinitesimal generator~\citep{oksendal}. The Markov property is needed to ensure the existence of transition density and also enables the use of Bayesian filtering and smoothing for regression. The infinitesimal generator can be used to discretize the SDEs as we do in Section~\ref{sec:dgp-ss-reg}. 
	
	It is also possible to extend the SDE representations of temporal DGPs to stochastic partial differential equation (SPDE) representations of spatio-temporal DGPs. \citet{simoMagazine2013} given the following result. Suppose $v\colon \mathbb{X}\times \mathbb{T} \to \R$ is a spatio-temporal stationary GP on a suitable domain, such that $v(\cu{x}, t) \sim \mathcal{GP}(0, C(\cu{x}, \cu{x}', t, t'))$. Then $\cu{v}(\cu{x}, t)$ can be constructed as a solution to an evolution type of SPDE
	\begin{equation}
		\tash{\cu{v}(\cu{x}, t)}{t} = \mathcal{A}\cu{v}(\cu{x}, t) + \mathcal{B}w(\cu{x}, t), \nonumber
	\end{equation}
	where $\cu{v}(\cu{x}, t)$ is the state of $v$, $\mathcal{A}$ and $\mathcal{B}$ are spatial operators, and $w(\cu{x}, t)$ is the spatio-temporal white noise. \citet{Emzir2020} build a deep Gaussian field based on the \matern SPDE by~\citet{lindtrom2011}, which provides another path to the spatio-temporal case.
	
	\subsection{Deep \matern Process}
	\label{sec:deep-matern}
	In this section, we present a \matern construction of SS-DGP~\eqref{equ:sdes-ensemble-one-line}. The coefficients are chosen such that each SDE corresponds to a conditional GP with the \matern covariance function. The idea is to find an equivalent SDE representation for each \matern GP node, and then parametrize the covariance parameters (i.e., length-scale $\ell$ and magnitude $\sigma$) with another layer of \matern GPs. We are interested in a GP
	\begin{equation}
		u^i_{j,k}\mid \ell,\sigma \sim \mathcal{GP}(0, C(t,t';\ell,\sigma)),
		\label{equ:matern-u}
	\end{equation}
	with the \matern covariance function
	\begin{equation}
		\begin{split}
			C(t, t') &=\frac{\sigma^2\,2^{1-\nu}}{\Gamma(\nu)}\left(\kappa\,\abs{t - t'} \right)^\nu K_\nu \left(\kappa\,\abs{t - t'} \right),
			\label{equ:matern-cov}
		\end{split}
	\end{equation}
	where $K_\nu$ is the modified Bessel function of the second kind and $\Gamma$ is the Gamma function. We denote $\kappa = \sqrt{2\nu}/\ell$ and $\nu = \alpha + 1/2$.
	
	As shown by \citet{jouni2010} and~\citet{simoMagazine2013}, one possible SDE representation of \matern GP $u^i_{j,k}$ in Equation~\eqref{equ:matern-u} is
	\begin{equation}
		\diff \cu{u}^i_{j,k} = \cu{A}^i_k\, \cu{u}^i_{j,k} \diff t + \cu{L}^i_k \diff W^i_k,
		\label{equ:matern-sde}
	\end{equation}
	where the state
	\begin{equation}
		\cu{u}^i_{j,k}(t)= \begin{bmatrix}
			u^i_{j,k}(t) & Du^i_{j,k}(t) & \cdots & D^\alpha u^i_{j,k}(t)
		\end{bmatrix}^\trans \in\R^{d },\nonumber
	\end{equation}
	and the SDE coefficients $\cu{A}^i_k$ and $\cu{L}^i_k$ admit the form
	\begin{equation}
		\cu{A}^i_k = \begin{bmatrix}
			0 & 1 & & \\
			& 0 & 1 & \\
			\vdots & & \ddots&\\
			-\binom{\alpha}{0}\kappa^\alpha & -\binom{\alpha}{1}\kappa^{\alpha-1} & \cdots & -\binom{\alpha}{\alpha-1}\kappa
		\end{bmatrix},\nonumber
	\end{equation}
	and
	\begin{equation}
		\cu{L}^i_k = \begin{bmatrix}
			0 & 0 & \cdots & \frac{\sigma\,\Gamma(\alpha+1) }{\sqrt{\Gamma(2\alpha +1)}}\left(2\,\kappa \right)^{(\alpha+\frac{1}{2})}
		\end{bmatrix}^\trans,
		\label{equ:matern-L}
	\end{equation}
	respectively. Above, we denote by $\binom{\alpha}{1}$ a binomial coefficient and $W^i_k\in\R$ is a standard Wiener process. Next, to contruct the deep \matern process, we need to parametrize the length scale $\ell$ and magnitude $\sigma$ by the states of parent GPs and build the system as in Equation~\eqref{equ:sdes-ensemble}. For example, if we want to use $u^3_{1,1}$ and $u^3_{1,2}$ to model the length scale and magnitude of $u^2_{1,1}$, then $\ell^2_{1,1} = g(u^3_{1,1})$ and $\sigma^2_{1,1} = g(u^3_{1,2})$. The wrapping function $g\colon\R\to(0,\infty)$ is mandatory to ensure the positivity of \matern parameters. The minimal requirement for function $g$ is to be positive and Borel measurable. For instance, we can let $g(u) = \exp(u) + c$ or $g(u) = u^2 + c$ for some $c>0$. Another choice is to let $g(u) = 1/(u^2+c)$ that is bounded and Lipschitz on $\R$, which makes the deep \matern process an It\^{o} diffusion and we have the SS-DGP well-defined~\citep{oksendal}.
	
	Note that the resulting state-space model composed of~\eqref{equ:matern-sde} has a canonical form from control theory~\citep{Torkel_control_2000}, and the dimensionality is determined by the smoothness parameter $\alpha$. Moreover, the coefficient $\cu{A}^i_k$ is Hurwitz, because all the eigenvalues have strictly negative real part. The stability of such system is studied, for example, in~\citet{Khasminskii2012}.
	
	\begin{example}
		\label{example:matern-ss-gp}
		Corresponding to Example~\ref{example:matern-gp}, the SDE construction of the two layer (exponential) \matern process is formulated as follows:
		\begin{equation}
			\begin{split}
				\diff f &= -\frac{1}{g(u^2_{1,1})}\,f\diff t + \frac{\sqrt{2}\,g(u^2_{1,2})}{\sqrt{g(u^2_{1,1})}} \diff W_f, \\
				\diff u^2_{1,1} &= -\frac{1}{g(u^3_{1,1})}\,u^2_{1,1}\diff t + \frac{\sqrt{2}\,g(u^3_{1,2})}{\sqrt{g(u^3_{1,1})}} \diff W^2_{1,1}, \\
				\diff u^2_{1,2} &= -\frac{1}{g(u^3_{2,3})}\,u^2_{1,2}\diff t + \frac{\sqrt{2}\,g(u^3_{2,4})}{\sqrt{g(u^3_{2,3})}} \diff W^2_{1,2},
			\end{split}\nonumber
		\end{equation}
		where we have states $\cu{U} = \begin{bmatrix}
			f & u^2_{1,1} & u^2_{1,2}\end{bmatrix}^\trans$ and the SDE coefficient functions 
		\begin{equation}
			\bm{\Lambda}(\cu{U}) = \begin{bmatrix}
				-\frac{1}{g(u^2_{1,1})} & &\\
				& -\frac{1}{g(u^3_{1,1})} &\\
				& & -\frac{1}{g(u^3_{2,3})}
			\end{bmatrix}\,\cu{U},\nonumber
		\end{equation}
		and $\bm{\beta}(\cu{U}) = \diag\left(
		\frac{\sqrt{2}\,g(u^2_{1,2})}{\sqrt{g(u^2_{1,1})}}, \frac{\sqrt{2}\,g(u^3_{1,2})}{\sqrt{g(u^3_{1,1})}}, \frac{\sqrt{2}\,g(u^3_{2,4})}{\sqrt{g(u^3_{2,3})}}
		\right)$. The length scale $\ell^2_{1,1}$ and magnitude $\sigma^2_{1,2}$ of $f$ are given by $\ell^2_{1,1} = g(u^2_{1,1})$ and $\sigma^2_{1,2} = g(u^2_{1,2})$, respectively.
	\end{example}
	
	\section{State-space Deep Gaussian Process Regression}
	\label{sec:dgp-ss-reg}
	In this section, we formulate sequential state-space regression by DGPs. By using the result in Equation~\eqref{equ:sdes-ensemble-one-line}, the state-space regression model is
	\begin{equation}
		\begin{split}
			\diff \cu{U}(t) &= \bm{\Lambda}(\cu{U}(t))\diff t + \bm{\beta}(\cu{U}(t))\diff\cu{W}(t), \\
			y_k &= \cu{H}\,\cu{U}(t_k) + r_k,
		\end{split}
		\label{equ:sde-reg-model}
	\end{equation}
	where the initial condition $\cu{U}(t_0)\sim\mathcal{N}(\cu{0}, \cu{P}_0)$ is independent of $\cu{W}(t)$ for $t\geq 0$, and $\cu{H}\,\cu{U}(t_k) = f(t_k)$ extracts the top GP $f$ from the state. We also assume that the functions $\bm{\Lambda}$ and $\bm{\beta}$ are selected suitably such that the SDE~\eqref{equ:sde-reg-model} has a weakly unique solution and imply Markov property~\citep{SDE-Friedman1975}. The deep \matern process and Example~\ref{example:matern-ss-gp} satisfy the required two conditions, provided that function $g$ is chosen properly.
	
	Suppose we have a set of observations $\cu{y}_{1:N} = \left\lbrace y_1, y_2,\ldots, y_N \right\rbrace$, then the posterior density of interests is
	\begin{equation}
		p(\cu{U}(t)\mid \cu{y}_{1:N}),
		\label{equ:ss-dgp-post}
	\end{equation}
	for any $t_1\leq t\leq t_N$. Since we have discrete-time measurements, let us denote by \begin{equation}
		\cu{U}_k\coloneqq\cu{U}(t_k),\nonumber
	\end{equation}
	for $k=1,2,\ldots, N$ and use $\cu{U}_{1:N} = \left\lbrace \cu{U}_1,\ldots, \cu{U}_N \right\rbrace $. Also, it would be is possible to extend the regression to classification by using a categorical measurement model \citep{gp-carl-edward,Garcia_et_al:2019a}. 
	
	\subsection{SDE Discretization}
	\label{sec:sde-discretize}
	To obtain the posterior density with discrete-time observations, we need the transition density of the SDE, such that $\cu{U}_{k+1} \sim p(\cu{U}_{k+1} \mid \cu{U}_{k})$. It is known that the transition density is the solution to the Fokker--Planck--Kolmogorov (FPK) partial differential equation~\citep[PDE, ][]{sarkka2019}. However, solving a PDE is not computationally cheap, and does not scale well for large-dimensional state. It is often more convenient to discretize the SDEs and approximate the continuous-discrete state-space model~\eqref{equ:sde-reg-model} with a discretized version
	\begin{equation}
		\begin{split}
			\cu{U}_{k+1} &= \cu{a}(\cu{U}_k) + \cu{q}(\cu{U}_k),\\
			y_k &= \cu{H}\,\cu{U}_k + r_k,
		\end{split}
		\label{equ:ss-dgp-disc}
	\end{equation}
	where $\cu{a}\colon\R^\varrho\to\R^\varrho$ is a function of state, and $\cu{q}\colon\R^\varrho\to\R^\varrho$ is a zero-mean random variable depending on the state. One of the most commonly used methods to derive functions $\cu{a}$ and $\cu{q}$ is the Euler--Maruyama scheme~\citep{peter-sde-num}. Unfortunately, the Euler--Maruyama is not applicable for many DGP models, as the covariance $\cu{q}$ would be singular. As an example, a smooth \matern ($\alpha\geq1$) GP's SDE representation gives a singular $\bm{\beta}(\cu{U}_k)\,\bm{\beta}^\trans(\cu{U}_k)$ (see Equation~\eqref{equ:matern-L}), thus the transition density $p(\cu{U}_{k+1} \mid \cu{U}_{k})$ is ill-defined. 
	
	The Taylor moment expansion (TME) is one way to proceed instead of Euler--Maruyama~\citep{zhao2020taylor, kessler1997taylorexoansion, florenz1989}. This method requires that the SDE coefficients $\bm{\Lambda}$ and $\bm{\beta}$ are differentiable and there exists an infinitesimal generator for the SDE~\citep{zhao2020taylor}. The deep \matern process satisfies these conditions provided that the wrapping function $g$ is chosen suitably. 
	
	We remark that at this point that we have formed an approximation to the SS-DGP in order to use its Markov property. This is different from the batch-DGP model where we do not utilize the Markov property for regression. In summary, we approximate the transition density
	\begin{equation}
		\begin{split}
			p(\cu{U}_{k+1} \mid \cu{U}_{k}) &\approx
			\mathcal{N}(\cu{U}_{k+1} \mid \cu{a}(\cu{U}_k), \cu{Q}(\cu{U}_k) ), \\
			\cu{Q}(\cu{U}_k) &= \cov\left[\cu{q}(\cu{U}_k)\mid \cu{U}_k\right], \nonumber
		\end{split}
	\end{equation}
	as a non-linear Gaussian, where a discretization such as Euler--Maruyama or TME is used. With the transition density formulated, we can now approximate the posterior density~\eqref{equ:ss-dgp-post} of SS-DGP using sequential methods in state-space.
	
	\subsection{State-space MAP Solution}
	\label{sec:ss-MAP}
	The MAP solution to the SS-DGP model is fairly similar to the batch-DGP model, except that we factorize the posterior density with the Markov property. Suppose that we are interested in the posterior density $p(\cu{U}_{0:N}\mid \cu{y}_{1:N})$ at $N$ discrete observation points, then we factorize the posterior density by
	\begin{equation}
		\begin{split}
			&p(\cu{U}_{0:N}\mid \cu{y}_{1:N}) \\
			&\propto p(\cu{y}_{1:N}\mid \cu{U}_{0:N})\,p(\cu{U}_{0:N}) \\
			&= \prod^N_{k=1}\mathcal{N}(y_k\mid \cu{H}\,\cu{U}_k, R_k)\,p(\cu{U}_0)\,\prod^N_{k=1} p(\cu{U}_k\mid \cu{U}_{k-1}).
		\end{split}
		\label{equ:ss-MAP-post}
	\end{equation}
	
	By taking the negative logarithm on both sides of Equation~\eqref{equ:ss-MAP-post}, the MAP estimate of SS-DGP is given by 
	\begin{equation}
		\cu{U}_{0:N}^{\mathrm{SMAP}}= \argmin_{\cu{U}_{0:N}} \mathcal{L}^{\mathrm{SMAP}}(\cu{U}_{0:N};\cu{y}_{1:N}),
	\end{equation}
	where
	\begin{equation}
		\begin{split}
			&\mathcal{L}^{\mathrm{SMAP}}(\cu{U}_{0:N};\cu{y}_{1:N}) \\
			&= -\log \left[p(\cu{y}_{1:N}\mid \cu{U}_{1:N})\,p(\cu{U}_0)\prod^N_{k=1} p(\cu{U}_k\mid \cu{U}_{k-1})\right] \\
			&=  \frac{1}{2}\sum^N_{k=1}\left[ \frac{1}{R_k}(y_k - \cu{H}\,\cu{U}_k)^2 + \log\det(2\,\pi\,R_k)\right] \\
			&+\sum^N_{k=1}\Big[\left(\cu{U}_k-\cu{a}(\cu{U}_{k-1})\right)^\trans\cu{Q}^{-1}(\cu{U}_{k-1})\left(\cu{U}_k-\cu{a}(\cu{U}_{k-1})\right) \\
			&\quad+ \log\det(2\,\pi\,\cu{Q}(\cu{U}_{k-1})) \Big]\times \frac{1}{2} \\
			&+\frac{1}{2} \left[ \cu{U}_0^\trans\,\cu{P}_0^{-1}\,\cu{U}_0 + \log\det(2\,\pi\,\cu{P}_0)\right] .
		\end{split}
		\label{equ:ss-MAP-loss-func}
	\end{equation}
	
	The corresponding gradient of~\eqref{equ:ss-MAP-loss-func} is given in Appendix~\ref{append:ss-MAP}. The computational complexity of this SS-DGP MAP estimation is $\mathcal{O}(N \,(d\,\sum^{L}_{i=1}L_i)^3)$ which is in contrast with the complexity $\mathcal{O}(N^3\,\sum^{L}_{i=1}L_i)$ of the batch-DGP. We see that the state-space MAP solution has an advantage with large dataset, as the computational complexity is linear with respect to the number of data points $N$. 
	
	The state-space MAP method also has the problem that it is inherently a point estimate. One way to proceed is to use a Bayesian filter and smoother instead of the MAP estimates~\citep{sarkka2013}. 
	
	\subsection{Bayesian Filtering and Smoothing Solution}
	\label{sec:ss-BFS}
	Recall the original SS-DGP model~\eqref{equ:sde-reg-model}. The estimation of the state from an observed process is equivalent to computing the posterior distribution \eqref{equ:ss-dgp-post} which in turn is equivalent to a continuous-discrete filtering and smoothing problem~\citep{Jazwinski1970, sarkka2019}. Compared to the MAP solution, the Bayesian smoothing approaches offer the full posterior distribution instead of a point estimate.
	
	The core idea of Bayesian smoothing is to utilize the Markov property of the process and approximate the posterior density recursively at each time step. In particular, we are interested in the filtering posterior 
	\begin{equation}
		p(\cu{U}_k\mid\cu{y}_{1:k}),
		\label{equ:filtering-post}
	\end{equation}
	and the smoothing posterior
	\begin{equation}
		p(\cu{U}_{k}\mid \cu{y}_{1:N}),
		\label{equ:smoothing-post}
	\end{equation}
	for any $k=1,2,\ldots, N$. There are many well-known methods to obtain the above posterior densities, such as the Kalman filter and Rauch--Tung--Striebel smoother for linear Gaussian state-space models. Typical methods for non-linear SS-DGP models are the Gaussian filters and smoothers~\citep{sarkka2013gaussian, Kushner-early1967, Kazufumi2000}. Some popular examples are the extended Kalman filter and smoother (EKF/EKS), and the unscented or cubature Kalman filter and smoothers (UKF/UKS/CKF/CKS). The significant benefit of Gaussian filters and smoothers is the computational efficiency, as they scale linearly with the number of measurements. 
	
	\begin{remark}
		\label{remark:ss-GFS-problem}
		Although the Gaussian filters and smoothers are beneficial choices in terms of computation, there are certain limitations when applying them to DGP regression. We elucidate this peculiar characteristic in Section~\ref{sec:restriction-GFS}. 
	\end{remark}
	
	Instead of Gaussian filters and smoothers, we can use a particle filter and smoother on a more general setting of DGPs~\citep{Simon2004, Christophe2010}. Typical choices are the bootstrap particle filter~\citep[PF,][]{Gorden1993} with resampling procedures~\citep{Kitagawa1996} and the backward-simulation particle smoother~\citep[PF-BS,][]{Simon2004}. However, particle filters and smoothers do not usually scale well with the dimension of state-space, as we need more particles to represent the probability densities in higher dimensions. Other non-Gaussian assumed density filters and smoothers might also apply, for example, the projection filter and smoother~\citep{projection1998, ShinsukeProjSm2018}.
	
	\section{Analysis on Gaussian Approximated DGP Posterior}
	\label{sec:restriction-GFS}
	Gaussian filters are particularly efficient methods, which approximate the DGP posterior~\eqref{equ:filtering-post} and the predictive density $p(\cu{U}_k\mid\cu{y}_{1:k-1})$ as Gaussian~\citep{Kazufumi2000}. Under linear additive Gaussian measurement models, the posterior density is approximated analytically by applying Gaussian identities. However, we are going to show that this type of Gaussian approximation is not useful for all constructions of DGPs. In particular, we show that the estimated posterior covariance of the observed GP $f(t)$ and an inner GP $\sigma(t)$ approaches to zero as $t\to\infty$. As a consequence, the Gaussian filtering update for $\sigma(t)$ will not use information from measurements as $t\to\infty$. 
	
	Hereafter, we restrict our analysis to a certain construction of DGPs and a class of Gaussian approximations (filters) for which we can prove the covariance vanishing property. 
	Therefore, in Section~\ref{sec:prep-assump} we define a construction of DGPs, and in Algorithm~\ref{alg:GFS} we formulate a type of Gaussian filters. The main result is revealed in Theorem~\ref{thm:cov-post}. 
	
	We organize the proofs as follows. First we show that at every time step the predictions from DGPs give vanishing prior covariance (in Lemma~\ref{lemma:cov-form}). Then we show that the Gaussian filter update step also shrinks the covariance (in Theorem~\ref{thm:cov-post}). Finally we prove the vanishing posterior covariance by mathematical induction over all time steps as in Theorem~\ref{thm:cov-post}. 
	
	\subsection{Preliminaries and Assumptions}
	\label{sec:prep-assump}
	Let $f\colon\mathbb{T}\to\R$ and $u_\sigma\colon\mathbb{T}\to\R$ be the solution to the pair of SDEs
	\begin{equation}
		\begin{split}
			\diff f(t) &= \mu(u_\ell(t))\,f(t) \diff t + \theta(u_\ell(t), u_\sigma(t) )\diff W_f(t), \\
			\diff u_\sigma(t) &= a(u_v(t))\,u_\sigma(t)\diff t + b(u_v(t))\diff W_\sigma(t),
		\end{split}
		\label{equ:analysis-sde}
	\end{equation}
	for $t \geq t_0$ starting from random initial conditions $f(t_0), u_\sigma(t_0)$ which are independent of the Wiener processes $W_f(t) \in\R$ and $W_\sigma(t) \in\R$. In addition, $u_\ell\colon\mathbb{T}\to\R$ and $u_v\colon\mathbb{T}\to\R$ are two independent processes driving the SDEs~\eqref{equ:analysis-sde},  which are also independent of $W_f(t) \in\R$ and $W_\sigma(t) \in\R$ for $t\geq t_0$. 
	
	Let $y(t_k) = f(t_k) + r(t_k)$ be the noisy observation of $f(t)$ at time $t_k$, where $r(t_k)\sim\mathcal{N}(0,R_k)$ and $k=1,2,\ldots$. Also let $\cu{y}_{1:k} = \left\lbrace y_1,\ldots, y_k \right\rbrace $ and $\Delta t = t_k-t_{k-1}>0$ for all $k$. We make the following assumptions.
	\begin{assumption}\label{assump:solution}
		The functions $\mu\colon\R\to(-\infty, 0)$, $\theta\colon \R\times\R\to\R$, $a\colon\R\to(-\infty, 0)$, and $b\colon\R\times\R\to\R$ and the initial conditions $f(t_0)$, $u_\sigma(t_0)$, $u_\ell(t_0)$, and $u_v(t_0)$ are chosen regular enough so that the solution to SDEs~\eqref{equ:analysis-sde} exists.
	\end{assumption}
	\begin{assumption}\label{assump:init}
		$\expec\left[f^2(t_0)\right]<\infty$, $\expec\left[u_\sigma^2(t_0)\right]<\infty$, and $\expec\left[(f(t_0)\,u_\sigma(t_0))^2\right]<\infty$.
	\end{assumption}
	\begin{assumption}\label{assump:neg}
		There exists constants $C_\mu<0$ and $C_a<0$ such that $(\mu\circ u_\ell)(t)\leq C_\mu$ and $(a\circ u_v)(t)\leq C_a$ almost surely.
	\end{assumption}
	\begin{assumption}\label{assump:bound}
		$\expec\left[(\mu(u_\ell(t))f(t))^2\right]\leq C<\infty$ almost everywhere and $\expec\left[\theta^2(u_\ell(t), u_\sigma(t))\right]<\infty$. Also $\expec\left[\theta^2(u_\ell(t), u_\sigma(t))\right] \geq C_\theta>0$ almost everywhere.
	\end{assumption}
	\begin{assumption}\label{assump:var-R}
		There exists a constant $C_R>0$ such that $R_k\geq C_R$ for all $k=1,2,\ldots,$ or there exists a $k$ such that $R_k=0$.
	\end{assumption}
	
	The solution existence in Assumption~\ref{assump:solution} is the prerequisite for the analysis of SDEs~\eqref{equ:analysis-sde}~\citep{Kuo2006Book, oksendal}. Assumption~\ref{assump:init} ensures that the SDEs start from a reasonable condition which is used in Lemma~\ref{lemma:cov-form}. Assumption~\ref{assump:neg} postulates negativity on functions $\mu$ and $a$. It implies that the sub-processes $f$ and $u_\sigma$ stay near zero. Also, the negativity guarantees the positivity of lengthscale (e.g., the lengthscale of $f(t)$ is $-\mu(u_\ell(t))$). Assumption~\ref{assump:bound} yields a lower bound on the variance of $f$ as stated in Corollary~\ref{cor:var-f-bound}. Finally, Assumption~\ref{assump:var-R} means that the measurement noise admits a lower bound uniformly which is used in Theorem~\ref{thm:cov-post}. This assumption also allows for perfect measurements (i.e., no measurement noises).
	
	The above SDEs~\eqref{equ:analysis-sde} and Assumptions~\ref{assump:solution}-\ref{assump:var-R} correspond to a type of DGP constructions. In particular, $f$ is a conditional GP given $u_\sigma$ and $u_\ell$. Also, $u_\sigma$ is another conditional GP given $u_v$. The processes $u_\ell$ and $u_v$ are two independent processes that drive $f$ and $u_\sigma$. The \matern DGP in Example~\ref{example:matern-ss-gp} satisfies the above assumptions, if we choose Gaussian initial conditions and a regular wrapping function by, for example, $g(u)=u^2 + c$  and $c>0$. 
	
	\subsection{Theoretical Results}
	The following Lemma~\ref{lemma:cov-form} shows that the covariance of $f(t)$ and $u_\sigma(t)$ approaches to zero as $t\to\infty$. 
	\begin{lemma}
		\label{lemma:cov-form}
		Under Assumptions~\ref{assump:solution} to~\ref{assump:neg}, 
		\begin{equation}
			\lim_{t\to \infty} \cov[f(t),u_\sigma(t)] = 0.
			\label{equ:cov-limit}
		\end{equation} 
	\end{lemma}
	\begin{proof}
		Let $m_f(t) \coloneqq \expec[f(t)]$, $m_\sigma(t) \coloneqq \expec[u_\sigma(t)]$. By It\^{o}'s lemma~\citep[see, e.g., Theorem 4.2 of][]{sarkka2019}, 
		\begin{equation}
			\begin{split}
				&\diff \left(f(t)\,u_\sigma(t)\right) \\
				&\quad= \left[ u_\sigma(t)\,\mu(u_\ell(t))\,f(t) +f(t)\, a(u_v(t))\,u_\sigma(t)\right] \diff t \\
				&\quad\quad+\frac{1}{2} [u_\sigma(t)\,\theta(u_\ell(t),u_\sigma(t))\diff W_f(t)\\
				&\quad\quad\quad+ f(t)\,b(u_v(t))\diff W_\sigma(t) ].
			\end{split}
			\label{equ:thm1-ito-lemma}
		\end{equation}
		To analyze the relation between $f$ and $u_\sigma$, we need to fix the information from $u_v$ and $u_\ell$. Hence, let $\mathcal{F}^v_t$ and $\mathcal{F}^\ell_t$ be the generated filtrations of $u_v(t)$ and $u_\ell(t)$, respectively. Taking conditional expectations on the above Equation~\eqref{equ:thm1-ito-lemma} gives
		\begin{equation}
			\begin{split}
				&\diff \expec\left[f(t)\,u_\sigma(t)\mid \mathcal{F}^v_t, \mathcal{F}^\ell_t\right] \\
				&= \expec\left[ \left( \mu(u_\ell(t)) + a(u_v(t))\right) \,f(t)\,u_\sigma(t)\mid \mathcal{F}^v_t, \mathcal{F}^\ell_t\right]\diff t \\
				&= (\mu(u_\ell(t)) + a(u_v(t)))\, \expec\left[f(t)\,u_\sigma(t)\mid \mathcal{F}^v_t, \mathcal{F}^\ell_t\right]\,\diff t. \nonumber
			\end{split}
		\end{equation}
		Thus
		\begin{equation}
			\begin{split}
				&\expec\left[f(t)\,u_\sigma(t)\mid \mathcal{F}^v_t, \mathcal{F}^\ell_t\right] \\
				&= \expec[f(t_0)\,u_\sigma(t_0)\mid u_v(t_0), u_\ell(t_0)] \,e^{\int^t_{t_0}\mu(u_\ell(s))+a(u_v(s))\diff s}.\nonumber
			\end{split}
		\end{equation}
		Using the same approach, we derive 
		\begin{equation}
			\begin{split}
				\expec\left[f(t)\mid \mathcal{F}^\ell_t\right] &= \expec[f(t_0)\mid u_\ell(t_0)]\,e^{\int^t_{t_0}\mu(u_\ell(s))\diff s}, \\
				\expec\left[u_\sigma(t)\mid \mathcal{F}^v_t\right] &= \expec[u_\sigma(t_0)\mid u_v(t_0)]\,e^{\int^t_{t_0}a(u_v(s))\diff s}.
			\end{split}
		\end{equation}
		Then by law of total expectation, we recover \begin{equation}
			\begin{split}
				&\cov[f(t), u_\sigma(t)] \\
				&= \expec\Big[\expec[f(t_0)\, u_\sigma(t_0)\mid u_\ell(t_0), u_v(t_0)] \\
				&\quad\quad\times e^{\int^t_{t_0}\mu(u_\ell(s))+a(u_v(s))\diff s}\Big] \\
				&\quad - \expec\left[\expec[f(t_0)\mid u_\ell(t_0)]\,e^{\int^t_{t_0}\mu(u_\ell(s))\diff s}\right] \\
				&\quad\quad\times\expec\left[\expec[u_\sigma(t_0)\mid u_v(t_0)]\,e^{\int^t_{t_0}a(u_v(s))\diff s}\right].
			\end{split}
			\label{equ:total-cov}
		\end{equation}
		Taking the limit of Equation~\eqref{equ:total-cov} gives 
		\begin{equation}
			\begin{split}
				&\lim_{t\to\infty}\cov[f(t), u_\sigma(t)] \\
				&= \lim_{t\to\infty}\expec\Big[\expec[f(t_0)\, u_\sigma(t_0)\mid u_\ell(t_0), u_v(t_0)]\\
				&\quad\quad \quad\quad\times e^{\int^t_{t_0}\mu(u_\ell(s))+a(u_v(s))\diff s}\Big] \\
				&\quad - \lim_{t\to\infty}\expec\left[\expec[f(t_0)\mid u_\ell(t_0)]\,e^{\int^t_{t_0}\mu(u_\ell(s))\diff s}\right]\\
				&\quad\quad\times\lim_{t\to\infty}\expec\left[\expec[u_\sigma(t_0)\mid u_v(t_0)]\,e^{\int^t_{t_0}a(u_v(s))\diff s}\right],
				\label{equ:cov-limit-all}
			\end{split}
		\end{equation}
		where all the three limits on the right side turn out to be zero. Let us first focus on $\expec\left[\expec[f(t_0)\mid u_\ell(t_0)]\,e^{\int^t_{t_0}\mu(u_\ell(s))\diff s}\right]$. By Jensen's inequality~\citep[see, e.g., Theorem 7.9 of][]{AchimKlenke2014}
		\begin{equation}
			\begin{split}
				&\abs{\expec\left[\expec[f(t_0)\mid u_\ell(t_0)]\,e^{\int^t_{t_0}\mu(u_\ell(s))\diff s}\right]} \\
				& \leq\expec\left[\abs{\expec[f(t_0)\mid u_\ell(t_0)]\,e^{\int^t_{t_0}\mu(u_\ell(s))\diff s}}\right], \nonumber
			\end{split}
		\end{equation}
		for $t \in \mathbb{T}$. Then by H\"{o}lder's inequality~\citep[see, e.g., Theorem 7.16 of][]{AchimKlenke2014}, the above inequality continues as
		\begin{equation}
			\begin{split}
				&\expec\left[\abs{\expec[f(t_0)\mid u_\ell(t_0)]\,e^{\int^t_{t_0}\mu(u_\ell(s))\diff s}}\right] \\
				&\leq \sqrt{\expec\left[\expec^2[f(t_0)\mid u_\ell(t_0)]\right]}\, \sqrt{ \expec\left[e^{2\int^t_{t_0}\mu(u_\ell(s))\diff s}\right]}.\nonumber
			\end{split}
		\end{equation}
		Now by using Assumption~\ref{assump:neg}, we know that there exists a constant $C_\mu<0$ such that $(\mu\circ u_\ell)(t)\leq C_\mu$ almost surely. Hence
		\begin{equation}
			\begin{split}
				\expec\left[e^{\int^t_{t_0}\mu(u_\ell(s))\diff s}\right] &\leq \expec\left[e^{C_\mu\,(t-t_0)} \right] = e^{C_\mu\,(t-t_0)},\nonumber
			\end{split}
		\end{equation}
		for all $t>t_0$. Therefore
		\begin{equation}
			\begin{split}
				&\lim_{t\to\infty}\expec\left[\expec[f(t_0)\mid u_\ell(t_0)]\,e^{\int^t_{t_0}\mu(u_\ell(s))\diff s}\right] \\
				&\quad= \sqrt{\expec\left[\expec^2[f(t_0)\mid u_\ell(t_0)]\right]}\,\lim_{t\to\infty}e^{2\,C_\mu\,(t-t_0)} =0.\nonumber
			\end{split}
		\end{equation}
		Assumption~\ref{assump:init} ensures that $\expec\left[\expec^2[f(t_0)\mid u_\ell(t_0)]\right]$ is finite. Similarly, we obtain the zero limits for the rest of the terms in Equation~\eqref{equ:cov-limit-all}. Thus limit~\eqref{equ:cov-limit} holds. 
	\end{proof}
	
	The almost sure negativity (i.e., Assumption~\ref{assump:neg}) on functions $\mu(\cdot)$ and $a(\cdot)$ is the key condition we need to have for the covariance to vanish to zero in infinite time. These conditions are often true in an SDE representation of a DGP because $\mu(\cdot)$ and $a(\cdot)$ ensure the positivity of lengthscales. 
	
	Before analyzing the posterior covariance, we need to construct a positive lower bound on the variance of $f(t)$, which is given in Lemma~\ref{lemma:var-f-bound} and Corollary~\ref{cor:var-f-bound}.
	
	\begin{lemma}
		\label{lemma:var-f-bound}
		Under Assumption~\ref{assump:solution}, for any $\epsilon > 0$, there is $\zeta>0$ such that
		\begin{equation}
			\begin{split}
				&\varr[f(t)] \geq \frac{1}{z(t)}\int^t_{t_0}z(s)\,\Big(\expec\left[\theta^2(u_\ell(s), u_\sigma(s))\right] \\
				&\quad\quad\quad\quad\quad\quad\quad- 2\,\epsilon\,\sqrt{\expec\left[(\mu(u_\ell(s))\,f(s))^2\right]}\Big) \diff s,
			\end{split}
			\label{equ:var-f-bound2}
		\end{equation}
		where $z(t) = \exp{\int^t_{t_0}2\,\zeta\,\sqrt{\expec[(\mu(u_\ell(s))\,f(s))^2]}\diff s}$.
	\end{lemma}
	\begin{proof}
		Denote by $P(t) \coloneqq \varr[f(t)] = \expec[(f(t) - \expec[f(t)])^2]$. By applying It\^{o}'s lemma on $(f(t) - \expec[f(t)])^2$ and taking expectation, we obtain
		\begin{equation}
			\begin{split}
				P(t) &= P(t_0) + 2 \int^t_{t_0} \expec[\mu(u_\ell(s))\,f(s)(f(s) - \expec[f(s)])]\,\diff s \\
				&\quad+ \int^t_{t_0} \expec\left[\theta^2(u_\ell(s), u_\sigma(s))\right]\,\diff s,
				\label{equ:lemma-var-f-Pt}
			\end{split}
		\end{equation}
		where the initial $P(t_0)>0$. By Jensen's and H\"{o}lder's inequalities~\citep{AchimKlenke2014}, 
		\begin{equation}
			\begin{split}
				\begin{split}
					&\abs{\expec[\mu(u_\ell(t))\,f(t)(f(t) - \expec[f(t)])]} \\
					&\quad\leq \sqrt{\expec[(\mu(u_\ell(t))\,f(t))^2]} \sqrt{P(t)}.\nonumber
				\end{split}
			\end{split}
		\end{equation}
		We now form a linear bound on $\sqrt{P(t)}$ such that for any $\epsilon>0$, there is $\zeta>0$ such that $\sqrt{P(t)}\leq \epsilon + \zeta P(t)$. Next, to prove the bound~\eqref{equ:var-f-bound2}, we use the differential form of~\eqref{equ:lemma-var-f-Pt} and get
		\begin{equation}
			\begin{split}
				&\frac{\diff P(t)}{\diff t} \\
				&\geq -2\,\sqrt{\expec[(\mu(u_\ell(t))\,f(t))^2]} \sqrt{P(t)} + \expec[\theta^2(u_\ell(t), u_\sigma(t))] \\
				&\geq -2\,\zeta\, \sqrt{\expec[(\mu(u_\ell(t))\,f(t))^2]}\, P(t) \\
				&\quad+\left(\expec\left[\theta^2(u_\ell(t), u_\sigma(t))\right] - 2\,\epsilon\,\sqrt{\expec[(\mu(u_\ell(t))\,f(t))^2]}\right),\nonumber
			\end{split}
		\end{equation}
		Now, we introduce $z(t) = \exp{\int^t_{t_0}2\,\zeta\,\sqrt{\expec[(\mu(u_\ell(s))\,f(s))^2]}\diff s}$, and then by integrating factor method on $\frac{\diff}{\diff t} (z(t)\,P(t))$, we recover the bound~\eqref{equ:var-f-bound2}. 
	\end{proof}
	
	\begin{corollary}
		\label{cor:var-f-bound}
		Under Assumptions~\ref{assump:solution} and~\ref{assump:bound}, there exists $\epsilon>0$ and $C_{F}(\Delta t)>0$ such that
		\begin{equation}
			\varr[f(t)] \geq C_F(\Delta t).
			\label{equ:var-f-bound-dt}
		\end{equation}
		\begin{proof}
			From Lemma~\ref{lemma:var-f-bound}, we know that for any $\epsilon>0$, there is $\zeta>0$ such that Equation~\eqref{equ:var-f-bound2} holds. By Assumption~\ref{assump:bound}, we have $1\leq z(t)\leq \exp(2\,\zeta\,\Delta t\,\sqrt{C})$. Also, we have $\expec\left[\theta^2(u_\ell(t), u_\sigma(t))\right] - 2\,\epsilon\,\sqrt{\expec\left[(\mu(u_\ell(t))\,f(t))^2\right]} \geq C_\theta - 2\,\epsilon\,\sqrt{C}$ almost everywhere. Thus let us choose any small enough $\epsilon < \frac{C_\theta}{2\,\sqrt{C}}$ so that $C_\theta - 2\,\epsilon\,\sqrt{C}>0$. Now let $C_F = \frac{\left( C_\theta - 2\,\epsilon\,\sqrt{C}\right) \Delta t}{\exp(2\,\zeta\,\Delta t\,\sqrt{C})}$ hence Equation~\eqref{equ:var-f-bound-dt} holds. Note that the inequality~\eqref{equ:var-f-bound-dt} only depends on $\Delta t$ and some fixed parameters of the SDEs. 
		\end{proof}
	\end{corollary}
	
	The following Algorithm~\ref{alg:GFS} formulates a partial procedure for estimating the posterior density using a Gaussian approximation. In particular, Algorithm~\ref{alg:GFS} gives an approximation
	\begin{equation}
		P^{f,\sigma}_k \approx \cov[f(t_k), u_\sigma(t_k)\mid \cu{y}_{1:k}],\nonumber
	\end{equation} 
	to the posterior covariance for $k=1,2,\ldots$. In order to do so, we need to make predictions through SDEs~\eqref{equ:analysis-sde} based on different starting conditions at each time step. Hence let us introduce two notations as following. We denote by 
	\begin{equation}
		\cov[f(t), u_\sigma(t)](c_{0}),\nonumber
	\end{equation} 
	and 
	\begin{equation}
		\varr[f(t)](s_0),\nonumber
	\end{equation} 
	the functions of $t$ in Equations~\eqref{equ:total-cov} and~\eqref{equ:lemma-var-f-Pt} starting from initial values $c_0\in\R$ and $s_0\in(0,+\infty)$ at $t_{0}$, respectively. 
	
	\begin{algorithm}[Gaussian posterior approximation for $P^{f,\sigma}_k$]\label{alg:GFS}
		Let us approximate the posterior densities $p(f(t_k), u_\sigma(t_k)\mid \cu{y}_{1:k})$ by Gaussian densities for $k=1,2,\ldots$. Suppose that the initial condition is known and particularly $P^{f,\sigma}_0 \coloneqq \cov[f(t_0), u_\sigma(t_0)]$ and $P^{f,f}_0 \coloneqq \varr[f(t_0)]$. Then starting from $k=1$ we calculate
		\begin{equation}
			\Bar{P}^{f,\sigma}_k = \cov[f(t_k), u_\sigma(t_k)](P^{f,\sigma}_{k-1}),
			\label{equ:GF-pred-fs}
		\end{equation}
		and 
		\begin{equation}
			\Bar{P}^{f,f}_k=\varr[f(t_k)](P^{f,f}_{k-1}),
			\label{equ:GF-pred-f}
		\end{equation}
		through the SDEs~\eqref{equ:analysis-sde} and update
		\begin{equation}
			P^{f,\sigma}_k = \Bar{P}^{f,\sigma}_k - \frac{\Bar{P}^{f,f}_k\, \Bar{P}^{f,\sigma}_k}{\Bar{P}^{f,f}_k + R_k},
			\label{equ:GF-update}
		\end{equation}
		for $k=1,2,\ldots$.
	\end{algorithm}
	
	\begin{remark}
		The above Algorithm~\ref{alg:GFS} is a abstraction of continuous-discrete Gaussian filters~\citep{Kazufumi2000, sarkka2019}, except that the predictions through SDEs~\eqref{equ:analysis-sde} are done exactly in Equations~\eqref{equ:GF-pred-fs} and~\eqref{equ:GF-pred-f}. The derivation of Equation~\eqref{equ:GF-update} is shown in Appendix~\ref{append:derivation-kf-update}. Note that in practice the predictions might also involve various types of Gaussian approximations and even numerical integrations (e.g., sigma-point methods). 
	\end{remark}
	
	\begin{theorem}
		\label{thm:cov-post}
		Suppose that Assumptions~\ref{assump:solution} to~\ref{assump:var-R} hold. Further assume that $\abs{\cov[f(t), u_\sigma(t)](c_0)}\leq \abs{c_0}$ for all $c_0\in\R$, then Algorithm~\ref{alg:GFS} gives
		\begin{equation}
			\lim_{k\to\infty} P^{f,\sigma}_k = 0.
			\label{equ:cov-fs-post}
		\end{equation}
	\end{theorem}
	\begin{proof}
		We are going to use induction to prove that the claim 
		\begin{equation}
			\abs{P^{f,\sigma}_k}\leq \abs{P_0^{f,\sigma}}\, \prod^k_{i=1} M_i
			\label{equ:induction}
		\end{equation}
		holds for all $k=1,2,\ldots$, where $M_i = \frac{R_i}{\Bar{P}^{f,f}_i + R_i}$. To do so, we expand $\abs{P^{f,\sigma}_k}$ by 
		\begin{equation}
			\begin{split}
				\abs{P^{f,\sigma}_k} &= \abs{\Bar{P}^{f,\sigma}_k - \frac{\Bar{P}^{f,f}_k\, \Bar{P}^{f,\sigma}_k}{\Bar{P}^{f,f}_k + R_k}} \\
				&= \abs{\frac{\Bar{P}^{f,\sigma}_k\,R_k}{\Bar{P}^{f,f}_k + R_k}} = M_k\, \abs{\Bar{P}^{f,\sigma}_k}\leq M_k \abs{P^{f,\sigma}_{k-1}}.
				\label{equ:induction2}
			\end{split}
		\end{equation}
		Now we can verify that $\abs{P^{f,\sigma}_1}\leq\abs{P_0^{f,\sigma}}\,M_1$ when $k=1$, which satisfies the induction claim~\eqref{equ:induction}. Suppose that Equation~\eqref{equ:induction} holds for a given $k>1$, then we can calculate Equation~\eqref{equ:induction2} at $k+1$ giving
		\begin{equation}
			\begin{split}
				\abs{P^{f,\sigma}_{k+1}} &= M_{k+1}\, \abs{\Bar{P}^{f,\sigma}_{k+1}} \leq M_{k+1}\, \abs{P^{f,\sigma}_{k}} \\
				&\leq M_{k+1} \abs{P^{f,\sigma}_0}\, \prod_{i=1}^{k} M_i = \abs{P_0^{f,\sigma}}\, \prod^{k+1}_{i=1} M_i,
			\end{split}
		\end{equation}
		which satisfies the induction claim~\eqref{equ:induction}. Thus Equation~\eqref{equ:induction} holds. Above, we used the assumption $\abs{\cov[f(t), u_\sigma(t)](c_0)}\leq \abs{c_0}$ for all $c_0\in\R$ to get $\abs{\Bar{P}^{f,\sigma}_{k+1}} \leq \abs{P^{f,\sigma}_{k}}$ for any $k$. 
		
		By Corollary~\ref{cor:var-f-bound}, Assumption~\ref{assump:var-R}, and a fixed non-zero $\Delta t$, we know that $\Bar{P}^{f,f}_k$ are lower bounded uniformly over all $k$, thus $\lim_{k\to\infty}\prod^k_{i=1}M_i=0$. Hence, by taking the limit on Equation~\eqref{equ:induction}, the Equation~\eqref{equ:cov-fs-post} holds. Also, this theorem trivially holds if $R_k=0$ for some $k$ or $P^{f,\sigma}_0 = 0$ because $M_k = 0$ for all $k=1,2,\ldots$.
	\end{proof}
	
	\begin{remark}
		Note that in Theorem~\ref{thm:cov-post}, the initial bounding assumption $\abs{\cov[f(t), u_\sigma(t)](c_0)}\leq \abs{c_0}$ for all $c_0\in\R$ is needed because it is not always followed from Lemma~\ref{lemma:cov-form}. On the other hand, for any choice of $c_0\in\R$, there always exists a threshold $\eta>0$ such that for all $t > \eta $ we have $\abs{\cov[f(t), u_\sigma(t)](c_0)}\leq \abs{c_0}$ because of Lemma~\ref{thm:cov-post}.
	\end{remark}
	
	Under the result of bounded $\varr{[f(t)]}$ in Corollary~\ref{cor:var-f-bound}, the consequence of the vanishing posterior covariance in Theorem~\ref{thm:cov-post} is that the so-called Kalman gain for $u_\sigma(t)$ approaches zero asymptotically. It entails that the Kalman update for $u_\sigma(t)$ will use no information from measurements when $t\to\infty$. In the later experiment as shown in Figure~\ref{fig:exp-GFS-vanishing} we see that the corresponding estimated $u_\sigma(t)$ and covariance rapidly stabilizes to zero. 
	
	The previous Theorem~\ref{thm:cov-post} is formulated in a general sense which applies to DGP methods that use Algorithm~\ref{alg:GFS} and satisfy Assumptions~\ref{assump:solution} to~\ref{assump:var-R}. A concrete example is shown in the following Example~\ref{cor:KF-Bucy}.
	
	\begin{example}
		\label{cor:KF-Bucy}
		Consider a system of SDEs,
		\begin{equation}
			\begin{split}
				\diff f(t) &= \mu \, f(t) \diff t + u_\sigma(t)\,\diff W_f(t), \\
				\diff u_\sigma(t) &= a \, u_\sigma(t) \diff t + b\,\diff W_u(t),
				\label{equ:cor-2-sde}
			\end{split}
		\end{equation}
		starting from a Gaussian initial condition $f(t_0),u_\sigma(t_0)$, where constants $\mu<0$, $a<0$, and $b>0$. The conditions of Theorem~\ref{thm:cov-post} are now satisfied, and thus $\lim_{k\to\infty} P^{f,\sigma}_k = 0$.
	\end{example}
	
	\section{Experiments}
	\label{sec:experiment}
	\begin{figure}[t!]
		\begin{subfigure}[t]{.49\linewidth}
			\centering
			\resizebox{.98\linewidth}{!}{%
\begin{tikzpicture}[x=0.6pt,y=0.6pt,yscale=-1,xscale=1]

\draw  [line width=1.5]  (190,30) .. controls (190,18.95) and (198.95,10) .. (210,10) .. controls (221.05,10) and (230,18.95) .. (230,30) .. controls (230,41.05) and (221.05,50) .. (210,50) .. controls (198.95,50) and (190,41.05) .. (190,30) -- cycle ;

\draw  [line width=1.5]  (30,30) .. controls (30,18.95) and (38.95,10) .. (50,10) .. controls (61.05,10) and (70,18.95) .. (70,30) .. controls (70,41.05) and (61.05,50) .. (50,50) .. controls (38.95,50) and (30,41.05) .. (30,30) -- cycle ;

\draw  [fill={rgb, 255:red, 218; green, 217; blue, 217 }  ,fill opacity=1 ][line width=1.5]  (110,30) .. controls (110,18.95) and (118.95,10) .. (130,10) .. controls (141.05,10) and (150,18.95) .. (150,30) .. controls (150,41.05) and (141.05,50) .. (130,50) .. controls (118.95,50) and (110,41.05) .. (110,30) -- cycle ;
\draw [line width=1.5]    (70,30) -- (106,30) ;
\draw [shift={(110,30)}, rotate = 180] [fill={rgb, 255:red, 0; green, 0; blue, 0 }  ][line width=0.08]  [draw opacity=0] (13.4,-6.43) -- (0,0) -- (13.4,6.44) -- (8.9,0) -- cycle    ;
\draw [line width=1.5]    (190,30) -- (154,30) ;
\draw [shift={(150,30)}, rotate = 360] [fill={rgb, 255:red, 0; green, 0; blue, 0 }  ][line width=0.08]  [draw opacity=0] (13.4,-6.43) -- (0,0) -- (13.4,6.44) -- (8.9,0) -- cycle    ;

\draw (210,30) node  [font=\large]  {$f$};
\draw (130,30) node  [font=\large]  {$y$};
\draw (50,30) node  [font=\large]  {$r$};

\end{tikzpicture}
			}
			\caption{GP}
			\label{fig:reg-models-GP}
		\end{subfigure}
		\hfill
		\begin{subfigure}[t]{.49\linewidth}
			\centering
			\resizebox{.98\linewidth}{!}{%
\begin{tikzpicture}[x=0.6pt,y=0.6pt,yscale=-1,xscale=1]

\draw  [line width=1.5]  (190,30) .. controls (190,18.95) and (198.95,10) .. (210,10) .. controls (221.05,10) and (230,18.95) .. (230,30) .. controls (230,41.05) and (221.05,50) .. (210,50) .. controls (198.95,50) and (190,41.05) .. (190,30) -- cycle ;

\draw  [line width=1.5]  (30,30) .. controls (30,18.95) and (38.95,10) .. (50,10) .. controls (61.05,10) and (70,18.95) .. (70,30) .. controls (70,41.05) and (61.05,50) .. (50,50) .. controls (38.95,50) and (30,41.05) .. (30,30) -- cycle ;

\draw  [fill={rgb, 255:red, 218; green, 217; blue, 217 }  ,fill opacity=1 ][line width=1.5]  (110,30) .. controls (110,18.95) and (118.95,10) .. (130,10) .. controls (141.05,10) and (150,18.95) .. (150,30) .. controls (150,41.05) and (141.05,50) .. (130,50) .. controls (118.95,50) and (110,41.05) .. (110,30) -- cycle ;
\draw [line width=1.5]    (70,30) -- (106,30) ;
\draw [shift={(110,30)}, rotate = 180] [fill={rgb, 255:red, 0; green, 0; blue, 0 }  ][line width=0.08]  [draw opacity=0] (13.4,-6.43) -- (0,0) -- (13.4,6.44) -- (8.9,0) -- cycle    ;
\draw  [dash pattern={on 5.63pt off 4.5pt}][line width=1.5]  (110,100) .. controls (110,88.95) and (118.95,80) .. (130,80) .. controls (141.05,80) and (150,88.95) .. (150,100) .. controls (150,111.05) and (141.05,120) .. (130,120) .. controls (118.95,120) and (110,111.05) .. (110,100) -- cycle ;
\draw  [dash pattern={on 5.63pt off 4.5pt}][line width=1.5]  (190,100) .. controls (190,88.95) and (198.95,80) .. (210,80) .. controls (221.05,80) and (230,88.95) .. (230,100) .. controls (230,111.05) and (221.05,120) .. (210,120) .. controls (198.95,120) and (190,111.05) .. (190,100) -- cycle ;
\draw [line width=1.5]    (210,80) -- (210,54) ;
\draw [shift={(210,50)}, rotate = 450] [fill={rgb, 255:red, 0; green, 0; blue, 0 }  ][line width=0.08]  [draw opacity=0] (13.4,-6.43) -- (0,0) -- (13.4,6.44) -- (8.9,0) -- cycle    ;
\draw [line width=1.5]    (130,80) -- (196.32,51.58) ;
\draw [shift={(200,50)}, rotate = 516.8] [fill={rgb, 255:red, 0; green, 0; blue, 0 }  ][line width=0.08]  [draw opacity=0] (13.4,-6.43) -- (0,0) -- (13.4,6.44) -- (8.9,0) -- cycle    ;
\draw [line width=1.5]    (190,30) -- (154,30) ;
\draw [shift={(150,30)}, rotate = 360] [fill={rgb, 255:red, 0; green, 0; blue, 0 }  ][line width=0.08]  [draw opacity=0] (13.4,-6.43) -- (0,0) -- (13.4,6.44) -- (8.9,0) -- cycle    ;

\draw (210,30) node  [font=\large]  {$f$};
\draw (130,30) node  [font=\large]  {$y$};
\draw (50,30) node  [font=\large]  {$r$};
\draw (130,100) node  [font=\large]  {$\ell^2_{1,1} $};
\draw (210,100) node  [font=\large]  {$\sigma^2_{1,2} $};

\end{tikzpicture}
			}
			\caption{NS-GP}
			\label{fig:reg-models-NSGP}
		\end{subfigure}
		\vfill
		\begin{subfigure}[t]{.49\linewidth}
			\centering
			\resizebox{.98\linewidth}{!}{%
\begin{tikzpicture}[x=0.6pt,y=0.6pt,yscale=-1,xscale=1]

\draw  [line width=1.5]  (190,30) .. controls (190,18.95) and (198.95,10) .. (210,10) .. controls (221.05,10) and (230,18.95) .. (230,30) .. controls (230,41.05) and (221.05,50) .. (210,50) .. controls (198.95,50) and (190,41.05) .. (190,30) -- cycle ;

\draw  [line width=1.5]  (30,30) .. controls (30,18.95) and (38.95,10) .. (50,10) .. controls (61.05,10) and (70,18.95) .. (70,30) .. controls (70,41.05) and (61.05,50) .. (50,50) .. controls (38.95,50) and (30,41.05) .. (30,30) -- cycle ;

\draw  [fill={rgb, 255:red, 218; green, 217; blue, 217 }  ,fill opacity=1 ][line width=1.5]  (110,30) .. controls (110,18.95) and (118.95,10) .. (130,10) .. controls (141.05,10) and (150,18.95) .. (150,30) .. controls (150,41.05) and (141.05,50) .. (130,50) .. controls (118.95,50) and (110,41.05) .. (110,30) -- cycle ;
\draw [line width=1.5]    (70,30) -- (106,30) ;
\draw [shift={(110,30)}, rotate = 180] [fill={rgb, 255:red, 0; green, 0; blue, 0 }  ][line width=0.08]  [draw opacity=0] (13.4,-6.43) -- (0,0) -- (13.4,6.44) -- (8.9,0) -- cycle    ;
\draw [line width=1.5]    (190,30) -- (154,30) ;
\draw [shift={(150,30)}, rotate = 360] [fill={rgb, 255:red, 0; green, 0; blue, 0 }  ][line width=0.08]  [draw opacity=0] (13.4,-6.43) -- (0,0) -- (13.4,6.44) -- (8.9,0) -- cycle    ;
\draw  [line width=1.5]  (110,100) .. controls (110,88.95) and (118.95,80) .. (130,80) .. controls (141.05,80) and (150,88.95) .. (150,100) .. controls (150,111.05) and (141.05,120) .. (130,120) .. controls (118.95,120) and (110,111.05) .. (110,100) -- cycle ;
\draw  [line width=1.5]  (190,100) .. controls (190,88.95) and (198.95,80) .. (210,80) .. controls (221.05,80) and (230,88.95) .. (230,100) .. controls (230,111.05) and (221.05,120) .. (210,120) .. controls (198.95,120) and (190,111.05) .. (190,100) -- cycle ;
\draw [line width=1.5]    (130,80) -- (196.32,51.58) ;
\draw [shift={(200,50)}, rotate = 516.8] [fill={rgb, 255:red, 0; green, 0; blue, 0 }  ][line width=0.08]  [draw opacity=0] (13.4,-6.43) -- (0,0) -- (13.4,6.44) -- (8.9,0) -- cycle    ;
\draw [line width=1.5]    (210,80) -- (210,54) ;
\draw [shift={(210,50)}, rotate = 450] [fill={rgb, 255:red, 0; green, 0; blue, 0 }  ][line width=0.08]  [draw opacity=0] (13.4,-6.43) -- (0,0) -- (13.4,6.44) -- (8.9,0) -- cycle    ;

\draw (210,30) node  [font=\large]  {$f$};
\draw (130,30) node  [font=\large]  {$y$};
\draw (50,30) node  [font=\large]  {$r$};
\draw (130,100) node  [font=\large]  {$\ell^2_{1,1} $};
\draw (210,100) node  [font=\large]  {$\sigma^2_{1,2} $};

\end{tikzpicture}
			}
			\caption{DGP-2}
			\label{fig:reg-models-DGP2}
		\end{subfigure}
		\hfill
		\begin{subfigure}[t]{.49\linewidth}
			\centering
			\resizebox{.98\linewidth}{!}{%
\begin{tikzpicture}[x=0.6pt,y=0.6pt,yscale=-1,xscale=1]

\draw  [line width=1.5]  (190,30) .. controls (190,18.95) and (198.95,10) .. (210,10) .. controls (221.05,10) and (230,18.95) .. (230,30) .. controls (230,41.05) and (221.05,50) .. (210,50) .. controls (198.95,50) and (190,41.05) .. (190,30) -- cycle ;

\draw  [line width=1.5]  (30,30) .. controls (30,18.95) and (38.95,10) .. (50,10) .. controls (61.05,10) and (70,18.95) .. (70,30) .. controls (70,41.05) and (61.05,50) .. (50,50) .. controls (38.95,50) and (30,41.05) .. (30,30) -- cycle ;

\draw  [fill={rgb, 255:red, 218; green, 217; blue, 217 }  ,fill opacity=1 ][line width=1.5]  (110,30) .. controls (110,18.95) and (118.95,10) .. (130,10) .. controls (141.05,10) and (150,18.95) .. (150,30) .. controls (150,41.05) and (141.05,50) .. (130,50) .. controls (118.95,50) and (110,41.05) .. (110,30) -- cycle ;
\draw [line width=1.5]    (70,30) -- (106,30) ;
\draw [shift={(110,30)}, rotate = 180] [fill={rgb, 255:red, 0; green, 0; blue, 0 }  ][line width=0.08]  [draw opacity=0] (13.4,-6.43) -- (0,0) -- (13.4,6.44) -- (8.9,0) -- cycle    ;
\draw [line width=1.5]    (190,30) -- (154,30) ;
\draw [shift={(150,30)}, rotate = 360] [fill={rgb, 255:red, 0; green, 0; blue, 0 }  ][line width=0.08]  [draw opacity=0] (13.4,-6.43) -- (0,0) -- (13.4,6.44) -- (8.9,0) -- cycle    ;
\draw  [line width=1.5]  (110,100) .. controls (110,88.95) and (118.95,80) .. (130,80) .. controls (141.05,80) and (150,88.95) .. (150,100) .. controls (150,111.05) and (141.05,120) .. (130,120) .. controls (118.95,120) and (110,111.05) .. (110,100) -- cycle ;
\draw  [line width=1.5]  (190,100) .. controls (190,88.95) and (198.95,80) .. (210,80) .. controls (221.05,80) and (230,88.95) .. (230,100) .. controls (230,111.05) and (221.05,120) .. (210,120) .. controls (198.95,120) and (190,111.05) .. (190,100) -- cycle ;
\draw [line width=1.5]    (130,80) -- (196.32,51.58) ;
\draw [shift={(200,50)}, rotate = 516.8] [fill={rgb, 255:red, 0; green, 0; blue, 0 }  ][line width=0.08]  [draw opacity=0] (13.4,-6.43) -- (0,0) -- (13.4,6.44) -- (8.9,0) -- cycle    ;
\draw [line width=1.5]    (210,80) -- (210,54) ;
\draw [shift={(210,50)}, rotate = 450] [fill={rgb, 255:red, 0; green, 0; blue, 0 }  ][line width=0.08]  [draw opacity=0] (13.4,-6.43) -- (0,0) -- (13.4,6.44) -- (8.9,0) -- cycle    ;
\draw  [line width=1.5]  (30,160) .. controls (30,148.95) and (38.95,140) .. (50,140) .. controls (61.05,140) and (70,148.95) .. (70,160) .. controls (70,171.05) and (61.05,180) .. (50,180) .. controls (38.95,180) and (30,171.05) .. (30,160) -- cycle ;
\draw  [line width=1.5]  (80,160) .. controls (80,148.95) and (88.95,140) .. (100,140) .. controls (111.05,140) and (120,148.95) .. (120,160) .. controls (120,171.05) and (111.05,180) .. (100,180) .. controls (88.95,180) and (80,171.05) .. (80,160) -- cycle ;
\draw  [line width=1.5]  (140,160) .. controls (140,148.95) and (148.95,140) .. (160,140) .. controls (171.05,140) and (180,148.95) .. (180,160) .. controls (180,171.05) and (171.05,180) .. (160,180) .. controls (148.95,180) and (140,171.05) .. (140,160) -- cycle ;
\draw  [line width=1.5]  (190,160) .. controls (190,148.95) and (198.95,140) .. (210,140) .. controls (221.05,140) and (230,148.95) .. (230,160) .. controls (230,171.05) and (221.05,180) .. (210,180) .. controls (198.95,180) and (190,171.05) .. (190,160) -- cycle ;
\draw [line width=1.5]    (100,140) -- (126.67,122.22) ;
\draw [shift={(130,120)}, rotate = 506.31] [fill={rgb, 255:red, 0; green, 0; blue, 0 }  ][line width=0.08]  [draw opacity=0] (13.4,-6.43) -- (0,0) -- (13.4,6.44) -- (8.9,0) -- cycle    ;
\draw [line width=1.5]    (50,140) -- (106.42,111.79) ;
\draw [shift={(110,110)}, rotate = 513.4300000000001] [fill={rgb, 255:red, 0; green, 0; blue, 0 }  ][line width=0.08]  [draw opacity=0] (13.4,-6.43) -- (0,0) -- (13.4,6.44) -- (8.9,0) -- cycle    ;
\draw [line width=1.5]    (160,140) -- (196.42,121.79) ;
\draw [shift={(200,120)}, rotate = 513.4300000000001] [fill={rgb, 255:red, 0; green, 0; blue, 0 }  ][line width=0.08]  [draw opacity=0] (13.4,-6.43) -- (0,0) -- (13.4,6.44) -- (8.9,0) -- cycle    ;
\draw [line width=1.5]    (210,140) -- (210,124) ;
\draw [shift={(210,120)}, rotate = 450] [fill={rgb, 255:red, 0; green, 0; blue, 0 }  ][line width=0.08]  [draw opacity=0] (13.4,-6.43) -- (0,0) -- (13.4,6.44) -- (8.9,0) -- cycle    ;

\draw (210,30) node  [font=\large]  {$f$};
\draw (130,30) node  [font=\large]  {$y$};
\draw (50,30) node  [font=\large]  {$r$};
\draw (130,100) node  [font=\large]  {$\ell^2_{1,1} $};
\draw (210,100) node  [font=\large]  {$\sigma^2_{1,2} $};
\draw (50,160) node  [font=\large]  {$\ell^3_{1,1} $};
\draw (100,160) node  [font=\large]  {$\sigma^3_{1,2} $};
\draw (160,160) node  [font=\large]  {$\ell^3_{2, 3} $};
\draw (210,160) node  [font=\large]  {$\sigma^3_{2, 4} $};

\end{tikzpicture}
			}
			\caption{DGP-3}
			\label{fig:reg-models-DGP3}
		\end{subfigure}
		\caption{Graphs of four regression models. We denote by $y$ as the measurement of function $f$ contaminated by noise $r$. In (b), the processes $\ell^2_{1,1}$ and $\sigma^2_{1,2}$ in dashed circles are degenerate (learnable hyperparameters). }
		\label{fig:reg-models}
	\end{figure}
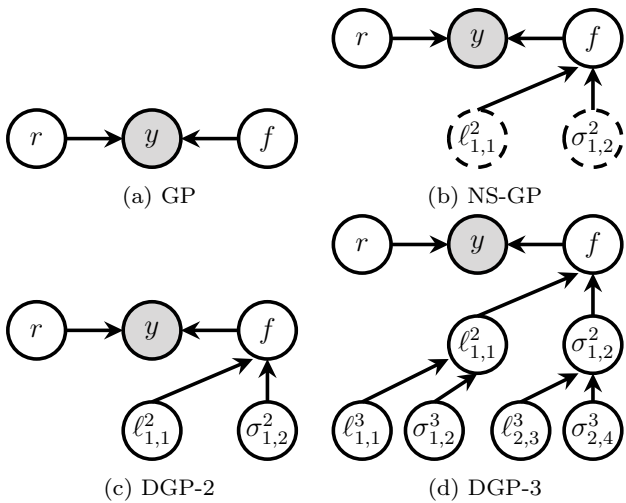
	In this section we numerically evaluate the proposed methods. The specific objectives of the experiments are as follows. First, we show the advantages of using DGPs over conventional GPs or non-stationary GPs (one-layer DGPs) in non-stationary regression. Then, we compare the batch and state-space constructions of DGPs. Finally, we examine the efficiencies of different DGP regression methods.
	
	We prepare four regression models as shown in Figure~\ref{fig:reg-models}. These models are the conventional GP~\citep{gp-carl-edward}, non-stationary GP~\citep[NS-GP,][]{paciorek2006spatial}, two-layer DGP (DGP-2), and three-layer DGP (DGP-3). The DGP-2 and DGP-3 are constructed using both the batch and state-space approaches as formulated in Sections~\ref{sec:dgp-batch-reg} and~\ref{sec:gp-sde}, respectively. In particular, we consider a \matern type of GP construction, which only has two hyperparameters (i.e., the length scale $\ell$ and magnitude $\sigma$). That is to say, we use the non-stationary \matern covariance function \citep{paciorek2006spatial} for the NS-GP and batch-DGP models, and the deep \matern process for SS-DGP model. For the wrapping function $g$, we choose $g(u)=\exp(u)$. For the discretization of SS-DGP, we use the 3rd-order TME method~\citep{zhao2020taylor}. We control the smoothness of $f$ and hyperparameter processes by using $\alpha=1$ and $0$, respectively (see Equation~\eqref{equ:matern-cov}). In addition, we draw samples from the DGP priors in Appendix~\ref{append:prior-samples}.
	\begin{figure}[t!]
		\centering
		\includegraphics[width=.8\linewidth]{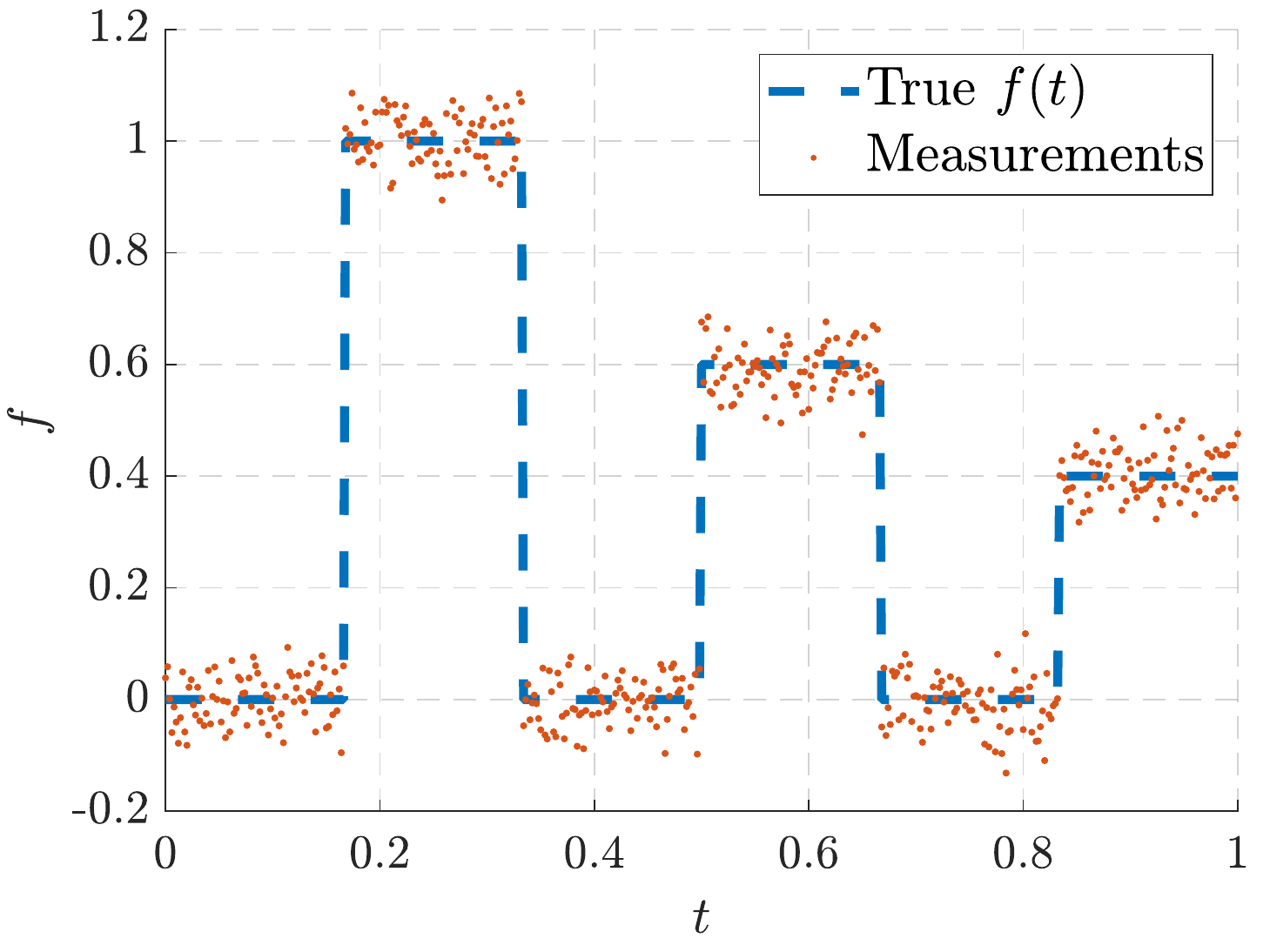}
		\caption{Demonstration of the magnitude-varying rectangle signal in Equation~\eqref{equ:exp-y} with 500 samples.}
		\label{fig:exp-y}
	\end{figure}
	
	There are unknown model hyperparameters. We use the maximum likelihood estimation (MLE) routine to optimize the hyperparameters for the GP and NS-GP models which have closed-form likelihood functions and gradients. For the DGP models, we find them by grid searches because the gradients are non-trivial to derive. We detail the found hyperparameters in Appendix~\ref{append:hyperparas}.
	
	As for the batch-DGP models, we use the proposed batch maximum a posteriori (B-MAP) method in Section~\ref{sec:batch-MAP}. Similarly for the SS-DGP, we apply the state-space MAP (SS-MAP), Gaussian filters and smoothers~\citep{sarkka2013}, and a bootstrap particle filter~\citep[PF,][]{Christophe2010, Doucet2000} and a backward-simulation particle smoother~\citep[PF-BS,][]{Simon2004}. 
	
	We use the limited-memory Broyden--Fletcher--Goldfarb--Shanno~\citep[l-BFGS,][]{numerical-opt} optimizer for MLE and MAP optimizations. For the Gaussian filters and smoothers, we exploit the commonly used linearization (EKFS) and spherical cubature method (CKFS)~\citep{sarkka2013}. As for the PF and PF-BS, we use 200,000 particles and 1,600 backward simulations. 
	
	The following experiments except the real application are computed with the Triton computing cluster at Aalto University, Finland\footnote{The companion code can be found at address https://github.com/zgbkdlm/ssdgp.}. We uniformly allocate 4 CPU cores and 4 gigabyte of memory for each of the individual experiment. In addition, the PF-BS method is implemented with CPU-based parallelization. All programs are implemented in MATLAB\textsuperscript{\textregistered} 2019b. 
	
	\subsection{Regression on Rectangle Signal}
	\label{sec:toy-data}
	In this section, we conduct regression on a magnitude-varying rectangle wave, as shown in Figure~\ref{fig:exp-y}. The regression model is formulated by
	\begin{equation}
		\begin{split}
			f(t) &= \begin{cases}
				0, & t \in [0, \frac{1}{6})\cup [\frac{2}{6}, \frac{3}{6})\cup [\frac{4}{6}, \frac{5}{6}),\\
				1, & t\in [\frac{1}{6}, \frac{2}{6}),\\
				0.6, &t\in [\frac{3}{6}, \frac{4}{6}),\\
				0.4, &t\in [\frac{5}{6}, 1],
			\end{cases}\\
			y(t) &= f(t) + r(t),
		\end{split}
		\label{equ:exp-y}
	\end{equation}
	where $f$ is the true function, $y$ is the measurement, and $r(t) \sim \mathcal{N}(0, 0.002)$. We evenly generate samples $y(t_1),\ldots, y(t_T)$, where $T=100$. The challenge of this type of signal is that the rectangle wave is continuous and flat almost everywhere while it is only right-continuous at a finite number of isolated points. Moreover, the jumps have different heights.
	
	\begin{figure}[t!]
		\centering
		\includegraphics[width=.49\linewidth]{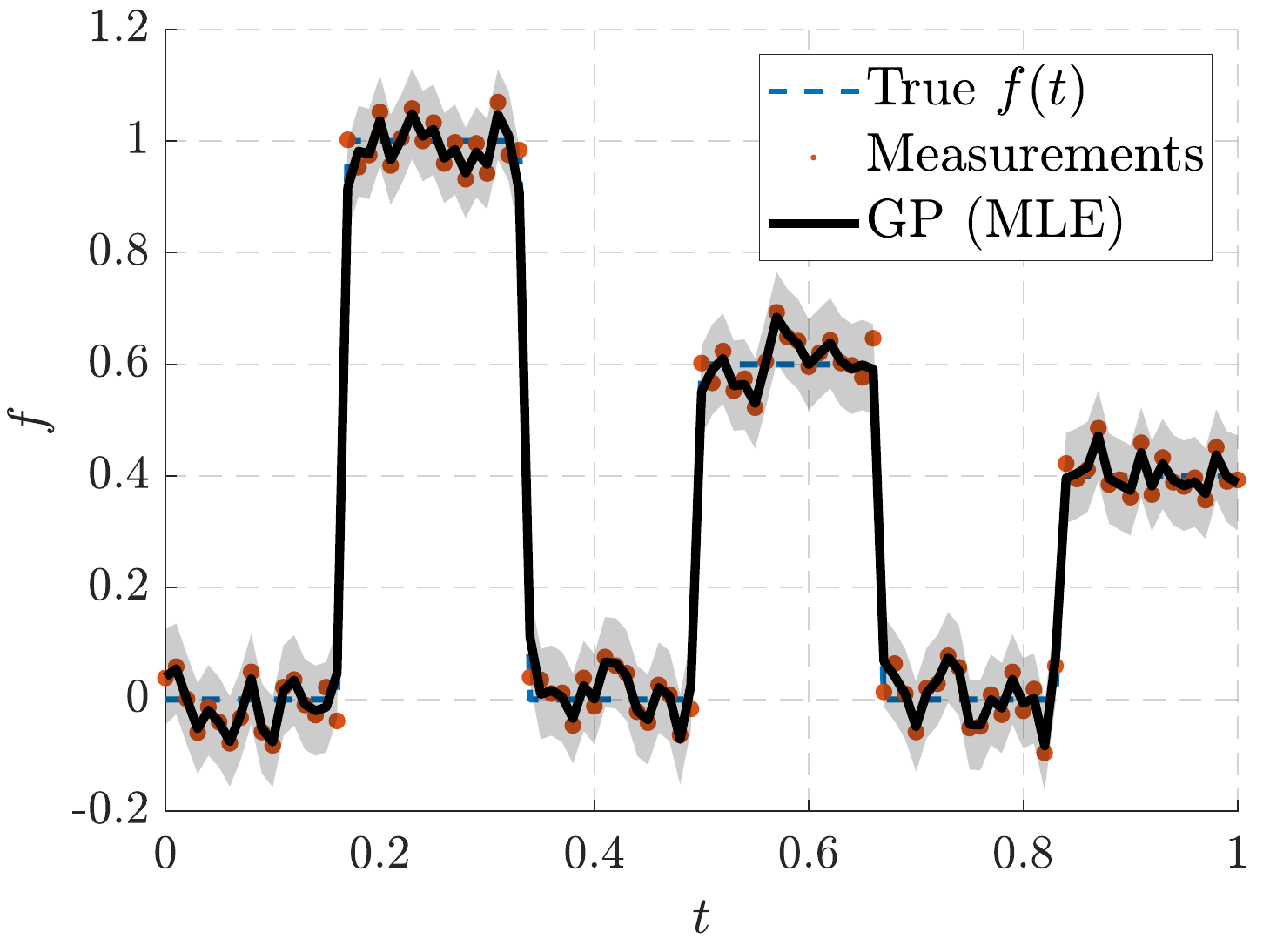}
		\includegraphics[width=.49\linewidth]{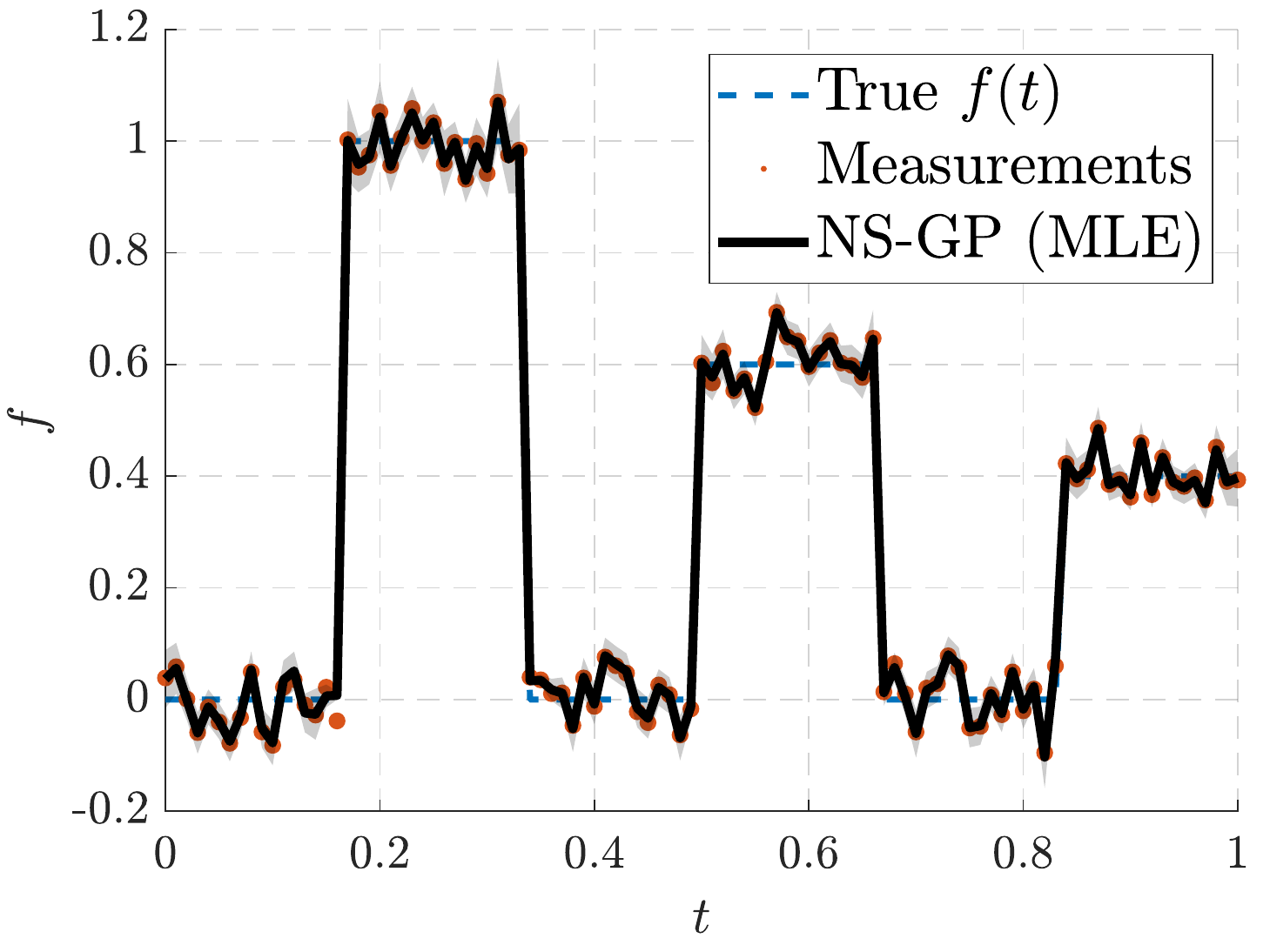}\\
		\includegraphics[width=.49\linewidth]{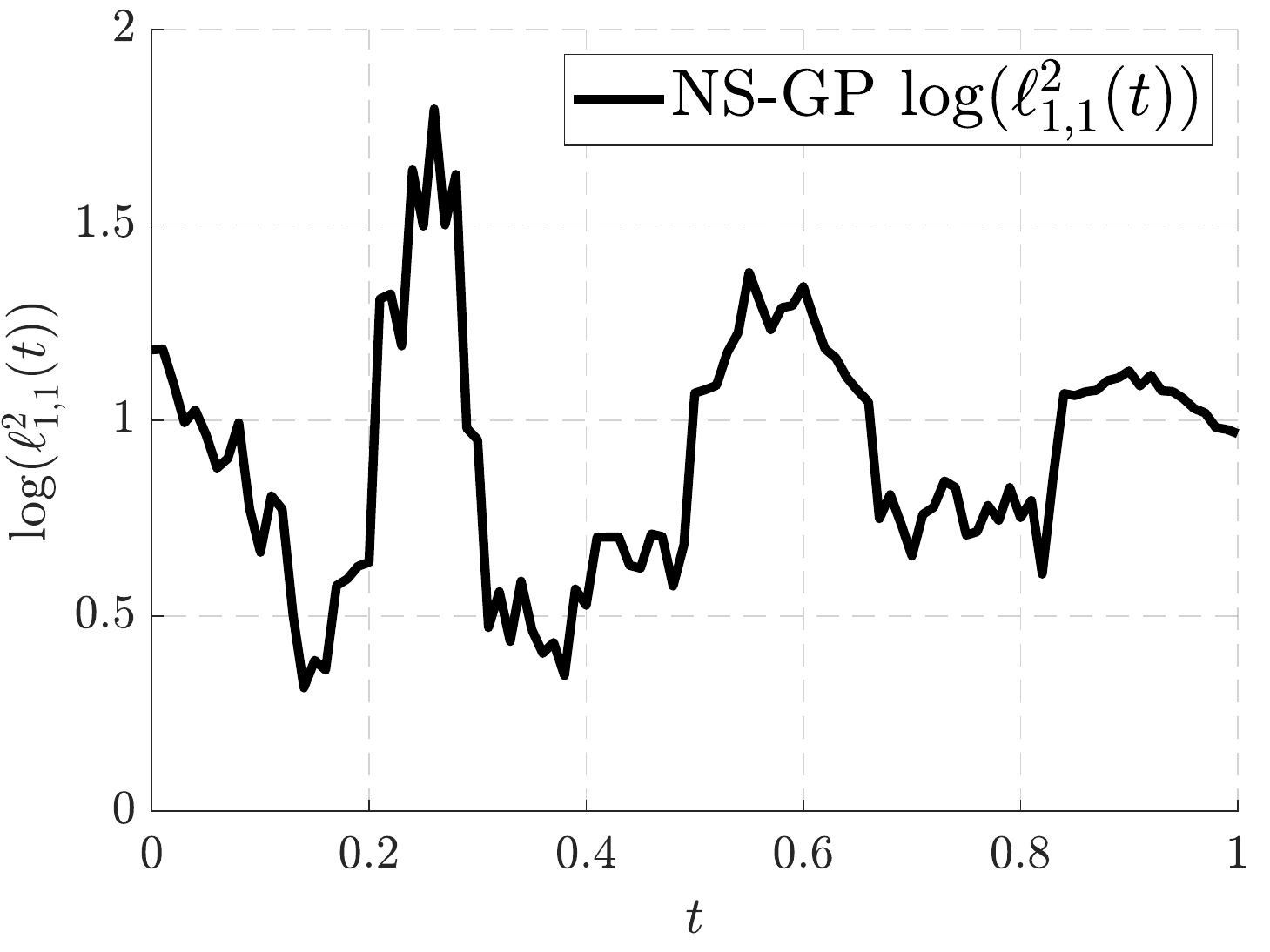}
		\includegraphics[width=.49\linewidth]{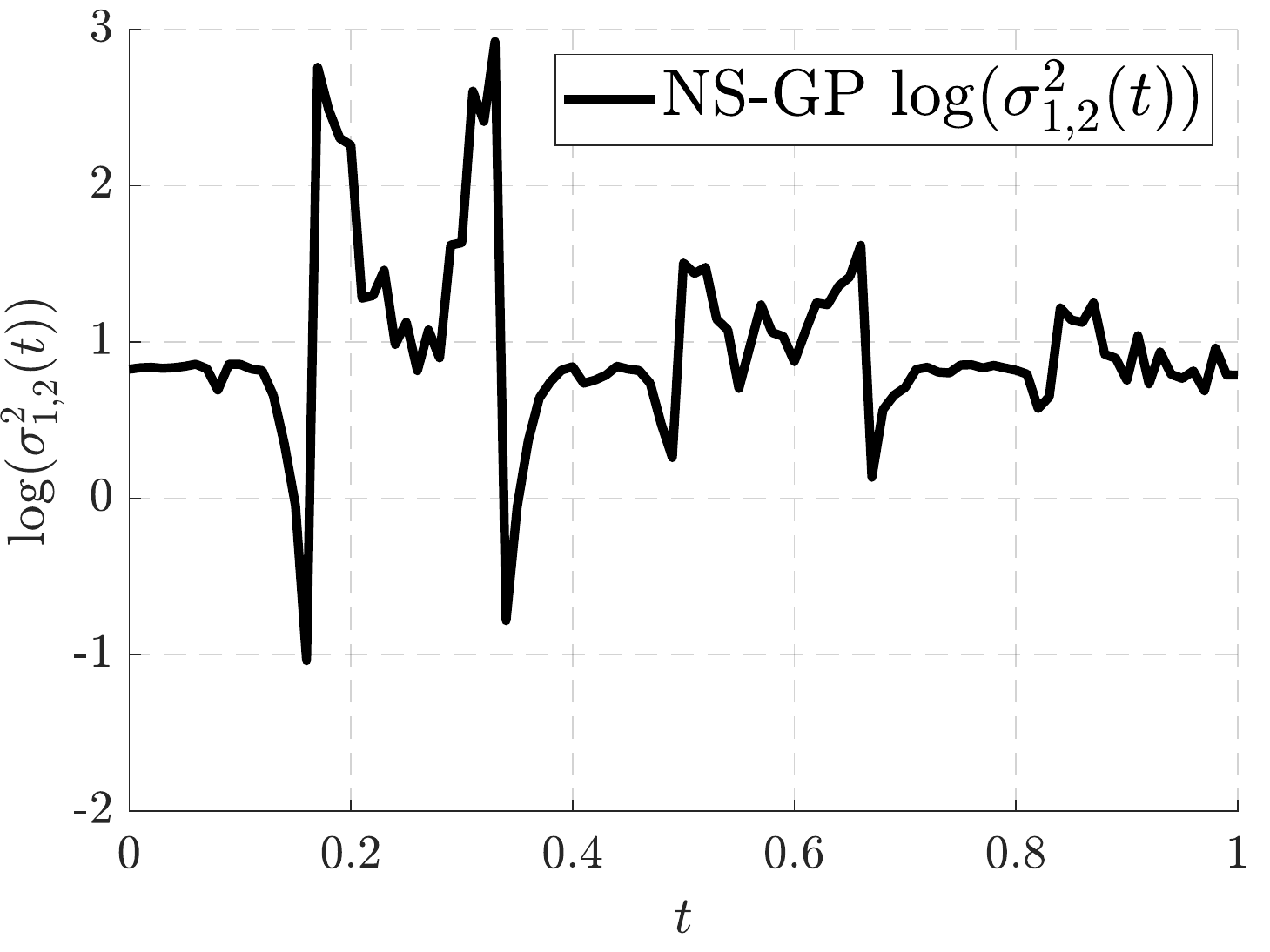}
		\caption{GP and NS-GP regressions on model~\eqref{equ:exp-y}. The shaded area stands for 95\% confidence interval.}
		\label{fig:exp-GP-NSGP}
	\end{figure}
	
	We formulate the commonly used root mean square error (RMSE)
	\begin{equation}
		\left( \frac{1}{T}\,\sum^T_{k=1} \left(f(t_k) - \tilde{f}(t_k)\right)^2\right)^{1/2},
		\label{equ:rmse}
	\end{equation}
	as well as the negative log predictive density (NLPD)
	\begin{equation}
		-\log  \int p(\cu{y}^\star \mid \cu{f}^\star) \,\tilde{p}(\cu{f}^\star \mid \cu{y}_{1:T}) \diff \cu{f}^\star,
	\end{equation}
	to numerically demonstrate the methods' effectiveness, where $\tilde{f}$ is the regression estimate, $\cu{y}^\star$ are the test data, and $\tilde{p}(\cu{f}^\star \mid \cu{y}_{1:T})$ is the estimated posterior density. Note that the NLPD metric is not applied to the MAP results. We run 100 independent Monte Carlo trials to average the RMSE and NLPD as well as the computational time. For visualization, we uniformly choose the results under the same random seed. 
	
	\begin{table}[t!]
		\centering
		\begin{tabular}{@{}llll@{}}
			\toprule
			Methods& \begin{tabular}[c]{@{}l@{}}RMSE\\ ($ 0^{-2}$)\end{tabular} & NLPD & Time (s) \\ \midrule
			GP (MLE) & $4.36\pm 0.3$  & $-136.6\pm 8$ & $2.0\pm 0.5$ \\
			NS-GP (MLE) & $4.28\pm 0.3$ & $-135.9\pm 12$ & $3.3\pm0.2$ \\
			B-MAP (DGP-2) & $3.89\pm 0.3$ & N/A & $454.9\pm 67$ \\
			B-MAP (DGP-3) & $3.80\pm 0.3$& N/A & $897.7\pm37$ \\
			SS-MAP (DGP-2) & $2.04\pm 0.4$ & N/A & $205.4\pm26$ \\
			SS-MAP (DGP-3) & $\mathbf{1.69}\pm0.3$ & N/A & $479.5\pm 70$ \\
			CKFS (DGP-2) & $4.50\pm 0.3$ & $-136.2\pm 7$ & $0.2\pm0.03$ \\
			CKFS (DGP-3) & N/A & N/A & N/A \\
			EKFS (DGP-2) & $5.32\pm0.2$ & $-135.4\pm 8$ & $\mathbf{0.1}\pm0.02$ \\
			EKFS (DGP-3) & $7.77\pm0.1$ & $-119.6\pm9$ & $0.2\pm0.01$ \\
			PF (DGP-2) & $4.25\pm2.3$ & $-135.5\pm17$ & $929.6\pm200$ \\
			PF (DGP-3) & $3.73\pm0.9$ & $-145.1\pm12$ & $1460.8\pm170$ \\
			PF-BS (DGP-2) & $4.08\pm2.7$ & $-140.6\pm25$ & 4-7 hrs \\ 
			PF-BS (DGP-3) & $3.35\pm0.9$ & $-\mathbf{149.2}\pm 11$ & 17-20 hrs \\\bottomrule
		\end{tabular}
		\caption{Averaged RMSEs ($\times10^{-2}$), NLPD, and computational time (in seconds) over different regression models and methods. }
		\label{tbl:rect-rmse-nlpd-time}
	\end{table}
	
	\begin{figure*}[t!]
		\centering
		\includegraphics[width=.32\linewidth]{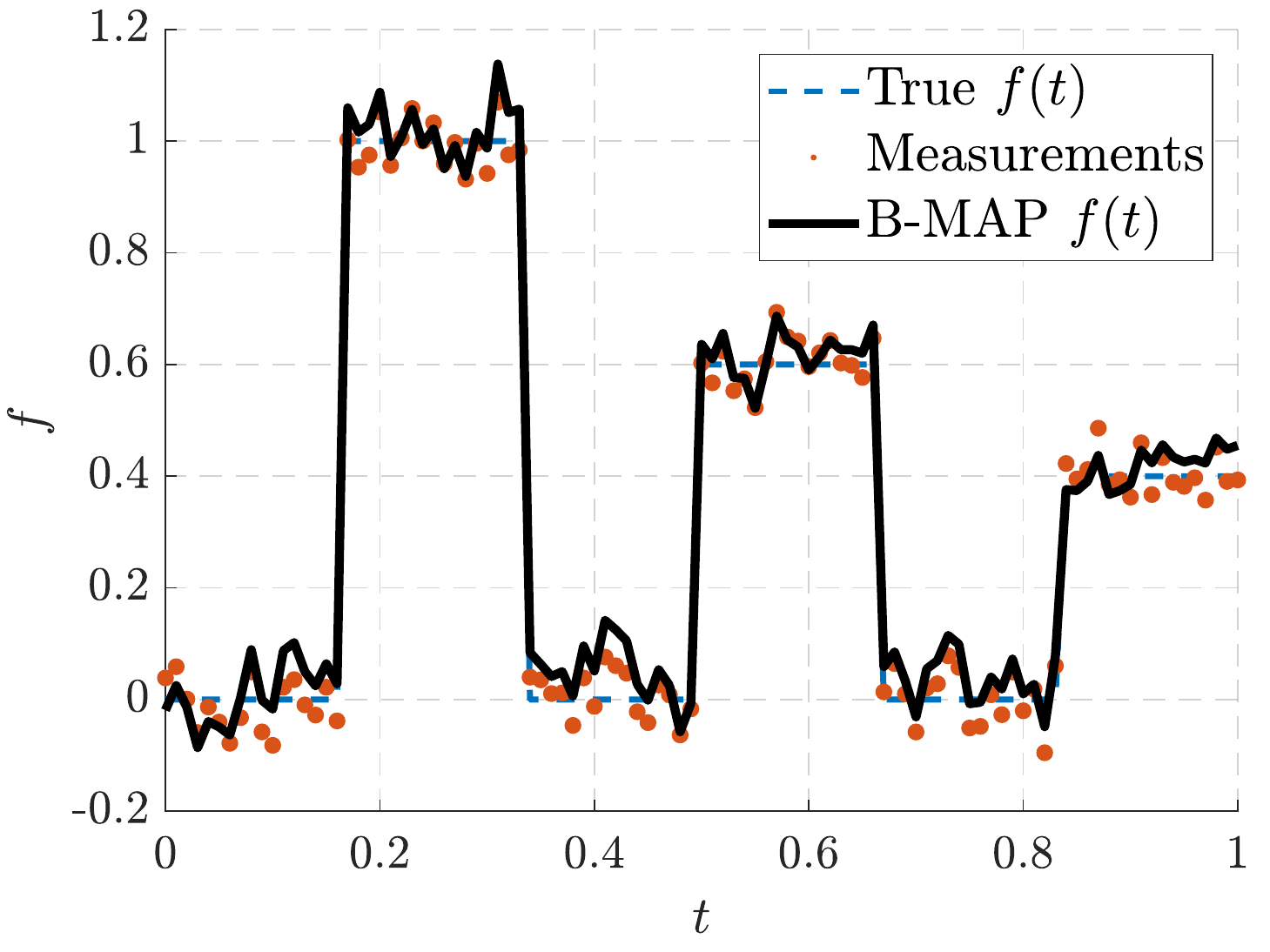}
		\includegraphics[width=.32\linewidth]{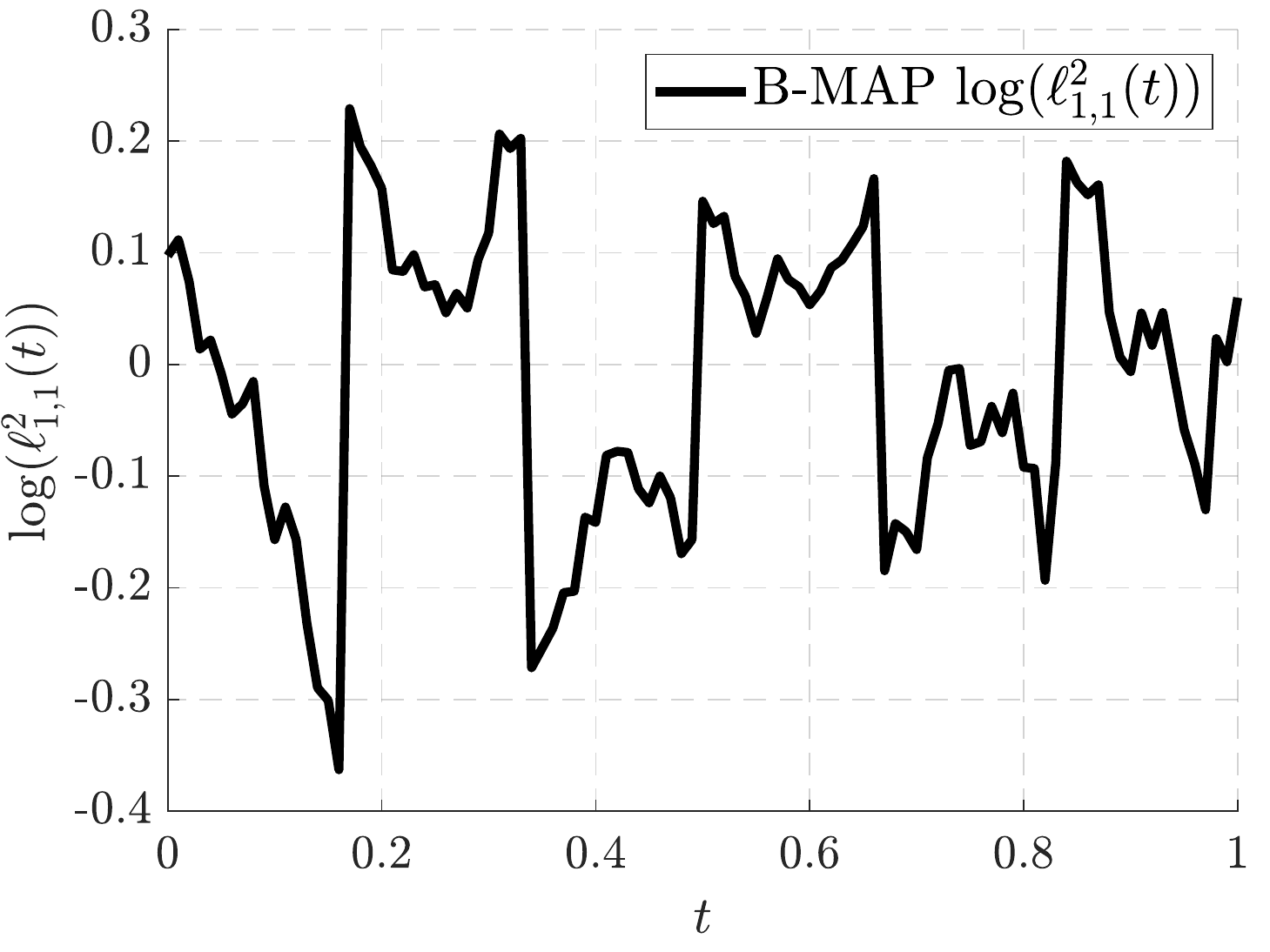}
		\includegraphics[width=.32\linewidth]{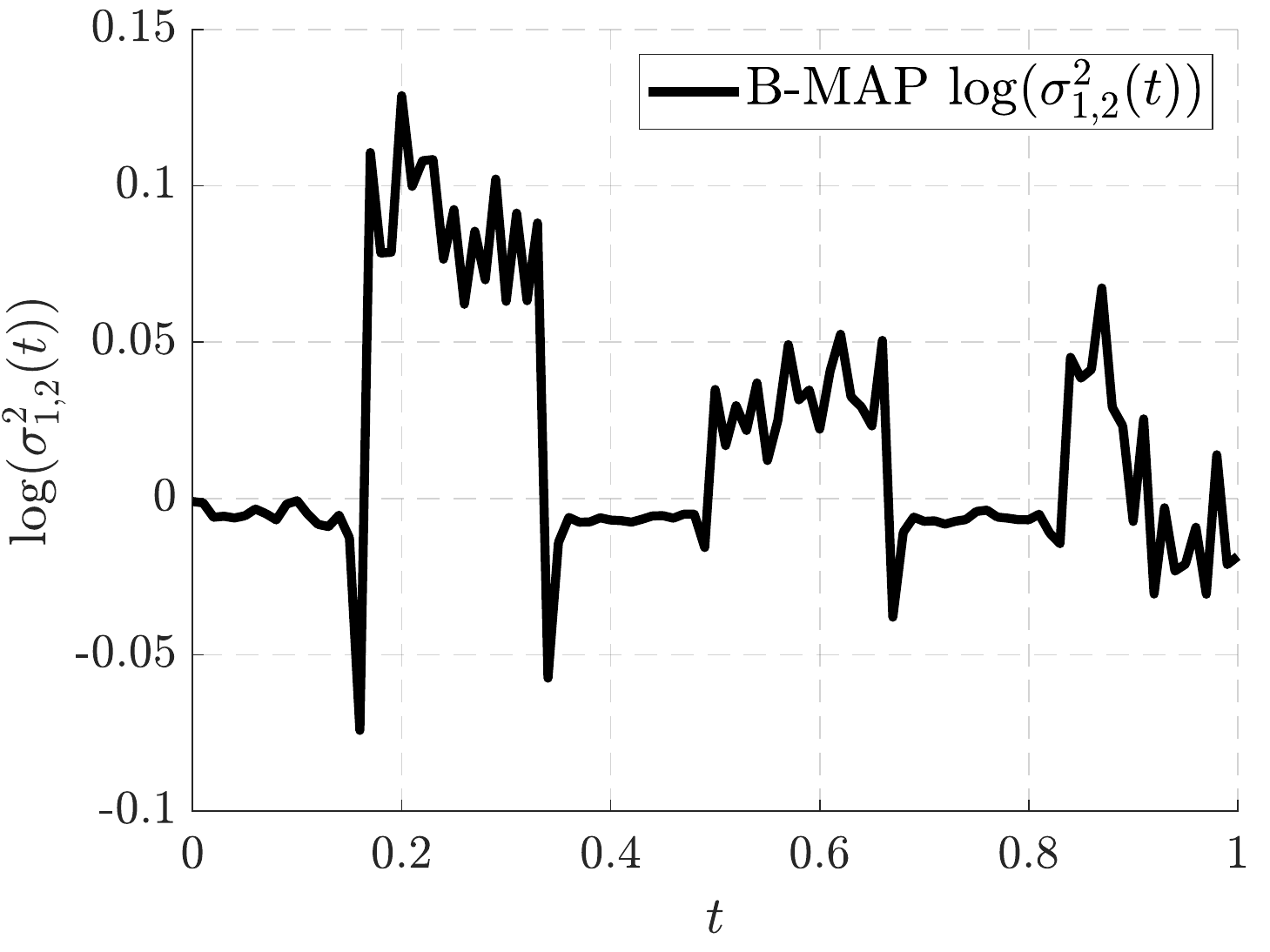}\\
		\includegraphics[width=.32\linewidth]{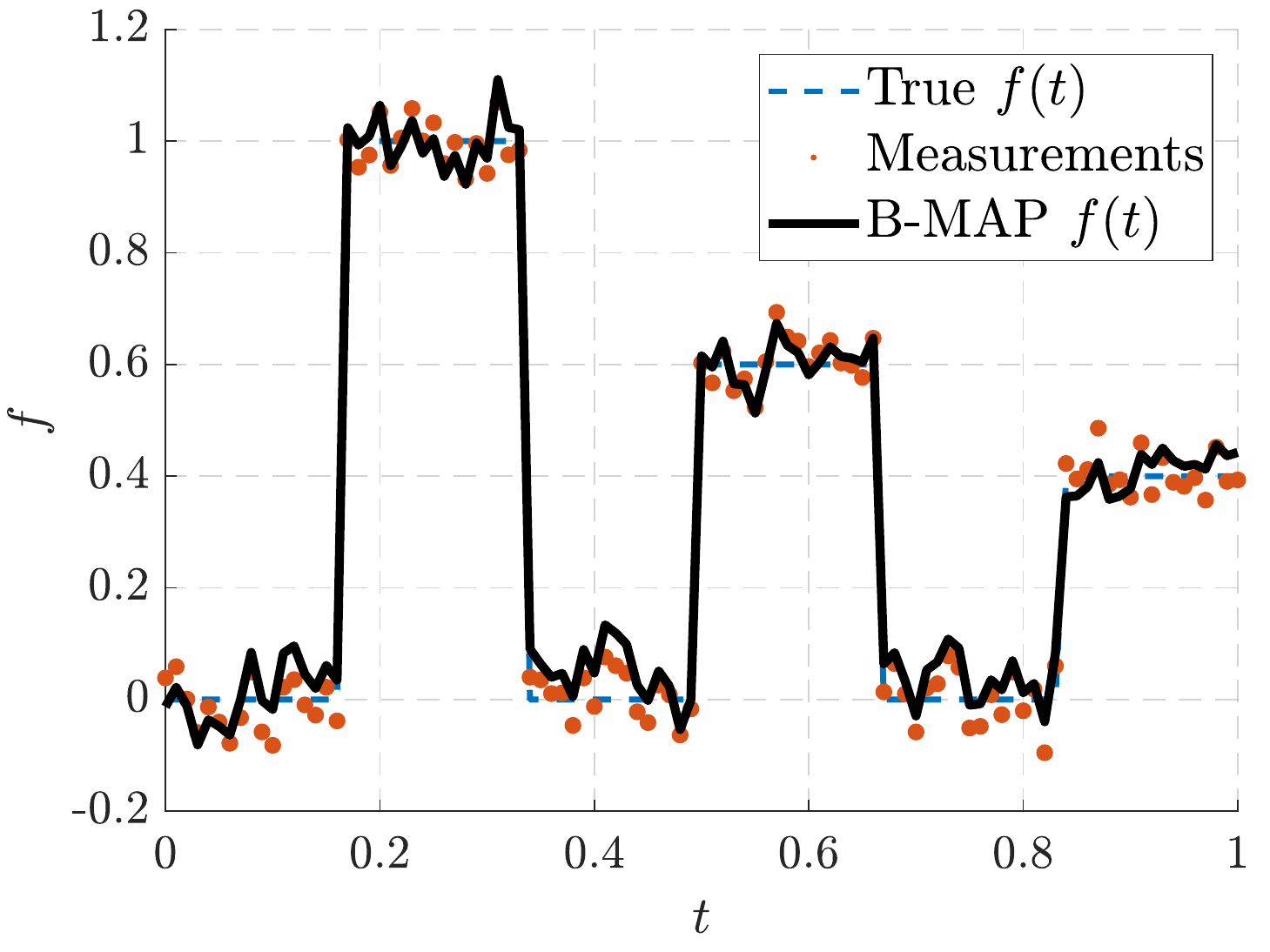}
		\includegraphics[width=.32\linewidth]{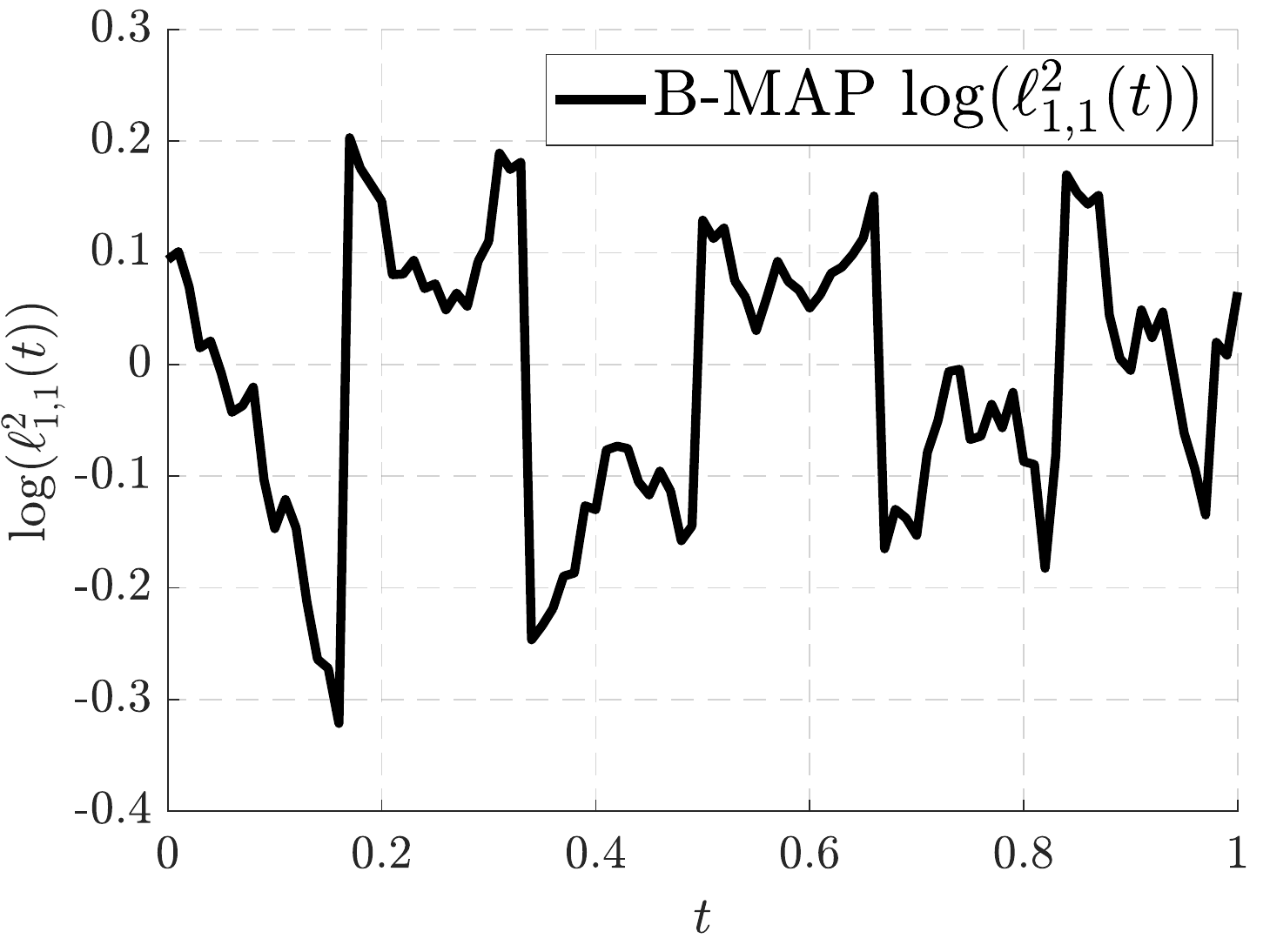}
		\includegraphics[width=.32\linewidth]{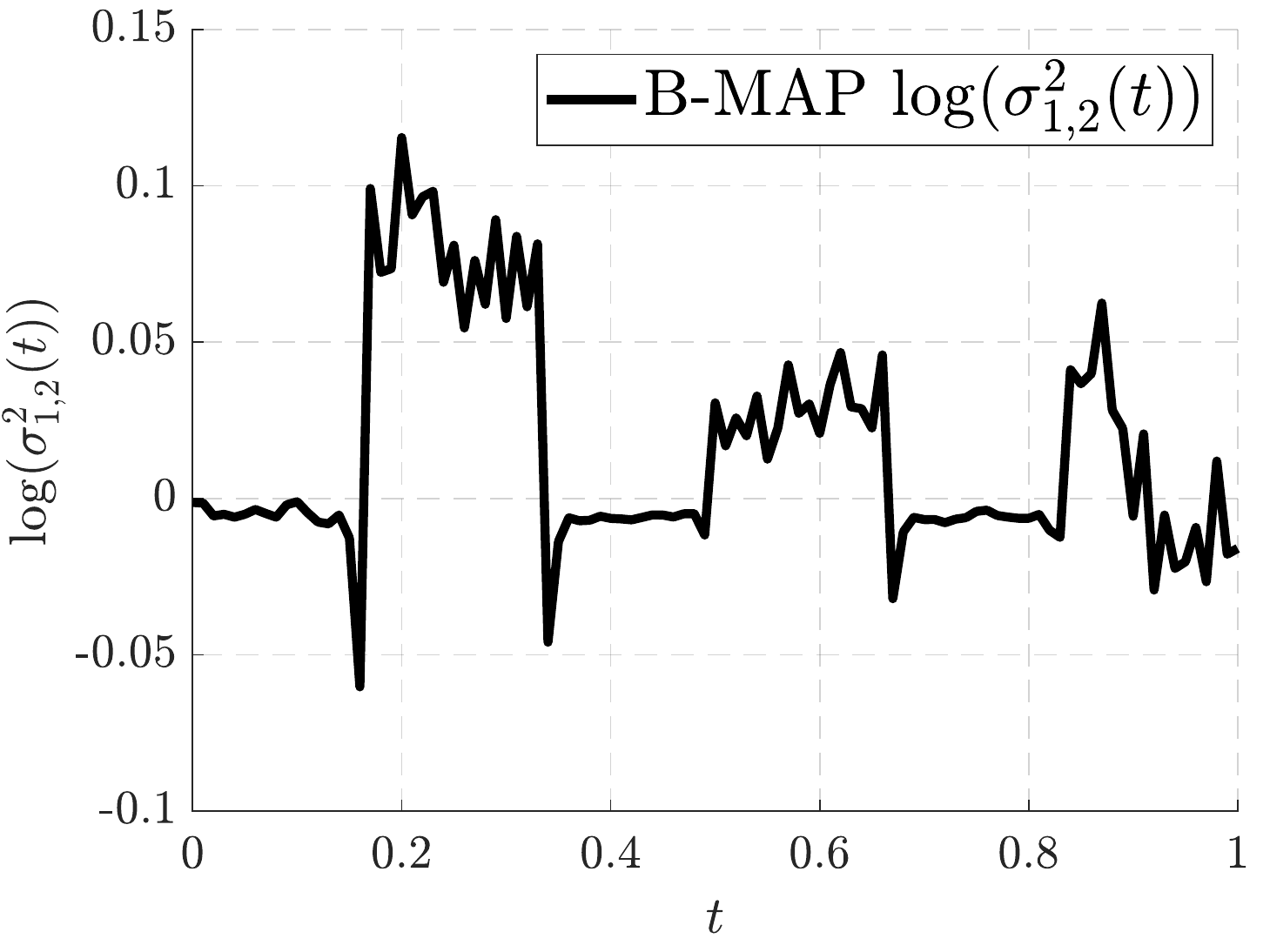}\\
		\includegraphics[width=.24\linewidth]{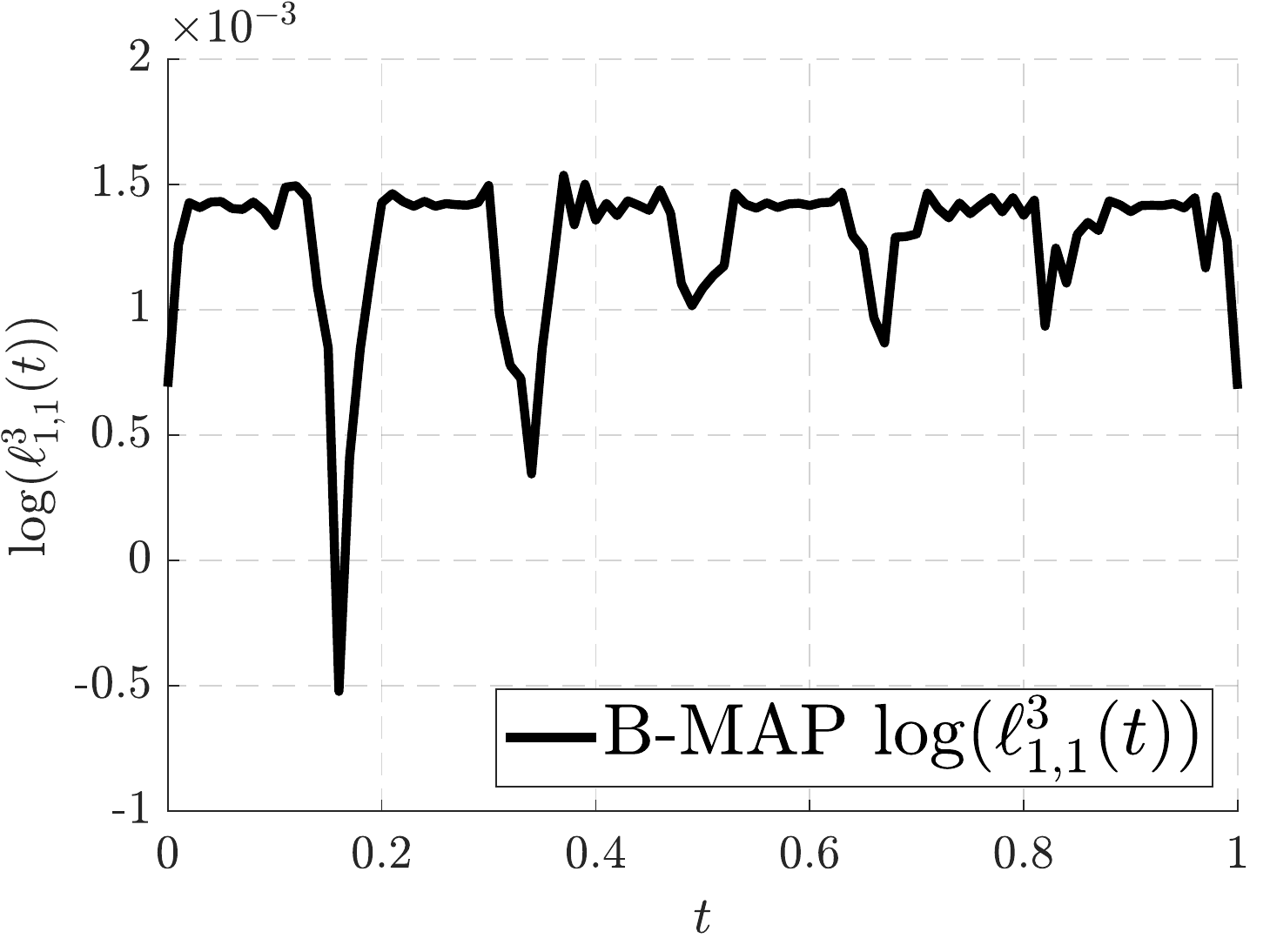}
		\includegraphics[width=.24\linewidth]{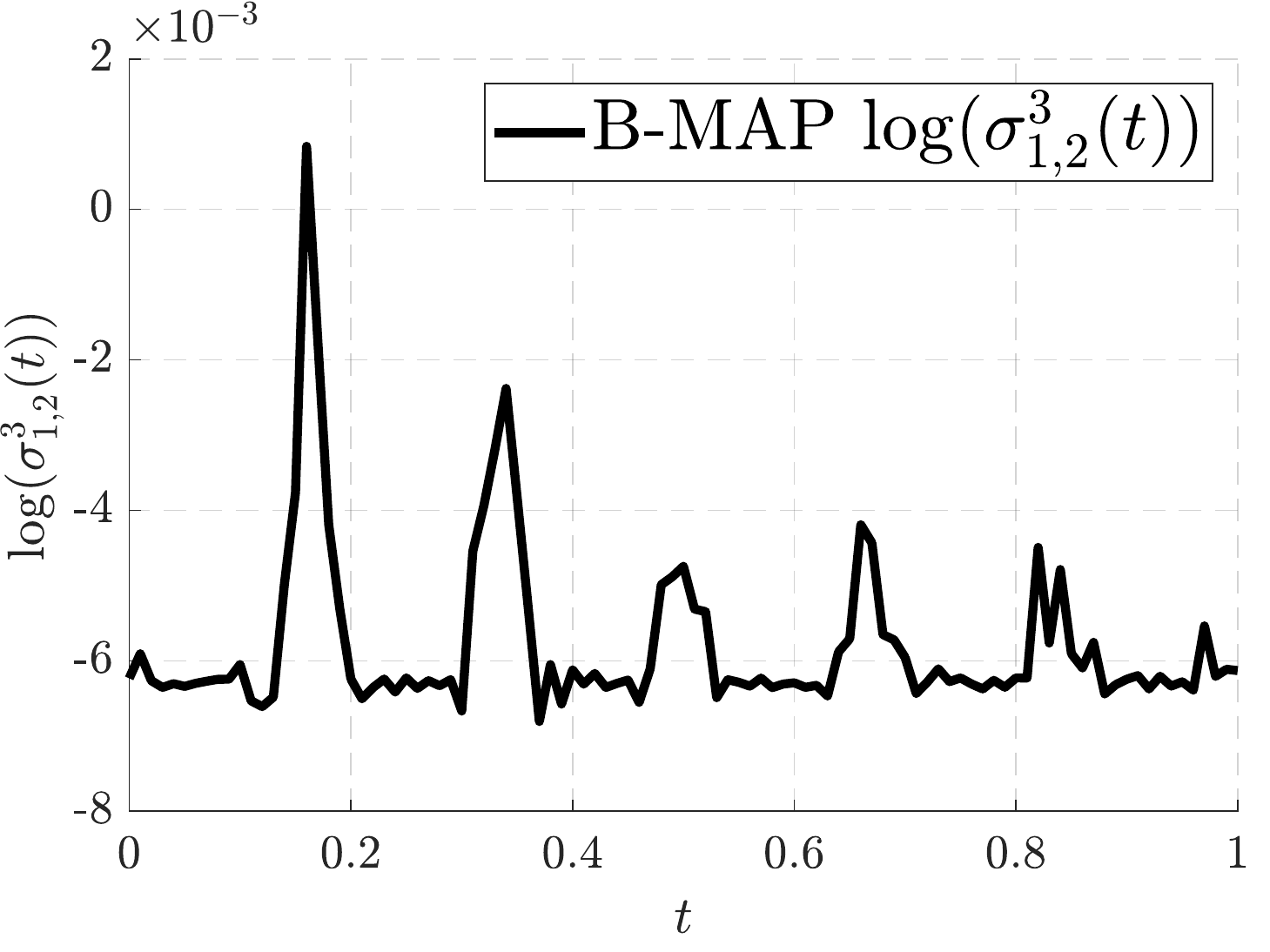}
		\includegraphics[width=.24\linewidth]{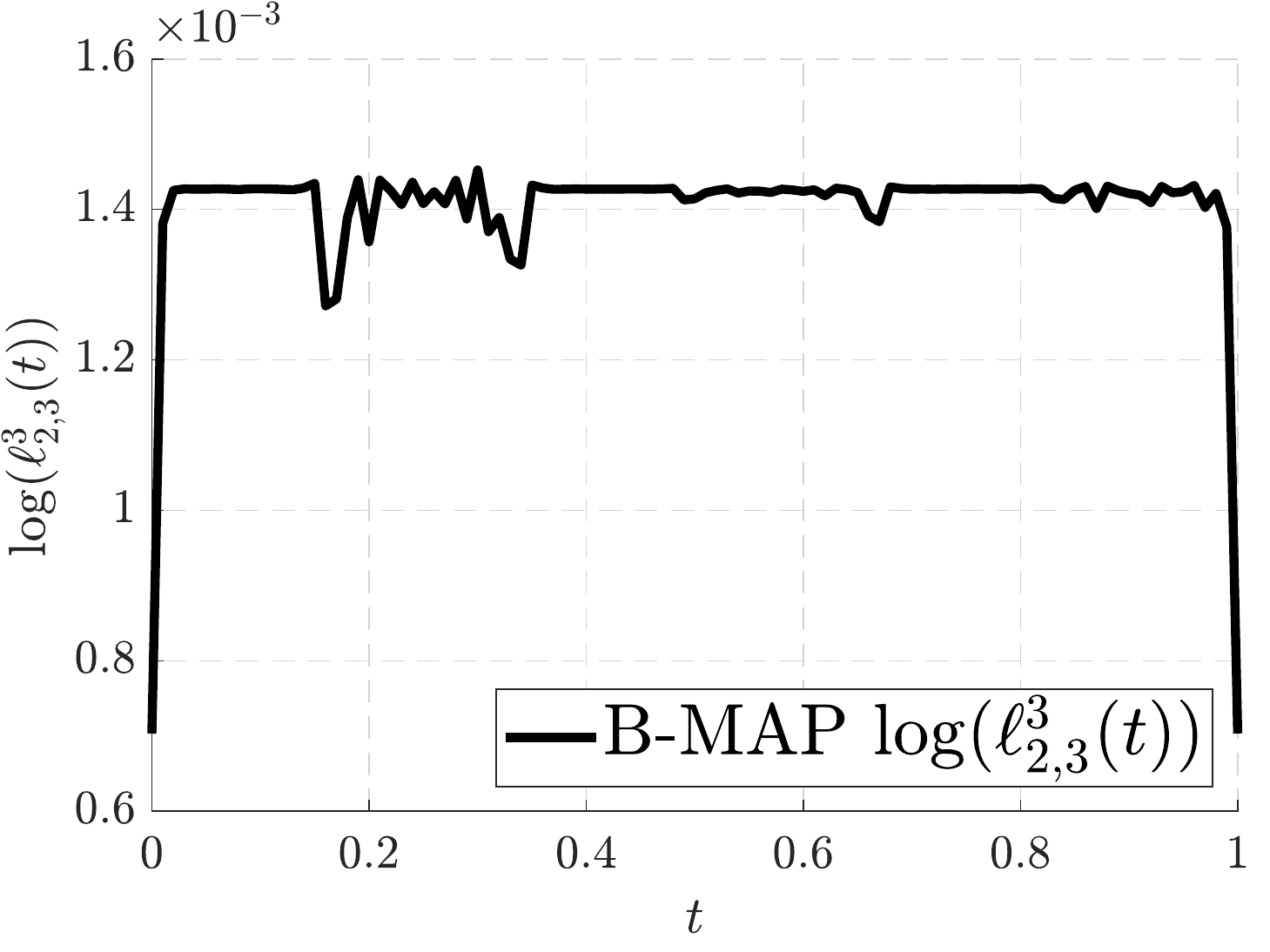}
		\includegraphics[width=.24\linewidth]{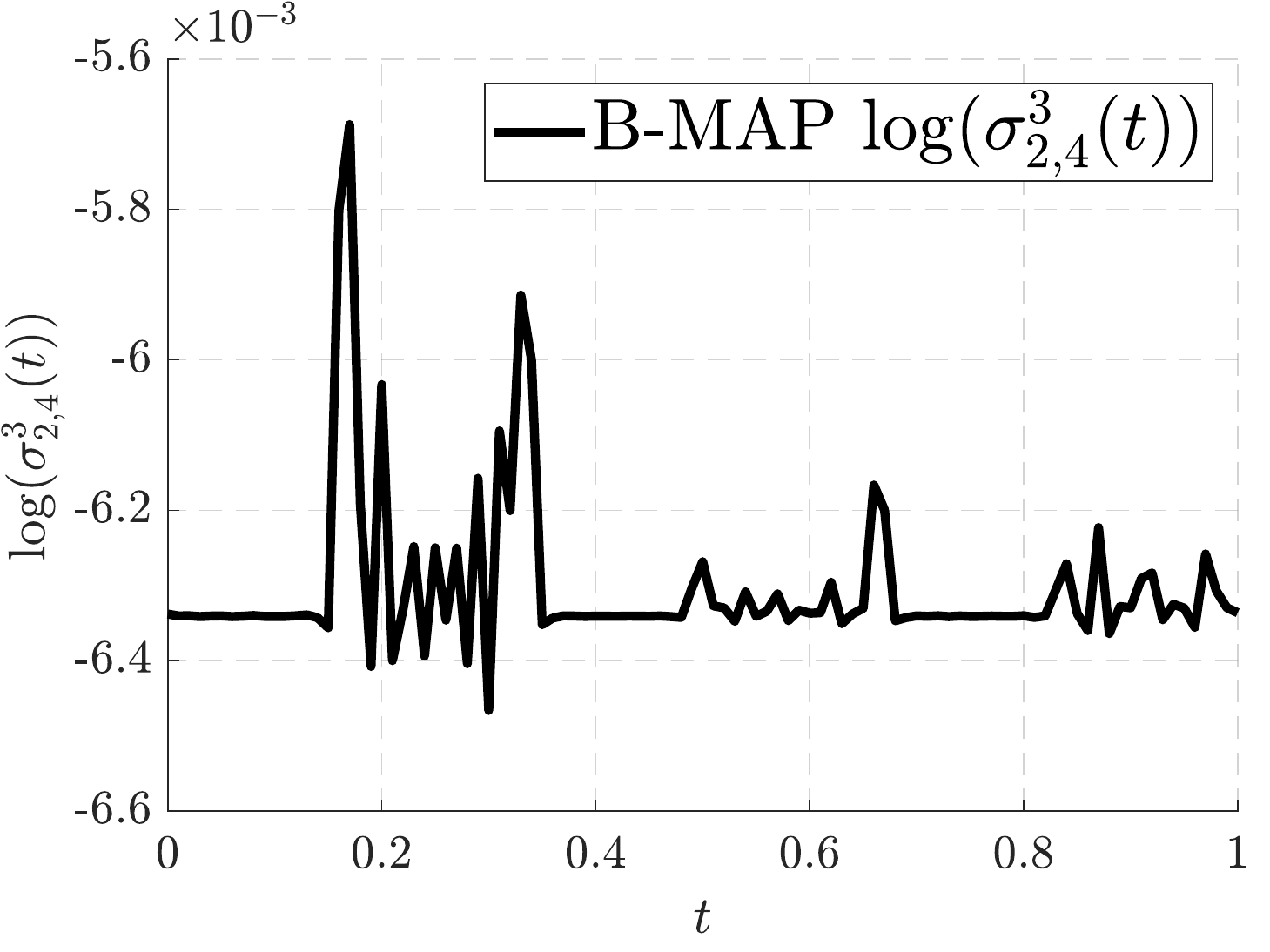}
		\caption{B-MAP regression results on model~\eqref{equ:exp-y} using DGP-2 (first row) and DGP-3 (second and third rows).}
		\label{fig:exp-BMAP}
	\end{figure*}
	
	\begin{figure*}[t!]
		\centering
		\includegraphics[width=.24\linewidth]{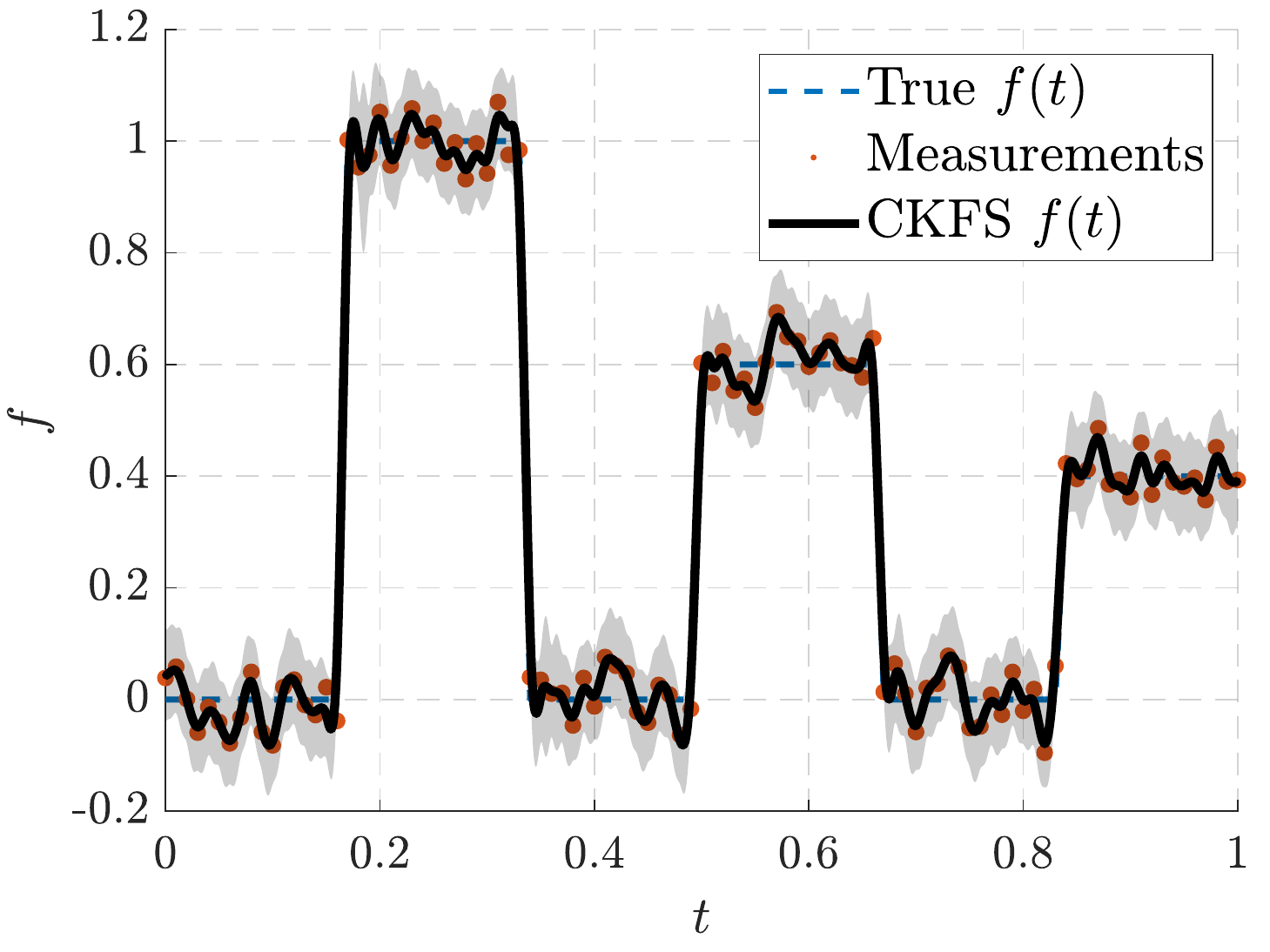}
		\includegraphics[width=.24\linewidth]{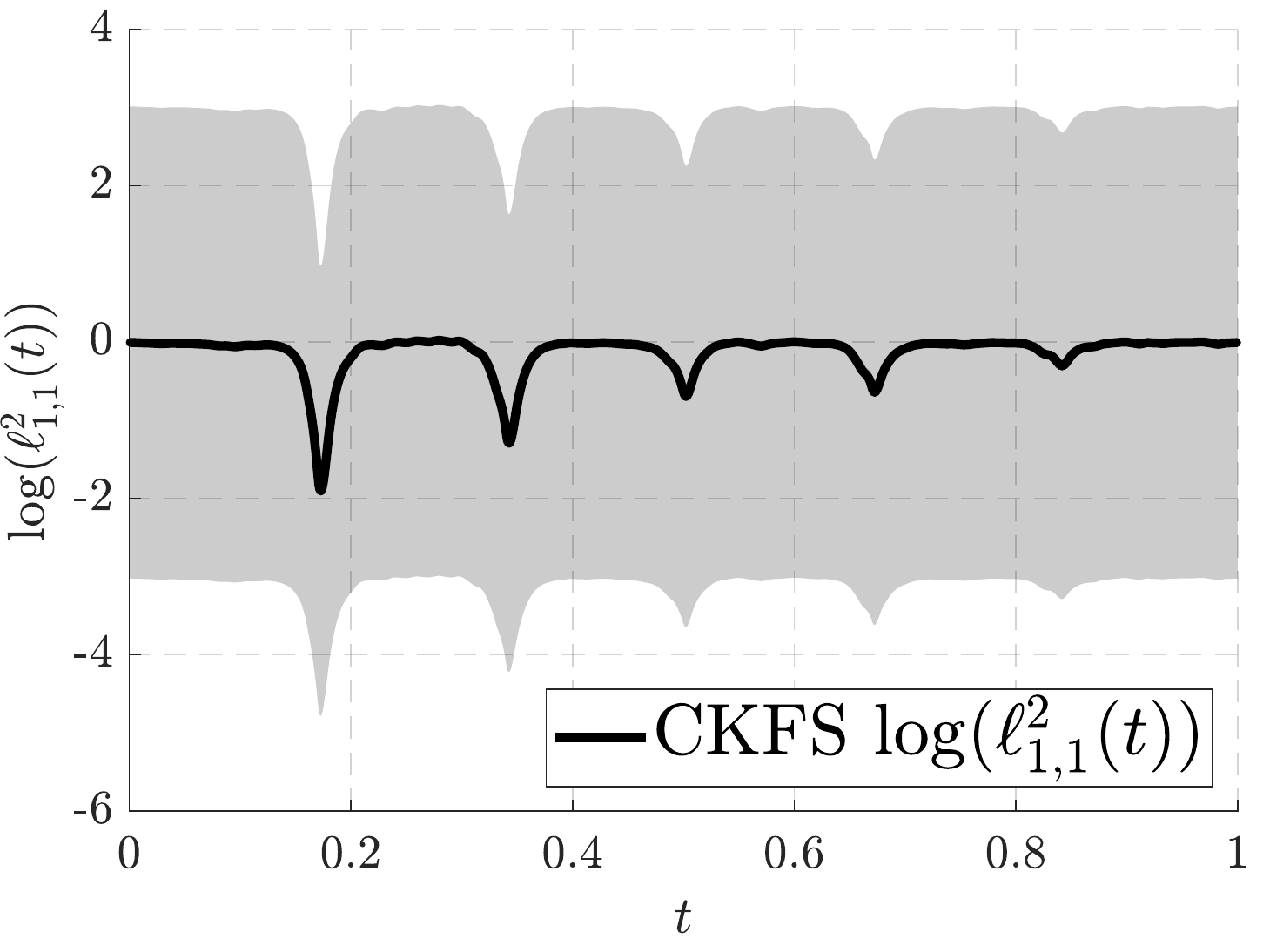}
		\includegraphics[width=.24\linewidth]{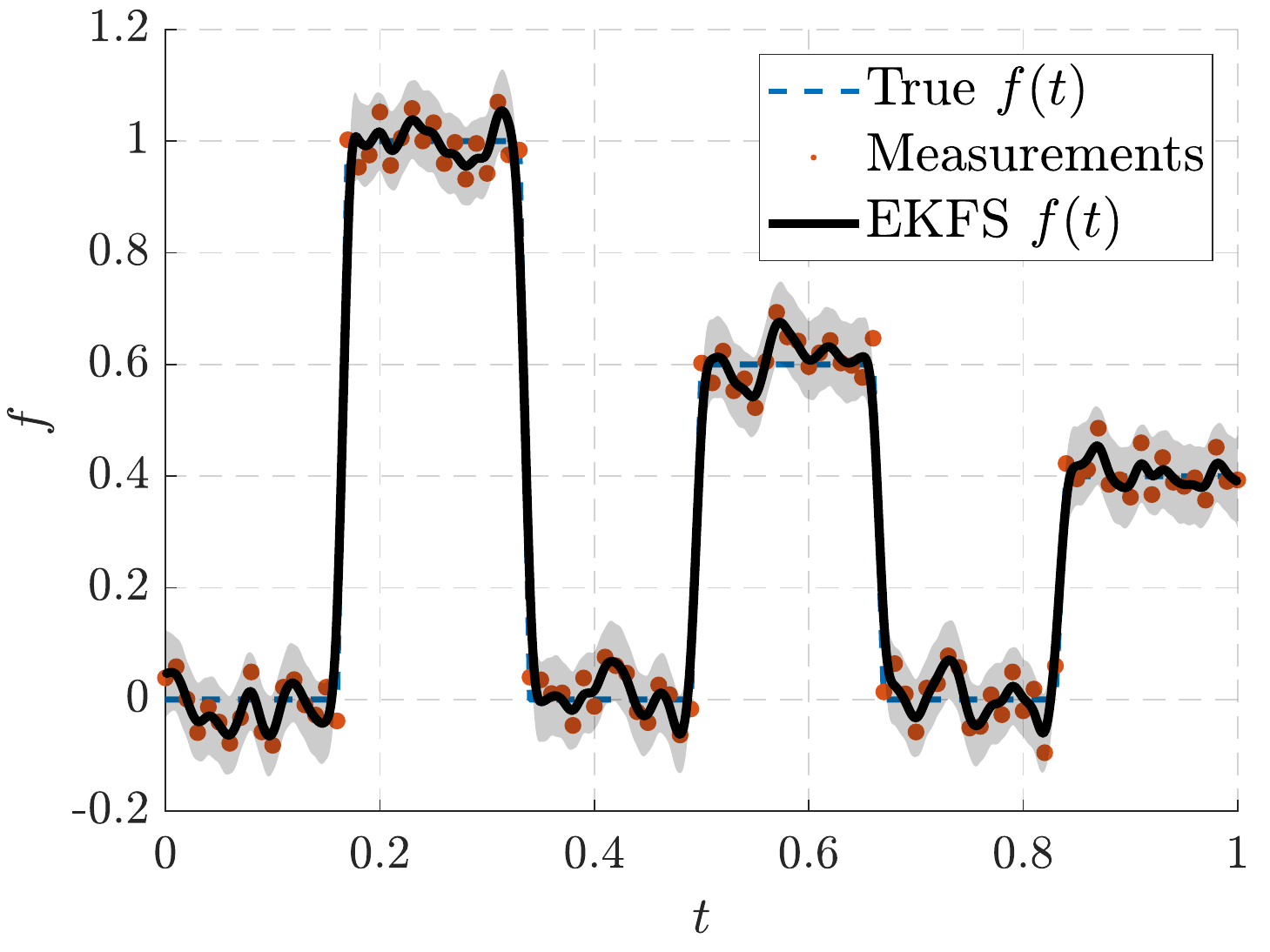}
		\includegraphics[width=.24\linewidth]{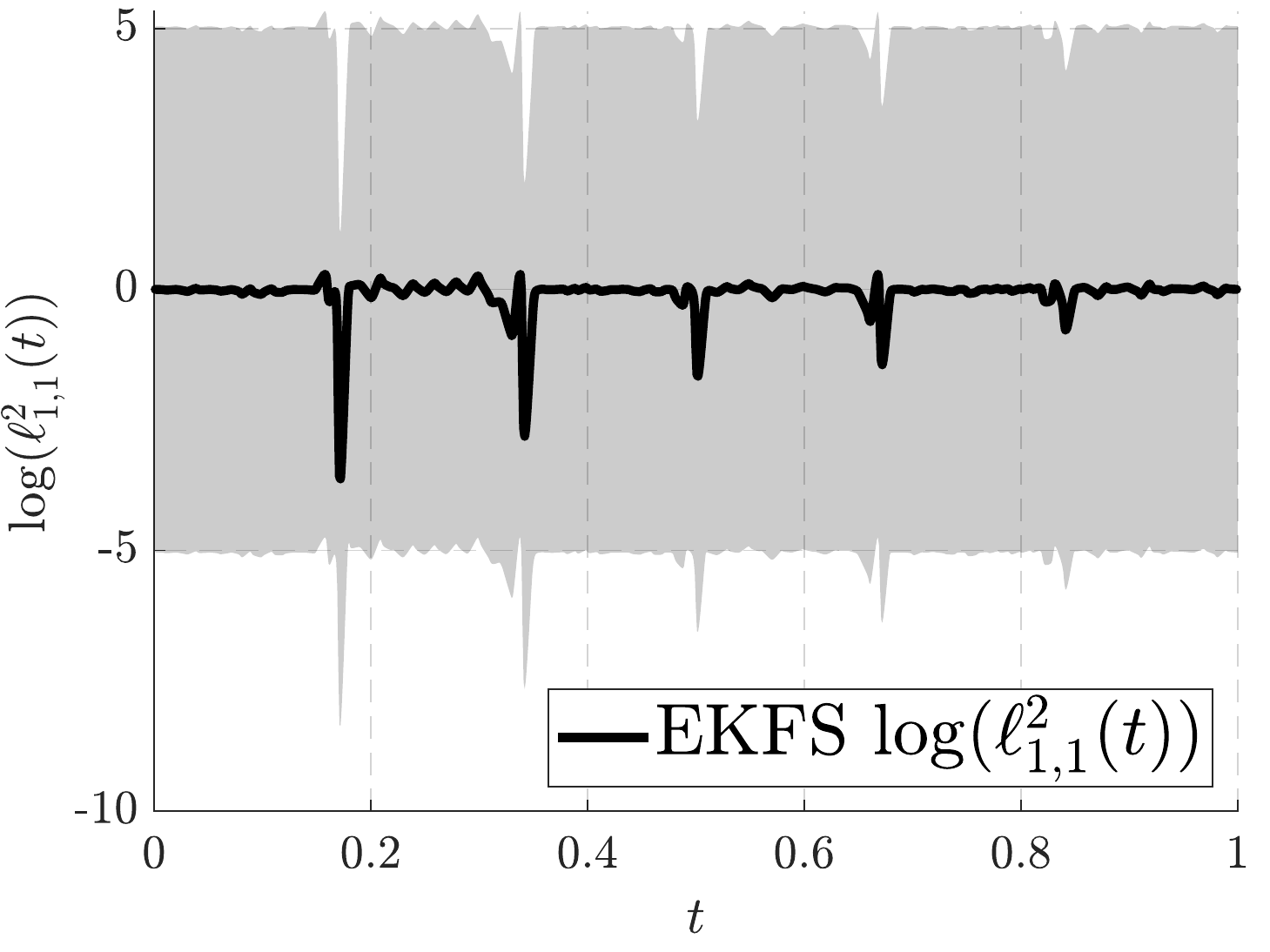}\\
		\includegraphics[width=.32\linewidth]{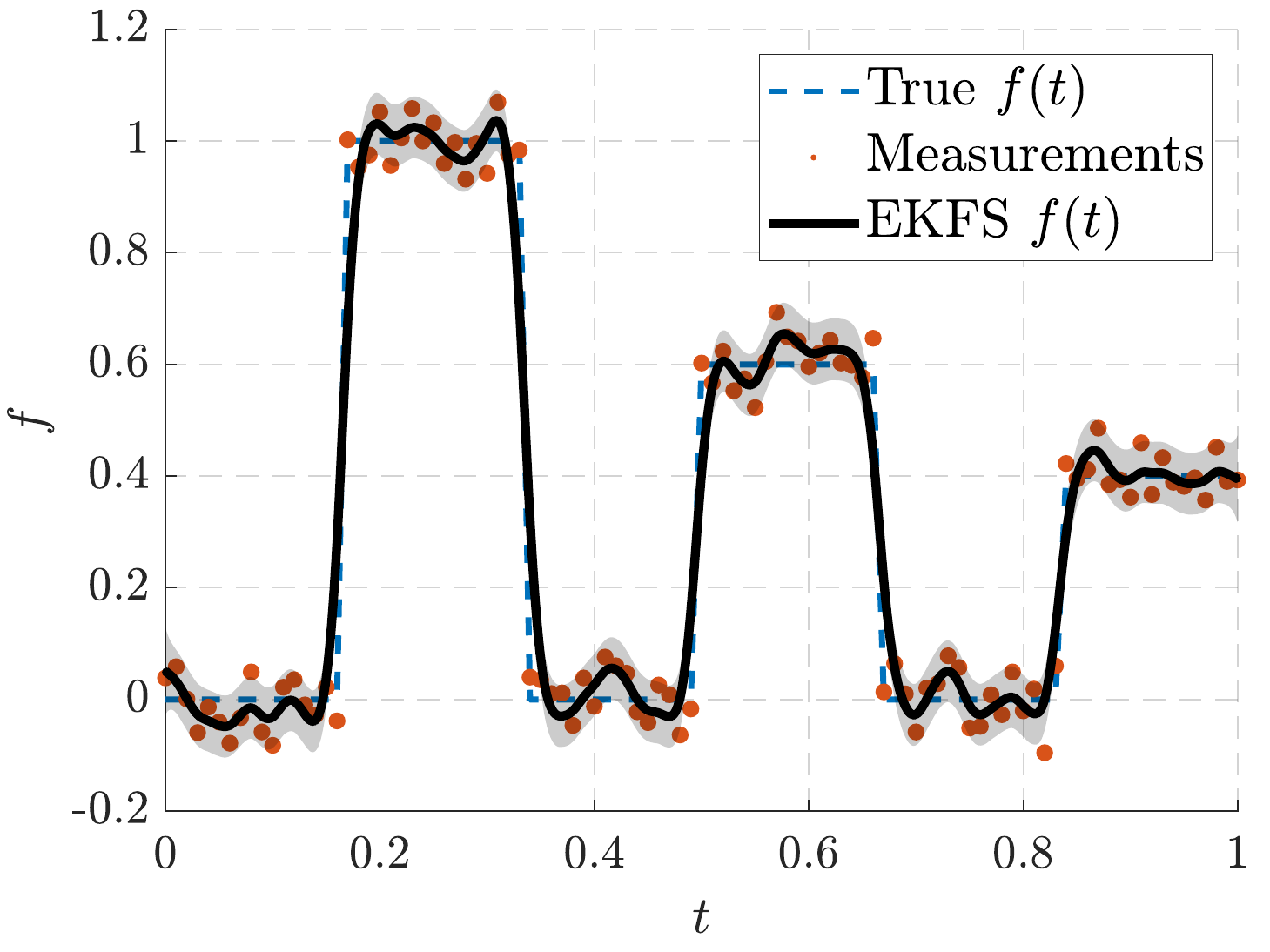}
		\includegraphics[width=.32\linewidth]{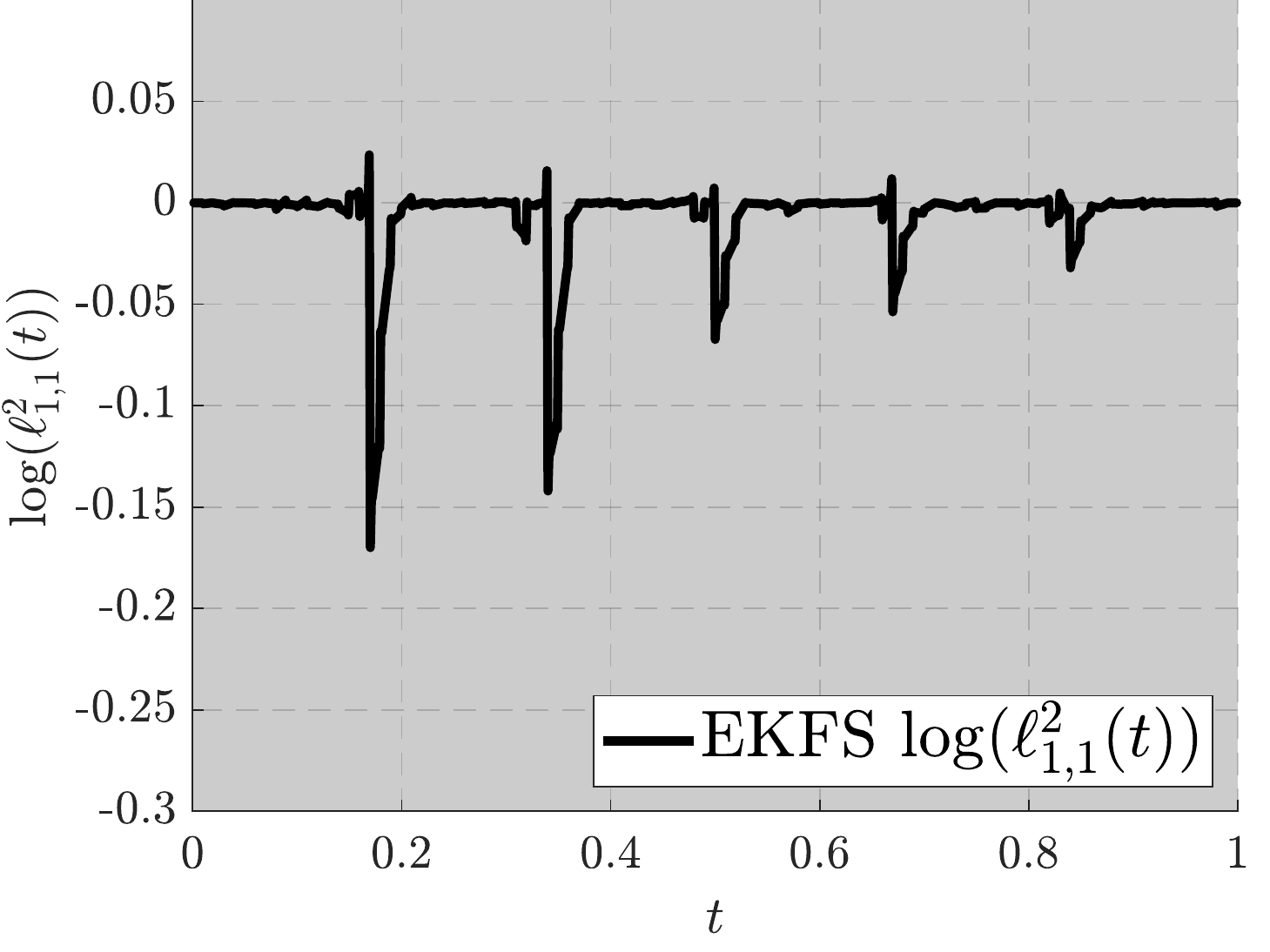}
		\includegraphics[width=.32\linewidth]{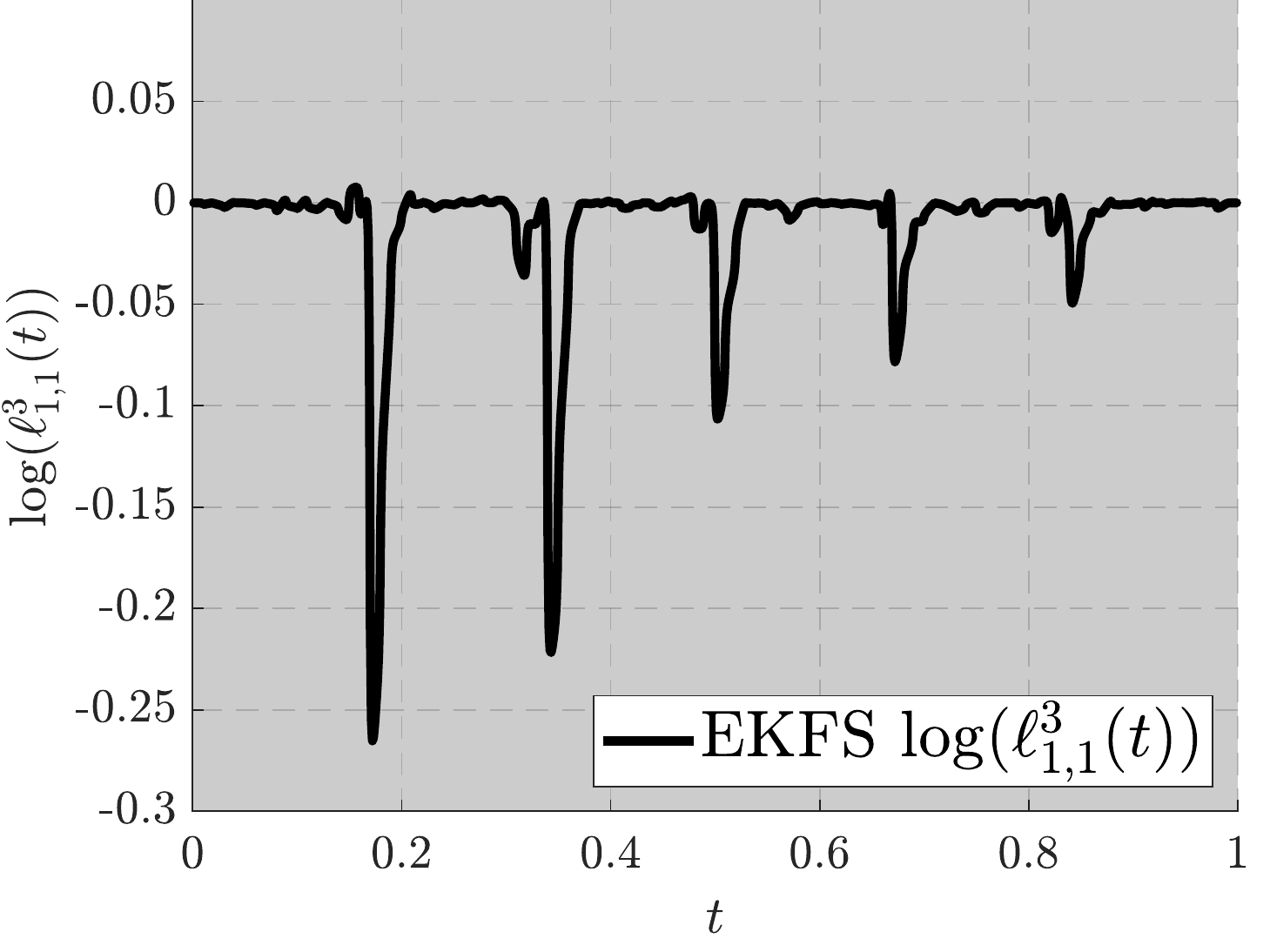}
		\caption{CKFS and EKFS regression results on model~\eqref{equ:exp-y} using DGP-2 (first row) and EKFS on DGP-3 (first and second rows). The shaded area stands for 95\% confidence interval. }
		\label{fig:exp-GFS}
	\end{figure*}
	
	\begin{figure*}[t!]
		\centering
		\includegraphics[width=.32\linewidth]{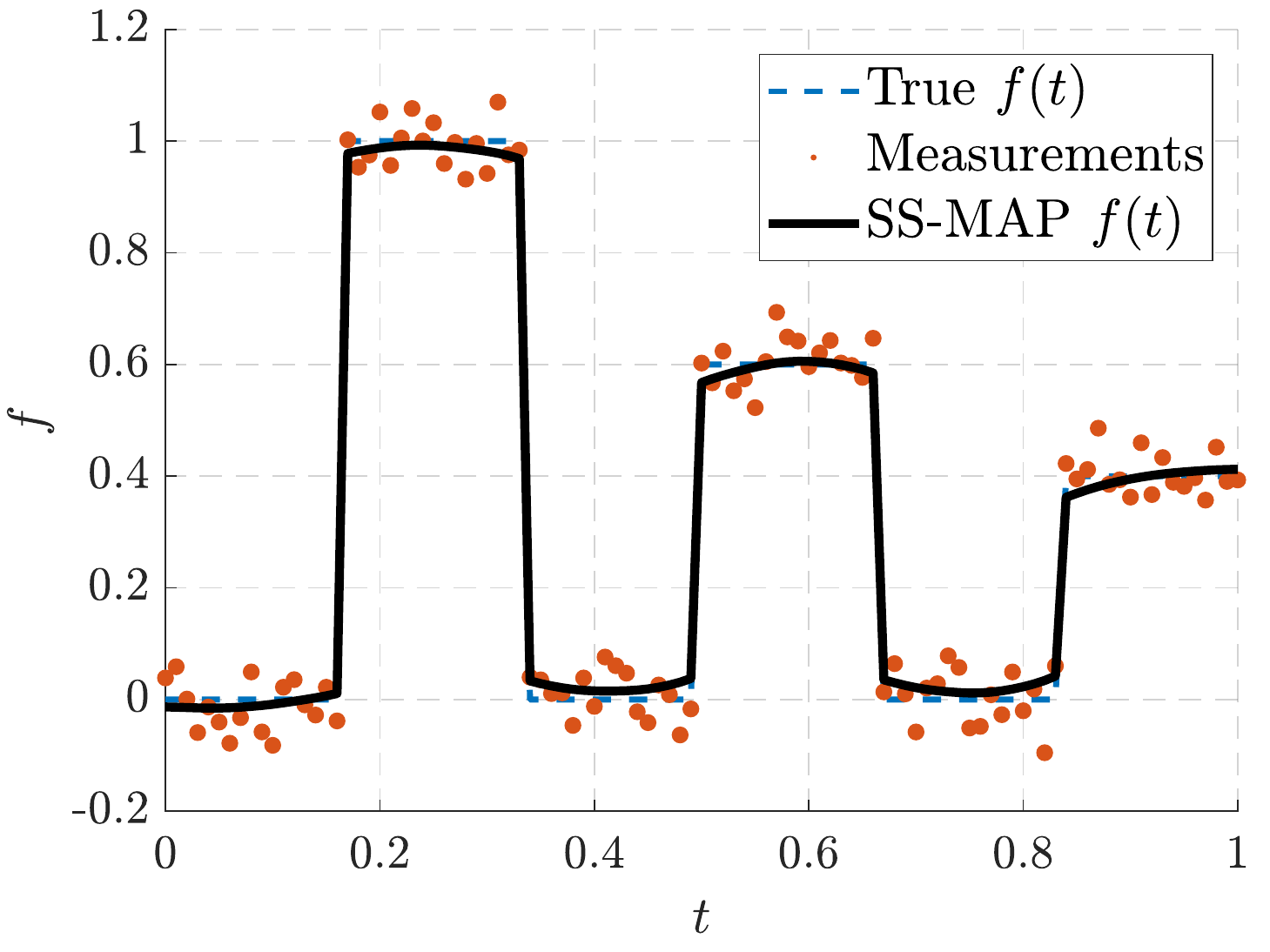}
		\includegraphics[width=.32\linewidth]{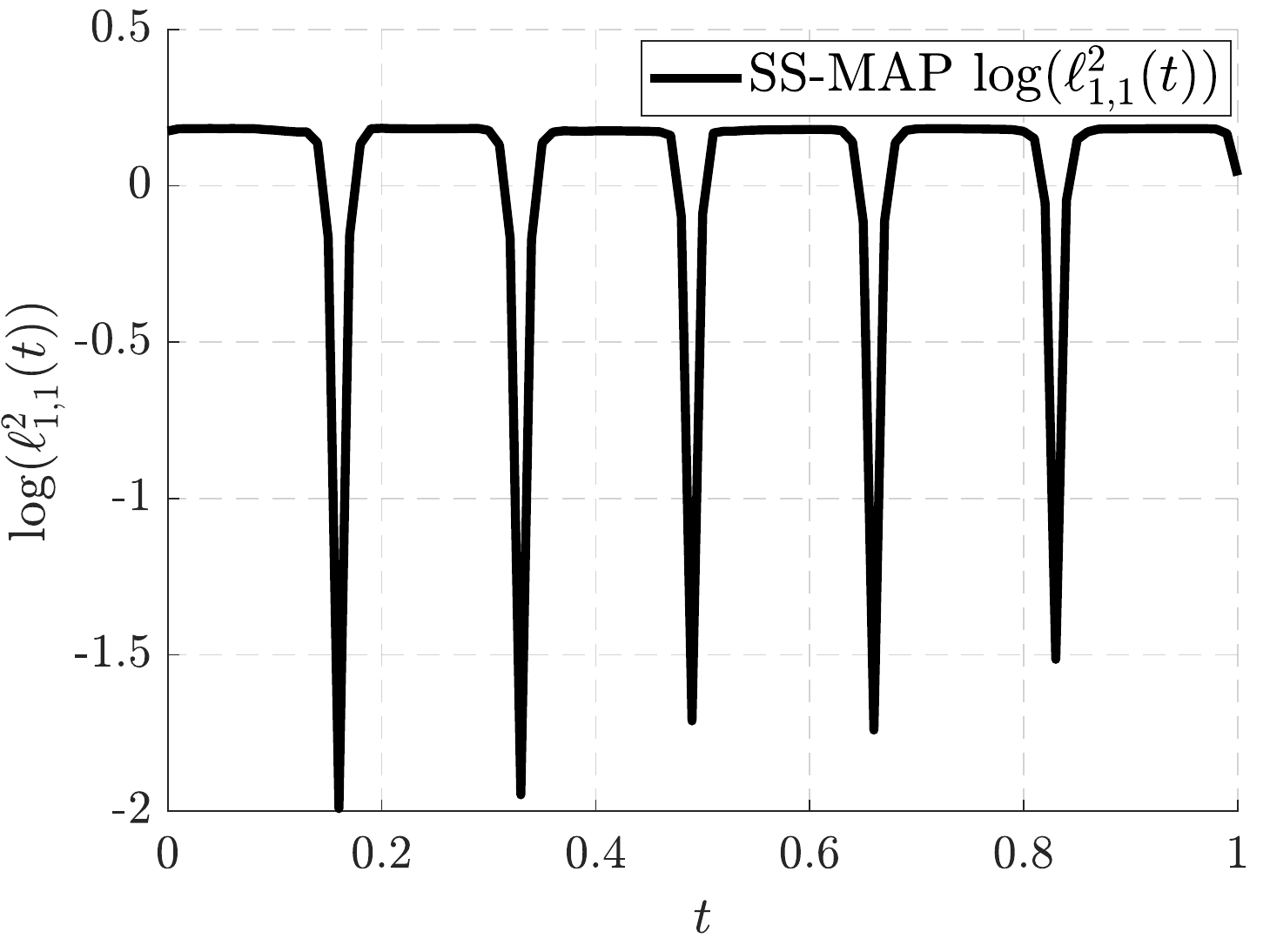}
		\includegraphics[width=.32\linewidth]{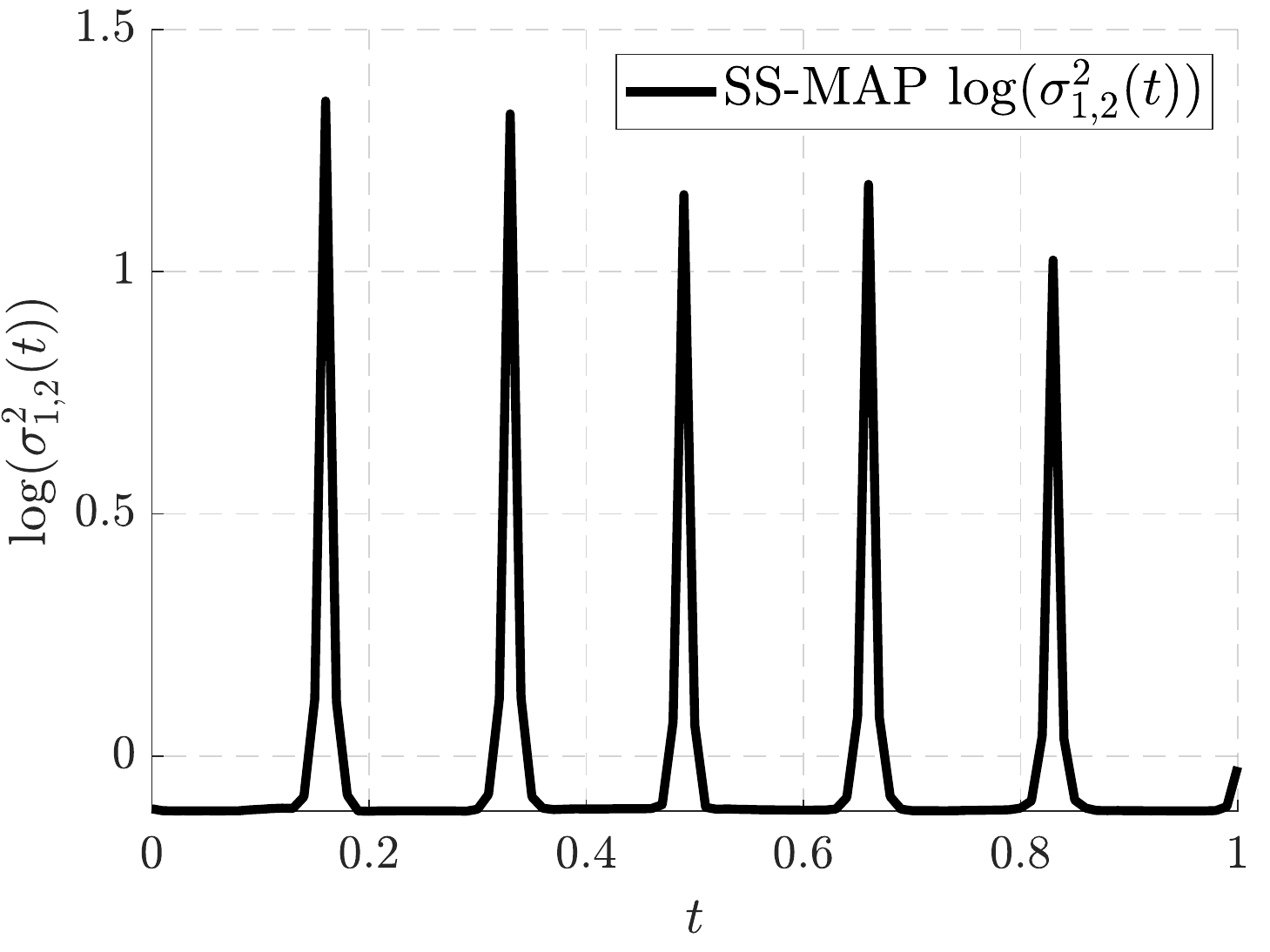}\\
		\includegraphics[width=.32\linewidth]{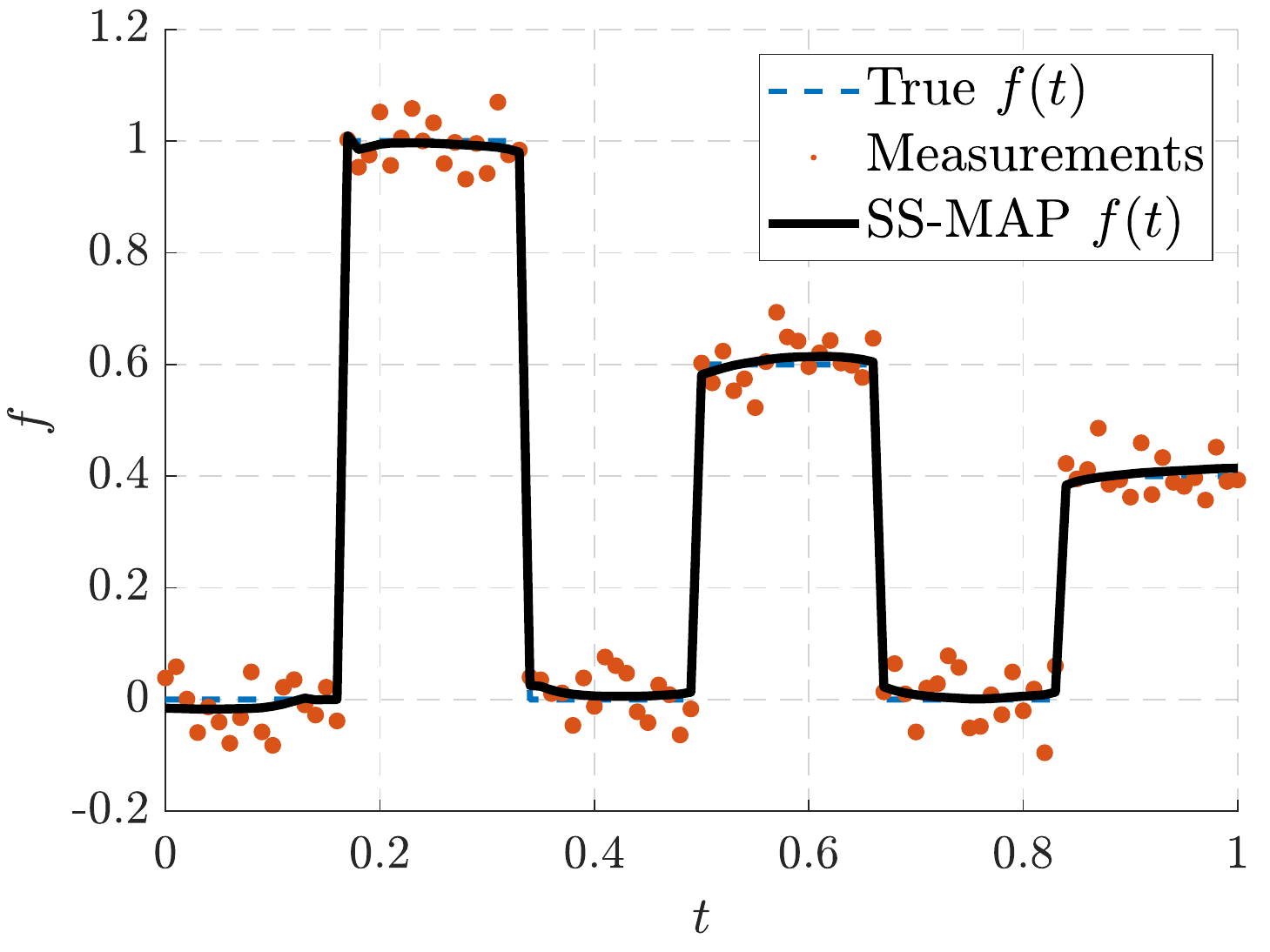}
		\includegraphics[width=.32\linewidth]{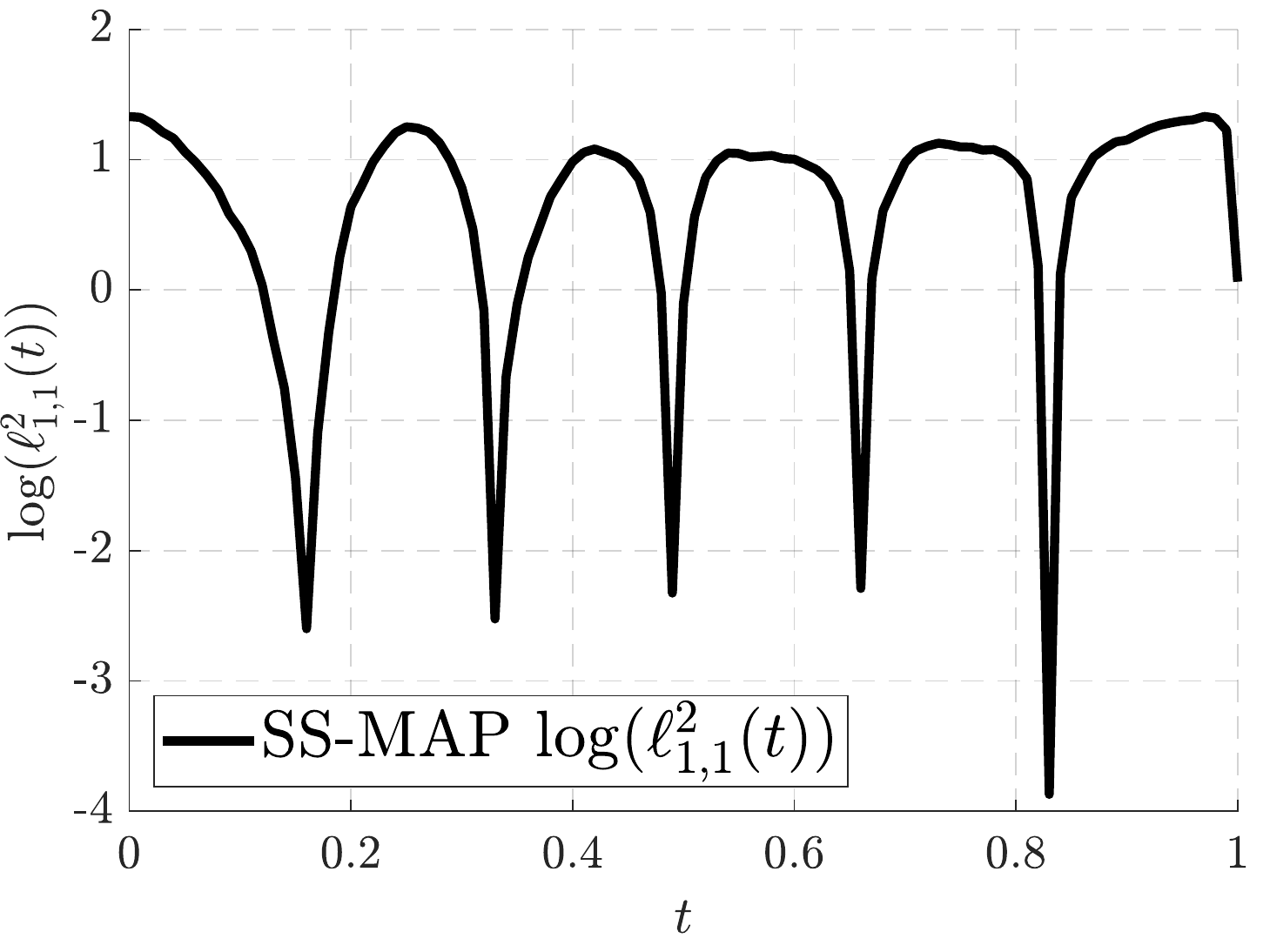}
		\includegraphics[width=.32\linewidth]{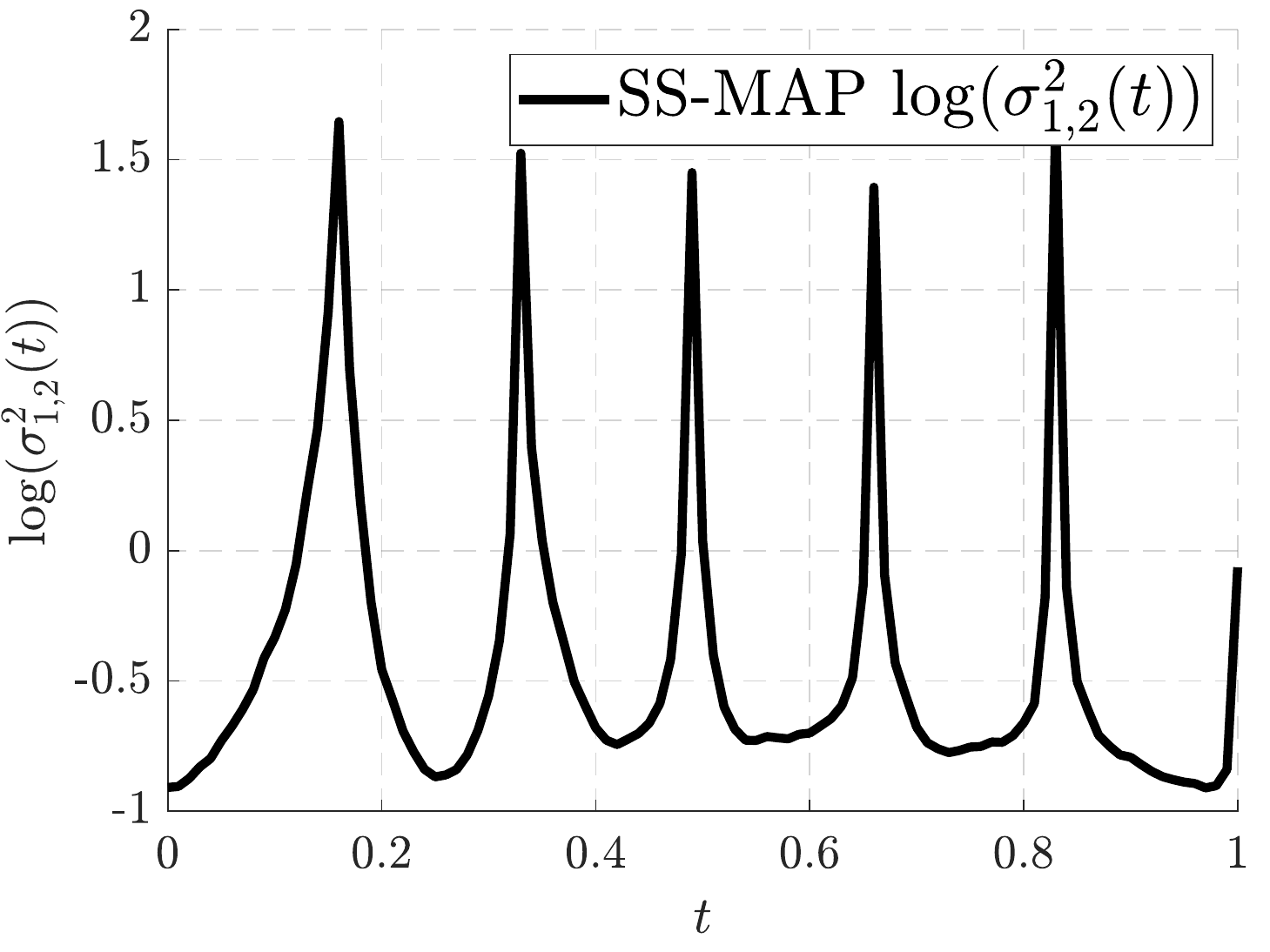}\\
		\includegraphics[width=.24\linewidth]{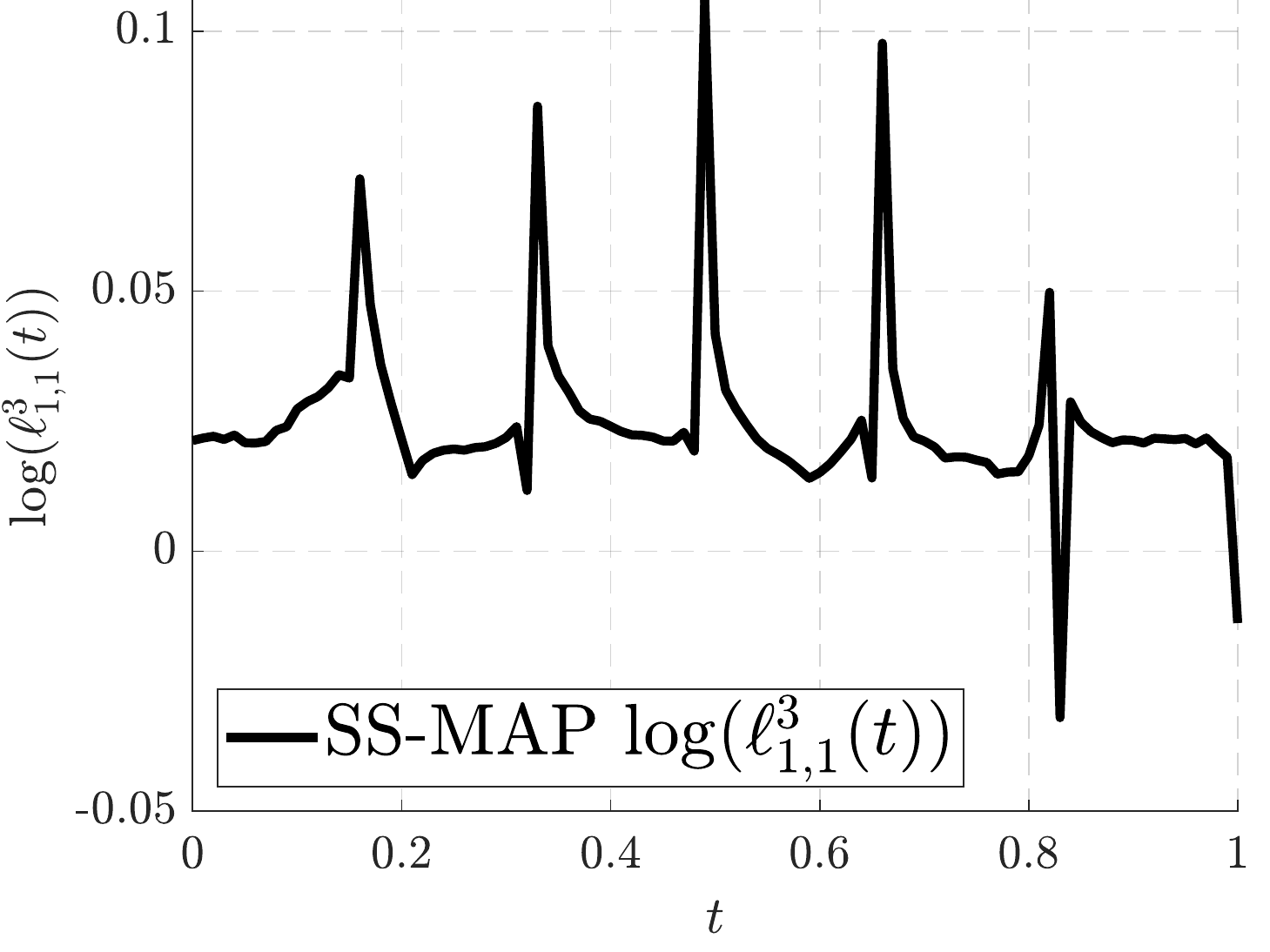}
		\includegraphics[width=.24\linewidth]{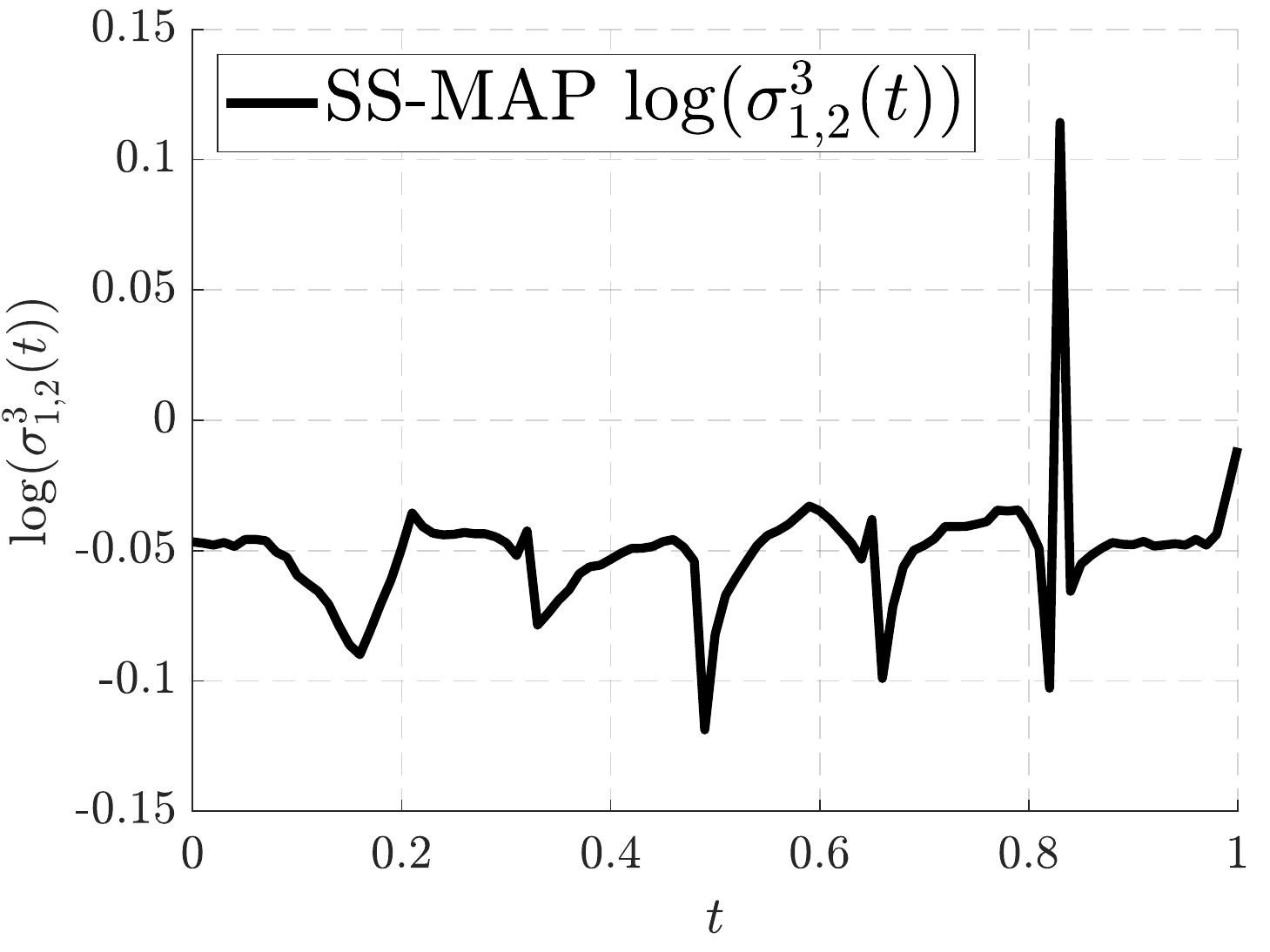}
		\includegraphics[width=.24\linewidth]{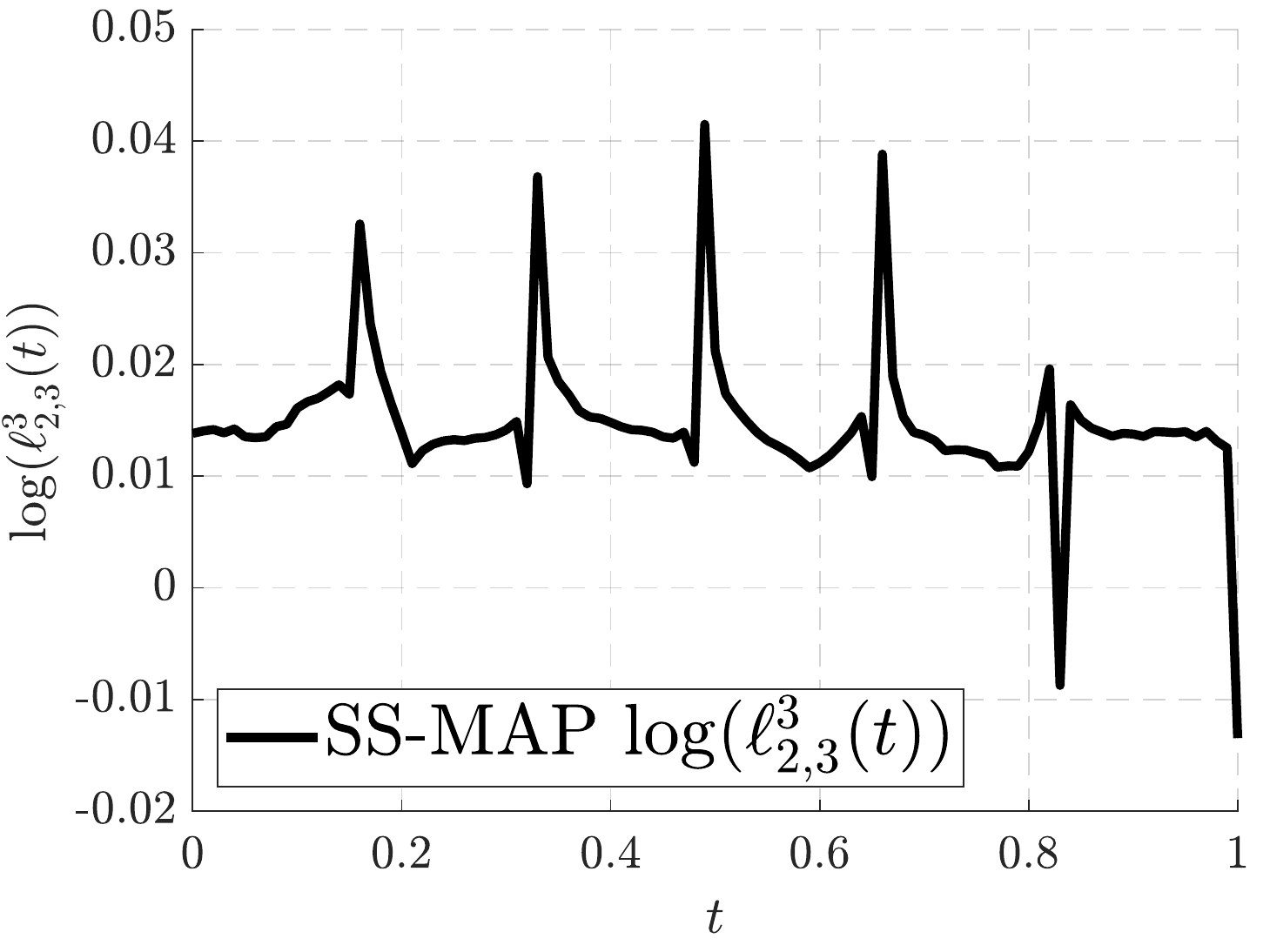}
		\includegraphics[width=.24\linewidth]{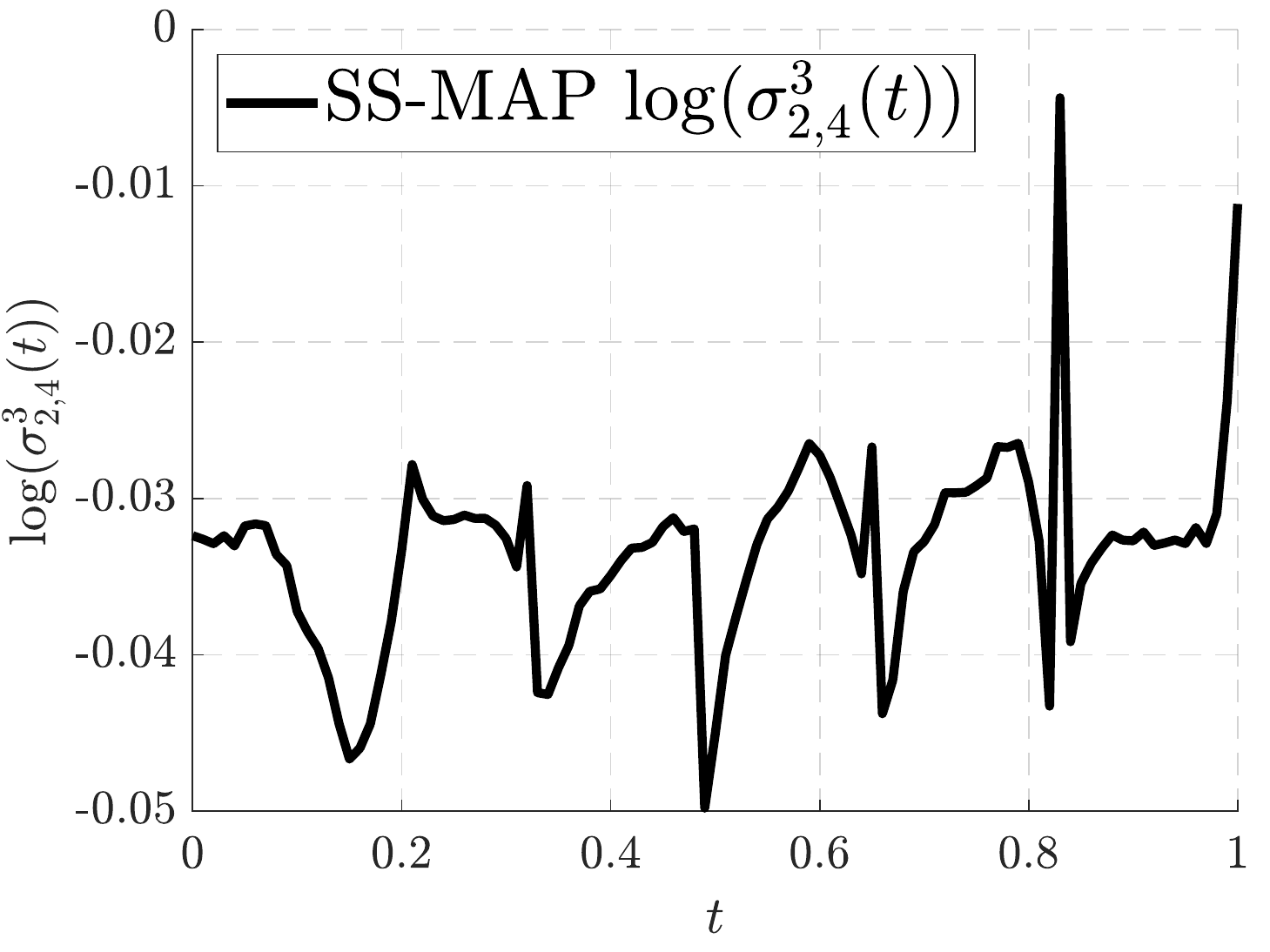}
		\caption{SS-MAP regression results on model~\eqref{equ:exp-y} using DGP-2 (first row) and DGP-3 (second and third rows).}
		\label{fig:exp-SSMAP}
	\end{figure*}
	
	\begin{figure}[h!]
		\centering
		\includegraphics[width=.49\linewidth]{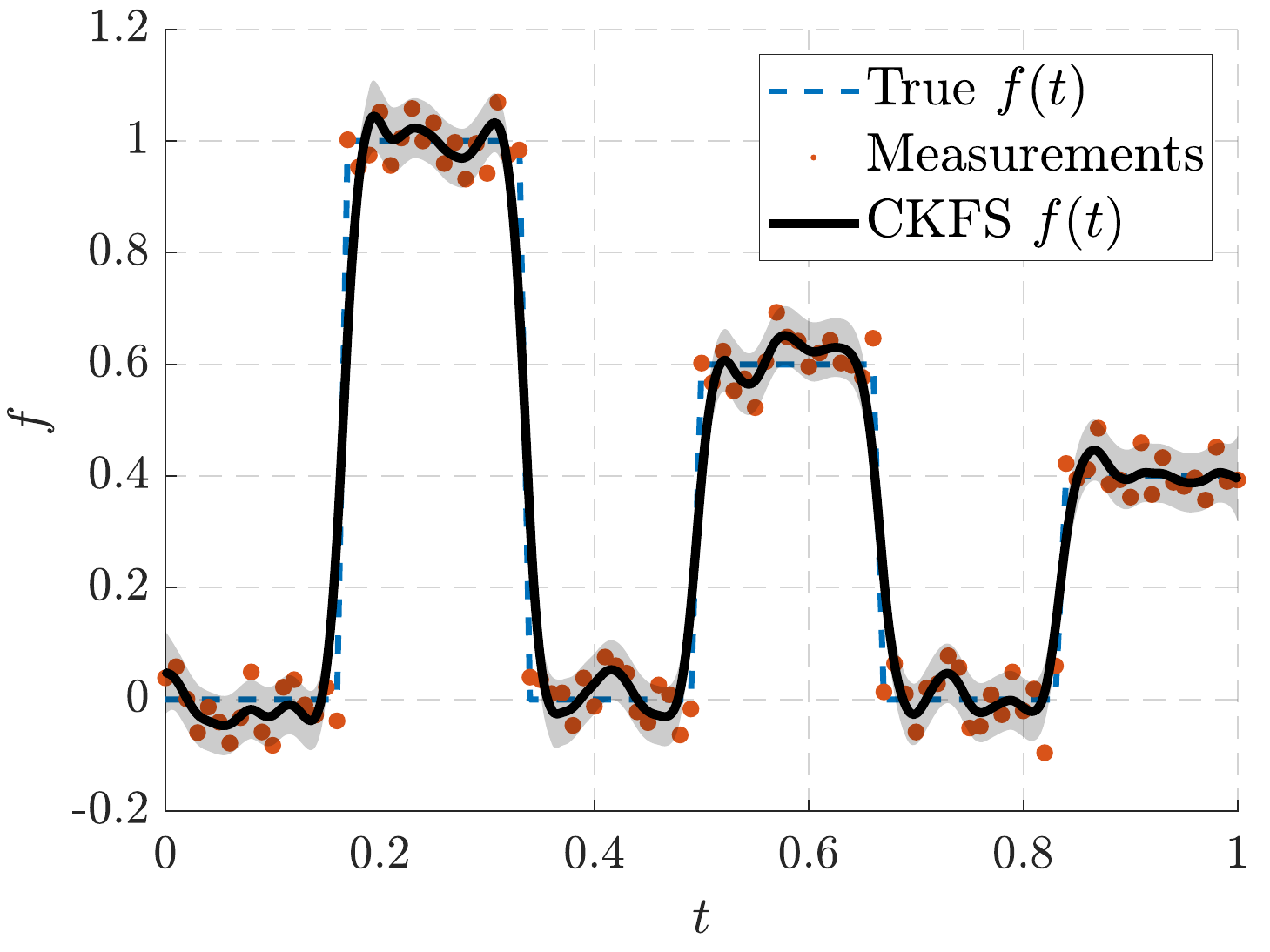}
		\includegraphics[width=.49\linewidth]{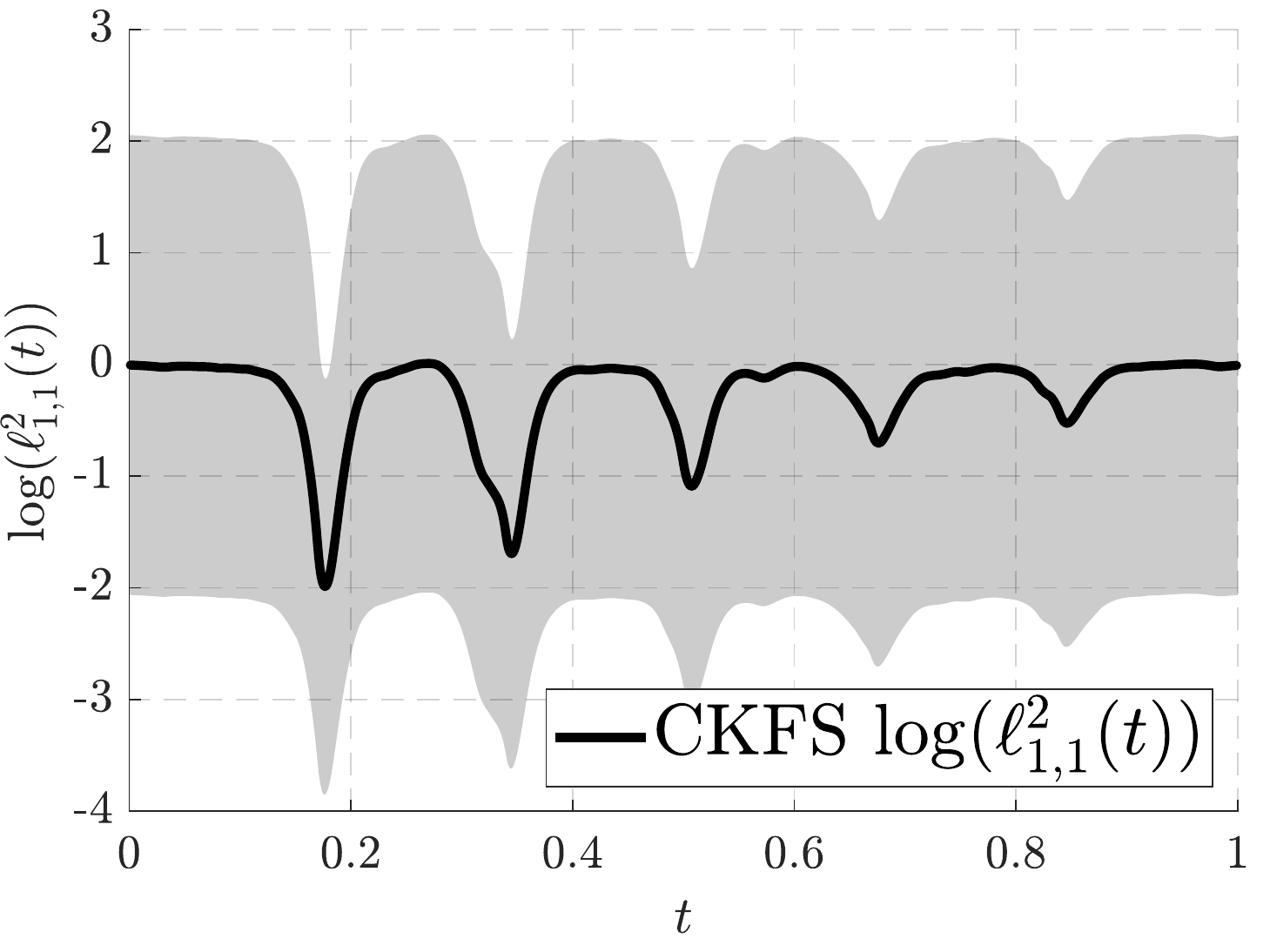}\\
		\includegraphics[width=.49\linewidth]{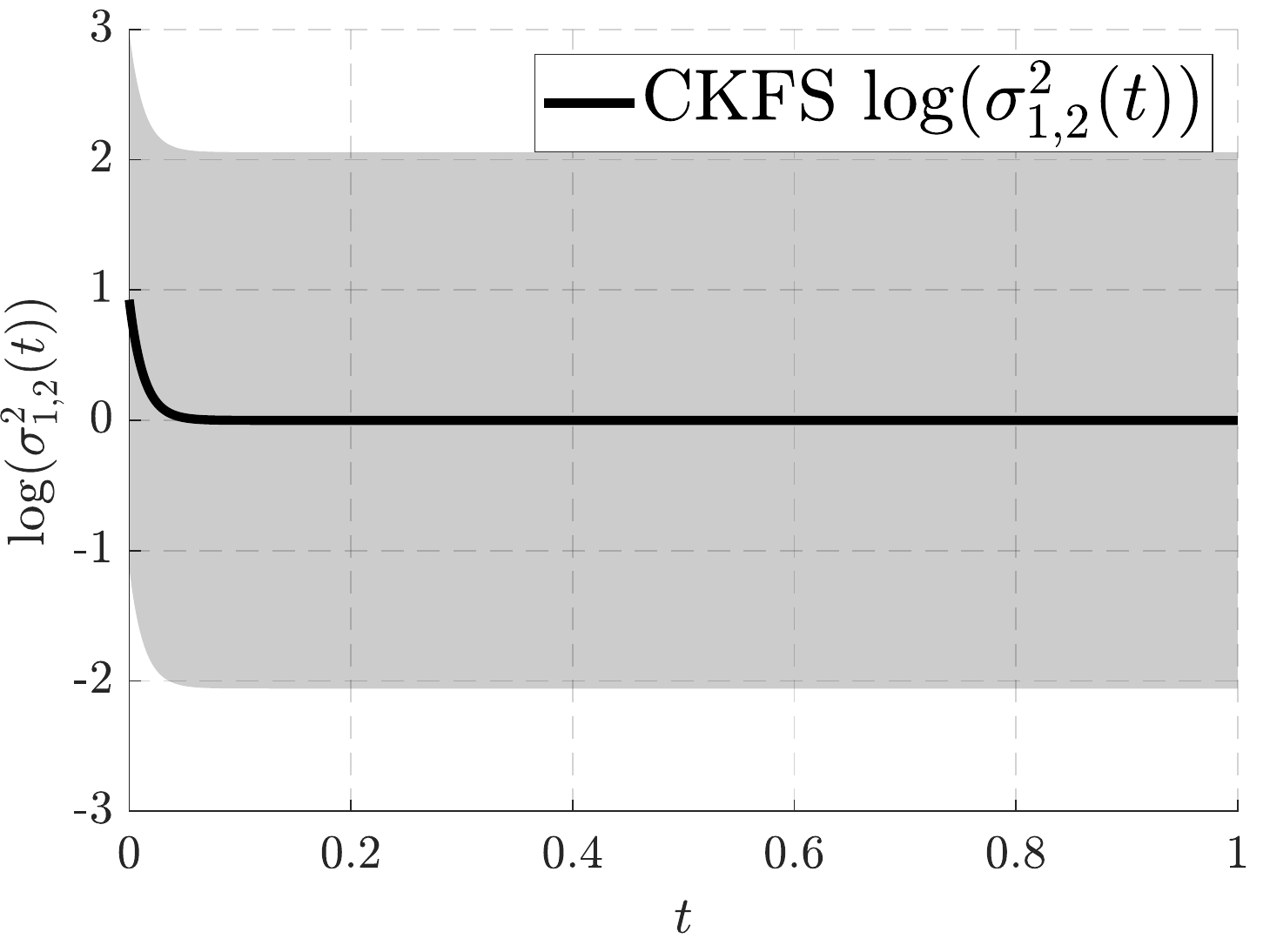}
		\includegraphics[width=.49\linewidth]{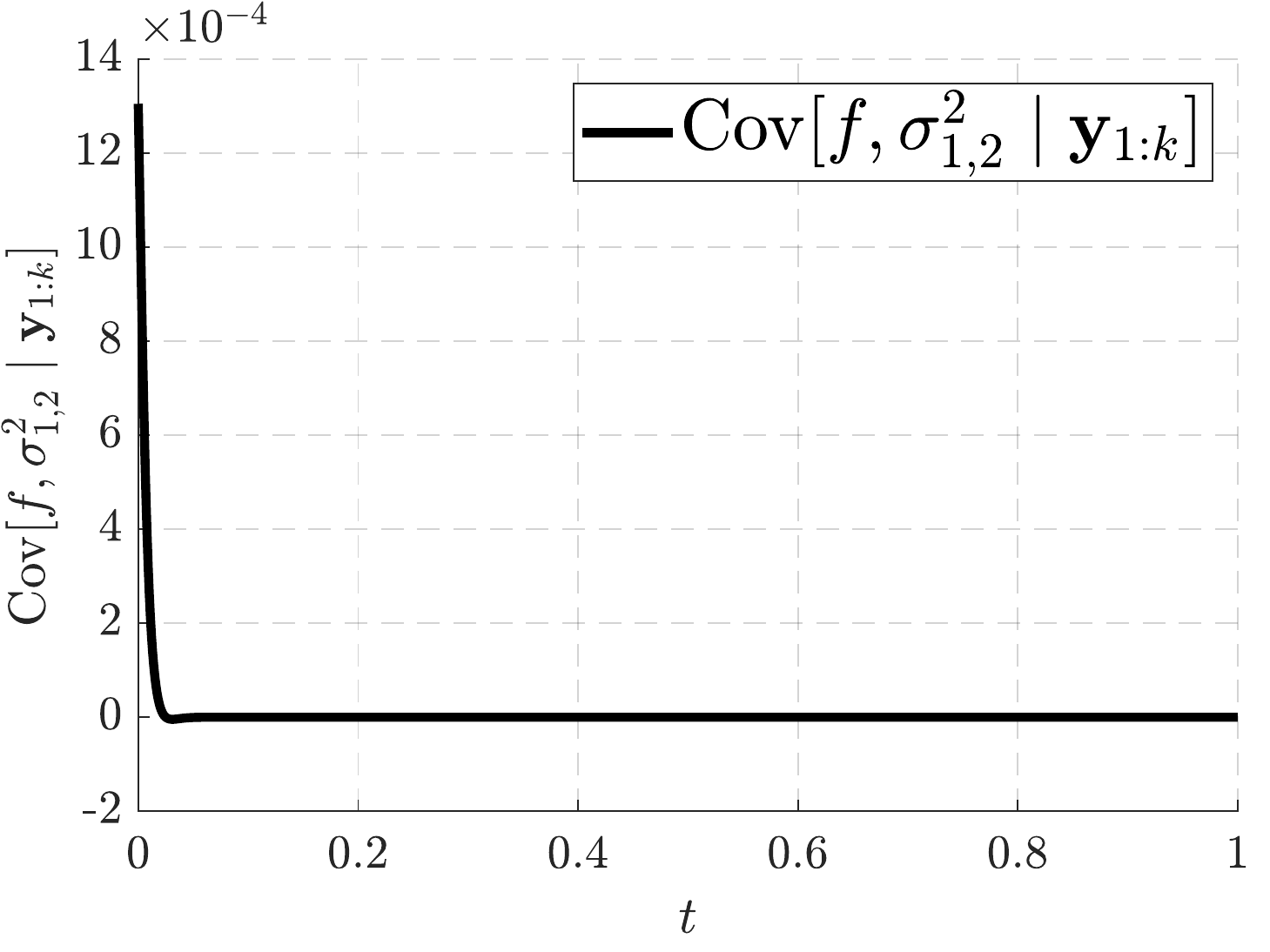}
		\caption{CKFS regression on model~\eqref{equ:exp-y} using DGP-2. }
		\label{fig:exp-GFS-vanishing}
	\end{figure}
	
	\begin{figure*}[t!]
		\centering
		\includegraphics[width=.32\linewidth]{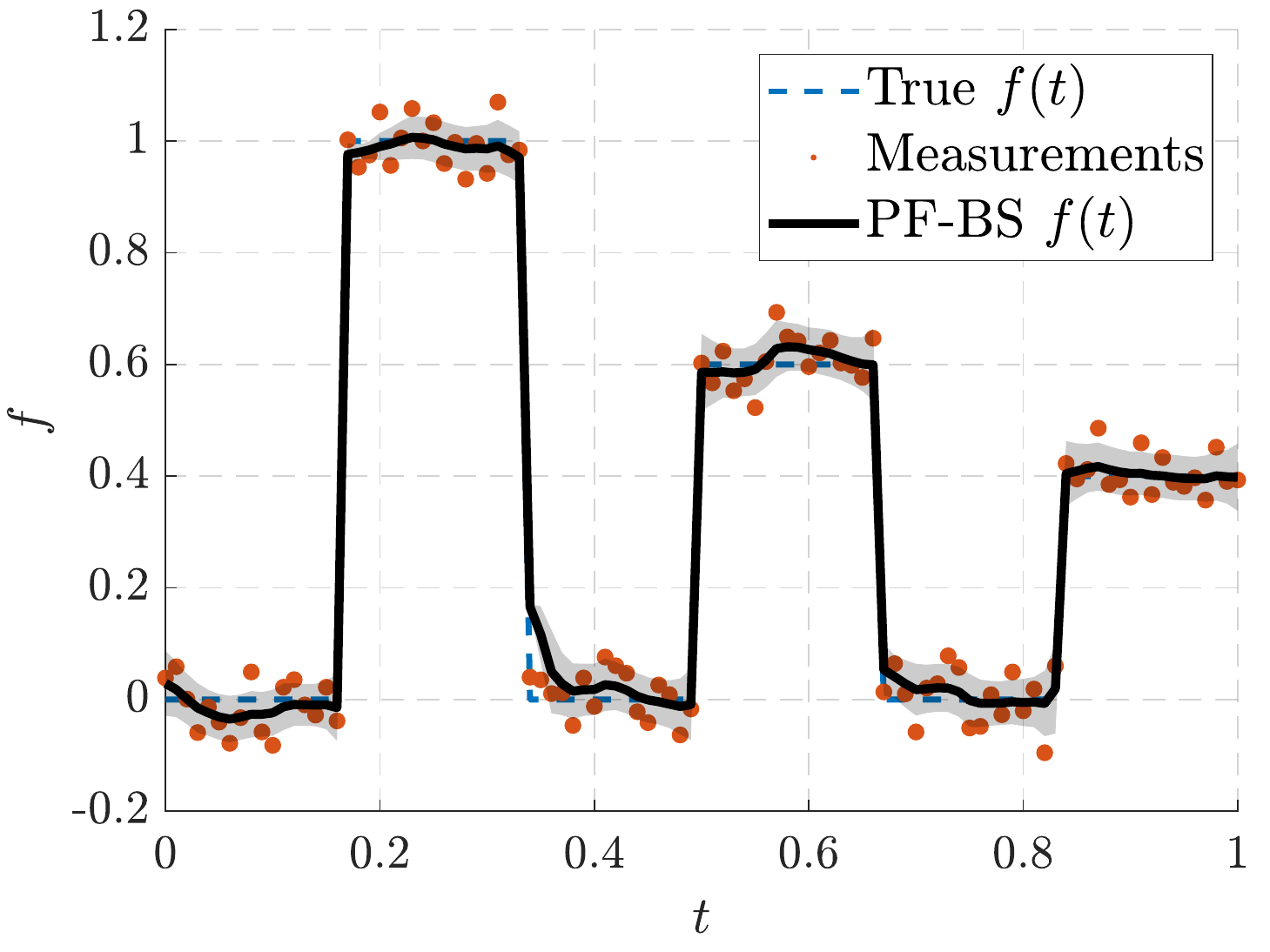}
		\includegraphics[width=.32\linewidth]{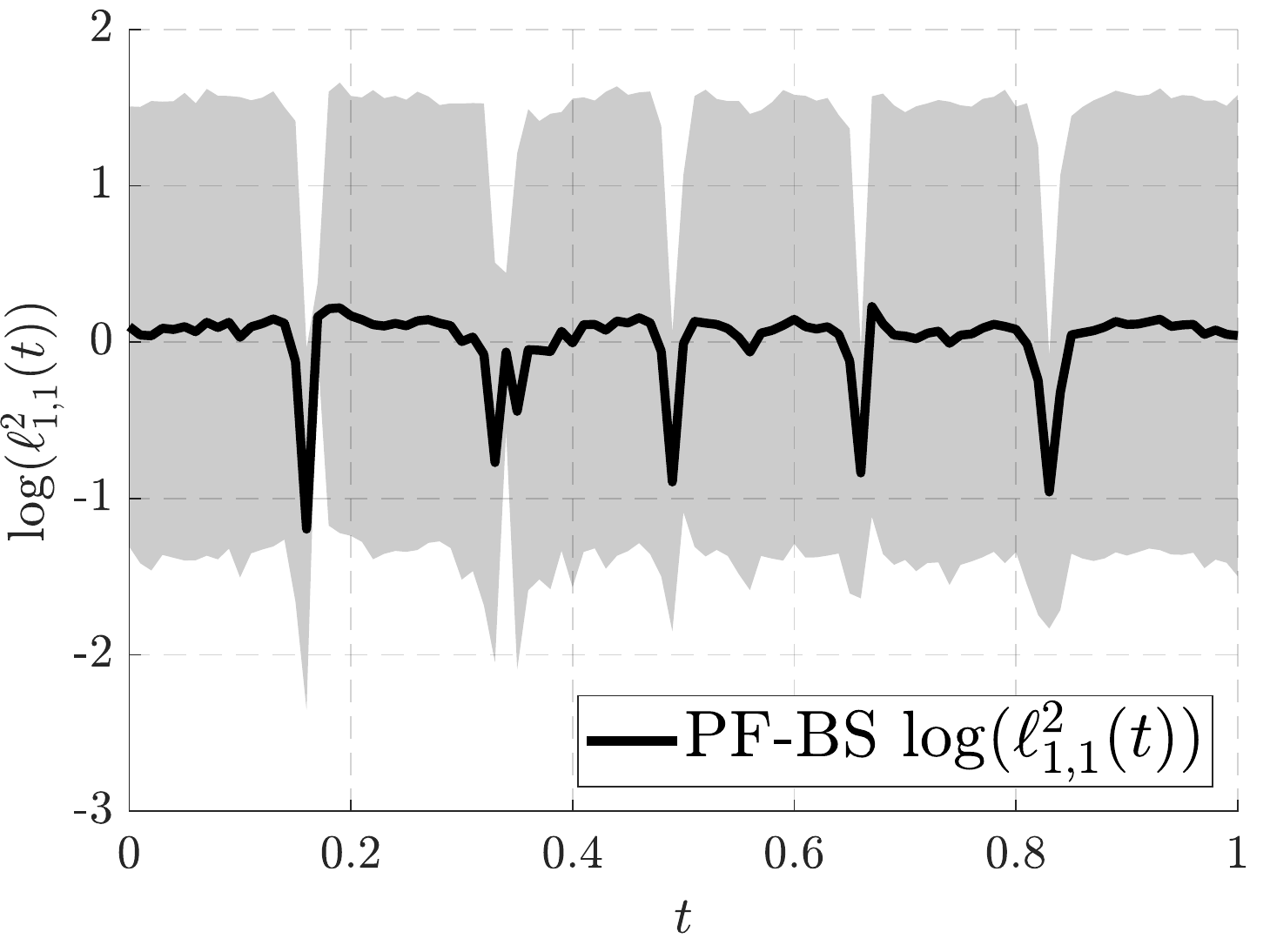}
		\includegraphics[width=.32\linewidth]{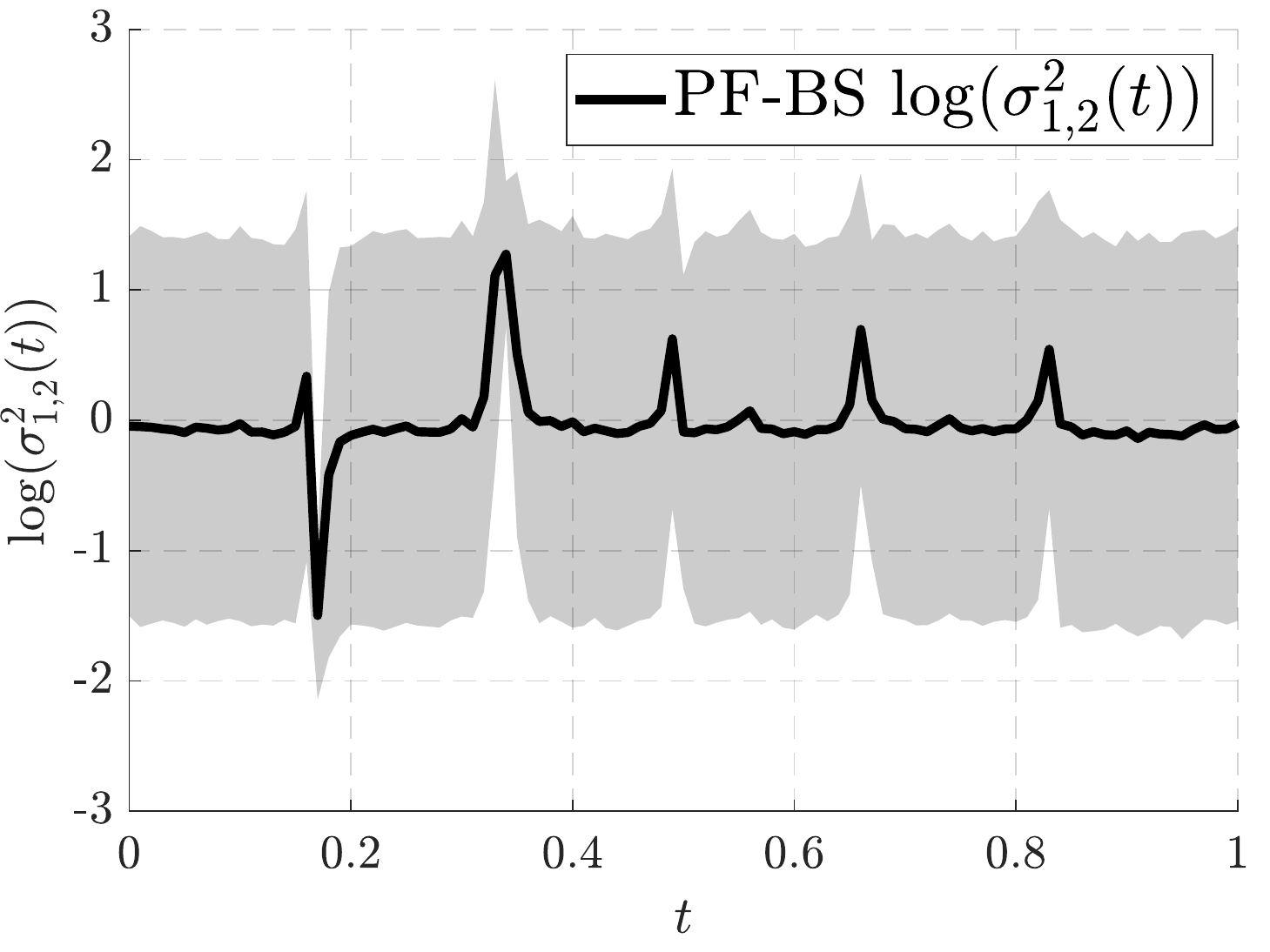}\\
		\includegraphics[width=.32\linewidth]{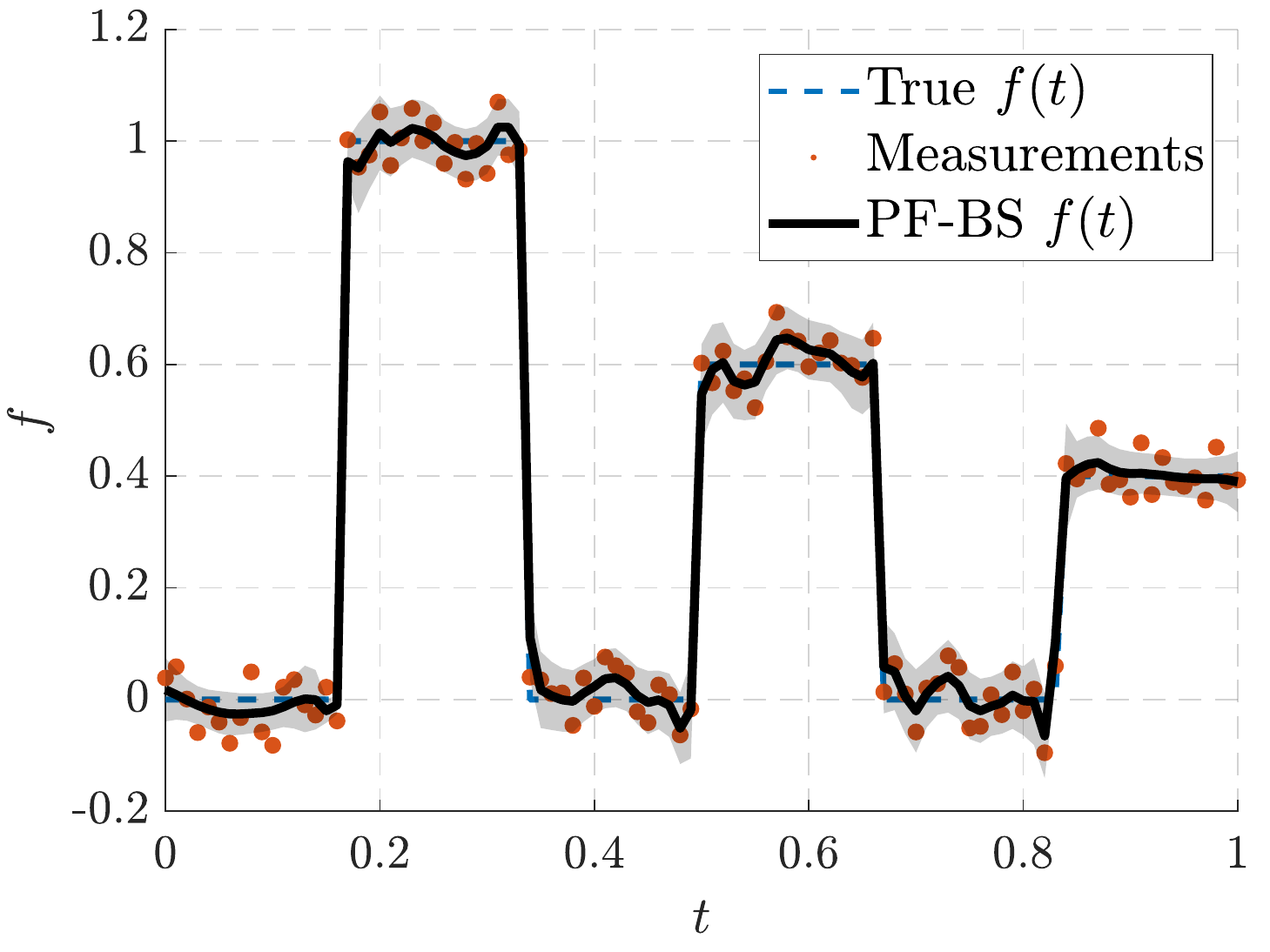}
		\includegraphics[width=.32\linewidth]{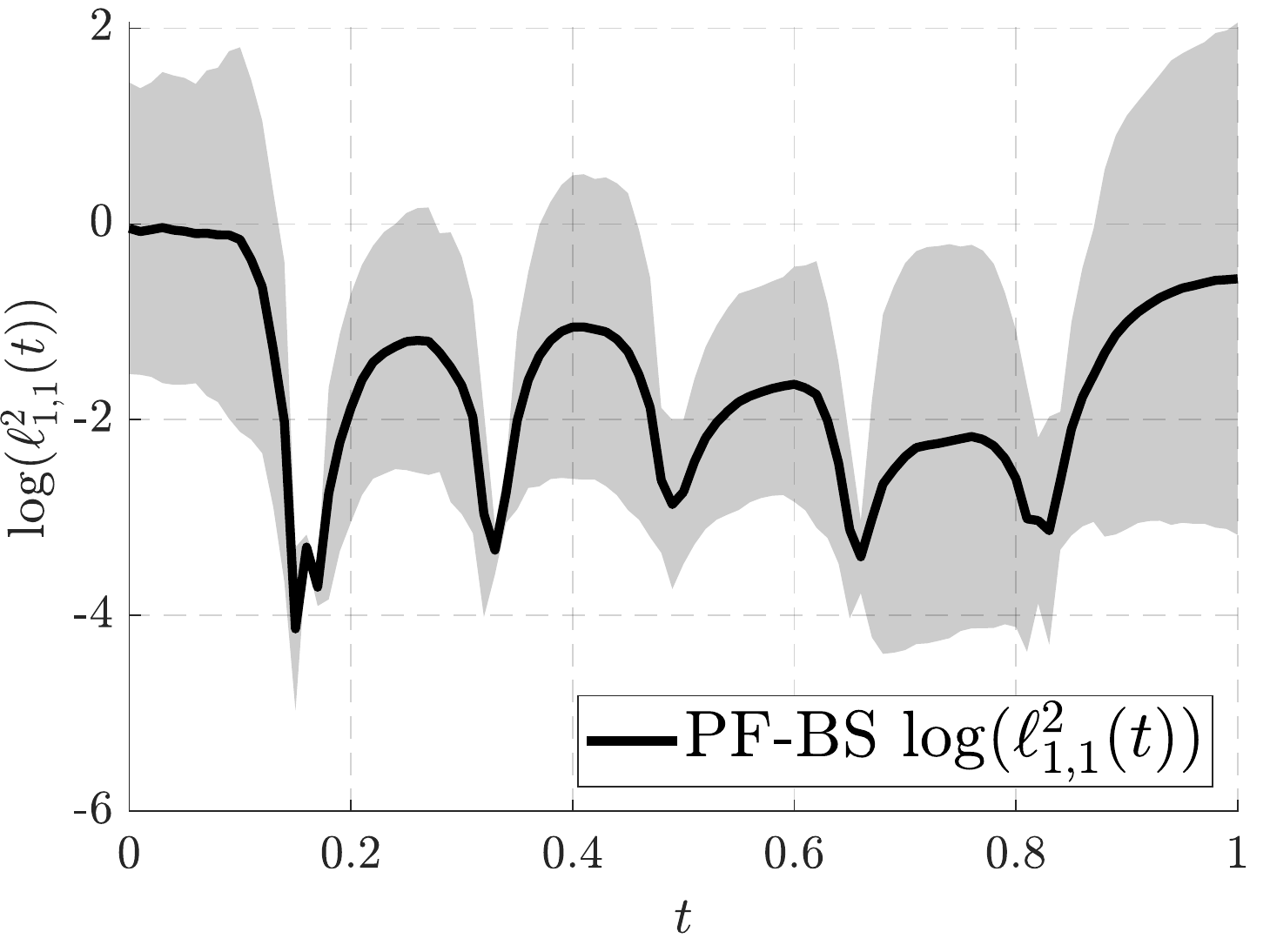}
		\includegraphics[width=.32\linewidth]{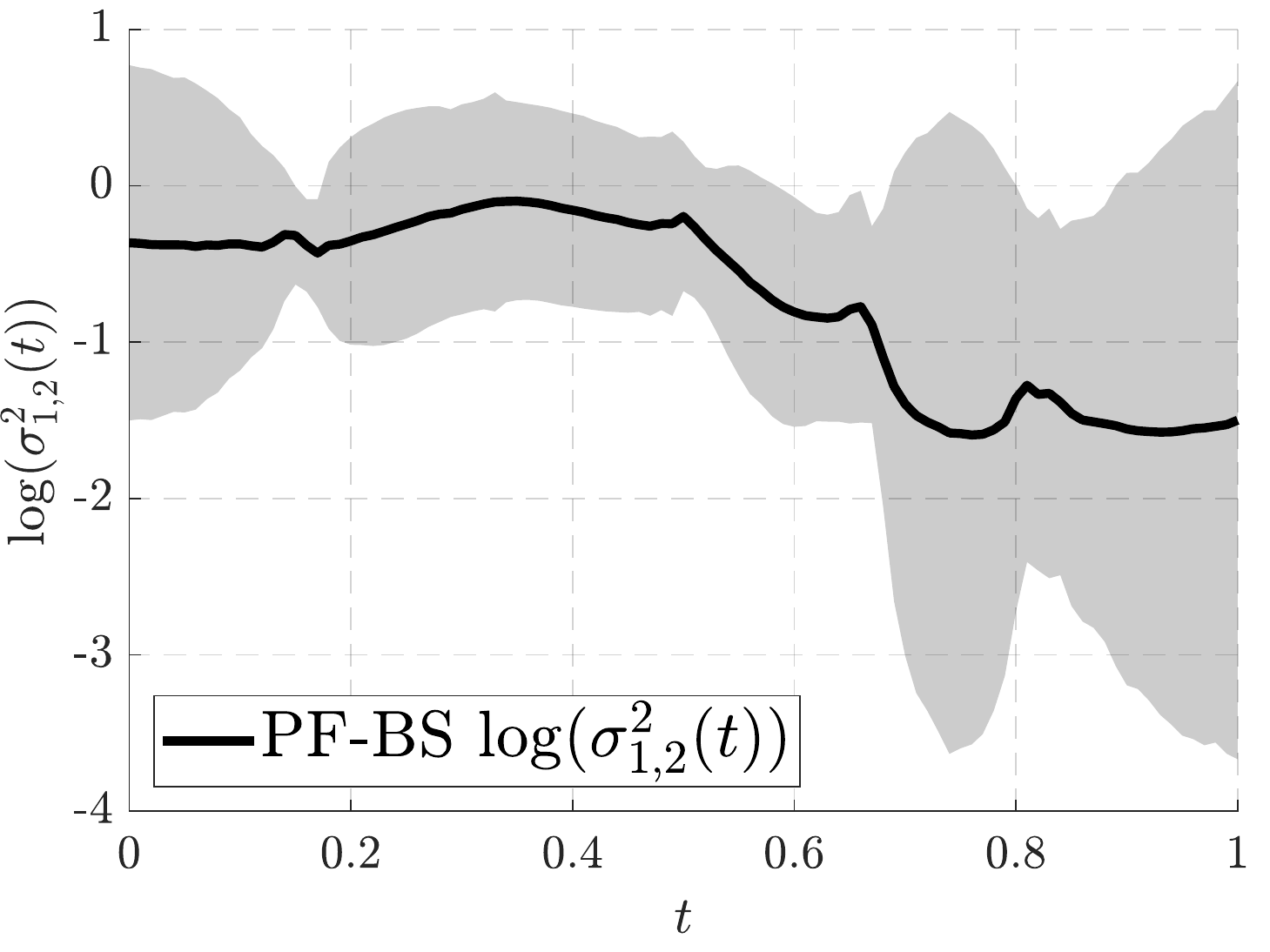}\\
		\includegraphics[width=.24\linewidth]{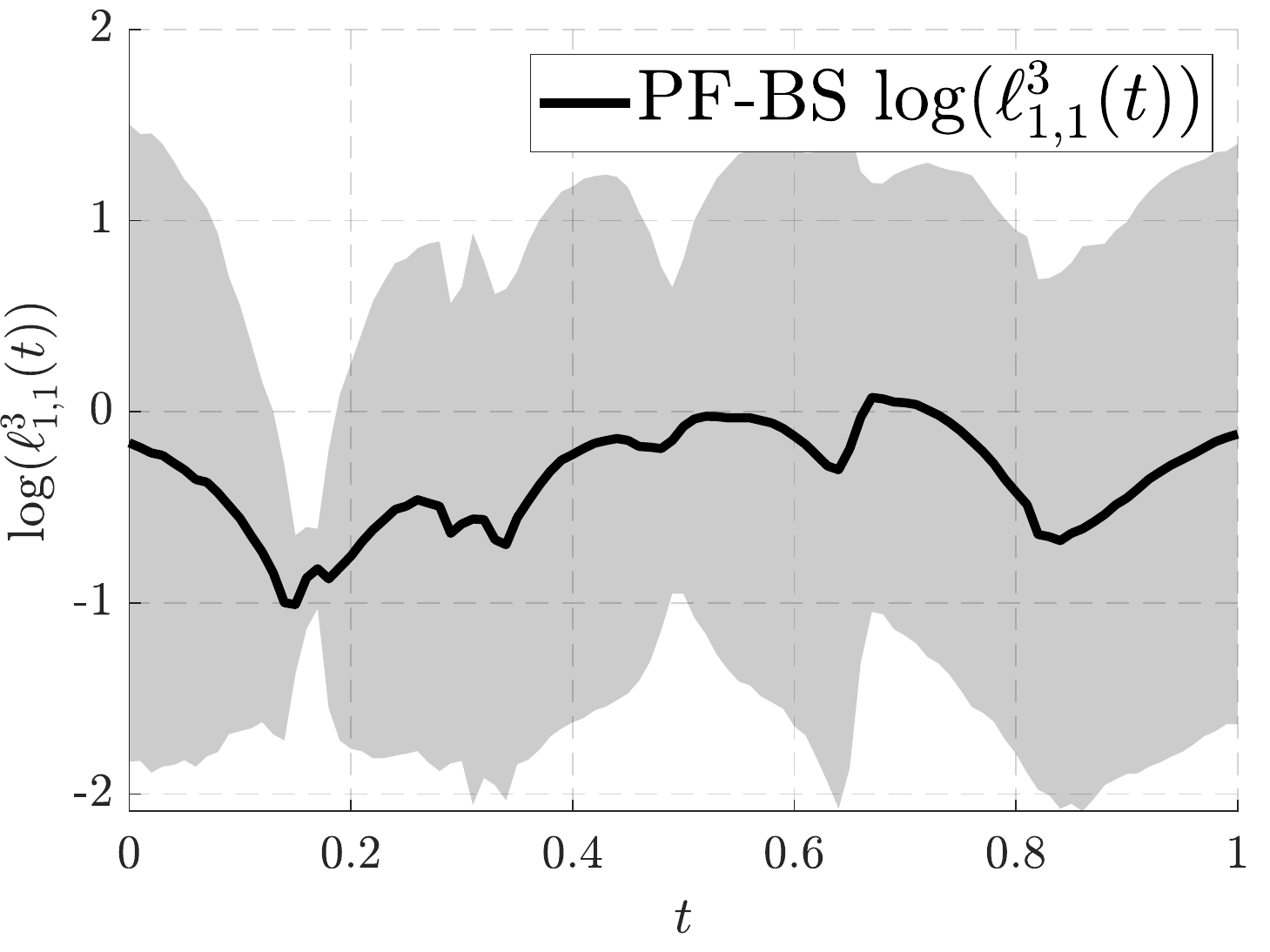}
		\includegraphics[width=.24\linewidth]{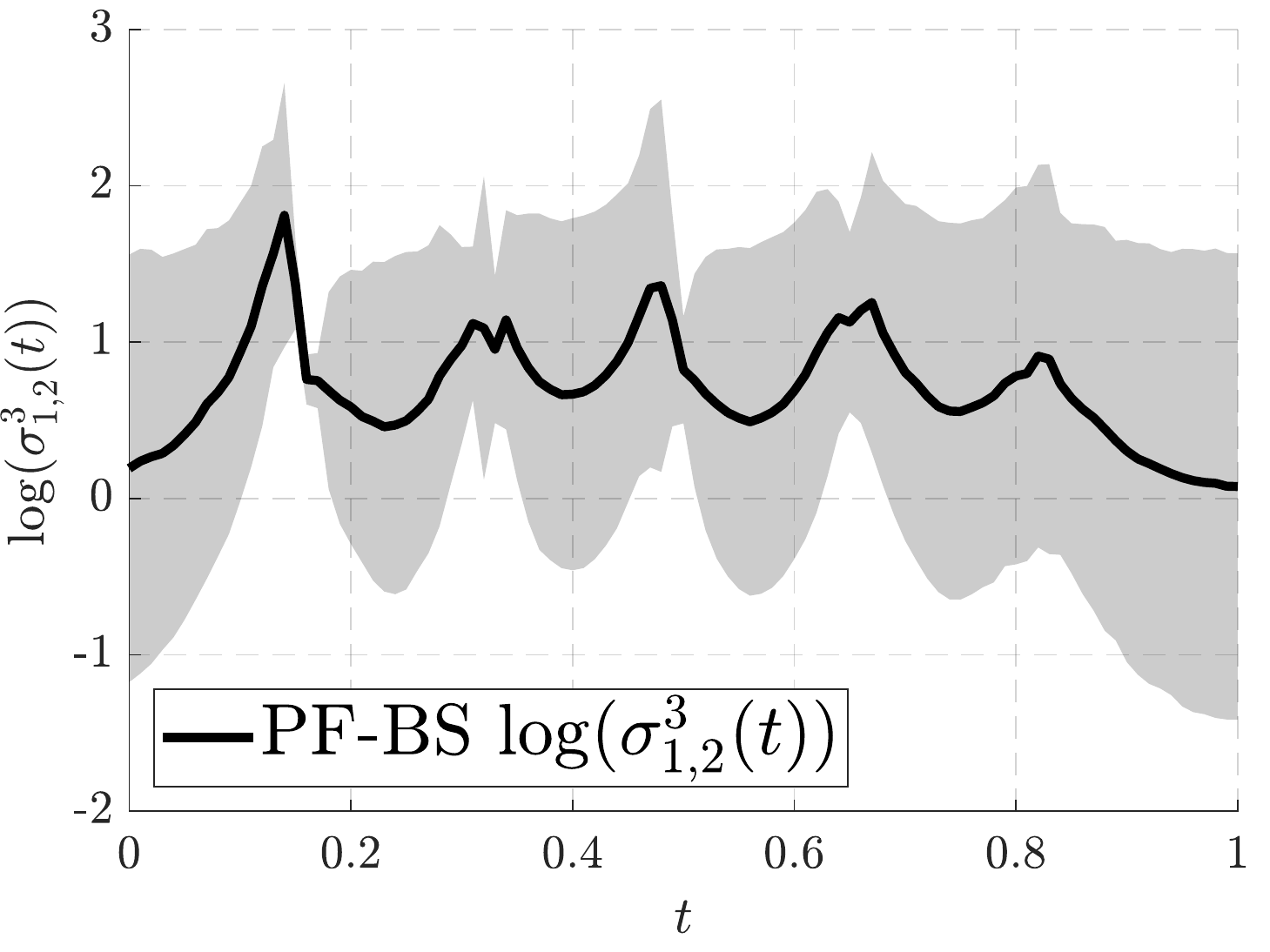}
		\includegraphics[width=.24\linewidth]{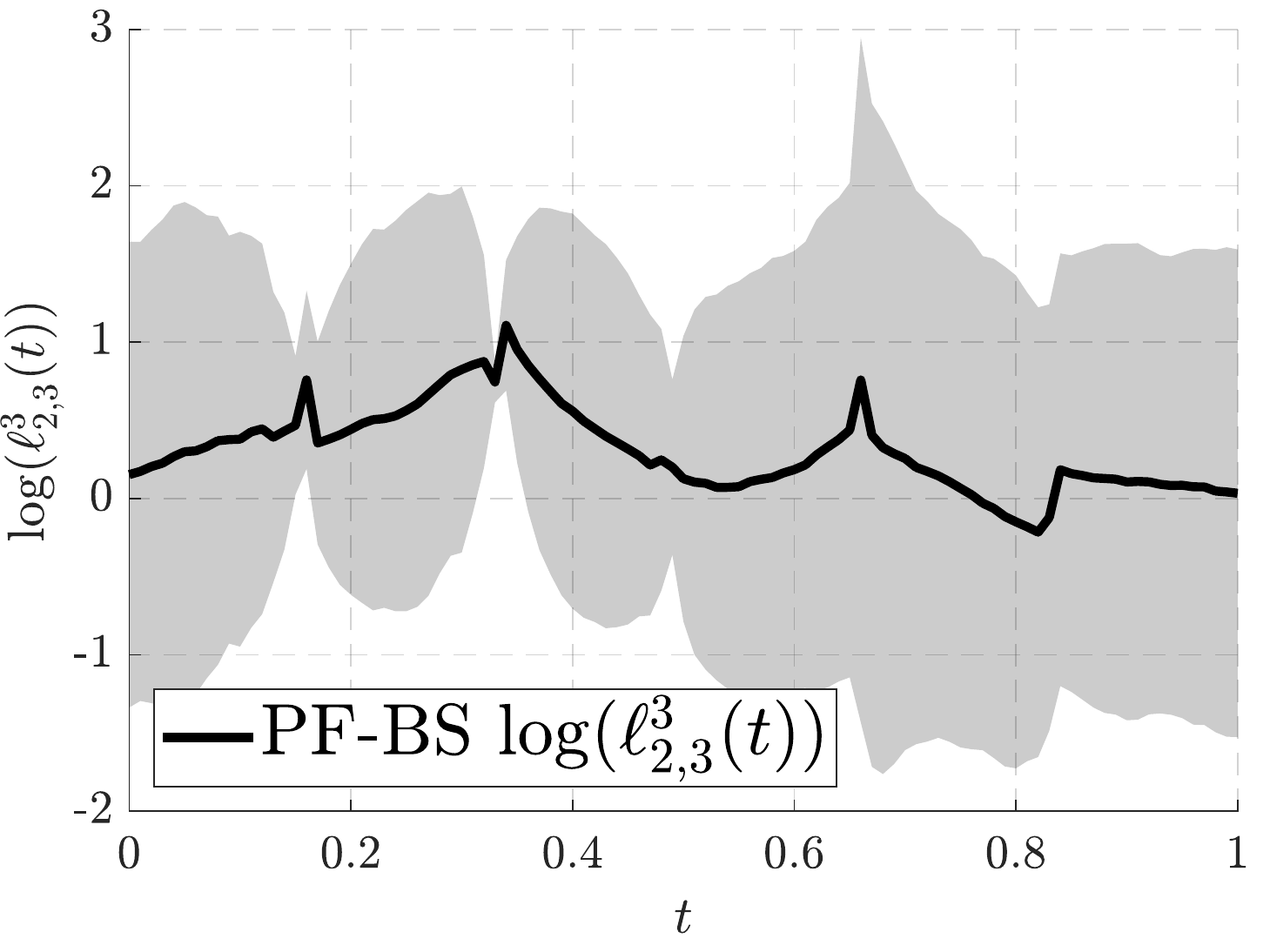}
		\includegraphics[width=.24\linewidth]{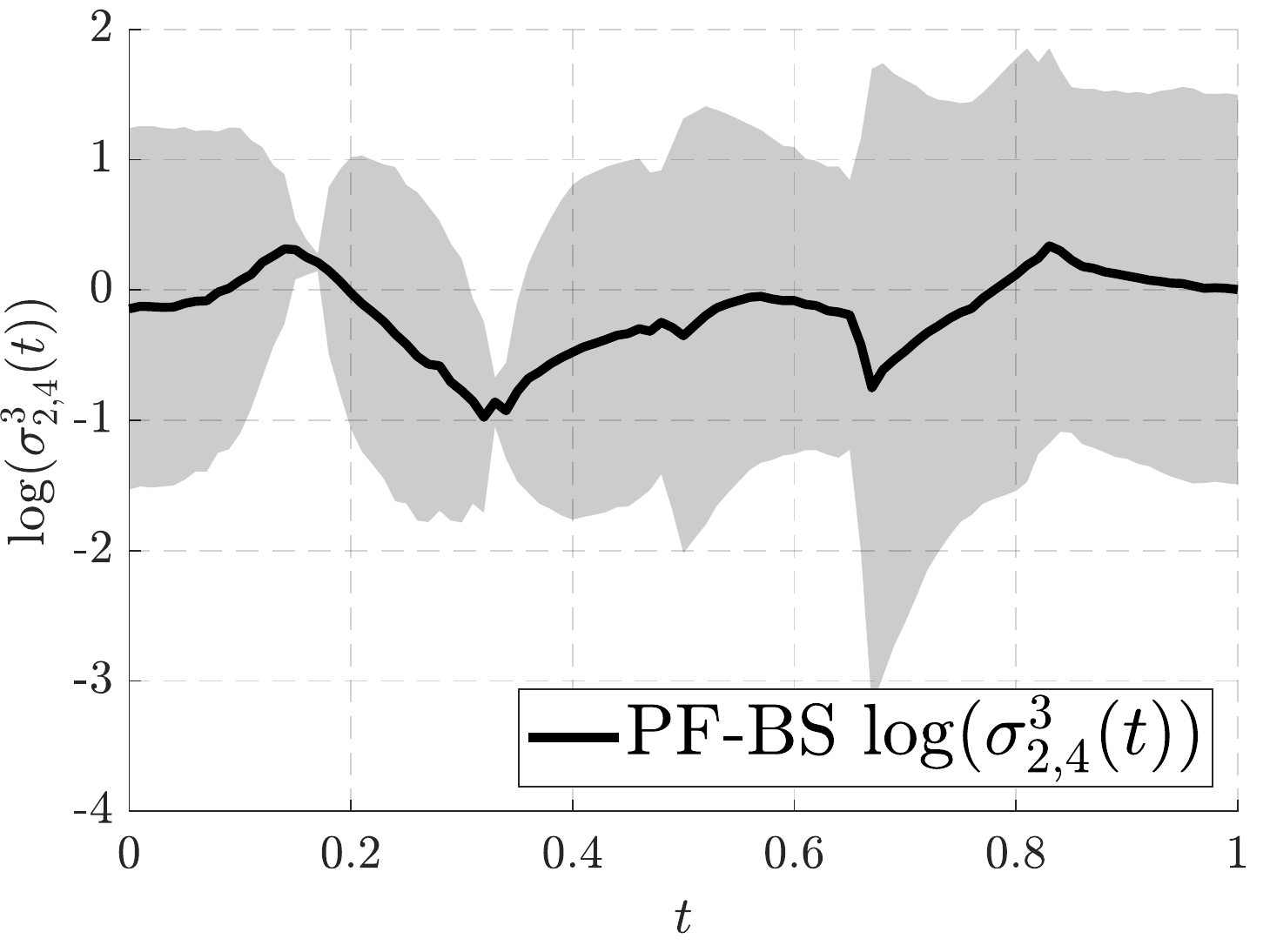}
		\caption{PF-BS regression results on model DGP-2 (first row) and DGP-3 (second and third rows). The shaded area stands for 95\% confidence interval. }
		\label{fig:exp-PFBS}
	\end{figure*}
	
	Figure~\ref{fig:exp-GP-NSGP} shows the results of GP and NS-GP regression. Both of GP and NS-GP experience overfitting problem on this rectangle signal, while the estimated posterior variance of NS-GP is significantly smaller than that of GP. The outcome of GP is expected, as the covariance function is stationary. Because there are no constraints (e.g., being time-continuous) on the parameters of NS-GP, the learnt $\ell^2_{1,1}$ and $\sigma^2_{1,2}$ overfit to the likelihood function individually at each time instant~\citep[cf. ][]{paciorek2006spatial}. From Table~\ref{tbl:rect-rmse-nlpd-time} we can see that the RMSE and NLPD of GP and NS-GP are very close.
	
	The results of B-MAP on batch-DGPs are shown in Figure~\ref{fig:exp-BMAP}. We can see a slight improvement in overfitting compared to GP and NS-GP. However, the learnt function $f(t)$ of B-MAP is not smooth enough and is jittering. For B-MAP on DGP-2, the estimated $\ell^2_{1,1}$ and $\sigma^2_{1,2}$ change abruptly on the jump points, and do not stay at flat levels, especially $\ell^2_{1,1}$. On the contrary, the estimated $\ell^3_{1,1}$ and $\sigma^3_{1,2}$ on the last layer of DGP-3 stay mostly flat while changing sharply on the jump points. From Figure~\ref{fig:exp-BMAP} and the RMSEs of Table~\ref{tbl:rect-rmse-nlpd-time} we can see that the results of B-MAP on DGP-2 and DGP-3 are almost identical with subtle differences.
	
	Compared to the batch-DGP, the SS-DGP method gives a better fit to the true function. This result is demonstrated in Figure~\ref{fig:exp-SSMAP}, where SS-MAP is used. There is no noticeable overfitting problem in the SS-MAP estimates. The learnt function $f$ is smooth and fits to the actual function to a reasonable extent. For SS-MAP on DGP-2, the estimated $\ell^2_{1,1}$ and $\sigma^2_{1,2}$ mostly stay at a constant level and change rapidly on the leap points. From the second and third rows of Figure~\ref{fig:exp-SSMAP} and Table~\ref{tbl:rect-rmse-nlpd-time}, we see that the SS-MAP achieves better result on DGP-3 compared to on DGP-2. We also find that the learnt parameters $\ell^2_{1,1}$ and $\sigma^2_{1,2}$ of DGP-3 appear to be smoother than of DGP-2. 
	
	Apart from the SS-MAP solution to the SS-DGP, we also demonstrate the Bayesian filtering and smoothing solutions in Figures~\ref{fig:exp-GFS-vanishing}, \ref{fig:exp-GFS}, and~\ref{fig:exp-PFBS}. Figure~\ref{fig:exp-GFS-vanishing} shows the results of CKFS on DGP-2. We find that the regression result on DGP-2 is acceptable though the estimate is overly smooth on the jump points. The learnt parameters $\ell^2_{1,1}$ also change significantly on the jump points and stay flat elsewhere. Moreover, we find that the estimated $\log(\sigma^2_{1,2})$ and $\cov[f,\sigma^2_{1,2}\mid\cu{y}_{1:k}]$ converge to zero in very fast speeds, especially the covariance estimate. This phenomenon resembles the vanishing covariance in Theorem~\ref{thm:cov-post}. In this case, the estimated $\log(\sigma^2_{1,2})$ converges to the prior mean of SS-DGP which is zero, due to the vanishing covariance. Therefore for this experiment and all the following experiments, we treat all the magnitude parameters of \matern (e.g., $\sigma^2_{1,2}$) as trainable hyperparameters learnt from grid searches. The results are illustrated in Figure~\ref{fig:exp-GFS}. However, we identify that there is a numerical difficulty when applying CKFS on DGP-3. With many hyperparameter settings, the CKFS fails due to numerical problems (e.g., singular matrix). The EKFS still works on DGP-3, thus we plot the results in the second row of Figure~\ref{fig:exp-GFS}. The estimated $f$ of EKFS appears to be over-smooth, especially on the jump points. Also, the estimated variances of $\ell^2_{1,1}$ and $\sigma^2_{1,2}$ are significantly large. 
	
	Figure~\ref{fig:exp-PFBS} illustrates the result of PF-BS. We find that the regression results are reasonably close to the ground truth. Also, the estimated $f$ is smooth. The estimated parameters $\ell^2_{1,1}$ and $\sigma^2_{1,2}$ for PF-BS on DGP-2 have a similar pattern as the results of SS-MAP, CKFS, and EKFS, which only change abruptly on the jump points. However, the $\ell^2_{1,1}$ of DGP-3 does not stay flat generally, and $\sigma^2_{1,2}$ does not change significantly on the jump points. In Table~\ref{tbl:rect-rmse-nlpd-time}, the RMSEs of PF-BS on DGP-3 are better than on DGP-2. Also, PF-BS is slightly better than PF. 
	
	\begin{figure}[h!]
		\centering
		\includegraphics[width=.9\linewidth]{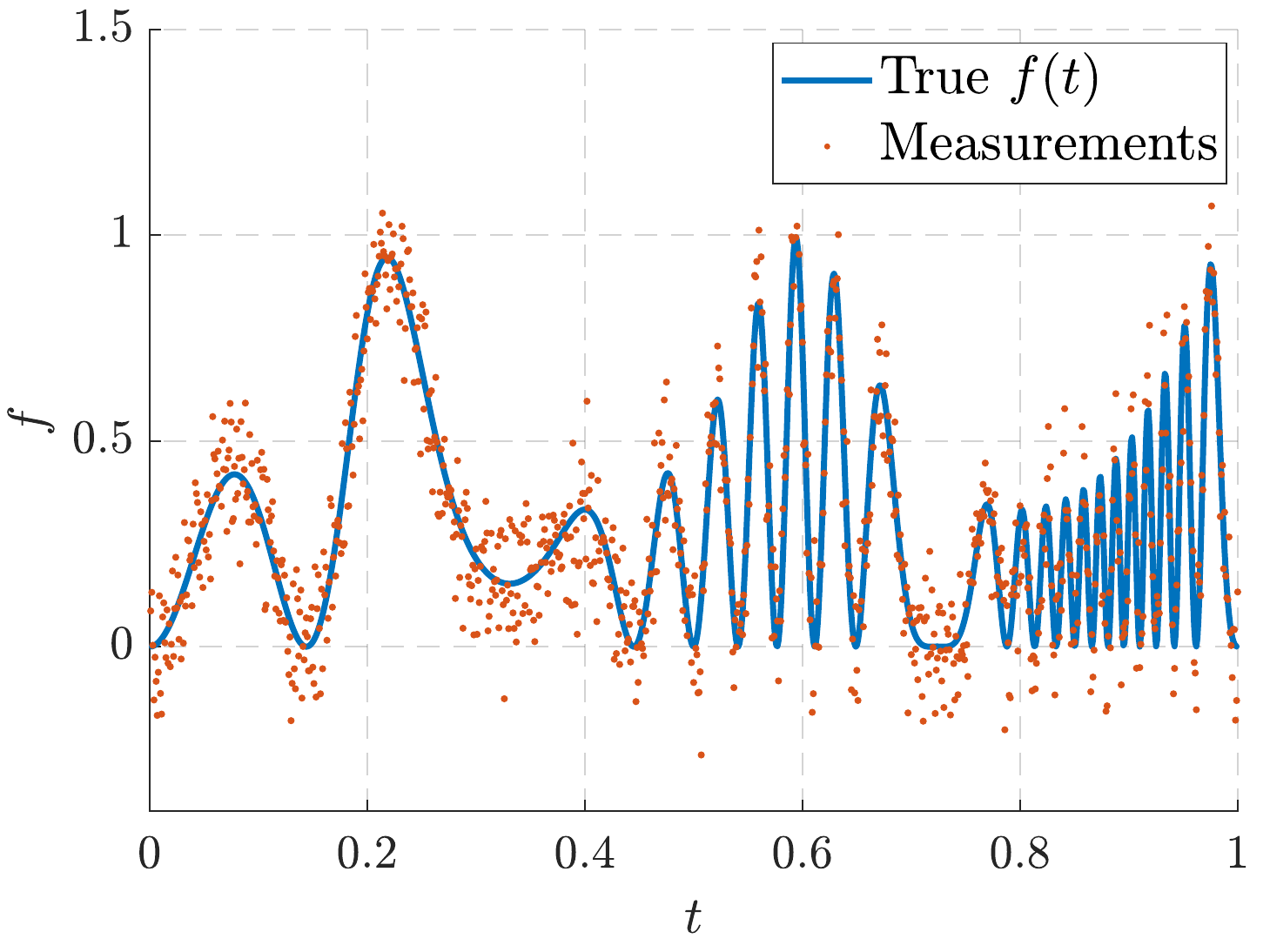}
		\caption{Demonstration of the composite sinusoidal signal~\eqref{equ:exp-sine}.}
		\label{fig:exp-sine-y}
	\end{figure}
	
	\begin{figure*}[t!]
		\centering
		\includegraphics[width=.32\linewidth]{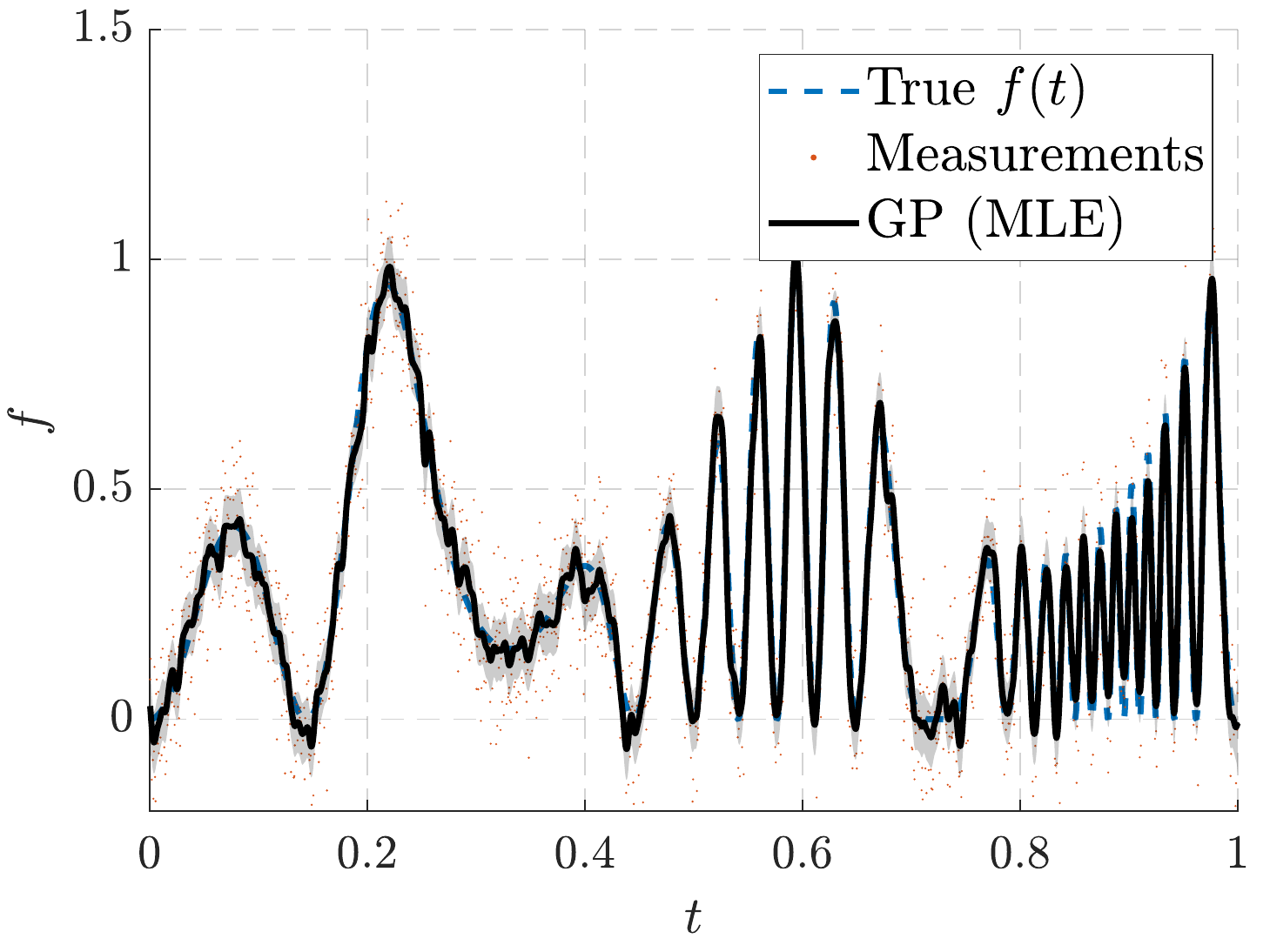}
		\includegraphics[width=.32\linewidth]{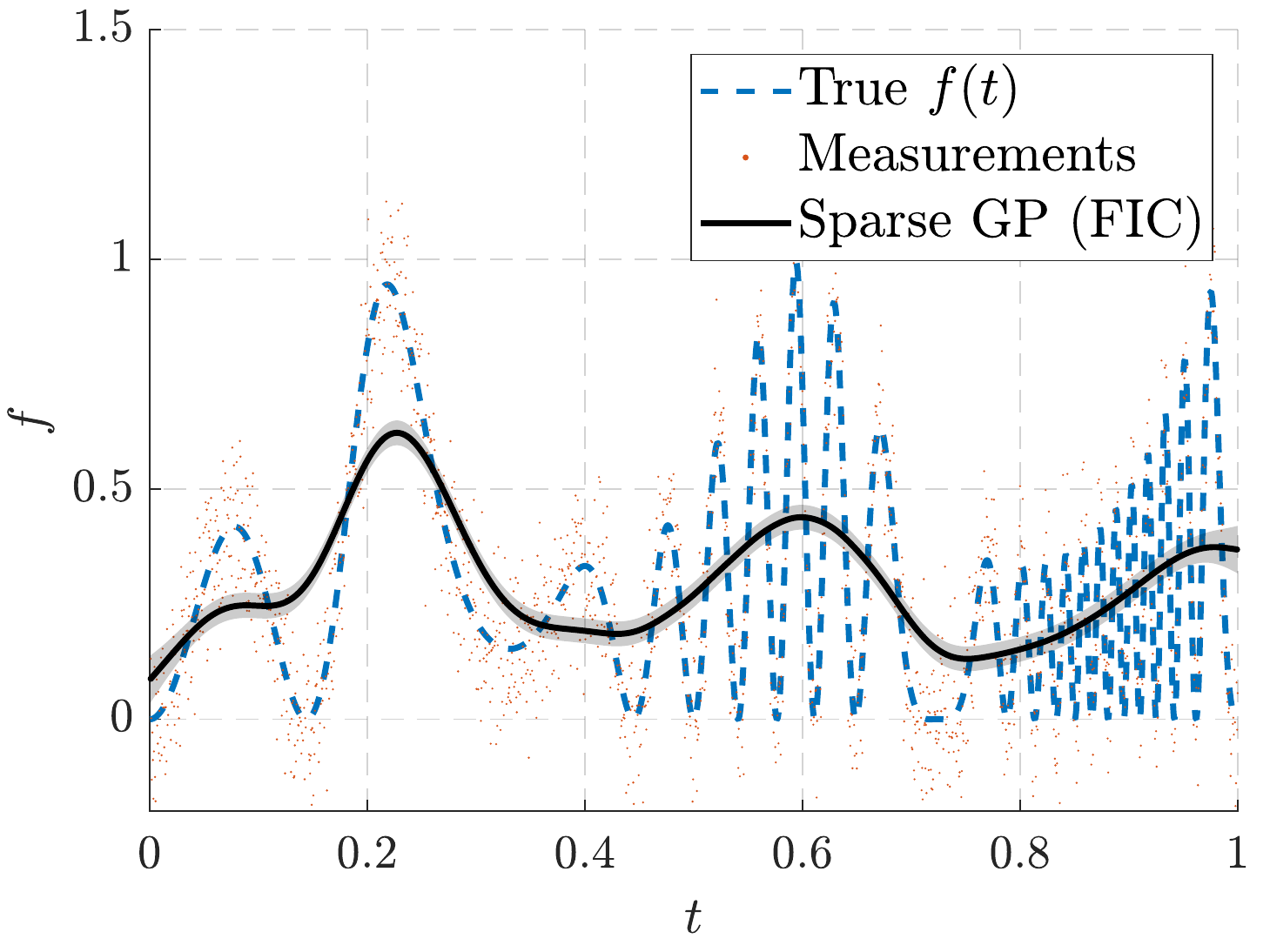}
		\includegraphics[width=.32\linewidth]{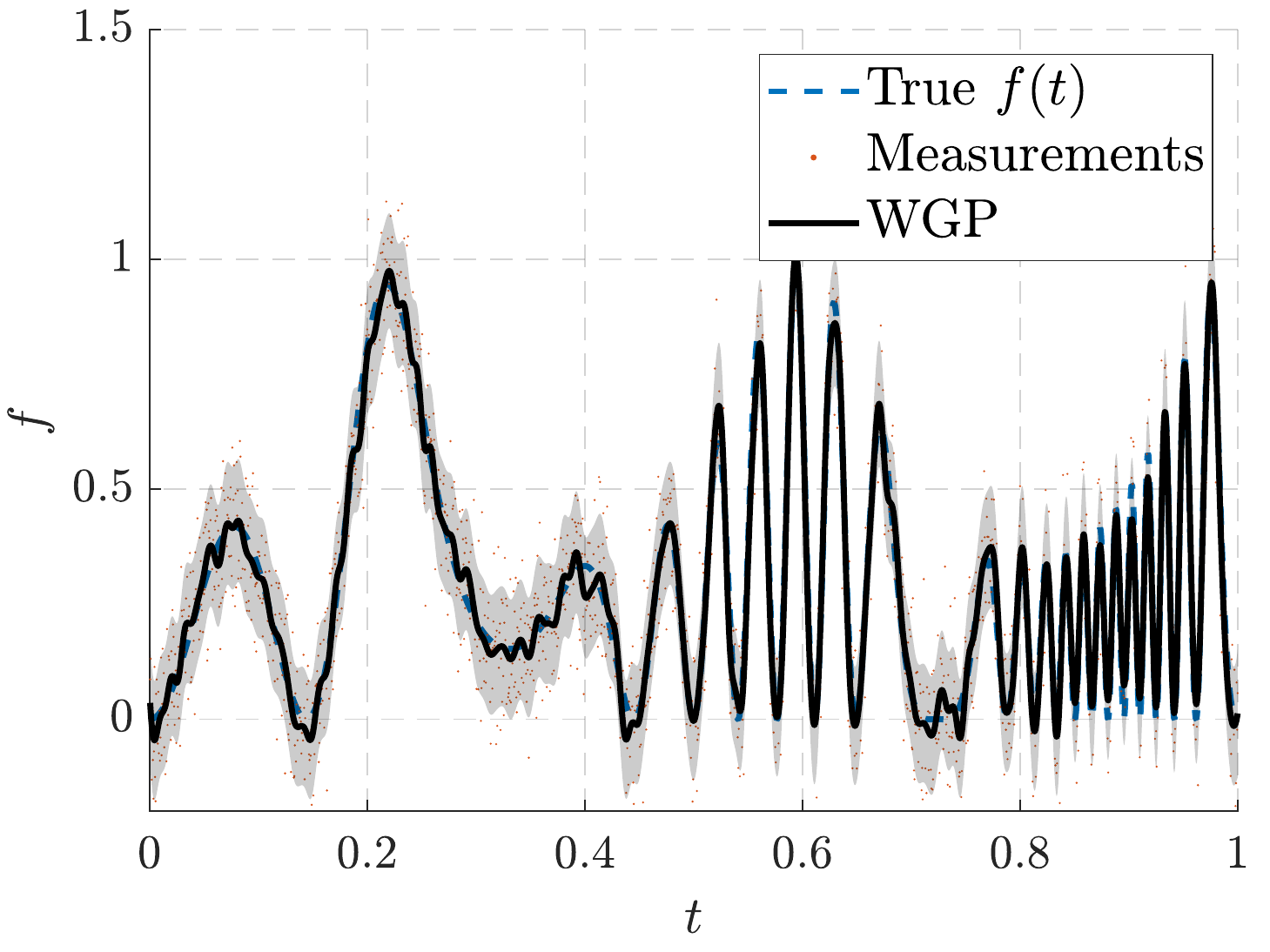}\\
		\includegraphics[width=.32\linewidth]{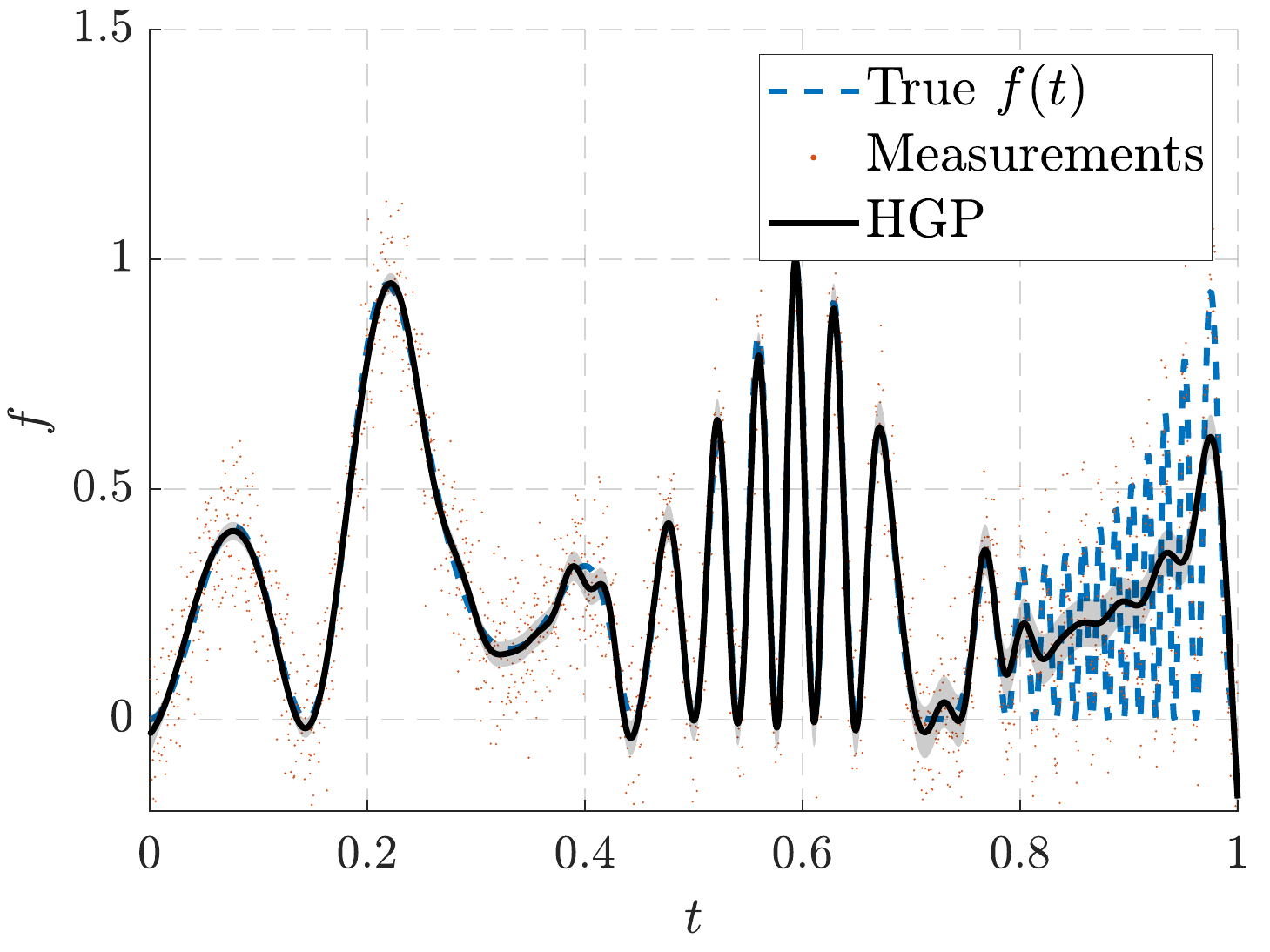}
		\includegraphics[width=.32\linewidth]{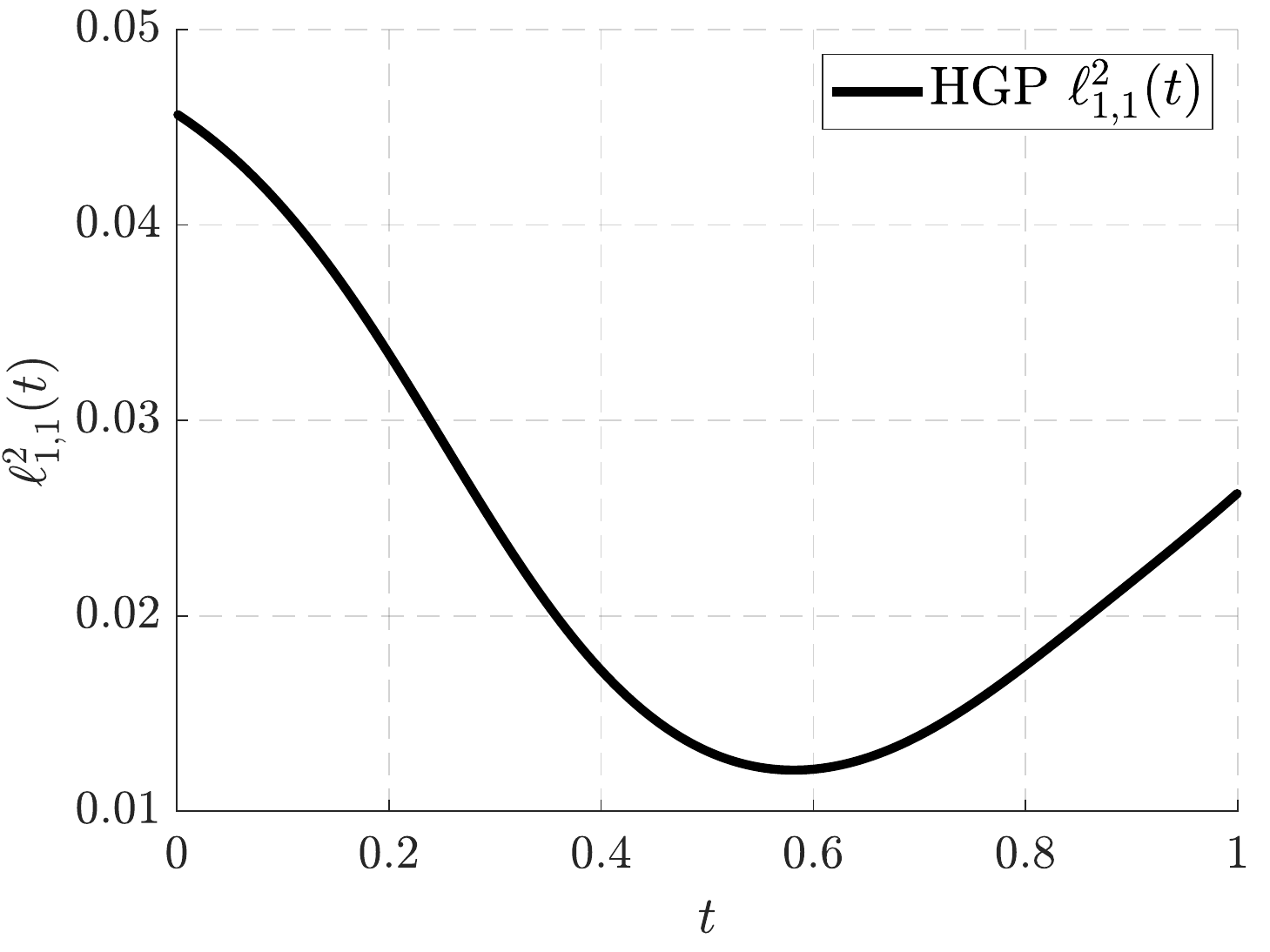}
		\includegraphics[width=.32\linewidth]{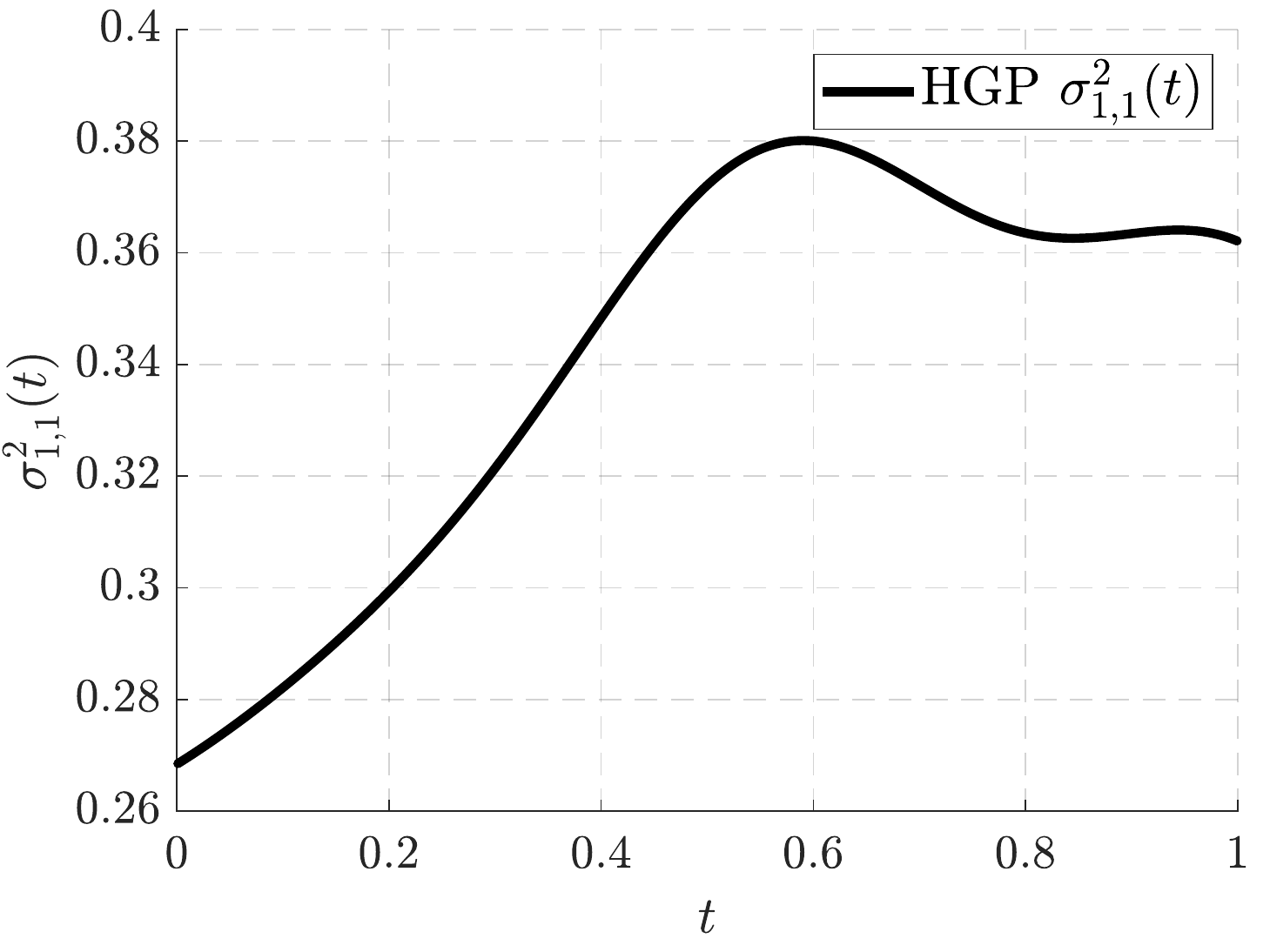}
		\caption{Regression results of GP, Sparse GP, WGP, and HGP on model~\eqref{equ:exp-sine}. The shaded area stands for 95\% confidence interval. }
		\label{fig:sin-exp-GP}
	\end{figure*} 
	
	We now summarize the numerical results in terms of the RMSEs, NLPD, and computational time from Table~\ref{tbl:rect-rmse-nlpd-time}. Table~\ref{tbl:rect-rmse-nlpd-time} demonstrates that the DGP methods using MAP, PF, and PF-BS outperform GP and NS-GP on this non-stationary signal regression. Moreover, the RMSEs and NLPDs are improved by using DGP-3 over DGP-2, except for Gaussian filters and smoothers. Among all regression methods, the SS-MAP is the best in terms of RMSE, followed by B-MAP and PF-BS. In terms of NLPD, PF-BS admits the lowest value. However, the NLPD and RMSE results of PF and PF-BS have very large deviations which are improved by using DGP-3 over DGP-2. We found that the Gaussian filters and smoothers (CKFS and EKFS) are the fastest, followed by GP and NS-GP. We also notice that for all methods, DGP-3 is more time-consuming than DGP-2. Even though we implemented PF-BS in CPU-based parallelization the time consumption is still significantly higher than of the others because of the large number of particles and backward simulations. 
	
	\subsection{Regression on Composite Sinusoidal Signal}
	\label{sec:exp-sine}
	In this section, we conduct another experiment on a non-stationary composite sinusoidal signal formulated by
	\begin{equation}
		\begin{split}
			f(t) &= \frac{\sin^2\left( 7\,\pi\,\cos\left( 2\,\pi\,t^2\right)\,t \right)}{\cos\left(5\,\pi\,t \right)+2 }, \quad t\in [0, 1], \\
			y(t) &= f(t) + r(t), 
			\label{equ:exp-sine}
		\end{split}
	\end{equation}
	where $f$ is the true function, and $r(t)\sim \mathcal{N}(0, 0.01)$. This type of signal has been used by, for example,~\citet{Tim2020DGP, Vannucci1999} and~\citet{Karla2020}. A demonstration is plotted in Figure~\ref{fig:exp-sine-y}. In contrast to the discontinuous rectangle wave in Equation~\eqref{equ:exp-y}, this composite sinusoidal is smooth. Thus it is appropriate to postulate a smooth \matern prior. This non-stationary signal is challenging in the sense that the frequencies and magnitudes are changing rapidly over time. 
	
	\begin{table}[h!]
		\centering
		\begin{tabular}{@{}llll@{}}
			\toprule
			Methods& \begin{tabular}[c]{@{}l@{}}RMSE\\ ($ 0^{-2}$)\end{tabular} & \begin{tabular}[c]{@{}l@{}}NLPD\\ ($ 0^{3}$)\end{tabular} & Time (s) \\ \midrule
			GP (MLE)& $3.08\pm 0.1$ & $-1.69\pm0.03$ & $1.4\pm0.5$ \\
			Sparse GP (FIC)& $17.52\pm 1.9$ & $1.24\pm0.71$ & $0.3\pm0.1$ \\
			WGP & $3.21\pm2.6$ & N/A & $2.6\pm1.0$ \\
			HGP & $9.35\pm1.7$ & $-0.93\pm0.36$ & $765.5\pm109$ \\
			CKFS (DGP-2) & $\mathbf{2.52}\pm 0.2$ & $\mathbf{-1.71}\pm0.03$ &$0.3\pm0.1$ \\
			CKFS (DGP-3) & $2.54\pm 0.2$ & $-1.70\pm 0.03$ &$0.9\pm 0.2$ \\
			EKFS (DGP-2) & $2.61\pm0.2$ & $-1.70\pm 0.03$ & $\mathbf{0.1}\pm 0.03$ \\
			EKFS (DGP-3) & $2.73\pm 0.2$ & $-1.70\pm0.03$ &$0.2\pm0.02$ \\ \bottomrule
		\end{tabular}
		\caption{Averaged RMSEs, NLPD, and computational time (in seconds) on model~\eqref{equ:exp-sine} over different regression models and solvers. }
		\label{tbl:rect-rmse-nlpd-time-sin}
	\end{table}
	
	The settings of this experiment are the same with the rectangle wave regression in Section~\ref{sec:toy-data}, except that we generate the signal with $2,000$ samples. With this number of measurements, the NS-GP and MAP-based solvers fail because they do not converge in a reasonable amount of time. Also, we select three other GP models from the literature for comparison, that are, the fully independent conditional \citep[FIC,][]{Joaquin2005FIC} sparse GP with 500 pseudo-inputs, the warped GP~\citep[WGP,][]{WGP_Snelson2003}, and a non-stationary GP (HGP) by~\citet{heinonen2016}. 
	
	\begin{figure*}[t!]
		\centering
		\includegraphics[width=.4\linewidth]{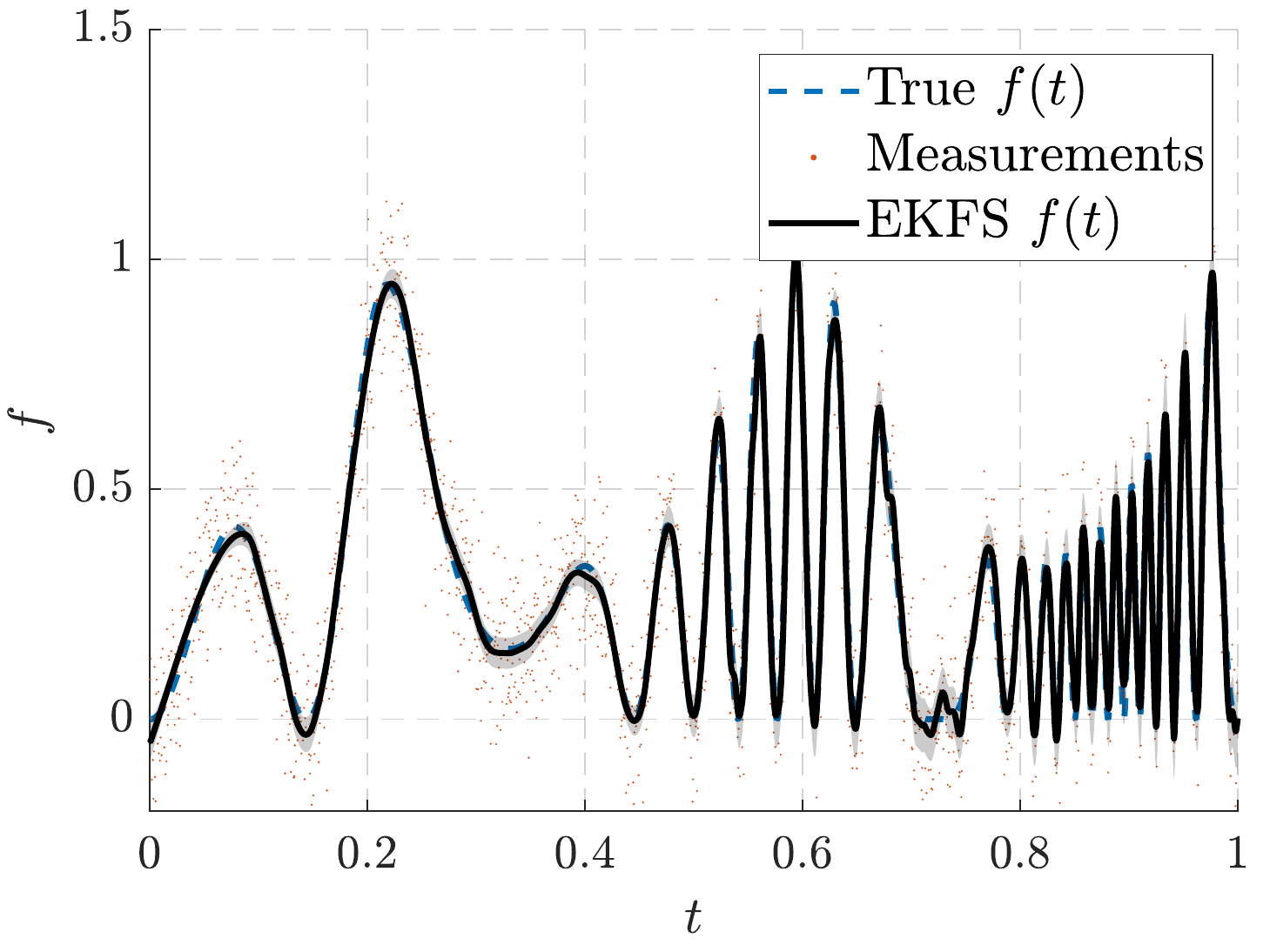}
		\includegraphics[width=.4\linewidth]{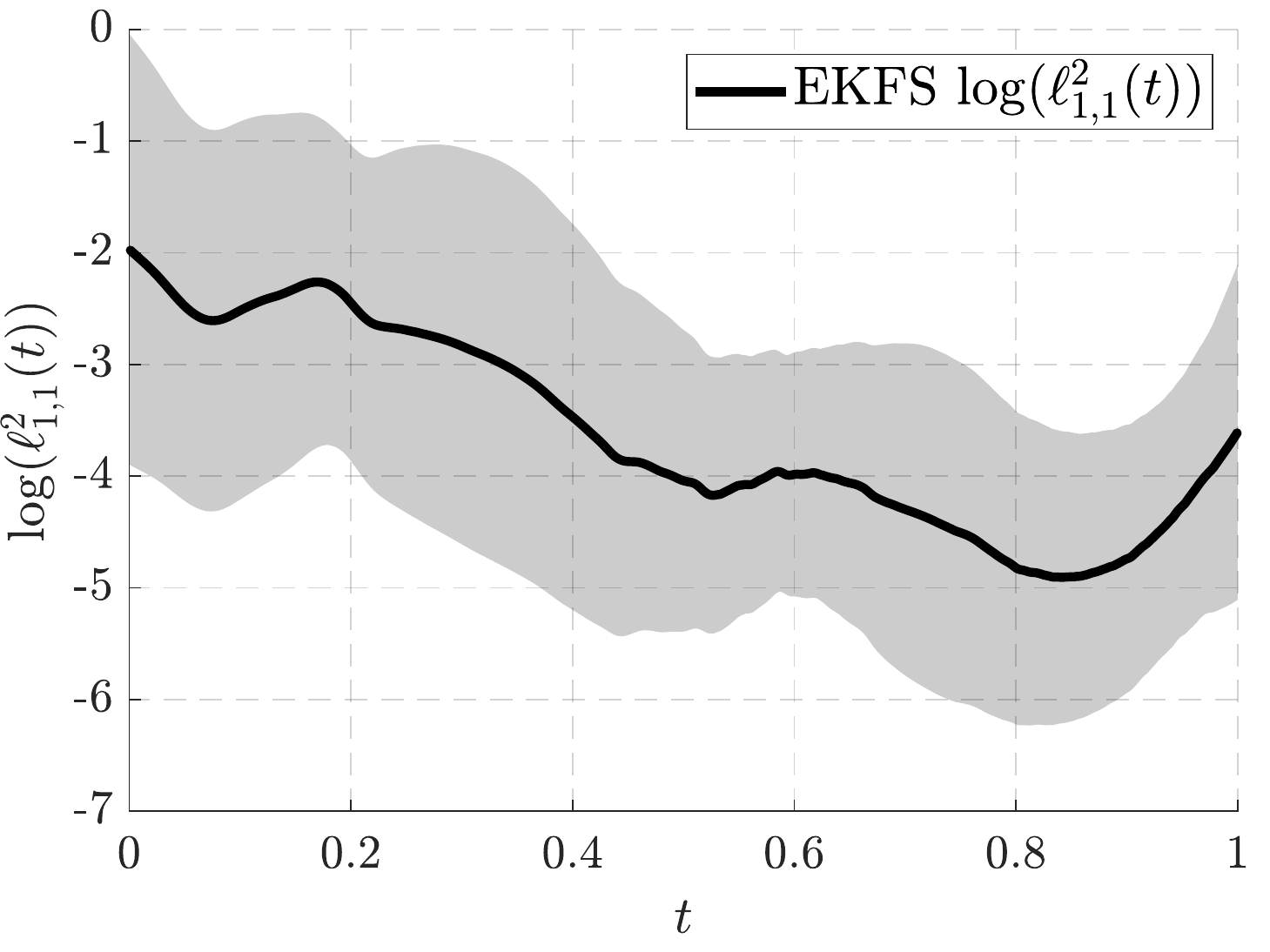}\\
		\includegraphics[width=.32\linewidth]{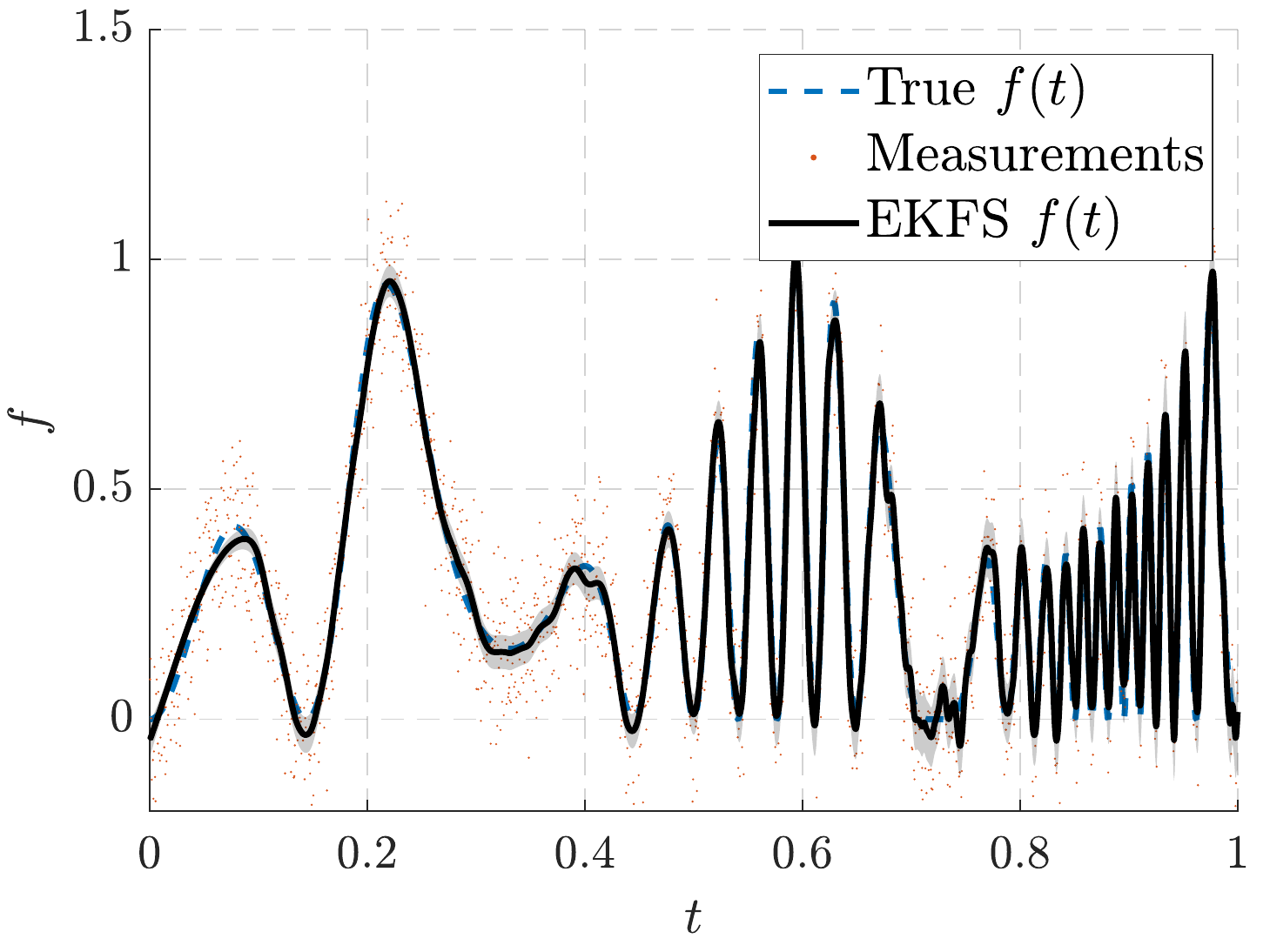}
		\includegraphics[width=.32\linewidth]{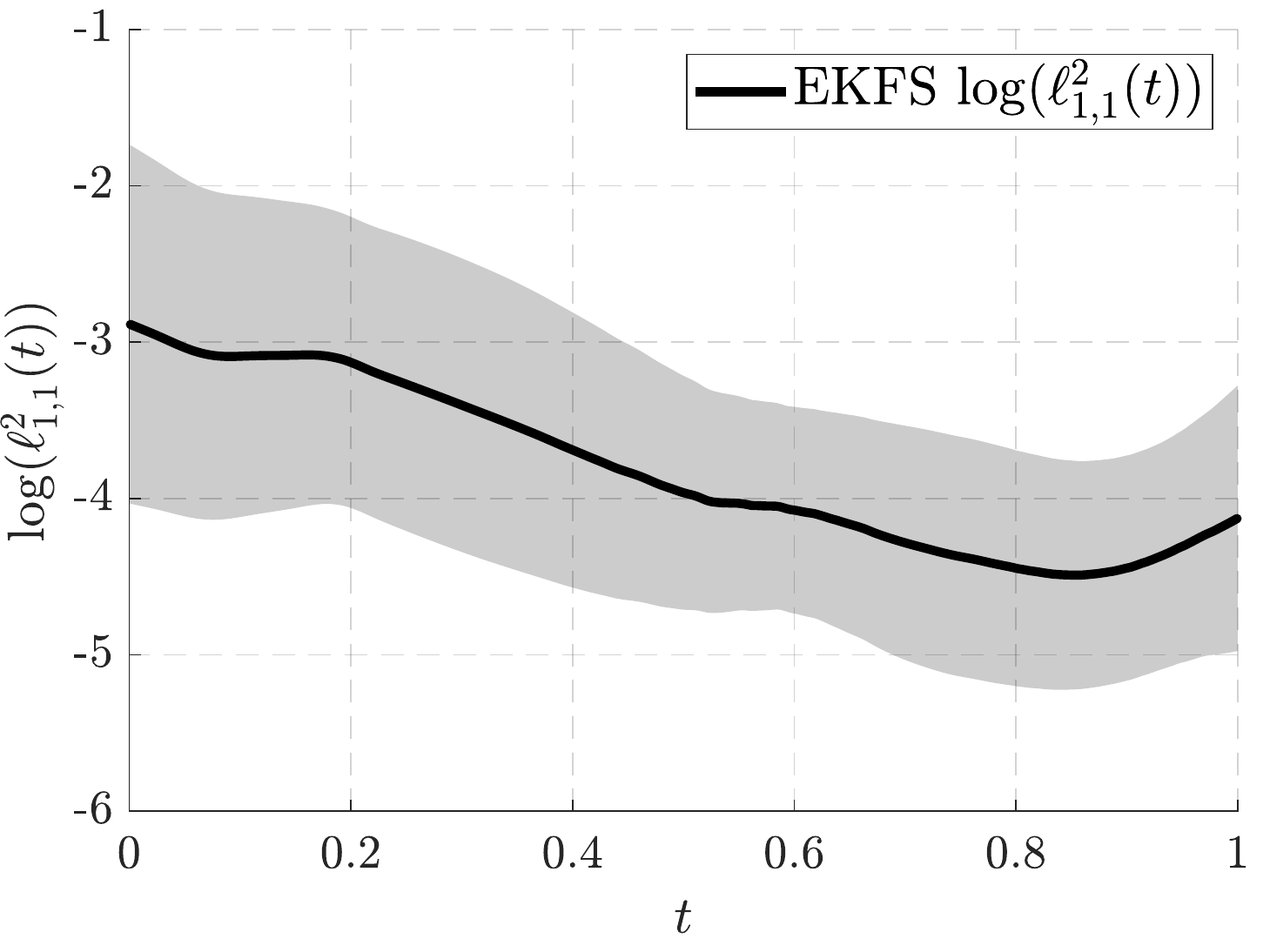}
		\includegraphics[width=.32\linewidth]{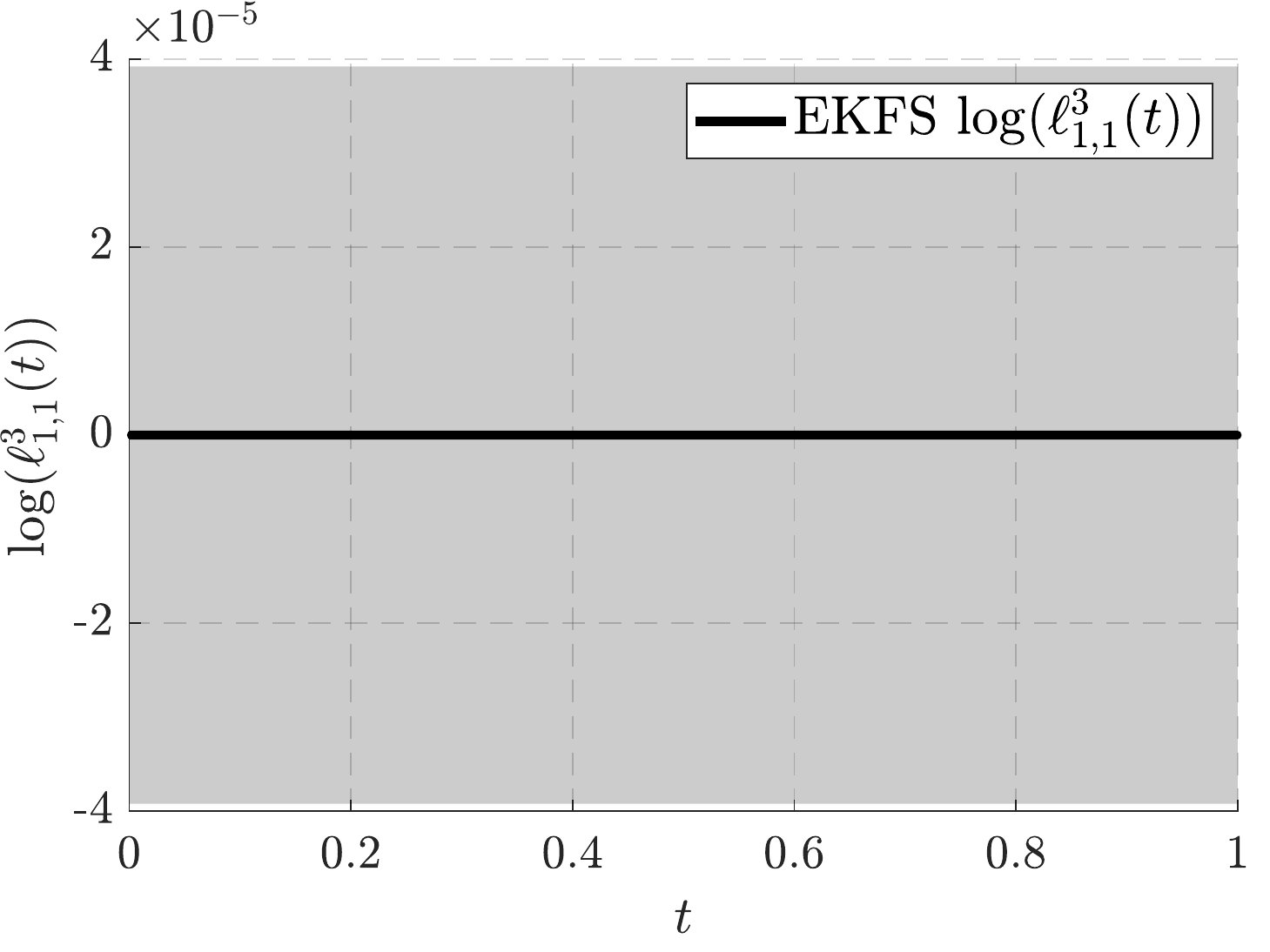}
		\caption{EKFS regression results on model~\eqref{equ:exp-sine} using DGP-2 (first row) and DGP-3 (second row). The shaded area stands for 95\% confidence interval. }
		\label{fig:sin-exp-EKFS}
	\end{figure*}
	
	\begin{figure*}[t!]
		\centering
		\includegraphics[width=.4\linewidth]{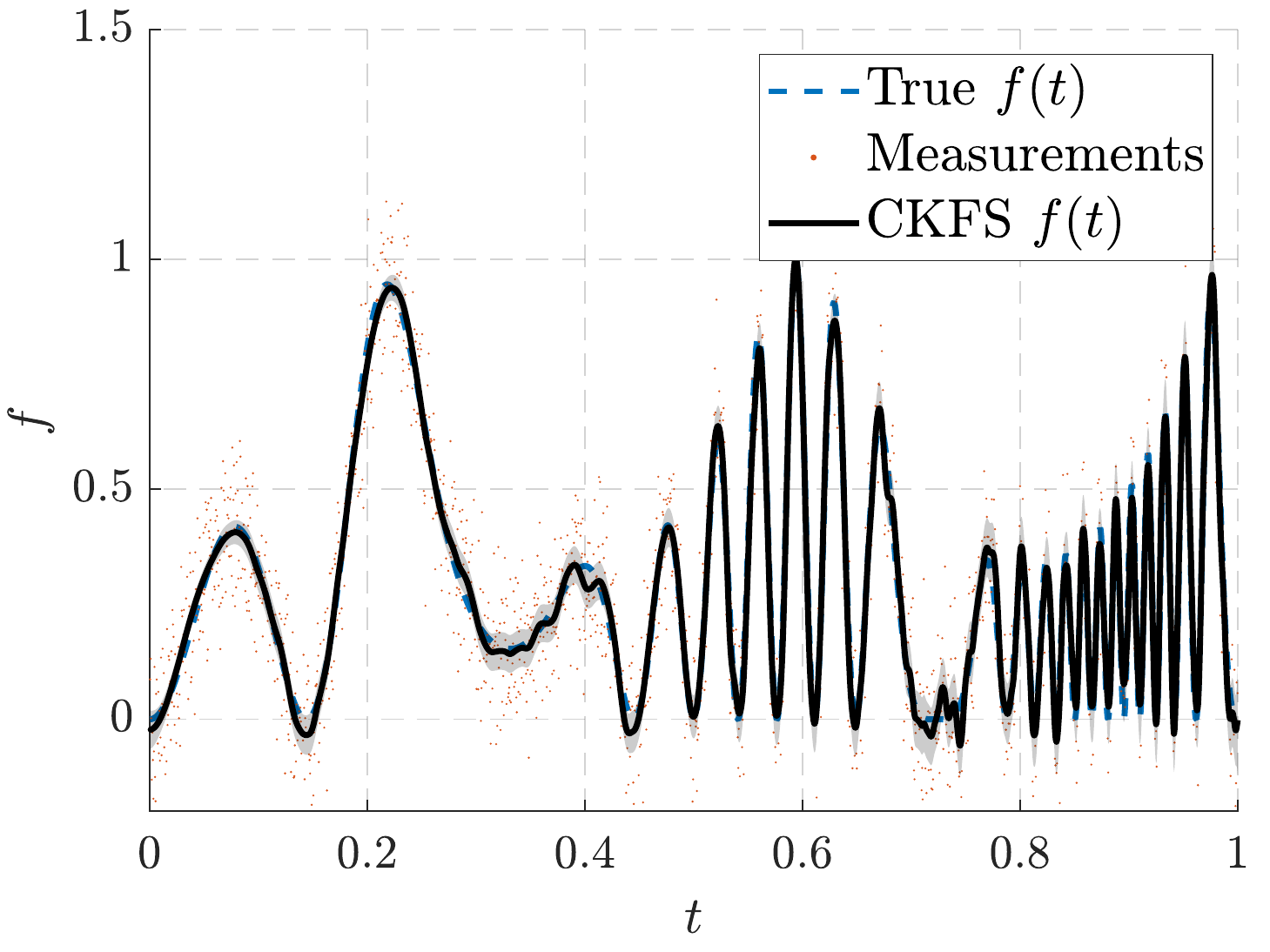}
		\includegraphics[width=.4\linewidth]{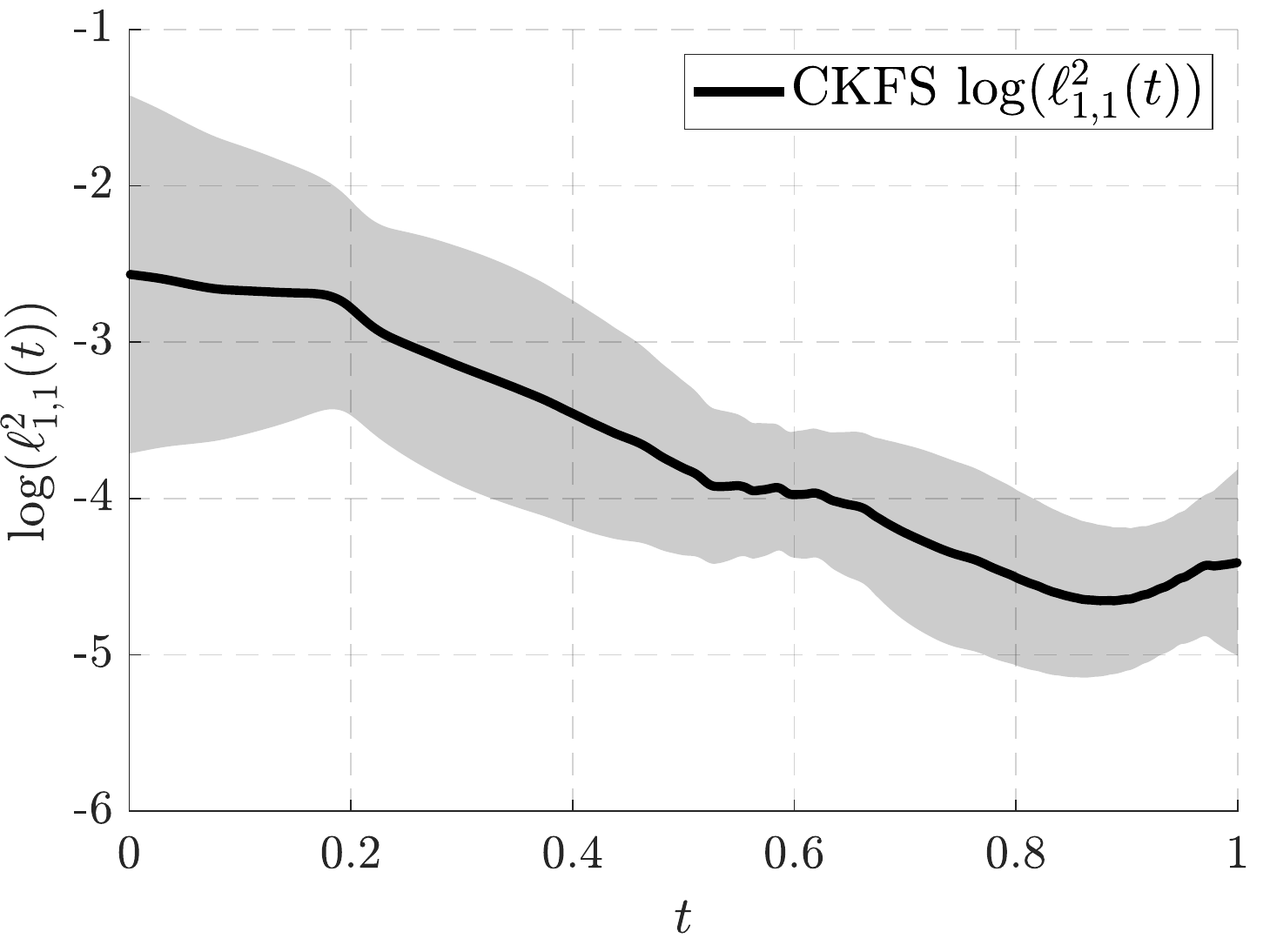}\\
		\includegraphics[width=.32\linewidth]{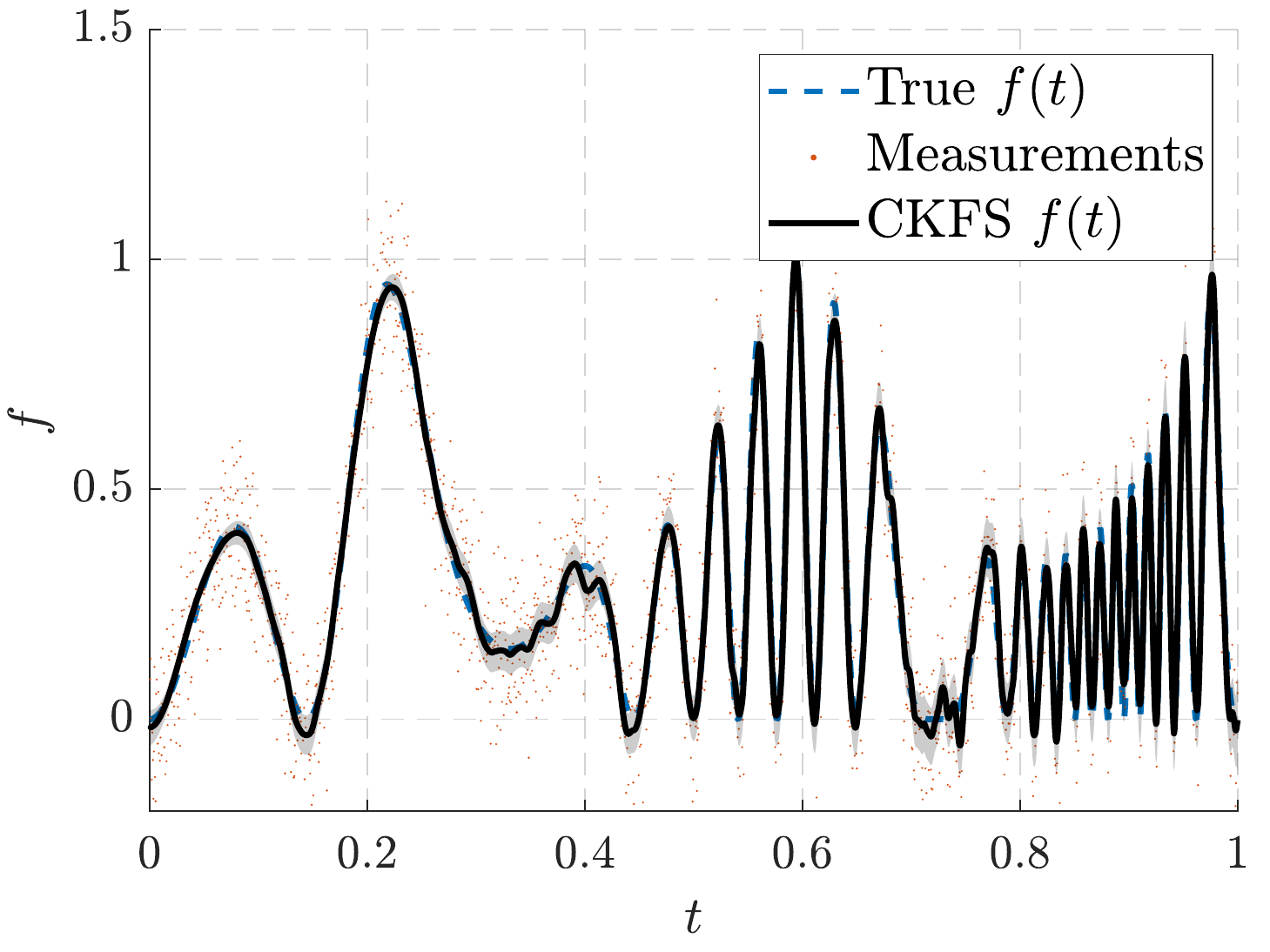}
		\includegraphics[width=.32\linewidth]{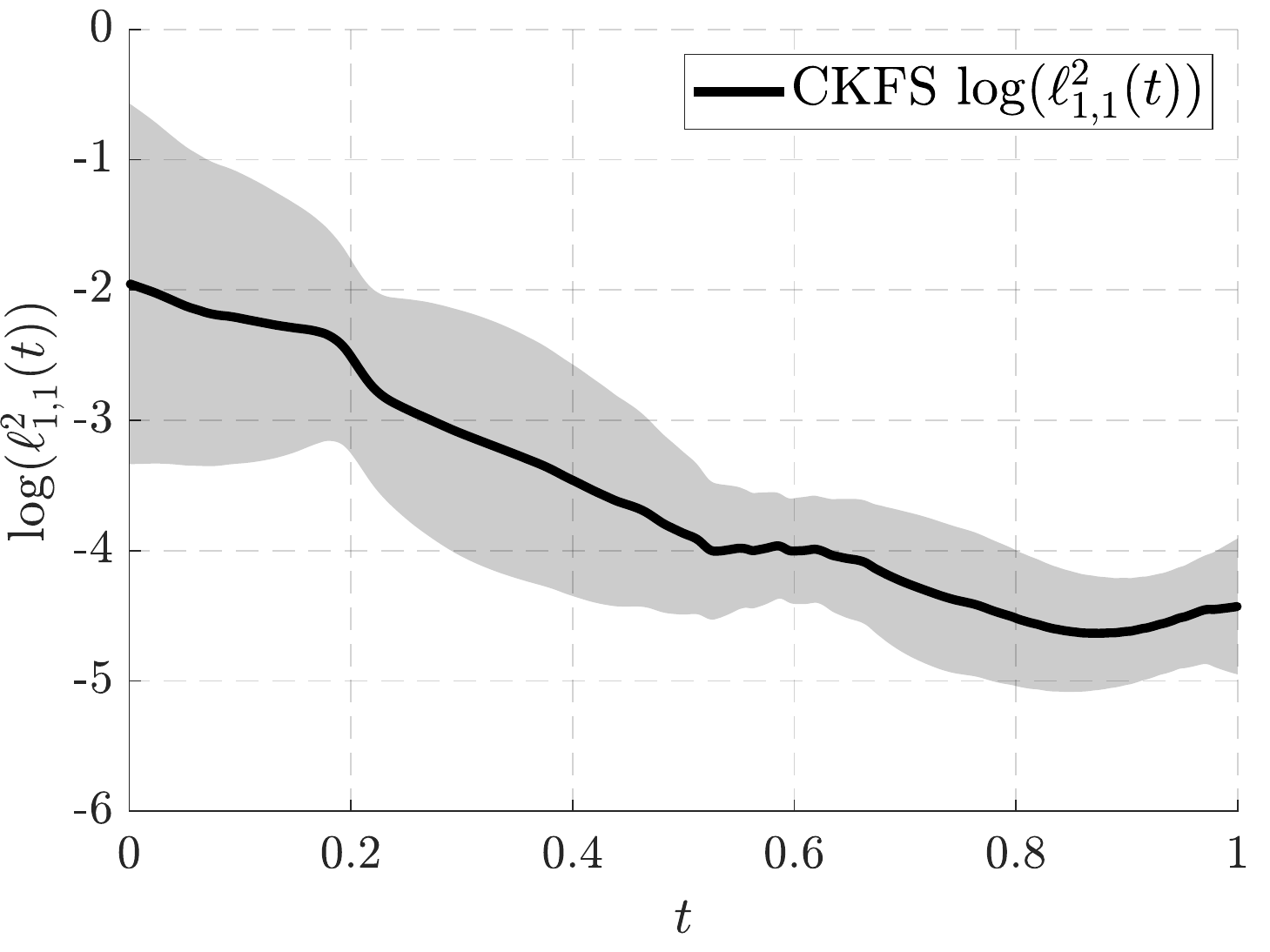}
		\includegraphics[width=.32\linewidth]{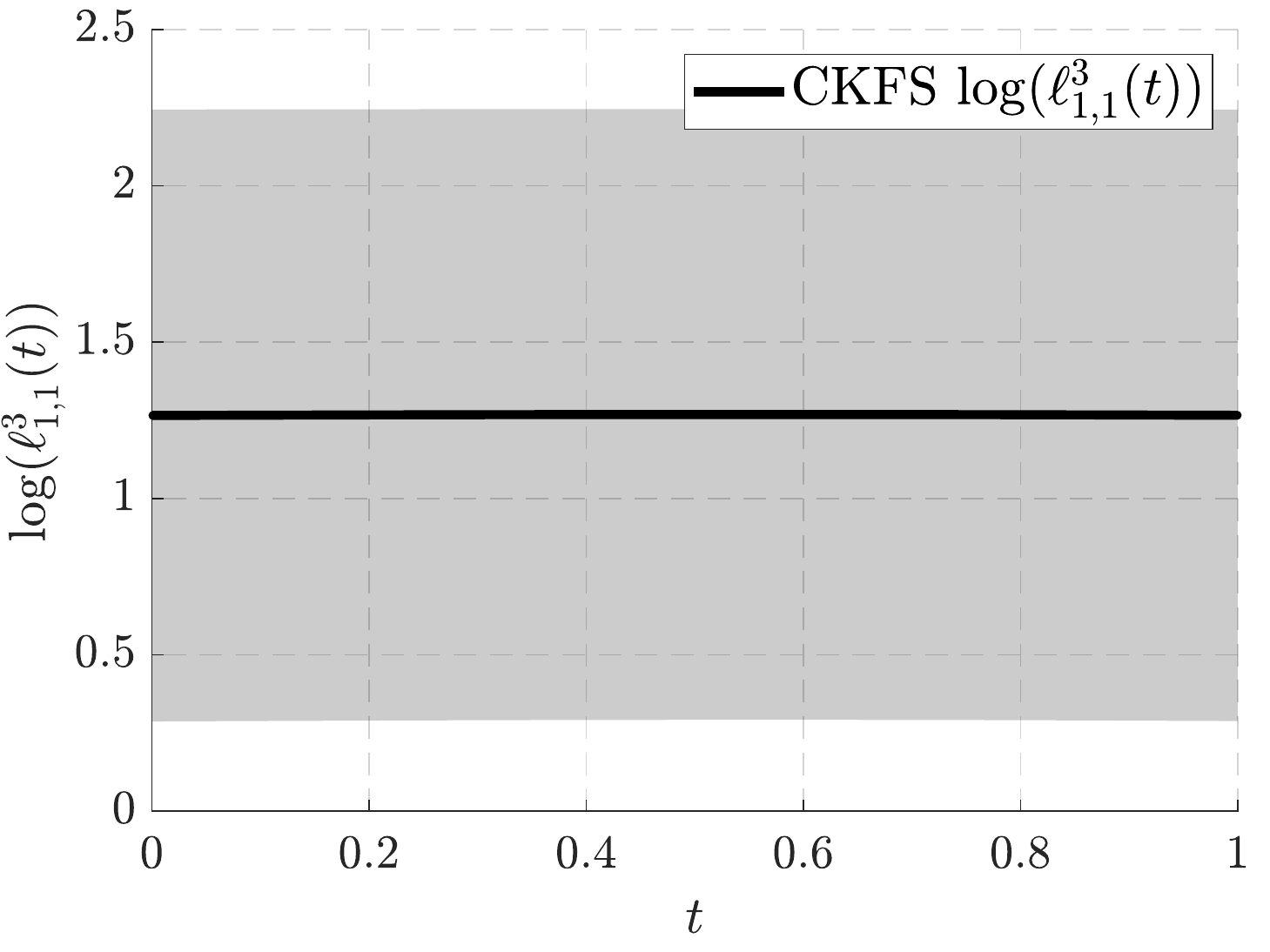}
		\caption{CKFS regression results on model~\eqref{equ:exp-sine} using DGP-2 (first row) and DGP-3 (second row). The shaded area stands for 95\% confidence interval. }
		\label{fig:sin-exp-GFS}
	\end{figure*}
	
	The results for GP, Sparse GP, WGP, and HGP are shown in Figure~\ref{fig:sin-exp-GP}. We find that the estimate of GP is overfitted to the measurements, and it is not smooth. On the contrary, the estimate of sparse GP is underfitted. The result of WGP is similar to GP, but the estimated variance of WGP is large. The HGP works well except for the part after $t>0.8$~s. The learnt $\ell^2_{1,1}$ and $\sigma^2_{1,2}$ from HGP are smooth. 
	
	\begin{figure*}[t!]
		\centering
		\includegraphics[width=.32\linewidth]{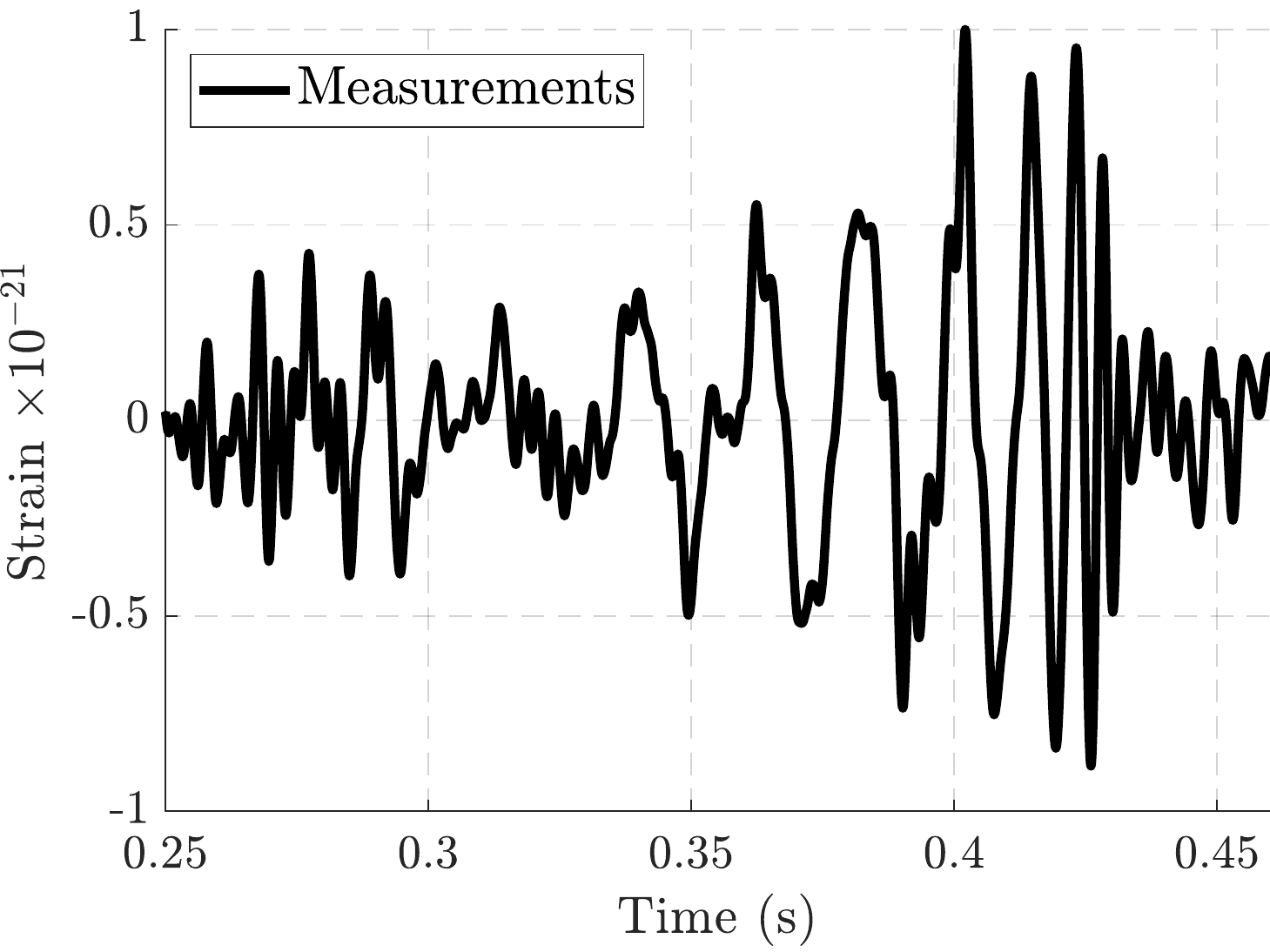}
		\includegraphics[width=.32\linewidth]{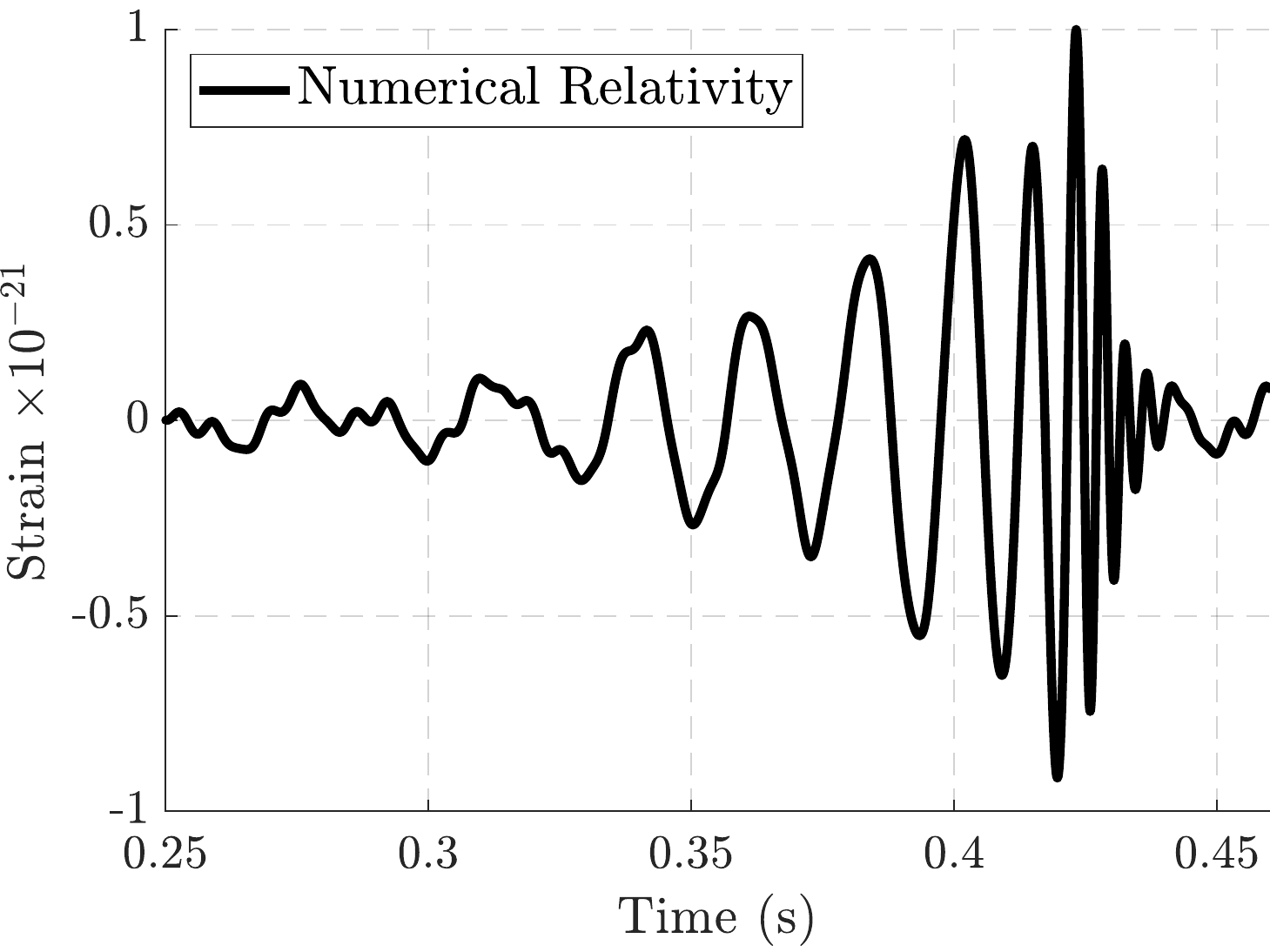}
		\includegraphics[width=.32\linewidth]{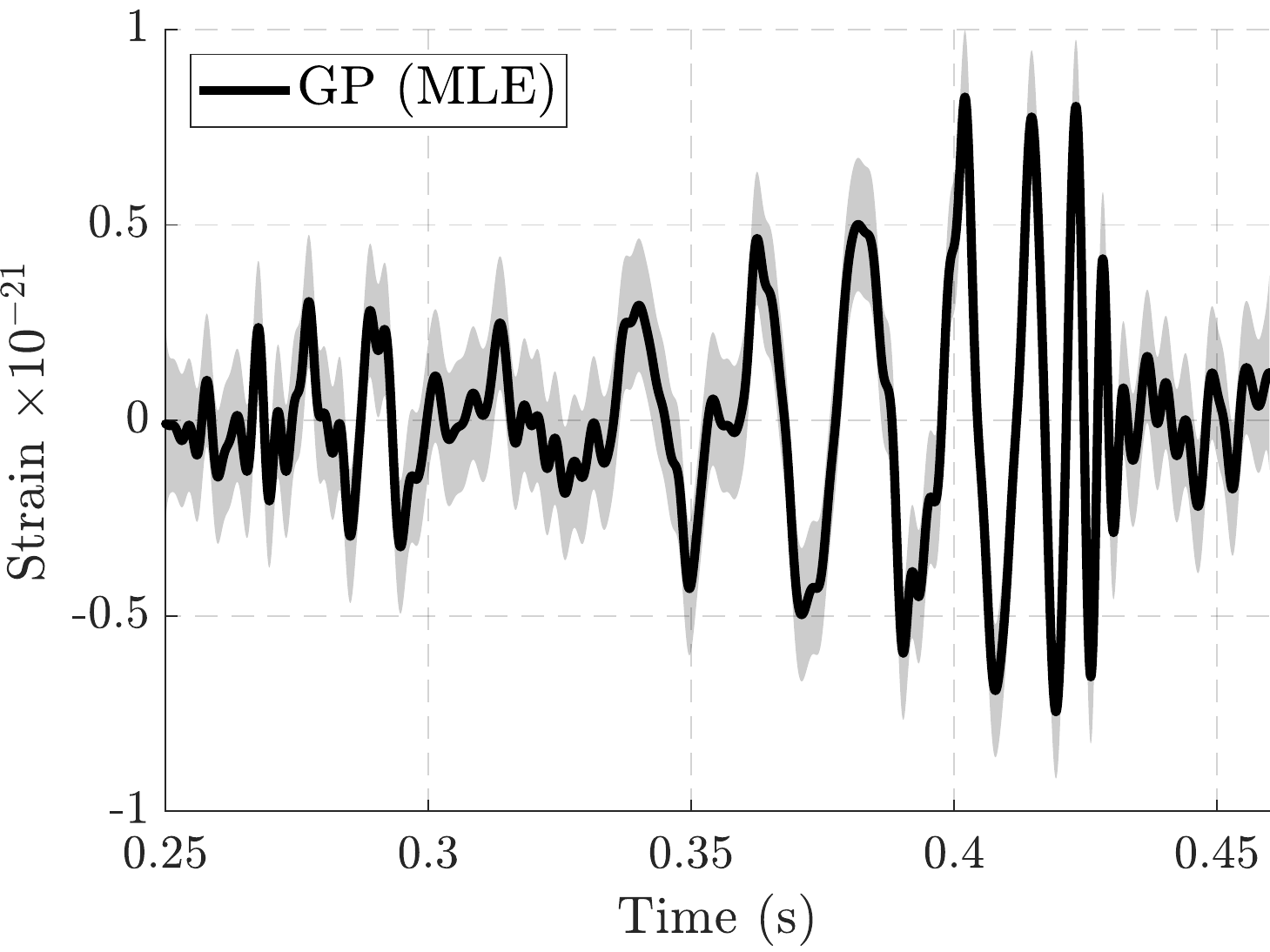}\\
		\includegraphics[width=.4\linewidth]{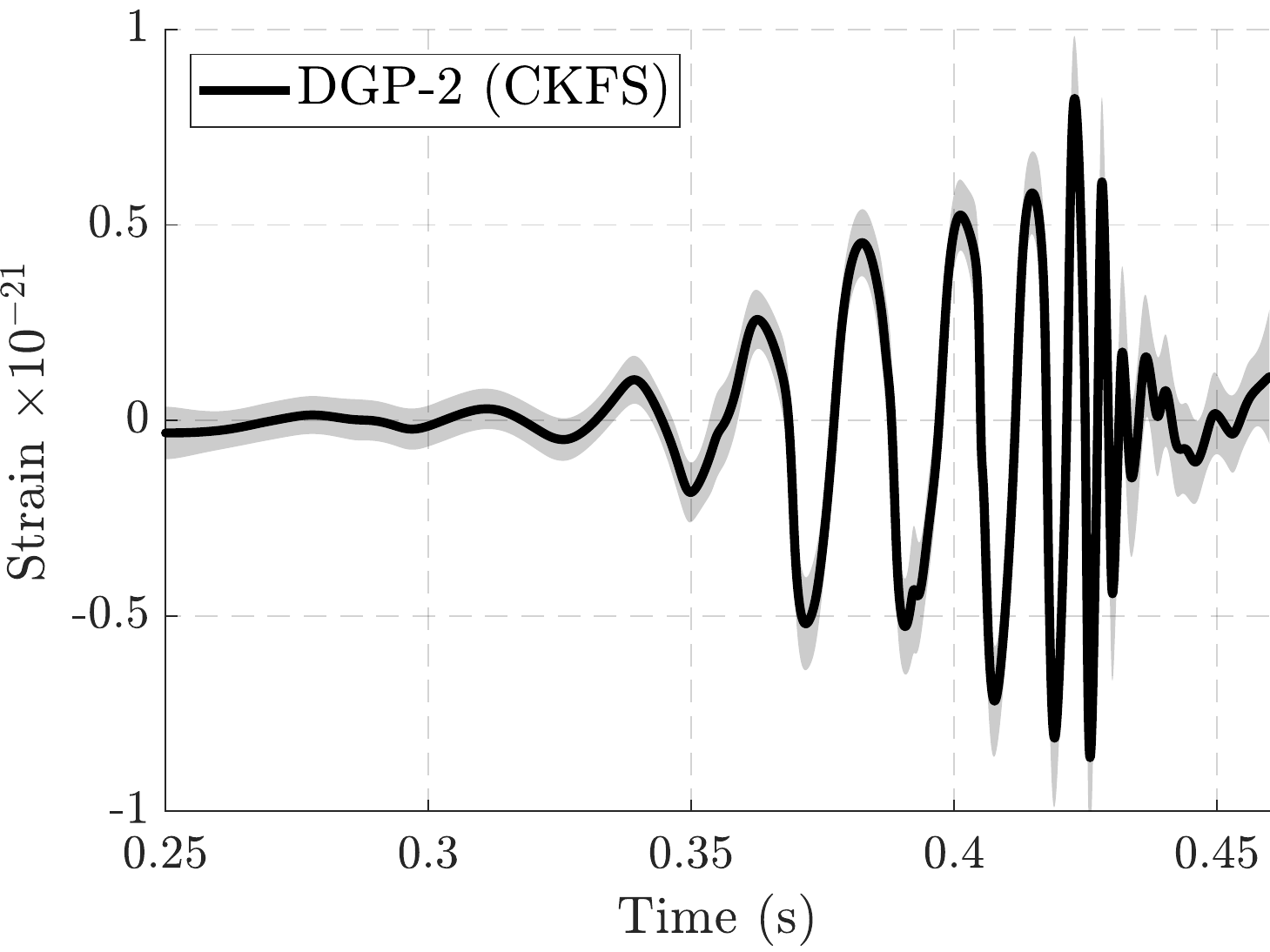}
		\includegraphics[width=.4\linewidth]{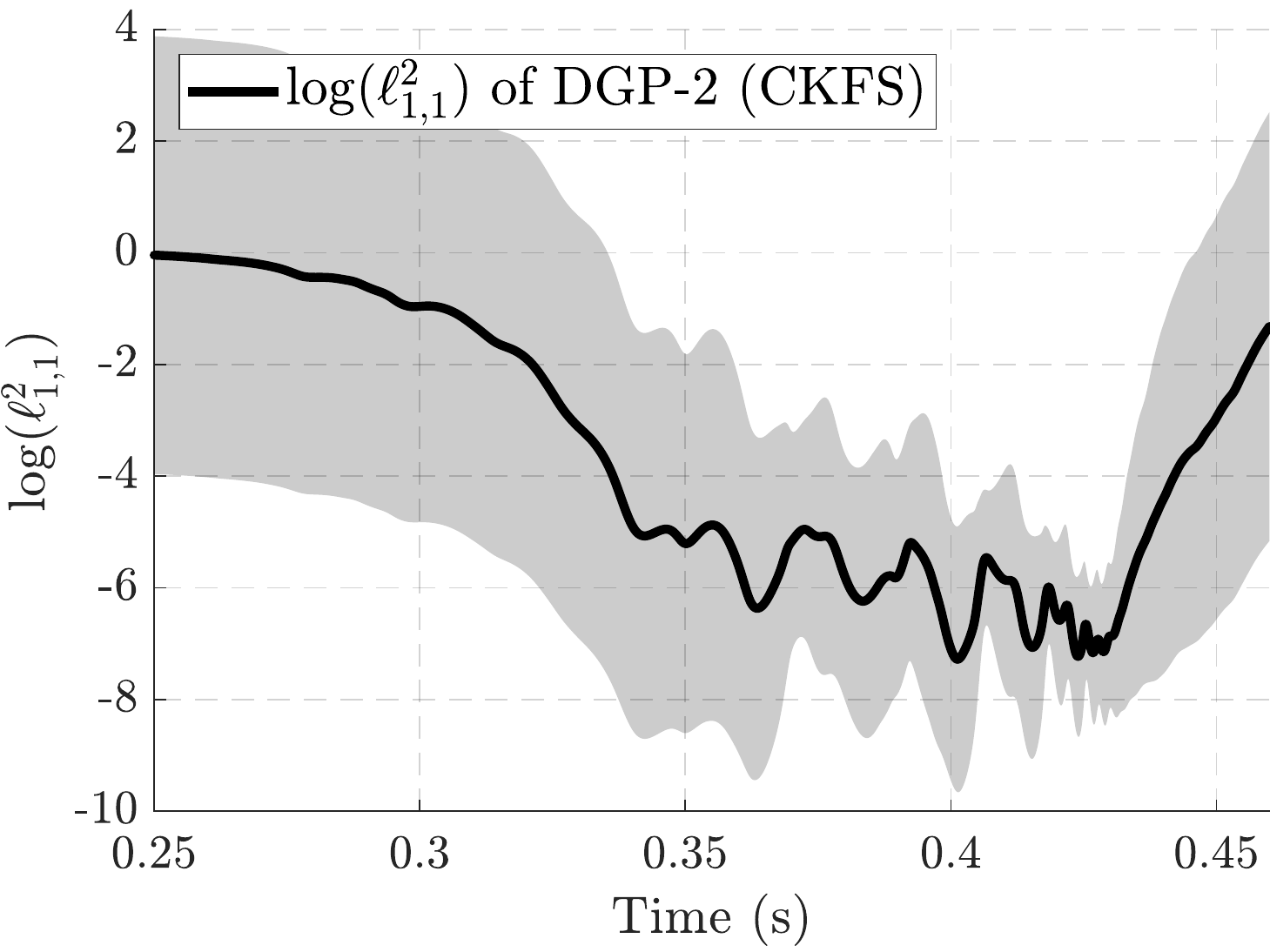}\\
		\includegraphics[width=.4\linewidth]{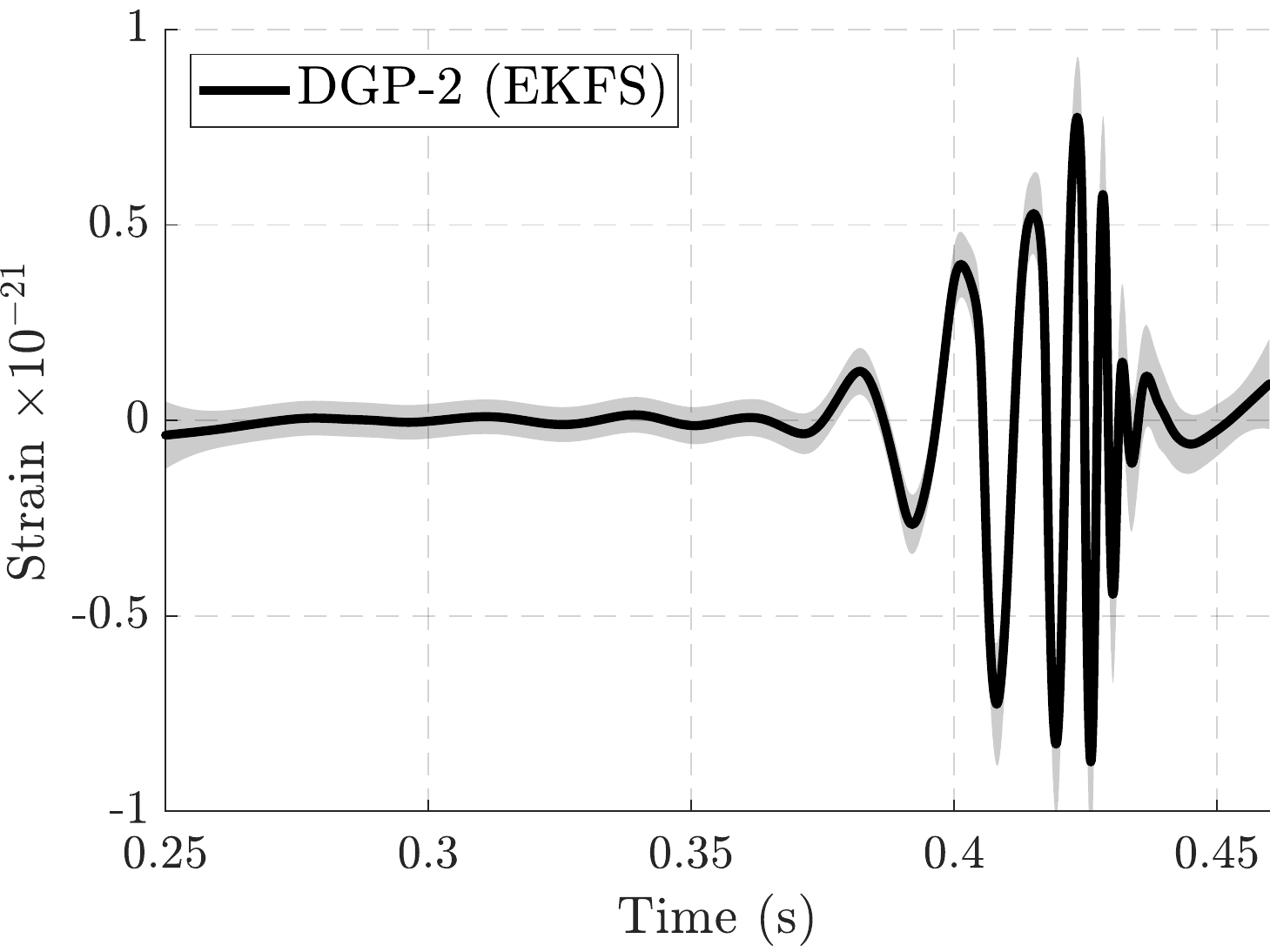}
		\includegraphics[width=.4\linewidth]{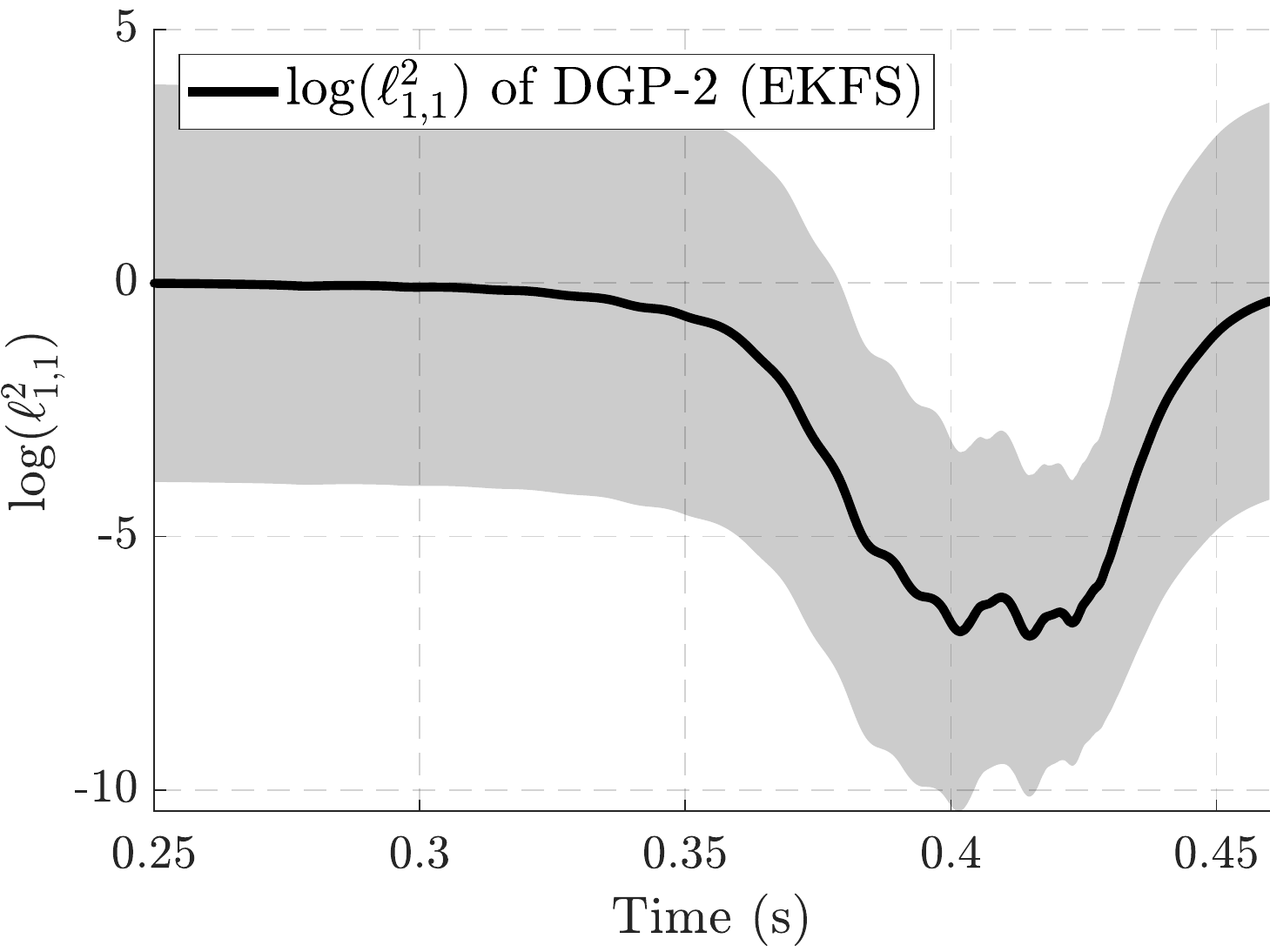}\\
		\caption{LIGO gravitational wave detection (event GW150914, Hanford, Washington) using (Mat\'{e}rn, $\alpha=1$) GP and DGP-2. The shaded area stands for 95\% confidence interval. }
		\label{fig:ligo}
	\end{figure*}
	
	Figures~\ref{fig:sin-exp-EKFS} and~\ref{fig:sin-exp-GFS} plot the results of EKFS and CKFS, respectively. From visual inspection, the Gaussian filters and smoothers based DGPs outperform GP, sparse GP, WGP, and HGP. We also find that the estimates from EKFS and CKFS are quite similar, whereas EKFS gives smoother estimate of $f$ compared to CKFS. The learnt $\ell^2_{1,1}$ and $\sigma^2_{1,2}$ also adapt to the frequency changes of the signal. It is worth noticing that the estimated $\ell^3_{1,1}$ in the third layer of DGP-3 is almost flat for both CKFS and EKFS. 
	
	The RMSE, NLPD, and computational time are listed in Table~\ref{tbl:rect-rmse-nlpd-time-sin}. This table verifies that the DGPs using Gaussian filters and smoothers (i.e., CKFS and EKFS) outperform other methods in terms of RMSE, NLPD, and computational time. Also, CKFS gives slightly better RMSE and NLPD than EKFS. For this signal, using DGP-3 yields no better RMSE and NLPD compared to DGP-2.
	
	\subsection{Real Data Application on LIGO Gravitational Wave Detection}
	\label{sec:ligo}
	The theoretical existence of gravitational waves was predicted by Albert Einstein in 1916 from a linearized field equation of general relativity~\citep{AMSGW2017, EinsteinGW1937}. In 2015, the laser interferometer gravitational-wave observatory (LIGO) team made the first observation of gravitational waves from a collision of two black holes, known as the event GW150914 \citep{LIGO2016}. The detection was originally done by using a matched-filter approach. It is of our interests to test if the GP and DGP approaches can detect the gravitational waves from the LIGO measurements. We now formulate the detection as a regression task. 
	
	We use the observation data provided by LIGO scientific collaboration and the Virgo collaboration\footnote{The data is available at \url{https://doi.org/10.7935/K5MW2F23} or \url{https://doi.org/10.7935/82H3-HH23}}. As shown in the first picture of Figure~\ref{fig:ligo}, the data contains 3,441 measurements sampled in frequency of 16,384 Hz. We use time interval $1 0^{-5}$~s to interpolate the data, which results in $10,499$ time steps. The reference gravitational wave calculated numerically from the general relativity theory is shown in Figure~\ref{fig:ligo}, and we use it as the ground truth for comparison. 
	
	We use the previously formulated regression models GP and DGP-2, as shown in Figure~\ref{fig:reg-models}. Unfortunately, the NS-GP and MAP-based solvers are not applicable due to a large number of observations and interpolation steps. Hence, we choose the Gaussian filters and smoothers (i.e., CKFS and EKFS) for DGP regression. 
	
	The detection results are shown in the second and third rows of Figure~\ref{fig:ligo}. We find that the DGP-2 model gives a better fit to the gravitational wave compared to GP. The DGP-2 estimate is almost identical to the numerical relativity result. GP however, fails because the estimate overfits to the measurements. Also, the outcomes of DGP-2 are explainable by reviewing the learnt parameter $\ell^2_{1,1}$. We see that the length scale $\ell^2_{1,1}$ adapts to the frequency changes of the gravitational wave, which is an expected feature by using the DGP model. The results of CKFS and EKFS are similar, while EKFS gives smoother results. 
	
	Moreover, the Gaussian filters and smoothers on DGP-2 have significantly smaller time consumption compared to GP. In one single run of the program, CKFS and EKFS take $1.5$~s and $0.4$~s, respectively, while GP takes $202.2$~s (including hyperparameter optimization). 
	
	\subsection{Summary of Experimental Results}
	
	In this section, we summarize the results of the state-space methods presented in the sections above. In the rectangular signal regression experiment, the state-space MAP and particle smoothing methods are better than Gaussian smoothers (e.g., EKFS and CKFS) in terms of RMSE and NLPD. Based on the results of the composite sinusoidal signal regression experiment, Gaussian smoothers are particularly efficient in computation. 
	However, Gaussian smoothers may not be suitable solvers for SS-DGP models that have both lengthscale and magnitude parameters included in the DGP hierarchy. This is proved in Section~\ref{sec:restriction-GFS}, and it is also numerically shown in Figure~\ref{fig:exp-GFS-vanishing}.
	
	\section{Conclusion}
	\label{sec:conclusion}
	In this paper, we have proposed a state-space approach to deep Gaussian process (DGP) regression. The DGP is formulated as a cascaded collection of conditional Gaussian processes (GPs). By using the state-space representation, we cast the DGP into a non-linear hierarchical system of linear stochastic differential equations (SDEs). Meanwhile, we propose the maximum a posteriori and Bayesian filtering and smoothing solutions to the DGP regression task. The experiment shows significant benefits when applying the DGP methods to simulated non-stationary regression problems as well as to a real data application in gravitational wave detection. 
	
	The proposed state-space DGPs (SS-DGPs) have the following major strengths. The DGP priors are capable of modeling larger classes of functions compared to the conventional and non-stationary GPs. In the construction of state-space DGP, one does not need to choose/design valid covariance functions manually like in~\citet{paciorek2006spatial} or \citet{Salimbeni2017ns}. In DGP regression in state-space form we do not need to evaluate the (full) covariance function either. Moreover, state-space methods are particularly efficient for temporal data as they have linear computational complexity with respect to time. 
	
	In addition, we have identified a wide class of SS-DGPs that are not suitable for Gaussian smothers to solve. More specifically, these SS-DGP models are the ones that have both their lengthscale and magnitude parameters modeled as GP nodes under the assumptions in Section~\ref{sec:restriction-GFS}. When applying Gaussian smoothers on these SS-DGPs, their Kalman gains converge to zero as time goes to infinity, which makes Gaussian smoothers use no information from data to update their posterior distributions. This is one limitation of SS-DGPs. Although one can use the MAP and particle smoothing methods in place of Gaussian smoothers, these methods can be computationally demanding. 
	
	For future investigation, enabling automatic differentiations is of interests. In this paper we have only applied grid search on a large number of trainable hyperparameters which results in a very crude optimization. By using libraries like TensorFlow or JAX we can also obtain Hessians which we can use to quantify the uncertainty in MAP. 
	
	Another useful future extension is to exploit data-scalable inference methods, such as sparse variational methods. For example, \citet{Paul2020VariationalSS} solve state-space GP regression problems (possibly with non-Gaussian likelihoods) by using a conjugate variational inference method while still retaining a linear computational complexity in time. Their work is extended by~\citet{Wilkinson2021} who introduce sparse inducing points to the said variational state-space GP inference, resulting in a computational complexity that is linear in the number of inducing points. Although these works are mainly concerned with standard state-space GPs (i.e., linear state-space models), it would be possible to apply these methods on SS-DGPs as well, for example, by linearizing the state-space models of SS-DGPs.
	
	Generalizing the temporal SS-DGPs to spatio-temporal SS-DGPs (see, the end of Section~\ref{sec:dgp-system-sdes}) would be worth studying as well, by extending the methodologies introduced in~\citet{simoMagazine2013, Emzir2020}. 
	
	\begin{appendices}
		\section{Derivatives of Loss Function~\eqref{equ:MAP-loss-old}}
		\label{append:MAP-old-deriv}
		We define the derivatives in a set 
		\begin{equation}
			\tash{\mathcal{L}^{\mathrm{BMAP}}}{U_{1:N}} = \left\lbrace \tash{\mathcal{L}}{u^i_{j,k\mid 1:N}}\colon i=1,\ldots,L, k=1,2,\ldots, L_i \right\rbrace \nonumber
		\end{equation}
		for all nodes, where each element is a column vector. For the top GP $f\coloneqq u^1_{1,1}$, the derivative is
		\begin{equation}
			\tash{\mathcal{L}^{\mathrm{BMAP}}}{f_{1:N}} = -\cu{R}^{-1}\,(\cu{y}_{1:N}-f_{1:N}) + \cu{C}^{-1}\,f_{1:N}.\nonumber
		\end{equation}
		The derivatives of other GP nodes are given by
		\begin{equation}
			\begin{split}
				\tash{\mathcal{L}^{\mathrm{BMAP}}}{u^i_{j,k\mid 1:N}} &= \tash{-\log p(u^i_{j,k\mid 1:N}\mid U^{i+1}_{k,\cdot\mid 1:N})}{u^i_{j,k\mid 1:N}}\\
				&\quad+ \tash{-\log p(u^{i-1}_{j,k\mid 1:N}\mid U^{i}_{k,\cdot\mid 1:N})}{u^i_{j,k\mid 1:N}}\\
				&=\left(\cu{C}^i_k \right)^{-1}u^i_{j,k\mid 1:N} +\frac{1}{2}\, \cu{g}^i_k.\nonumber
			\end{split}
		\end{equation}
		Above, the $m$-th element of $\cu{g}^i_k\in\R^{N }$ is
		\begin{equation}
			\begin{split}
				\left[\cu{g}^i_k \right]_m &= \tash{\mathcal{L}^{\mathrm{BMAP}}}{u^i_{j,k\mid m}} \\
				&= \tash{-\log\abs{2\pi\,\cu{C}^{i-1}_k}}{u^i_{j,k\mid m}} \\
				&\quad+ \tash{-(u^{i-1}_{j,k\mid 1:N})^\trans\,(\cu{C}^{i-1}_k)^{-1}\,u^{i-1}_{j,k\mid 1:N}}{u^i_{j,k\mid m}}\\
				&= \tr\left[\left(\left( \cu{C}^{i-1}_k\right)^{-1} - \bm{\tau}\,\bm{\tau}^\trans \right) \,\tash{\cu{C}^{i-1}_k}{u^i_{j,k\mid m}} \right],\nonumber
			\end{split}
		\end{equation}
		where $u^i_{j,k\mid m}$ is the $m$-th element of $u^i_{j,k\mid 1:N}$ and $\bm{\tau} = \left( \cu{C}^{i-1}_k\right)^{-1}\,u^{i-1}_{j,k\mid 1:N}$.
		
		\section{Derivatives of Loss Function~\eqref{equ:ss-MAP-loss-func}}
		\label{append:ss-MAP}
		We collect the derivates of the state in a set
		\begin{equation}
			\tash{\mathcal{L}^{\mathrm{SMAP}}}{\cu{U}_{1:N}} = \left\lbrace \tash{\mathcal{L}^{\mathrm{SMAP}}}{\cu{U}_k}, k=0,\ldots, N \right\rbrace,\nonumber
		\end{equation}
		for all time step, where each element is a column vector. For the initial condition, its derivative is
		\begin{equation}
			\tash{\mathcal{L}^{\mathrm{SMAP}}}{\cu{U}_0} = \cu{P}^{-1}_0\, \cu{U}_0 + \frac{1}{2}\,\cu{z}_0.\nonumber
		\end{equation}
		For $k=1,2,\ldots, N-1$, the derivative is
		\begin{equation}
			\begin{split}
				\tash{\mathcal{L}^{\mathrm{SMAP}}}{\cu{U}_k} &= \frac{1}{R_k}\cu{H}^\trans\,(\cu{H}\,\cu{U}_k - y_k) \\
				&\quad+ \cu{Q}^{-1}(\cu{U}_{k-1})\,(\cu{U}_k - \cu{a}(\cu{U}_{k-1})) + \frac{1}{2}\,\cu{z}_k.\nonumber
			\end{split}
		\end{equation}
		Above, $\cu{z}_k\in\R^{\varrho }$ is a vector for $k=0,1,\ldots,N-1$. Now let us temporarily use $u^m_k$ as the $m$-th component of state $\cu{U}_k$, then the $m$-th element of $\cu{z}_k$ is 
		\begin{equation}
			\begin{split}
				[\cu{z}_k]_m &= -\cu{U}^\trans_{k+1}\,\cu{Q}^{-1}(\cu{U}_k)\,\tash{\cu{Q}(\cu{U}_k)}{u^m_k}\,\cu{Q}^{-1}(\cu{U}_k)\,\cu{U}_{k+1} \\
				&\quad+2\,\tash{\cu{a}^\trans(\cu{U}_k)}{u^m_k}\,\cu{Q}^{-1}(\cu{U}_k)\,\left(\cu{a}(\cu{U}_k) - \cu{U}_{k+1} \right) \\
				&\quad+ \cu{a}^\trans(\cu{U}_k)\,\cu{Q}^{-1}(\cu{U}_k)\,\tash{\cu{Q}(\cu{U}_k)}{u^m_k}\,\cu{Q}^{-1}(\cu{U}_k)\\
				&\quad\quad\times \left( 2\,\cu{U}_{k+1} - \cu{a}(\cu{U}_k) \right) \\
				&\quad + \tr\left( \cu{Q}^{-1}(\cu{U}_k)\,\tash{\cu{Q}(\cu{U}_k)}{u^m_k}\right).
			\end{split}
		\end{equation}
		Finally, for the derivative on the last time step
		\begin{equation}
			\begin{split}
				\tash{\mathcal{L}^{\mathrm{SMAP}}}{\cu{U}_N} &= \frac{1}{R_N}\cu{H}^\trans\,(\cu{H}\,\cu{U}_N - y_N) \\
				&\quad+\cu{Q}^{-1}(\cu{U}_{N-1})\,\left(\cu{U}_N - \cu{a}(\cu{U}_{N-1}) \right).
			\end{split}
		\end{equation}
		
		\section{Derivation of Equation~\eqref{equ:GF-update}}
		\label{append:derivation-kf-update}
		Let us denote by 
		\begin{equation}
			\Bar{P}_k = \begin{bmatrix}
				\Bar{P}^{f,f}_k & \Bar{P}^{f,\sigma}_k \\
				\Bar{P}^{f,\sigma}_k & \Bar{P}^{\sigma,\sigma}_k
			\end{bmatrix}. \nonumber
		\end{equation}
		Then by the update step of Gaussian filters~\citep[see, e.g., Algorithm 6.3 of][]{sarkka2013}, we have
		\begin{equation}
			\begin{split}
				S_k &= H\, \Bar{P}_k \, H^\trans + R_k, \\
				K_k &= \Bar{P}_k\, H^\trans / S_k, \\
				P_k &= \Bar{P}_k - K_k\,K_k^\trans / S_k,
			\end{split}\nonumber
		\end{equation}
		where $H = \begin{bmatrix} 1 & 0 \end{bmatrix}$. Substituting $K_k$ and $S_k$ into $P_k$ gives
		\begin{equation}
			\begin{split}
				P_k &= \Bar{P}_k - \begin{bmatrix}
					\left( \Bar{P}^{f,f}_k\right)^2 & \Bar{P}^{f,f}_k\, \Bar{P}^{f,\sigma}_k \\
					\Bar{P}^{f,f}_k\, \Bar{P}^{f,\sigma}_k & \left( \Bar{P}^{\sigma,\sigma}_k\right)^2
				\end{bmatrix} / \left( \Bar{P}^{f,f}_k + R_k \right) 
			\end{split} \nonumber
		\end{equation}
		Hence, the $P^{f,\sigma}_k$ of $P_k$ is 
		\begin{equation}
			P^{f,\sigma}_k = \Bar{P}^{f,\sigma}_k - \frac{\Bar{P}^{f,f}_k\, \Bar{P}^{f,\sigma}_k}{\Bar{P}^{f,f}_k + R_k}.\nonumber
		\end{equation}
		
		\section{Samples from DGP Priors and Predictions from DGP Posterior Distributions}
		\label{append:prior-samples}
		\begin{figure*}[t!]
			\centering
			\includegraphics[width=.32\linewidth]{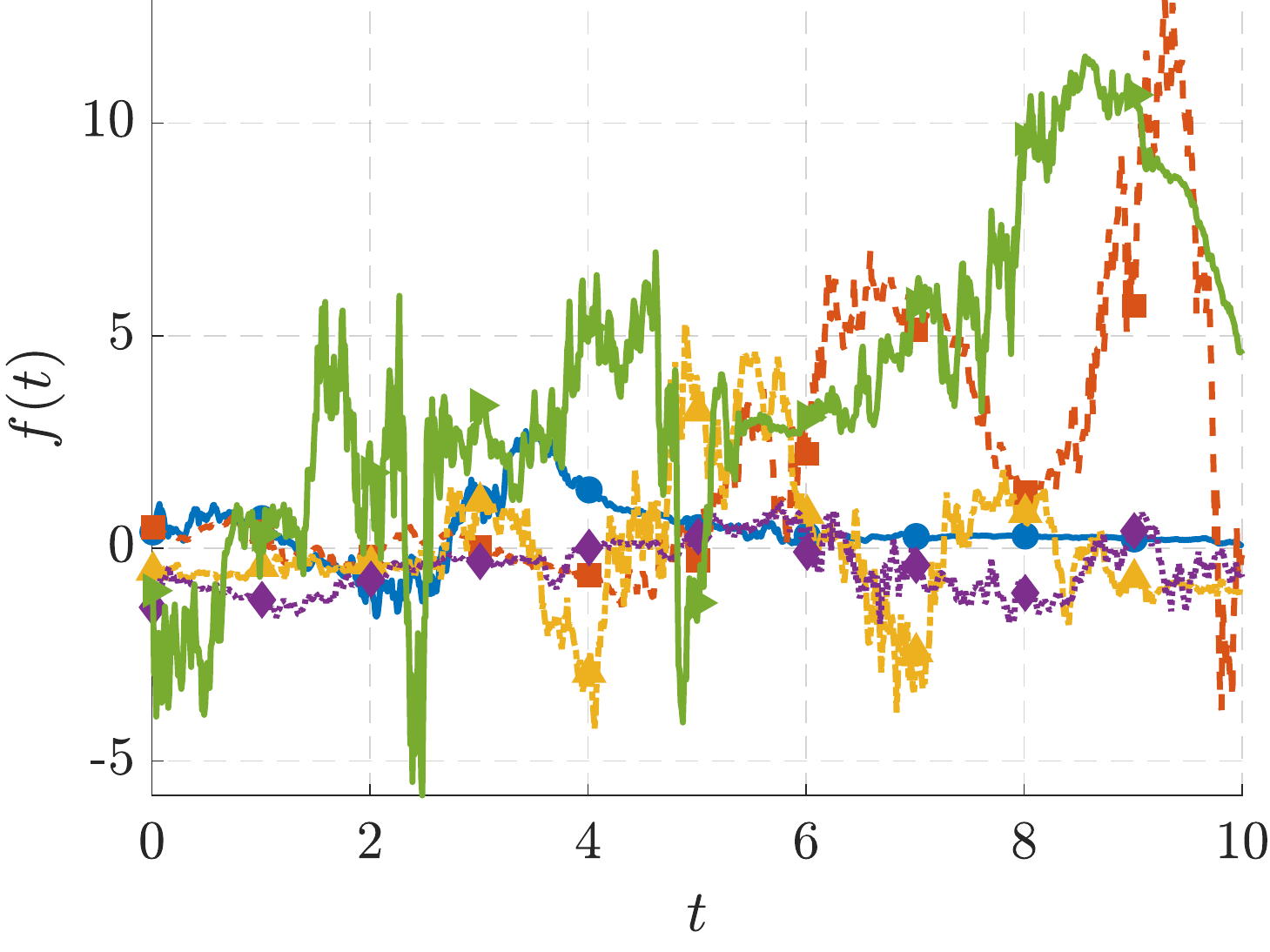}
			\includegraphics[width=.32\linewidth]{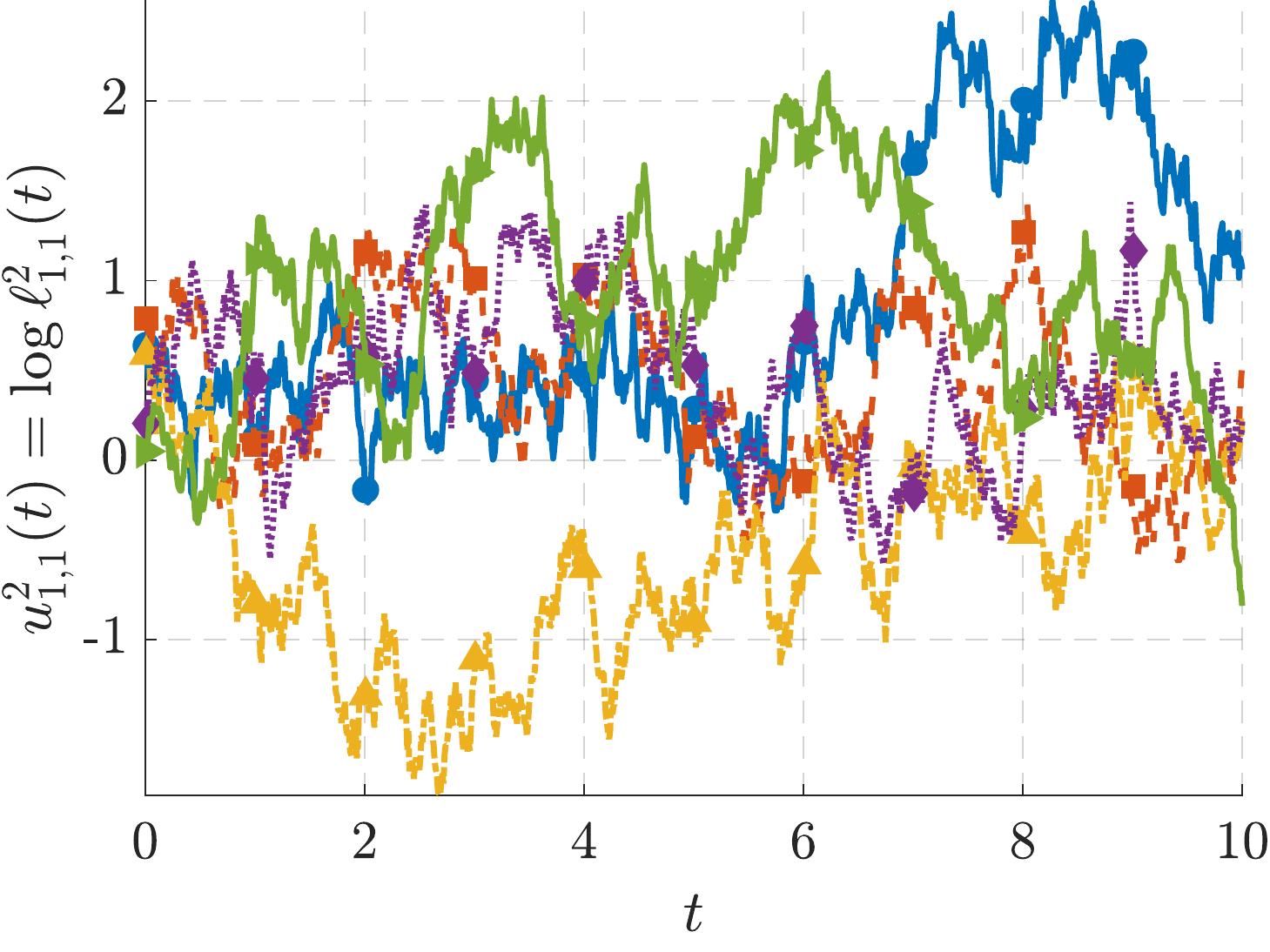}
			\includegraphics[width=.32\linewidth]{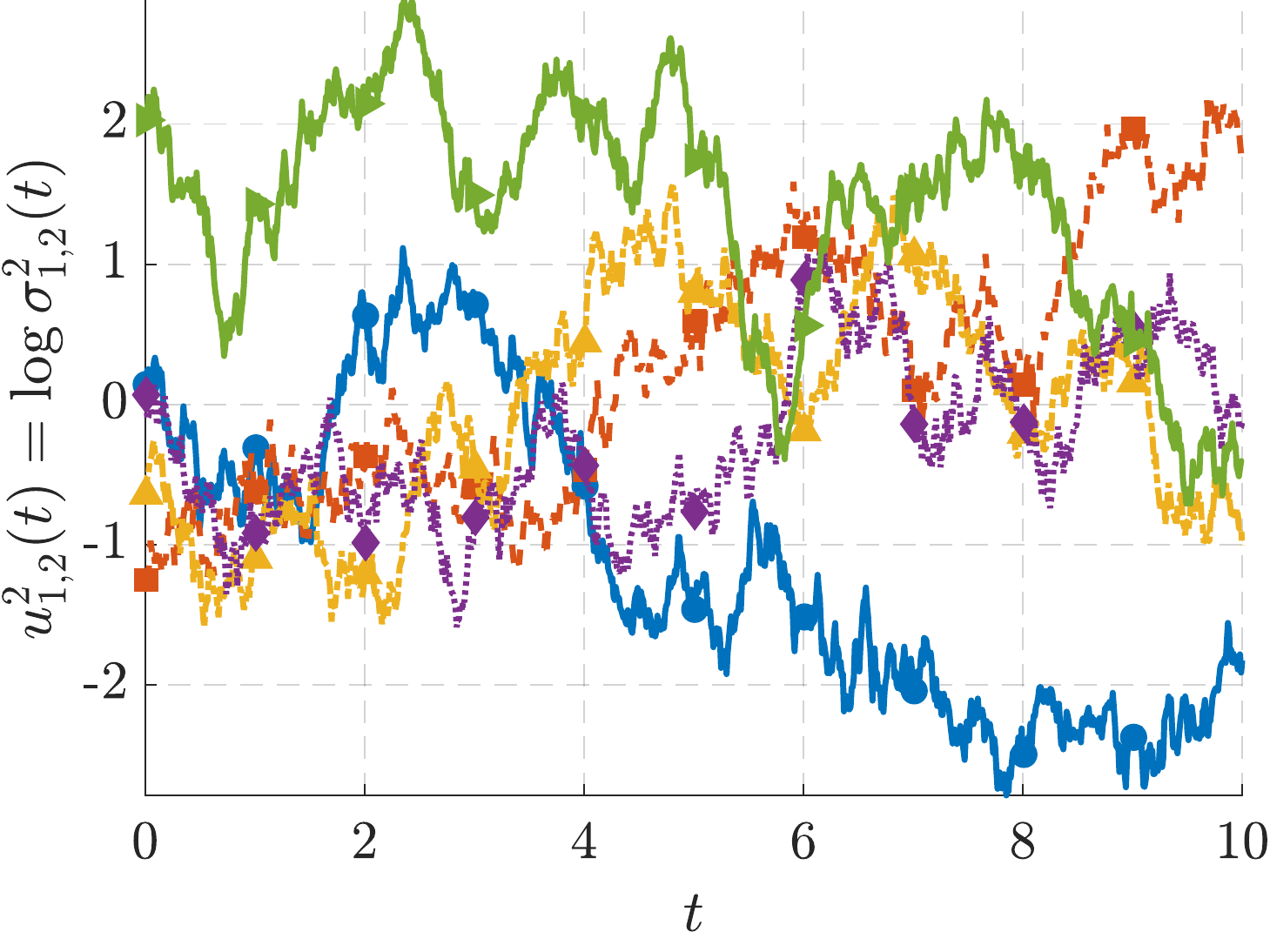}\\
			\includegraphics[width=.32\linewidth]{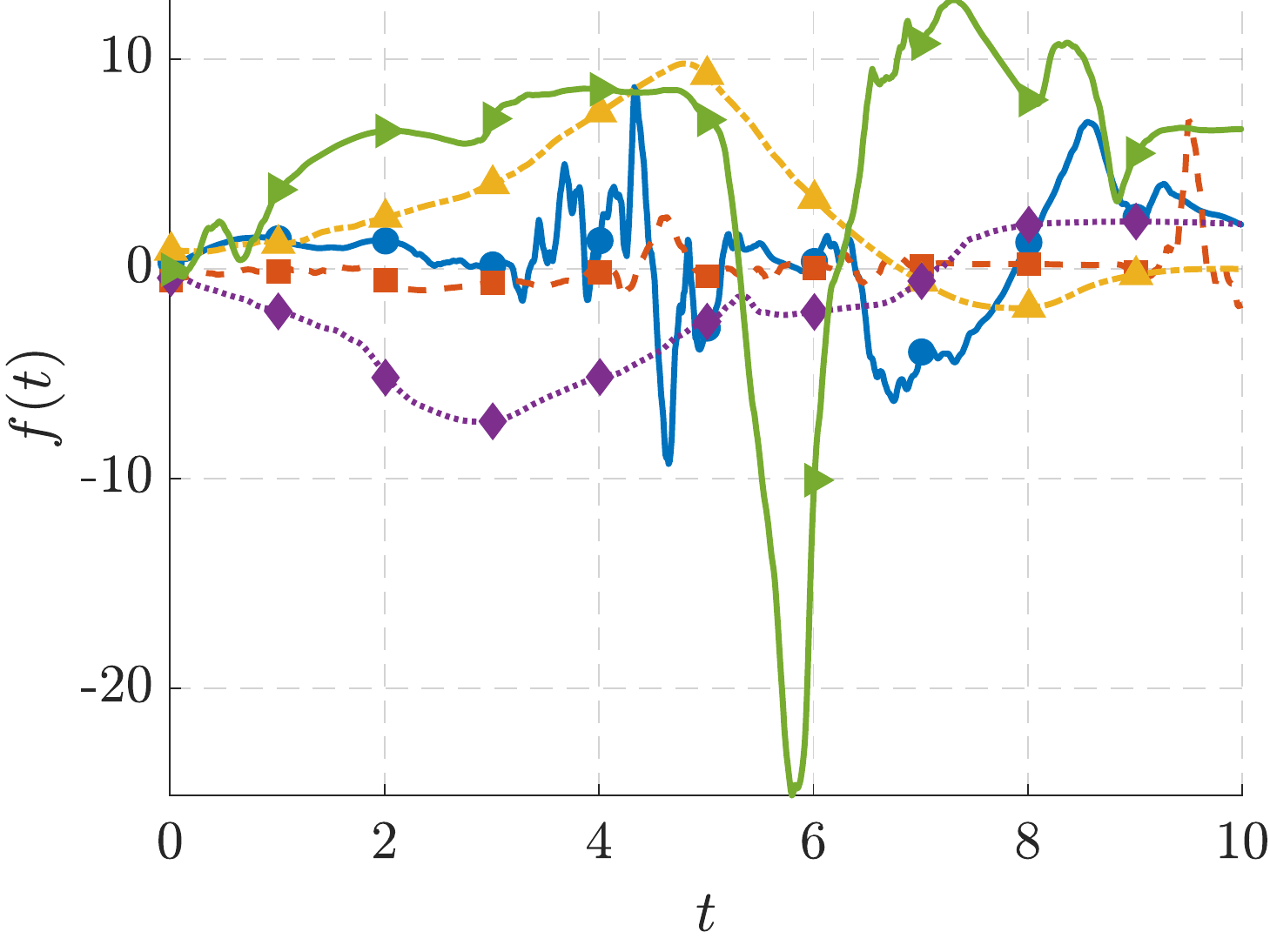}
			\includegraphics[width=.32\linewidth]{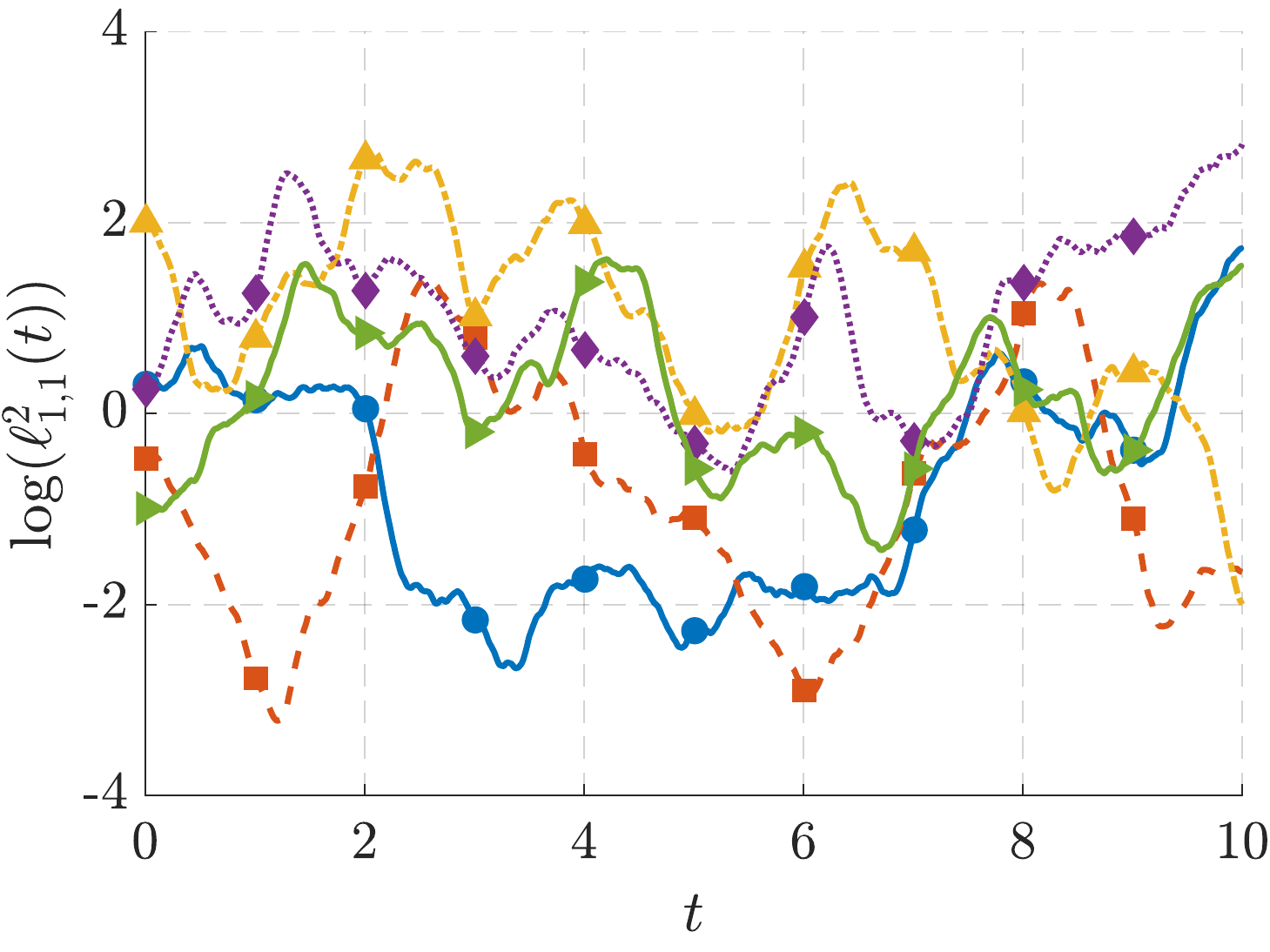}
			\includegraphics[width=.32\linewidth]{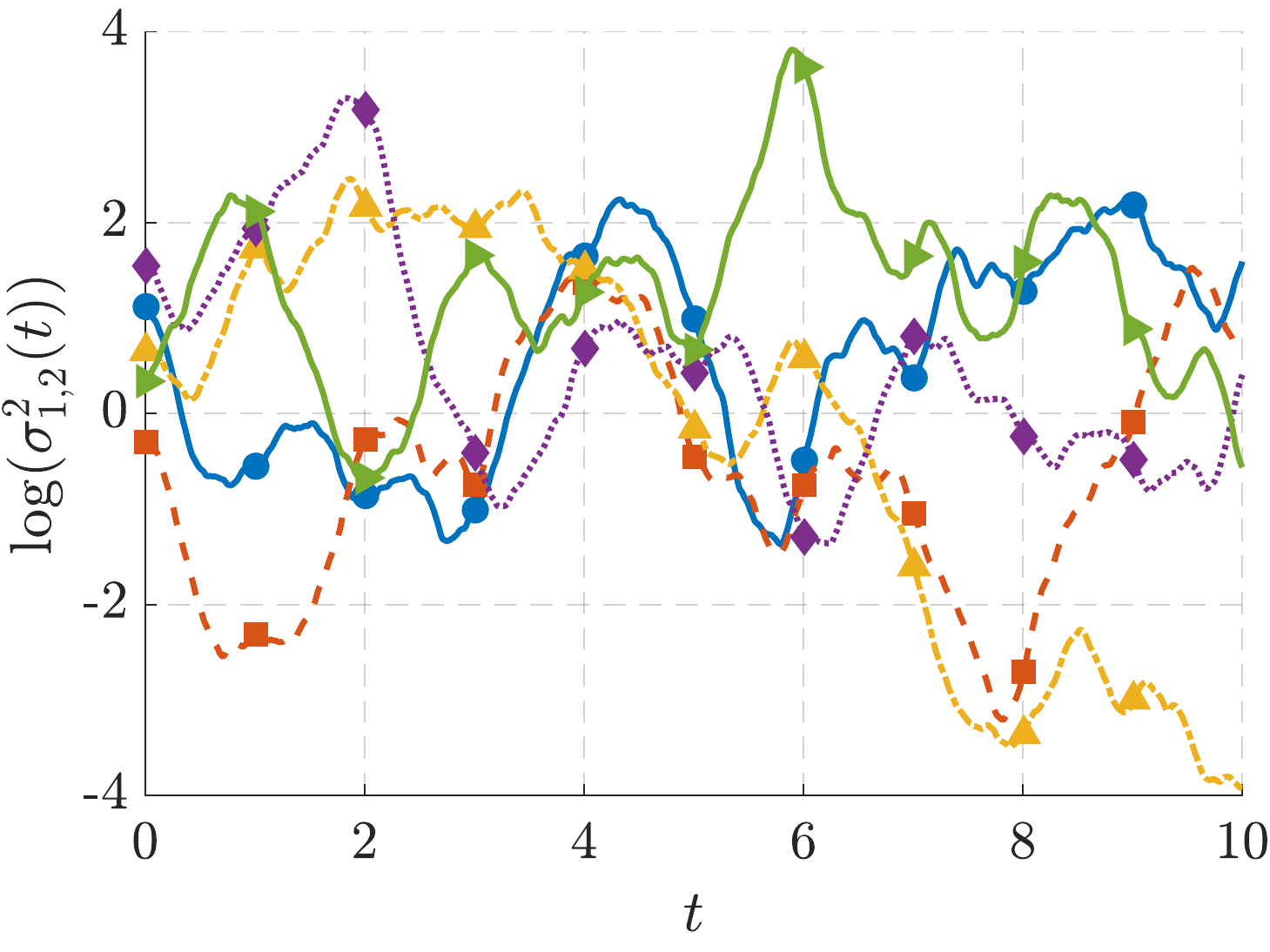}
			\caption{Samples of DGP-2 models defined in Example~\ref{example:matern-ss-gp} (first row) and Figure~\ref{fig:reg-models-DGP2} (second row). }
			\label{fig:prior-samples}
		\end{figure*}
		\begin{figure*}[t!]
			\centering
			\includegraphics[width=.99\linewidth]{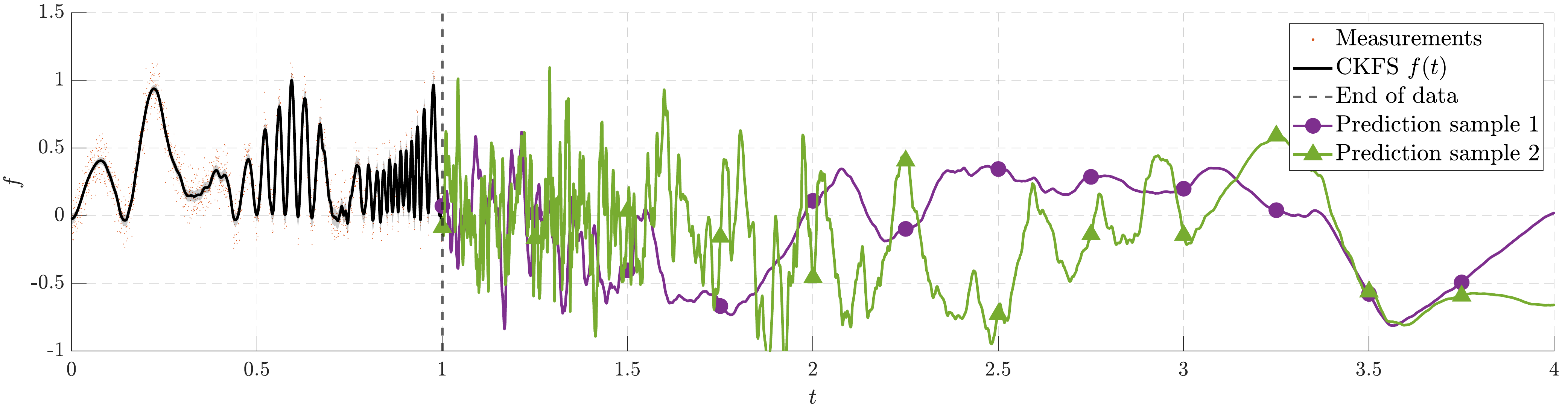}\\
			\includegraphics[width=.99\linewidth]{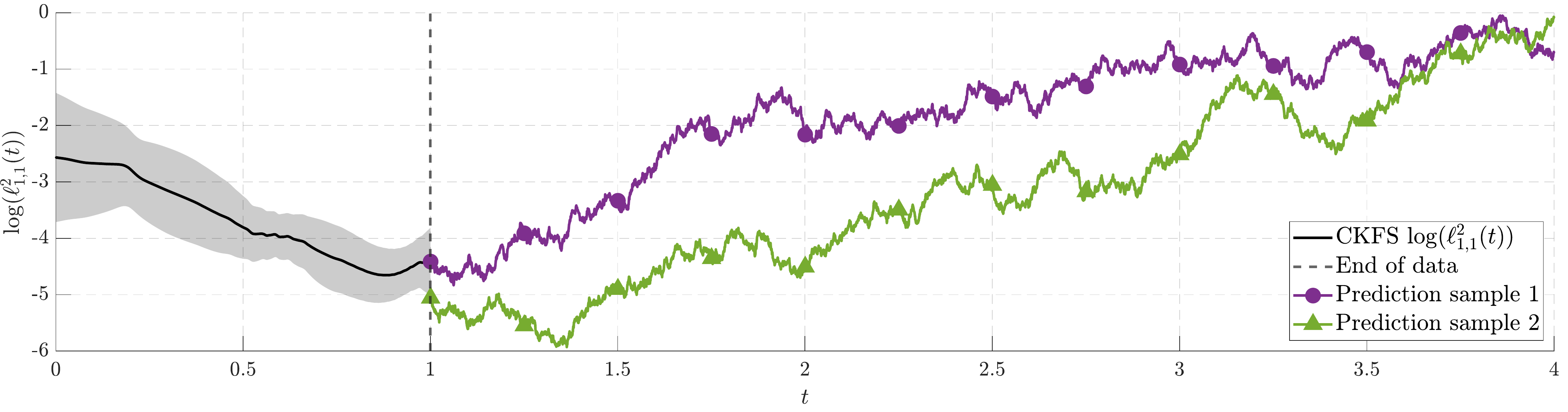}
			\caption{Prediction samples drawn from the CKFS DGP-2 model (i.e., continuation of Figure~\ref{fig:sin-exp-GFS}). Only two samples are shown for the sake of readability. }
		\end{figure*}
		To demonstrate the non-stationarity of the DGP models, we draw samples from the DGPs priors defined in Example~\ref{example:matern-ss-gp} and Figure~\ref{fig:reg-models-DGP2}. The samples are drawn by using the TME-3 discretization approach~\citep{zhao2020taylor} on $t\in[0, 10]$ with time interval $\Delta t = 0.01$~s. We show the samples in Figure~\ref{fig:prior-samples}, where we can clearly see the non-stationary features of process $f(t)$. The samples also switch the stationary and non-stationary behaviour randomly. 
		
		It is also of interests to see how does a fitted DGP model behave in the future (i.e., when extrapolated). For this purpose, we select the fitted CKFS DGP-2 on the sinusoidal experiments as the example. We draw prediction samples starting from the end of the smoothing posterior distribution, and predict until $t=4$~s. We see that at the beginning ($t=1$~s) the samples of $f$ retain similar features as the fitted $f$. As $t$ reaches the end, $f(t)$ is gradually becoming smoother because its lengthscale approach the stationary state. 
		
		\section{Hyperparameter Values Found via Grid Search}
		\label{append:hyperparas}
		For the sake of reproducibility we list the hyperparameters found by grid search in the following Table~\ref{tbl:hyperparas}. Due to a large number of unknown hyperparameters, the grid search routine assumes that GP nodes in the last layer share the same hyperparameters. Hereafter we use notations $\ell$ and $\sigma$ to represent the last layer lengthscale and magnitude.
		\begin{table}[h!]
			\centering
			\begin{tabular}{@{}lll@{}}
				\toprule
				Method & DGP-2 & DGP-3 \\ \midrule
				B-MAP & $\ell=0.087$, $\sigma=0.3$ & $\ell=0.04$, $\sigma=0.3$ \\
				SS-MAP & $\ell=0.008$, $\sigma=0.14$  & $\ell=0.001$, $\sigma=0.92$  \\
				EKFS & \begin{tabular}[c]{@{}l@{}}$\sigma^2_{1,2}=2$\\ $\ell=0.001$, $\sigma=2.1$\end{tabular} & \begin{tabular}[c]{@{}l@{}}$\sigma^2_{1,2}=12$, $\sigma^3_{1,2}=0.8$\\ $\ell=0.001$, $\sigma=9$\end{tabular} \\
				CKFS & \begin{tabular}[c]{@{}l@{}}$\sigma^2_{1,2}=2$\\ $\ell=0.001$, $\sigma=1.54$\end{tabular} & N/A \\
				PF-BS & $\ell=0.008$, $\sigma=0.54$ & $\ell=0.098$, $\sigma=0.79$  \\ \midrule
				CKFS & \begin{tabular}[c]{@{}l@{}}$\sigma^2_{1,2}=0.4$\\ $\ell=2.83$, $\sigma=1.49$\end{tabular} & \begin{tabular}[c]{@{}l@{}}$\sigma^2_{1,2}=0.4$, $\sigma^3_{1,2}=1.2$\\ $\ell=140$, $\sigma=0.7$\end{tabular} \\ 
				EKFS & \begin{tabular}[c]{@{}l@{}}$\sigma^2_{1,2}=1.6$\\ $\ell=0.23$, $\sigma=1.16$\end{tabular} & \begin{tabular}[c]{@{}l@{}}$\sigma^2_{1,2}=1.2$, $\sigma^3_{1,2}=0.9$\\ $\ell=0.22$, $\sigma=0.01$\end{tabular} \\ \bottomrule
			\end{tabular}
			\caption{Hyperparameters found via grid search for the rectangle (first block) and sinusoidal (second block) experiments. }
			\label{tbl:hyperparas}
		\end{table}
		
	\end{appendices}
	
	\bibliographystyle{spbasic}  
	\bibliography{refs}
	
\end{document}